\RequirePackage{fix-cm}
\documentclass[smallextended]{svjour3}      
\smartqed  
\usepackage[utf8]{inputenc}
\usepackage{algorithmicx}
\usepackage{algpseudocode}
\usepackage{parskip}
\usepackage{amsmath,amssymb,amstext,amsfonts}
\usepackage{float}
\usepackage{graphicx}
\usepackage{xcolor}
\usepackage{comment}
\usepackage{mathrsfs}

\newcommand{\OT}{\mathrm{OT}}

\newcommand{\ADM}{\mathrm{Joint}}
\newcommand{\Tr}{\mathrm{Tr}}
\newcommand{\E}{\mathbb{E}}
\newcommand{\Probs}{\mathcal{P}}
\newcommand{\N}{\mathcal{N}}
\newcommand{\GP}{\mathrm{GP}}
\newcommand{\KL}{{\mathrm{KL}}}
\newcommand{\trace}{\mathrm{Tr}}
\newcommand{\approach}{\ensuremath{\rightarrow}}
\newcommand{\Minh}{\textcolor{blue}}

\renewcommand{\H}{\mathcal{H}}
\newcommand{\Sym}{\mathrm{Sym}}
\newcommand{\HS}{\mathrm{HS}}
\newcommand{\tr}{\mathrm{tr}}
\newcommand{\Xbf}{\mathbf{X}}
\newcommand{\Ybf}{\mathbf{Y}}
\newcommand{\Ncal}{\mathcal{N}}
\newcommand{\1}{\mathbf{1}}
\newcommand{\X}{\mathcal{X}}
\newcommand{\R}{\mathbb{R}}
\newcommand{\la}{\langle}
\newcommand{\ra}{\rangle}
\newcommand{\mapto}{\ensuremath{\rightarrow}}
\renewcommand{\b}{\mathbf{b}}

\newcommand{ \ep}{\epsilon}
\newcommand{ \Lexp}{L^{\rm exp}_{\ep}}

\newcommand{\Pro}{\mathcal{P}}

\newcommand{\mysupp}{\mathrm{supp}}

\newcommand{\restr}[1]{\lower3pt\hbox{$|_{#1}$}}

\newcommand{\Ubb}{\mathbb{U}}
\newcommand{\Gauss}{\mathrm{Gauss}}
\newcommand{\dettwo}{{\rm det_2}}
\newcommand{\Bsc}{\mathscr{B}}
\newcommand{\Fcal}{\mathcal{F}}
\newcommand{\bE}{\mathbb{E}}
\newcommand{\Joint}{\mathrm{Joint}}
\newcommand{\Tcal}{\mathcal{T}}
\newcommand{\equivalent}{\ensuremath{\Longleftrightarrow}}
\newcommand{\imply}{\ensuremath{\Rightarrow}}
\newcommand{\mysqrt}{\mathrm{sqrt}}
\newcommand{\mysq}{\mathrm{sq}}
\newcommand{\compose}{\circ}
\DeclareMathOperator*{\argmin}{arg\,min}
\newcommand{\Lcal}{\mathcal{L}}
\newcommand{\Xcal}{\mathcal{X}}
\newcommand{\Pcal}{\mathcal{P}}
\newcommand{\Srm}{\mathrm{S}}
\newcommand{\Nbb}{\mathbb{N}}
\newcommand{\Kcal}{\mathcal{K}}
\newcommand{\Csc}{\mathscr{C}}
\newcommand{\myIm}{\mathrm{Im}}
\newcommand{\Ycal}{\mathcal{Y}}
\newcommand{\Zcal}{\mathcal{Z}}
\newcommand{\Gcal}{\mathcal{G}}
\newcommand{\myspan}{\mathrm{span}}
\newcommand{\Acal}{\mathcal{A}}
\title{
Entropic regularization of Wasserstein distance between infinite-dimensional Gaussian measures
and Gaussian processes}
\titlerunning{
	Entropic regularization of Wasserstein distance
	on Hilbert space}
\author{H\`a Quang Minh}
\institute{
	\at
	RIKEN Center for Advanced Intelligence Project, 1-4-1 Nihonbashi, Chuo-ku,
	Tokyo 103-0027, JAPAN
	\\
	\email{minh.haquang@riken.jp}
}
\date{\today}

\begin{document}

\maketitle

\begin{abstract}
	This work studies the entropic regularization formulation of the 2-Wasserstein distance on an infinite-dimensional Hilbert space, in particular for the Gaussian setting. We first present the Minimum Mutual Information property, namely the joint measures of two Gaussian measures on Hilbert space with the smallest mutual information 
	are joint Gaussian measures. This is the infinite-dimensional generalization of the Maximum Entropy property of Gaussian densities on Euclidean space. We then 
	give closed form formulas for the optimal entropic transport plan, entropic 2-Wasserstein distance, and Sinkhorn divergence between two Gaussian measures on a Hilbert space, along with the
	fixed point equations for the barycenter of a set of Gaussian measures. 
	Our formulations fully exploit the regularization aspect of the entropic formulation and are valid both in {\it singular} and {\it nonsingular} settings.
	In the infinite-dimensional setting, both the
	entropic 2-Wasserstein distance and Sinkhorn divergence are Fr\'echet differentiable, in contrast to the
	exact 2-Wasserstein distance, which is not differentiable.
	Our Sinkhorn barycenter equation is new and always has a unique solution. In contrast,  the finite-dimensional barycenter equation for the entropic 2-Wasserstein distance fails to generalize to the Hilbert space setting.
	In the setting of reproducing kernel Hilbert spaces (RKHS), our distance formulas are given explicitly in terms of the corresponding kernel Gram matrices, providing an interpolation between
	the kernel Maximum Mean Discrepancy (MMD) and the kernel 2-Wasserstein distance.
\end{abstract}

\section{Introduction}

In this work, we study the entropic regularization formulation of the $2$-Wasserstein distance
in the Hilbert space setting, with a particular focus on Gaussian measures and covariance operators on Hilbert space. This is the infinite-dimensional generalization
of recent work on the entropic $2$-Wasserstein distance between Gaussian measures on $\R^n$, as reported in 
\cite{Mallasto2020entropyregularized,Janati2020entropicOT,barrio2020entropic}. Our work is along the direction of entropic regularization in optimal transport, which has recently attracted much attention in various fields,
in particular machine learning and statistics
\cite{cuturi13,Sommerfeld2017WassersteinDO,feydy18,genevay16,genevay17,GigTam18,ramdas2017,RipThesis}, with applications in computer vision, density functional theory, and inverse problems (e.g.~\cite{genevay17,GerGroGor19,Lunz18,patrini18}). This direction of research is also closely connected
with the {\it Schr\"odinger bridge problem}~\cite{Schr31}, which has been studied extensively
~\cite{BorLewNus94,Csi75,peyre17,FraLor89,RusIPFP,Zam15,galsal,LeoSurvey,rus93,rus98}.

Our focus in the Gaussian setting stems not only from its use in elucidating various aspects of the abstract theory, since many quantities of interest admit closed form formulas, but also from numerous applications utilizing Gaussian measures and covariance matrices/operators.
These include brain  imaging~\cite{arsigny06,Dryden:2009}, computer vision~\cite{tuzel08,Tosato:PAMI2013}, and brain computer interfaces~\cite{Congedo:BCIreview2017}.
Many distances/divergences  have been studied and employed in practice, including the \emph{affine-invariant Riemannian metric} \cite{Pennec:IJCV2006}, corresponding to the Fisher-Rao distance between centered Gaussians, the Alpha Log-Determinant divergences \cite{Chebbi:2012Means}, corresponding to R\'enyi divergences between centered Gaussians, the \emph{Log-Euclidean metric} \cite{LogEuclidean:SIAM2007}, and  
recent work attempting to unify them 
~\cite{amari2018information,cichocki15,thanwerdas19,Minh:GSI2019}. 

{\bf Infinite-dimensional setting}. The generalization of distances/divergences for Gaussian measures and covariance matrices on $\R^n$ to the infinite-dimensional setting of Gaussian measures and covariance operators on Hilbert spaces has been carried out by various authors.
In general, the infinite-dimensional formulations are substantially more complex than the finite-dimensional ones and regularization is often necessary. This is 
the case for the affine-invariant Riemannian distance \cite{Larotonda:2007}, the Log-Hilbert-Schmidt metric
\cite{MinhSB:NIPS2014}, the Alpha and Alpha-Beta Log-Determinant divergences \cite{Minh:LogDet2016,Minh:LogDetIII2018,Minh:2019AlphaBeta,Minh:Positivity2020,Minh:2020regularizedDiv}.
The settings for these distances/divergences are the sets of positive definite unitized trace class/
Hilbert-Schmidt operators, which are positive trace class/Hilbert-Schmidt operators plus a positive scalar multiple of the identity operator so that operations such as inversion, logarithm, and determinant, are well-defined. A particular advantage of the $2$-Wasserstein distance compared to the above distances/divergences is that the finite and infinite-dimensional distance formulas \cite{Gelbrich:1990Wasserstein,cuesta1996:WassersteinHilbert} are the same and no regularization is necessary. 

{\bf Reproducing kernel Hilbert space (RKHS) setting}.
From the computational and practical viewpoint, this setting is particularly interesting since
many quantities of interest admit closed forms via kernel Gram matrices which can be efficiently computed.
Examples include the kernel Maximum Mean Discrepancy (MMD) \cite{Gretton:MMD12a} and the RKHS covariance operators, 
the latter resulting in powerful nonlinear algorithms with
substantial improvements over finite-dimensional covariance matrices, see e.g. \cite{ProbDistance:PAMI2006,Covariance:CVPR2014,MinhSB:NIPS2014,Minh:Covariance2017,zhang2019:OTRKHS} for examples of applications in computer vision. 

{\bf Contributions of this work}.
\begin{enumerate}
	\item We generalize the Maximum Entropy property of Gaussian densities in $\R^n$ to the 
	Hilbert space setting, namely the Minimum Mutual Information of joint measures of two Gaussian measures
	on Hilbert space.
	
	\item For two Gaussian measures on a Hilbert space $\H$, we provide
	closed form formulas for the optimal entropic transport plan, the entropic $2$-Wasserstein distance and the Sinkhorn divergence, generalizing results in \cite{Mallasto2020entropyregularized,Janati2020entropicOT,barrio2020entropic}.
	\item For a set of Gaussian measures, we show a new Sinkhorn barycenter equation, with always a unique non-trivial solution. In contrast,
	we show that the finite-dimensional barycenter equation for the entropic $2$-Wasserstein distance
	fails to generalize to the Hilbert space setting.
	 
	\item In the RKHS setting, we present closed form formulas for the distances via the finite kernel Gram matrices,
	providing an interpolation between the kernel MMD \cite {Gretton:MMD12a} and kernel Wasserstein distance \cite{zhang2019:OTRKHS,Minh:2019AlphaProcrustes}.
	 \item Our proofs and results fully exploit the regularization aspect of the entropic formulation and are valid both in {\it singular} and {\it non-singular} settings. This is novel also in the finite-dimensional setting (compared to \cite{Mallasto2020entropyregularized,Janati2020entropicOT,barrio2020entropic}).
\end{enumerate}
{\color{black}
\begin{remark}
As we discuss in detail below, many properties of the Gaussian case have {\it not} been proved in the general theory, due to (i) $\dim(\H) = \infty$, (ii) the cost function $c(x,y) = ||x-y||^2$ is unbounded on $\H$, and (iii) the support of Gaussian measures is unbounded.
\end{remark}
}

\section{Background and finite-dimensional results}
\label{section:background}
Let $(X,d)$ be a complete separable metric space equipped with a lower semi-continuous \emph{cost function} $c:X\times X \to \mathbb{R}_{\geq 0}$. 
Let $\Pcal(X)$ denote the set of all probability measures on $X$.
The {\it optimal transport} (OT) problem between two probability measures $\nu_0, \nu_1 \in \Probs(X)$ is  
(see e.g. \cite{villani2016})
\begin{equation}
    \OT_c(\nu_0, \nu_1) = \min_{\gamma\in \ADM(\nu_0,\nu_1)}\E_\gamma[c] = \min_{\gamma \in \ADM(\nu_0, \nu_1)}\int_{X \times X}c(x,y)d\gamma(x,y)
    \label{equation:OT-exact}
\end{equation}
where $\ADM(\nu_0,\nu_1)$ is the set of joint probabilities with marginals $\nu_0$ and $\nu_1$.
%
For $1 \leq p < \infty$, let $\Pcal_p(X)$ denote the set of all probability measures $\mu$ on $X$ of finite moment of order $p$,
i.e. 
$\int_{X}d^p(x_0,x)d\mu(x) < \infty$ for some (and hence any) $x_0 \in X$.
The {\it $p$-Wasserstein distance} $W_p$ between $\nu_0$ and $\nu_1$ is defined as
\begin{equation}
    W_p(\nu_0,\nu_1) = \OT_{d^p}(\nu_0, \nu_1)^{\frac{1}{p}}.
\end{equation}
This distance defines a metric on $\Pcal_p(X)$ (Theorem 7.3, \cite{villani2016}).
For two multivariate Gaussian distributions $\nu_i=\N(m_i,C_i)$, $i=0,1$, on $\R^n$,  $W_2(\nu_0, \nu_1)$ admits the following closed form \cite{givens84,dowson82,olkin82,knott84}
\begin{equation}
\label{equation:Gaussian-Wass-finite}
    W_2^2(\nu_0, \nu_1) = \|m_0-m_1\|^2 + \Tr(C_0) + \Tr(C_1) - 2 \Tr\left(C_1^\frac{1}{2} C_0 C_1^\frac{1}{2}\right)^\frac{1}{2}.
\end{equation}
{\bf Entropic regularization and Sinkhorn divergence.}
The exact OT problem \eqref{equation:OT-exact} is often computationally challenging and it is more numerically efficient to solve the following regularized optimization problem, for a given $\ep > 0$, 
\begin{equation}
\label{equation:OT-entropic}
\OT_c^\epsilon(\mu, \nu) = \min_{\gamma\in \ADM(\mu,\nu)}\left\lbrace\E_\gamma[c]
+ \epsilon \KL(\gamma || \mu \otimes \nu) \right\rbrace,
\end{equation}
where $\KL(\nu || \mu)$ denotes the Kullback-Leibler divergence between $\nu$ and $\mu$.
The KL
in \eqref{equation:OT-entropic} acts as a bias \cite{feydy18}, with the consequence that in general $\OT_c^\epsilon(\mu, \mu) \neq 0$. The following $p$-Sinkhorn divergence \cite{feydy18} removes this bias
\begin{equation}
\label{equation:sinkhorn}
S_{d^p}^\epsilon(\mu, \nu) = \OT_{d^p}^\epsilon(\mu, \nu) - \frac{1}{2}(\OT_{d^p}^\epsilon(\mu,\mu) + \OT_{d^p}^\epsilon(\nu,\nu) ).
\end{equation}
For $\nu_i=\Ncal(m_i,C_i)$, $i=0,1$, both $\OT^{\ep}_{d^2}(\nu_0, \nu_1)$ and $\Srm^{\ep}_{d^2}(\nu_0, \nu_1)$ admit closed form formulas. Let  $N^{\ep}_{ij} =  I + \left(I + \frac{16}{\epsilon^2}C_i^\frac{1}{2}C_jC_i^\frac{1}{2}\right)^\frac{1}{2}$, $i,j=0,1$, then \cite{Mallasto2020entropyregularized,Janati2020entropicOT,barrio2020entropic}
\begin{align}
\label{equation:gauss-entropic-Wass-finite}
\OT_{d^2}^\epsilon(\nu_0, \nu_1)
&= \|m_0 - m_1\|^2
+ \Tr(C_0) + \Tr(C_1)
\nonumber
\\
& \quad - \frac{\epsilon}{2}\left[
\Tr(N^{\epsilon}_{01}) - \log \det\left(N^{\epsilon}_{01}\right) + n\log{2} - 2n
\right],
\\
S_{d^2}^\epsilon(\nu_0,\nu_1)
&= {\color{black}\|m_0 - m_1\|^2} + \frac{\epsilon}{4} \left(
\Tr\left(
N_{00}^\epsilon - 2 N_{01}^\epsilon + N_{11}^\epsilon
\right)\phantom{\frac{M^2}{M^2}}\right.
\nonumber
\\
&\quad + \left.\log \left(
\frac{\det^2( N_{01}^\epsilon )}{\det(N_{00}^\epsilon)\det(N_{11}^\epsilon)}
\right)\right).
\label{equation:gauss-sinkhorn-finite}
\end{align}
In this case, the unique minimizer $\gamma$ in \eqref{equation:OT-entropic} is a joint Gaussian measure of $\nu_0$ and $\nu_1$, a direct consequence
of the Maximum Entropy of Gaussian densities (see below). 
In particular, $\lim\limits_{\epsilon\rightarrow 0}\OT^{\ep}_{d^2}(\nu_0, \nu_1) = \lim\limits_{\epsilon\rightarrow 0}\Srm_{d^2}^\epsilon(\nu_0, \nu_1) = W_2^2(\nu_0, \nu_1)$ and {\color{black}$\lim\limits_{\epsilon\rightarrow \infty} S_{d^2}^\epsilon(\nu_0, \nu_1) = \|m_0-m_1\|^2$}. For related work, see also \cite{kum2020penalization,ciccone2020regularizedtransport}.

\section{From finite to infinite-dimensional settings}
\label{section:finite-to-infinite}

In the current work, we generalize the results in \cite{Mallasto2020entropyregularized,Janati2020entropicOT,barrio2020entropic}
to the 
Hilbert space setting.
Throughout the following, let $(\H, \la, \ra)$ be a real, separable Hilbert space, with $\dim(\H) = \infty$ unless explicitly stated otherwise.
For two separable Hilbert spaces {\color{black}$(\H_i, \la,\ra_i)$},$i=1,2$, let $\Lcal(\H_1,\H_2)$ denote the Banach space of bounded linear operators from $\H_1$ to $\H_2$, with operator norm $||A||=\sup_{||x||_1\leq 1}||Ax||_2$.
For $\H_1=\H_2 = \H$, we use the notation $\Lcal(\H)$.

Let $\Sym(\H) \subset \Lcal(\H)$ be the set of bounded, self-adjoint linear operators on $\H$. Let $\Sym^{+}(\H) \subset \Sym(\H)$ be the set of
self-adjoint, {\it positive} operators on $\H$, i.e. $A \in \Sym^{+}(\H) \equivalent A^{*}=A, \la Ax,x\ra \geq 0 \forall x \in \H$. 
Let $\Sym^{++}(\H)\subset \Sym^{+}(\H)$ be the set of self-adjoint, {\it strictly positive} operator on $\H$,
i.e $A \in \Sym^{++}(\H) \equivalent A^{*}=A, \la x, Ax\ra > 0$ $\forall x\in \H, x \neq 0$.
We write $A \geq 0$ for $A \in \Sym^{+}(\H)$ and $A > 0$ for $A \in \Sym^{++}(\H)$.
If $\gamma I+A > 0$, where $I$ is the identity operator,$\gamma \in \R,\gamma > 0$, then $\gamma I+A$ is also invertible, in which case it is called
{\it positive definite}. {\color{black}In general, $A \in  \Sym(\H)$ is said to be positive definite if $\exists M_A > 0$ such that $\la x, Ax\ra \geq M_A||x||^2$ $\forall x \in \H$ - this condition is equivalent to $A$ being both strictly positive and invertible, see e.g. \cite{Petryshyn:1962}.}

The Banach space $\Tr(\H)$  of trace class operators on $\H$ is defined by (see e.g. \cite{ReedSimon:Functional})
$\Tr(\H) = \{A \in \Lcal(\H): ||A||_{\tr} = \sum_{k=1}^{\infty}\la e_k, (A^{*}A)^{1/2}e_k\ra < \infty\}$,
for any orthonormal basis {\color{black}$\{e_k\}_{k \in \Nbb} \subset \H$}.
For $A \in \Tr(\H)$, its trace is defined by $\trace(A) = \sum_{k=1}^{\infty}\la e_k, Ae_k\ra$, which is independent of choice of $\{e_k\}_{k\in \Nbb}$. 

The Hilbert space $\HS(\H_1,\H_2)$ of Hilbert-Schmidt operators from $\H_1$ to $\H_2$ is defined by 
(see e.g. \cite{Kadison:1983})
$\HS(\H_1, \H_2) = \{A \in \Lcal(\H_1, \H_2):||A||^2_{\HS} = \trace(A^{*}A) =\sum_{k=1}^{\infty}||Ae_k||_2^2 < \infty\}$,
for any orthonormal basis $\{e_k\}_{k \in \Nbb}$ in $\H_1$,
with inner product $\la A,B\ra_{\HS}=\trace(A^{*}B)$. For $\H_1 = \H_2 = \H$, we write $\HS(\H)$. We have 
{\color{black}$\trace(\H) \subsetneq \HS(\H) \subsetneq \Lcal(\H)$}
when $\dim(\H) = \infty$, with $||A||\leq ||A||_{\HS}\leq ||A||_{\tr}$.

Some key differences between the finite and infinite-dimensional settings are

\begin{enumerate}

\item On $\R^n$, for two random variables $X,Y$ with joint and marginal measures $\mu_{XY},\mu_X, \mu_Y$, having densities $f(x,y),f_X(x), f_Y(y)$, respectively, with respect to the Lebesgue measure, their {\it mutual information} is 
defined as \cite{CoverThomas1991:InformationTheory}
\begin{align}
I(X;Y) &= \int_{\R^{n}\times \R^n}\log\left[\frac{f(x,y)}{f_X(x)f_Y(y)}\right]f(x,y)dxdy = H(X)+H(Y)-H(X,Y)
\nonumber
\\
&= \KL(\mu_{XY}||\mu_X \otimes \mu_Y),
\end{align}
where $H(X) = -\int_{\R^n}\log[f_X(x)]f_X(x)dx$ is the {\it differential entropy} of $X$. The classical Maximum Entropy of Gaussian densities property (\cite{Covariance:CVPR2014}, Theorem 9.6.5) states that
if
$X$ has mean zero and covariance matrix $C$, then 
\begin{align*}
H(X) \leq \frac{1}{2}\log(2\pi e)^n\det(C), \text{with equality if and only if $X \sim \Ncal(0,C)$}.
\end{align*}
Thus if both $X$ and $Y$ have Gaussian densities, then
	{\it $I(X;Y)$ is 
minimum if and only if their joint density is Gaussian}, so that for $c(x,y) = ||x-y||^2$,
and $\nu_0,\nu_1$ being Gaussian, a minimizing $\gamma$ in \eqref{equation:OT-entropic} is necessarily a joint Gaussian measure of $\nu_0, \nu_1$. 
When $C$ is a covariance operator on $\H$ with $\dim(\H) = \infty$, the quantity $\det(C)$ is no longer well-defined.
However, $I(X;Y) = \KL(\mu_{XY}||\mu_X \otimes \mu_Y)$ is well-defined and finite whenever
$\mu_{XY}$ is absolutely continuous with respect to $\mu_X \otimes \mu_Y$. In the following, we show that
the above Minimum Mutual Information property of joint Gaussian measures generalizes to the infinite-dimensional setting.

\item In \cite{barrio2020entropic}, for $\mu,\nu \in \Pcal(\R^n)$, the following quantity is studied
\begin{align}
W_{2,\ep}^2(\mu, \nu) = \min_{\gamma \in \Joint(\mu, \nu)}\left\{\bE_{\gamma}||x-y||^2 -\ep H(\gamma)\right\}.
\end{align}
On $\Pcal(\R^n)$, if $\mu,\nu$ have positive densities, then $W^2_{2,\ep}(\mu,\nu)$ and $\OT^{\ep}_{d^2}(\mu, \nu)$ differ by a constant, with 
the minimizing joint measure $\gamma$ being the same. However, by the above discussion,
$W_{2,\ep}(\mu,\nu)$ is generally {\it not} well-defined on $\Pcal(\H)$
when $\dim(\H) = \infty$, in particular in the Gaussian setting. The same discussion applies to the formulations
studied in \cite{bigot19} and \cite{kum2020penalization}.

\item If $A$ is a strictly {\color{black}positive}, compact operator on $\H$, 
then
$A^{-1}$ is unbounded when $\dim(\H) = \infty$. Thus finite-dimensional methods that utilize matrix inversion extensively, e.g. in \cite{Mallasto2020entropyregularized,Janati2020entropicOT,barrio2020entropic}, are not applicable when $\dim(\H) = \infty$.
Instead, we fully exploit the regularization aspect of problem \eqref{equation:OT-entropic}
and invert operators of the form $\gamma I + A > 0$, 
thus our proofs fully resolve this issue and are valid in the general setting when $A$ can be singular.

\item The identity operator $I$ is not trace class when $\dim(\H) = \infty$.
This leads to the breakdown in the entropic barycenter problem (Theorem \ref{theorem:entropic-barycenter-Gaussian}) and has consequences
for the analysis of the existence of solutions of the barycenter equations
(detail given in Section \ref{section:compare-barycenter}). 


\end{enumerate}

\section{Main Results}
\label{section:main-results}

We first state the following generalization of the Maximum Entropy of Gaussian densities in $\R^n$.
To the best of our knowledge, this property has not been explicitly and rigorously presented in the literature
in the infinite-dimensional setting. 
It states that among all joint measures, with the same covariance operators, of two Gaussian measures $\mu_X,\mu_Y$ on two separable Hilbert spaces $\H_1,\H_2$, the ones with the 
minimum Mutual Information are precisely the joint Gaussian measures on $\H_1\times \H_2$. In the following, $\Gauss(\H)$ denotes the set of all Gaussian measures on $\H$ and $\Gauss(\mu_X, \mu_Y)$ denotes the
set of joint Gaussian measures having marginals $\mu_X$ and $\mu_Y$.
\begin{theorem}
	[\textbf{Minimum Mutual Information of Joint Gaussian Measures}]
	\label{theorem:minimum-Mutual-Info-Gaussian}
	Let $\H_1,\H_2$ be two separable Hilbert spaces.
	Let $\mu_X = \Ncal(m_X, C_X)\in \Gauss(\H_1)$, $\mu_Y = \Ncal(m_Y, C_Y)\in \Gauss(\H_2)$,
	$\ker(C_X) = \ker(C_Y) =\{0\}$.
	Let $\gamma \in \ADM(\mu_X, \mu_Y), \gamma_0 \in \Gauss(\mu_X, \mu_Y)$,
	$\gamma_0$ is equivalent to $\mu_X \otimes \mu_Y$. Assume that $\gamma$ and $\gamma_0$ have the same covariance operator $\Gamma$ and that $\mu_X\otimes \mu_Y$ has covariance operator $\Gamma_0$. 
	Then
	\begin{align}
	\label{equation:Mutual-Info-Gaussian}
	\KL(\gamma||\mu_X \otimes \mu_Y) \geq \KL(\gamma_0 ||\mu_X \otimes \mu_Y) = -\frac{1}{2}\log\det(I-V^{*}V).
	\end{align}
	Equality happens if and only if $\gamma = \gamma_0$.
	Here $V$ is the unique bounded linear operator satisfying $V \in \HS(\H_2,\H_1)$, $||V||< 1$, 
	such that $\Gamma = \Gamma_0^{1/2}\begin{pmatrix}I & V \\ V^{*} & I\end{pmatrix}\Gamma_0^{1/2}$.
\end{theorem}
The operator $V$ in Theorem \ref{theorem:minimum-Mutual-Info-Gaussian} is defined in Section \ref{section:mutual-info-Gauss}.
 It links the covariance operators $C_X$ and $C_Y$ with the cross-covariance operator $C_{XY}$ via
 the relation 
 $C_{XY} = C_X^{1/2}VC_Y^{1/2}$. In Eq.\eqref{equation:Mutual-Info-Gaussian},
 $\det$ refers to the {\it Fredholm determinant} (see e.g. \cite{Simon:1977}).
 Let $A \in \Tr(\H)$, then the Fredholm determinant of $I+A$ is given by $\det(I+A) =\prod_{j=1}^{\infty}(1+\lambda_j)$,
 where $\{\lambda_j\}_{j\in \Nbb}$ are the eigenvalues of $A$.
 
Theorem \ref{theorem:minimum-Mutual-Info-Gaussian} in turn follows from the following more general result on the KL divergence on Hilbert space. It states in particular that if $\mu$ is a Gaussian measure on $\H$,
then among all probability measures with the same mean and covariance operator, $\KL(\gamma ||\mu)$ is minimum if and only if $\gamma$ is Gaussian. 
 
\begin{theorem}
	\label{theorem:log-Radon-Nikodym-integral}
	Let $\mu = \Ncal(m_1, Q)$, $\ker(Q) = \{0\}$. Let $\nu = \Ncal(m_2, R_{\nu})$ be equivalent to $\mu$. Let $S \in \Sym(\H) \cap \HS(\H)$ be such that $R_{\nu} = Q^{1/2}(I-S)Q^{1/2}$.
	Let $\gamma \in \Pcal_2(\H)$ be absolutely continuous with respect to $\mu$, with mean $m_3$, $m_3-m_1 \in \myIm(Q^{1/2})$, and covariance operator
	$R_{\gamma} = Q^{1/2}AQ^{1/2}$, $A \in \Sym^{+}(\H)$. Assume further that one of the following (non-mutually exclusive) conditions hold
	\begin{enumerate}
		\item $S \in \Tr(\H)$.
		\item $I-A \in \HS(\H)$.
	\end{enumerate}
	Then the following decomposition holds
	\begin{align}
	\KL(\gamma||\mu) = \KL(\gamma ||\nu) 
	&-\frac{1}{2}||(I-S)^{-1/2}Q^{-1/2}(m_2 - m_1)||^2
	\nonumber
	\\
	&-\frac{1}{2}[\la S(I-S)^{-1}Q^{-1/2}(m_3-m_1), Q^{-1/2}(m_3-m_1)\ra]
	\nonumber
	\\
	& + \la (I-S)^{-1}Q^{-1/2}(m_2-m_1), Q^{-1/2}(m_3-m_1)\ra
	\nonumber
	\\
	& +\frac{1}{2}\trace[S(I-(I-S)^{-1}A)]- \frac{1}{2}\log\dettwo(I-S).
	\label{equation:log-RN-integral-1} 
	\end{align}
	In particular, for $m_3 = m_2$ and $A = I-S$, i.e. $R_{\gamma} = R_{\nu}$,
	\begin{align}
	\KL(\gamma||\mu) &= \KL(\gamma||\nu) + \frac{1}{2}||Q^{-1/2}(m_2-m_1)||^2 -\frac{1}{2}\log\dettwo(I-S)
	\\
	& = \KL(\gamma||\nu) + \KL(\nu ||\mu).
	\end{align}
	In this case $\KL(\gamma||\mu) \geq \KL(\nu||\mu)$, with equality if and only if $\gamma = \nu$,
	i.e. if and only if $\gamma$ is Gaussian.
\end{theorem}
In Theorem \ref{theorem:log-Radon-Nikodym-integral},
$\dettwo$ refers to the {\it Hilbert-Carleman determinant} (see e.g. \cite{Simon:1977}). For $A \in \HS(\H)$, 
the Hilbert-Carleman determinant of $I+A$ is defined by $\dettwo(I+A) = \det[(I+A)\exp(-A)]$, with $\det$ being the Fredholm determinant.
The different conditions in Theorem \ref{theorem:log-Radon-Nikodym-integral} can be satisfied simultaneously.
In particular, they are both automatically satisfied in the case
$\dim(\H) < \infty$.

{\bf Entropic $2$-Wasserstein distance between Gaussian measures}.
In the following, 
let $\mu_i = \Ncal(m_i, C_i)$, $i=0,1$, be two Gaussian measures on $\H$, 
where $m_i \in \H$, $C_i \in \Sym^{+}(\H) \cap \Tr(\H)$.
Consider the cost function $c(x,y) = ||x-y||^2$ on $\H \times \H$ and the corresponding optimization problem
\begin{align}
\label{equation:OT-ep-square}
\OT^{\ep}_{d^2}(\mu_0, \mu_1) = \min_{\gamma \in \ADM(\mu_0, \mu_1)}\bE_{\gamma}||x-y||^2 + \ep \KL(\gamma || \mu_0 \otimes \mu_1).
\end{align}
A direct consequence of Theorem \ref{theorem:minimum-Mutual-Info-Gaussian} is that if $C_0,C_1$ are nonsingular, then
a minimizer of problem \eqref{equation:OT-ep-square} is necessarily a joint Gaussian measure of $\mu_0$ and $\mu_1$. We show that this holds in the general setting, i.e. in {\it both } nonsingular and singular cases.
The following result gives the explicit formula for this minimizer, which is unique for any $\ep > 0$.
It is proved in Section \ref{section:OT-entropic-Gaussian}, using two different methods: (i) by directly solving 
the optimization \eqref{equation:OT-ep-square}, and (ii) by solving the corresponding Schr\"odinger system.
\begin{theorem}
	[\textbf{Optimal entropic transport plan}]
	\label{theorem:optimal-joint-square-gauss}
	Let $\mu_0=\Ncal(m_0, C_0)$, $\mu_1=\Ncal(m_1,C_1)$. For each fixed $\ep > 0$, problem \eqref{equation:OT-ep-square}
	has a unique minimizer $\gamma^{\ep}$, 
	which is the Gaussian measure 
	\begin{align}
	\label{equation:minimizing-Gaussian-measure}
	\gamma^{\ep} &= \Ncal\left(\begin{pmatrix} m_0 \\ m_1 \end{pmatrix},
	\begin{pmatrix} C_0 & C_{XY}
	\\
	C_{XY}^{*} & C_1\end{pmatrix}
	\right),
	\\
	\text{where  } C_{XY} &= \frac{2}{\ep}C_0^{1/2}\left(I+\frac{1}{2}M^{\ep}_{01}\right)^{-1}
	C_0^{1/2}C_1.
	\end{align}
	The Radon-Nikodym derivative of $\gamma^{\ep}$ with  respect to $\mu_0 \otimes \mu_1$ is given by
	\begin{align}
	\label{equation:gamma-opt-gauss}
	\frac{d\gamma^{\ep}}{d(\mu_0 \otimes \mu_1)}(x,y) = \alpha^{\ep}(x)\beta^{\ep}(y)\exp\left(-\frac{||x-y||^2}{\ep}\right),
	\end{align}
	where the functions $\alpha^{\ep}:\H \mapto \R$ and $\beta^{\ep}:\H \mapto \R$ take the form
	\begin{align}
	\alpha^{\ep}(x) &= \exp\left(\la x-m_0, A(x-m_0)\ra + \frac{2}{\ep}\la x-m_0, m_0 - m_1\ra + a\right),
	\\
	\beta^{\ep}(y) & = \exp\left(\la y-m_1, B(y-m_1)\ra + \frac{2}{\ep}\la y-m_1, m_1 - m_0\ra + b\right).
	\end{align}
	The constants $a,b \in \R$ and
	the operators $A, B: \H \mapto \H$ are given by
	\begin{equation}
	\begin{aligned}
	A &= \frac{1}{\epsilon}I - \frac{2}{\epsilon^2}C_1^{1/2}\left[I + \frac{1}{2}M^{\ep}_{10}\right]^{-1}C_1^{1/2},
	\\ 
	B &= \frac{1}{\epsilon}I -  \frac{2}{\epsilon^2}C_0^{1/2}\left[I + \frac{1}{2}M^{\ep}_{01}\right]^{-1}C_0^{1/2},
	\\
	\exp(a+b) &= \exp\left(\frac{||m_0-m_1||^2}{\ep}\right)\sqrt{\det\left(I+ \frac{1}{2}M^{\ep}_{01}\right)}.
	\end{aligned}
	\label{equation:ABab}
	\end{equation}
	Here $\det$ is the Fredholm determinant and
	$M^{\ep}_{ij}: \H \mapto \H$, 
	are defined by
	\begin{align}
	\label{equation:M-ep}
	M^{\ep}_{ij} = -I + \left(I + \frac{16}{\epsilon^2}C_i^{1/2}C_jC_i^{1/2}\right)^{1/2}, \;\;\; i,j=0,1.
	\end{align}
\end{theorem}
\begin{remark} The operator $M^{\ep}_{ij}$ as defined in Eq.\eqref{equation:M-ep} can be rewritten as
	\begin{equation}
	M^{\ep}_{ij} = \frac{16}{\epsilon^2}C_i^{1/2}C_jC_i^{1/2}\left[I + \left(I + \frac{16}{\epsilon^2}C_i^{1/2}C_jC_i^{1/2}\right)^{1/2}\right]^{-1},
	\end{equation}
	from which it follows that $M^{\ep}_{ij} \in \Sym^{+}(\H) \cap \Tr(\H)$.
	Thus for the operator
	\begin{equation}
	I + \frac{1}{2}M^{\ep}_{ij} = \frac{1}{2}I + \frac{1}{2}\left(I + \frac{16}{\epsilon^2}C_i^{1/2}C_jC_i^{1/2}\right)^{1/2},
	\end{equation}
	the Fredholm determinant $\det\left(I + \frac{1}{2}M^{\ep}_{ij}\right)$ is well-defined and positive.
	{\color{black}We note that the constants $a,b \in \R$ in Eq.\eqref{equation:ABab} (therefore the individual functions $\alpha^{\ep}(x)$, $\beta^{\ep}(y)$) are not uniquely specified, only their sum $a+b$ (therefore the product $\alpha^{\ep}(x)\beta^{\ep}(y)$) is uniquely specified, which in turn uniquely determines $\gamma^{\ep}$.}
\end{remark}
{\bf Finite-dimensional case}. For $\H = \R^n$ and $C_0, C_1 \in \Sym^{++}(n)$, the cross-covariance operator $C_{XY}$ in Theorem \ref{theorem:optimal-joint-square-gauss}
has the same expression as that given in Theorem 1 in \cite{Janati2020entropicOT}, namely
{\color{black}$C_{XY} = \frac{\ep}{4}[-I + C_0^{1/2}(I+ \frac{16}{\ep^2}C_0^{1/2}C_1C_0^{1/2})C_0^{-1/2}]$}.

\begin{theorem}
	[\textbf{Entropic 2-Wasserstein distance between Gaussian measures on Hilbert space}]
	\label{theorem:OT-regularized-Gaussian}
	Let $\mu_0 = \Ncal(m_0, C_0)$ and $\mu_1 = \Ncal(m_1, C_1)$. For each fixed $\ep > 0$,
	\begin{align}
	\OT^{\ep}_{d^2}(\mu_0, \mu_1) &= ||m_0-m_1||^2 + \trace(C_0) + \trace(C_1) - \frac{\ep}{2}\trace(M^{\ep}_{01})
	\nonumber
	\\
	& \quad +\frac{\ep}{2}\log\det\left(I + \frac{1}{2}M^{\ep}_{01}\right).
	\label{equation:gauss-entropic-2-Wasserstein-infinite}
	\end{align}
	
	\begin{remark}
While we use the term {\it entropic distance}, $\OT^{\ep}_{d^2}$ is neither a distance nor a divergence, since
generally $\OT^{\ep}_{d^2}(\mu, \mu) \neq 0$, as noted before.
	\end{remark}

\end{theorem}
{\bf Connection with the entropic Kantorovich duality formulation}.
Following \cite{DMaGer19}, let {\color{black}$(X, d)$} be a Polish space. {\color{black}For a probability measure $\mu$ on $X$}, the class of Entropy-Kantorovich potentials is defined by the set of measurable functions $\varphi$ on $X$ satisfying
\begin{equation}
\Lexp(X,\mu) = \left\lbrace \varphi:X \to [-\infty, \infty] \, : \,
0<\bE_\mu\left[\exp\left(\frac{1}{\epsilon} \varphi\right)\right] < \infty
\right\rbrace.
\end{equation}
The dual Kantorovich functional {\color{black}with the cost function $c(x,y)=d^2(x,y)$ is defined (see \cite{DMaGer19})} to be 
\begin{equation}
\label{equation:kanto-dual}
D(\varphi, \psi) = \bE_{\mu_0}[\varphi] + \bE_{\mu_1}[\psi]
-\ep \left(\bE_{\mu_0 \otimes \mu_1} \left[\exp\left(\frac{(\varphi\oplus \psi)-d^2}{\ep}\right)\right]-1\right),
\end{equation}
where $\left(\varphi \oplus \psi\right)(x,y) = \varphi(x) + \psi(y)$.
For $X = \H$, $c(x,y) = ||x-y||^2$, 
the \emph{entropic Kantorovich dual formulation} of $\OT^\epsilon_{d^2}$ is given by
\cite{DMaGer19,feydy18,genevay17,GigTamBB18,LeoSurvey}
\begin{align}
\label{equation:dual-problem}
\OT^{\ep}_{d^2}(\mu_0, \mu_1) = \sup_{\varphi \in \Lexp(\H, \mu_0), \psi \in \Lexp(\H, \mu_1)}D(\varphi,\psi).
\end{align}
The following shows that in our setting, the supremum in \eqref{equation:dual-problem} is attained.
\begin{corollary}
	\label{corollary:dual-attain}
	Let $\mu_i = \Ncal(m_i, C_i),i=0,1$.
	Let $\varphi^{\ep} = \ep \log{\alpha^{\ep}}$, {\color{black}$\psi^{\ep} = \ep\log{\beta^{\ep}}$},
	with $\alpha^{\ep}$, $\beta^{\ep}$ as defined 
	in Theorem \ref{theorem:optimal-joint-square-gauss}. Then $\varphi^{\ep} \in \Lexp(\H, \mu_0)$, $\psi^{\ep} \in \Lexp(\H, \mu_1)$, and 
	\begin{align}
	\OT^{\ep}_{d^2}(\mu_0, \mu_1) = D(\varphi^{\ep}, \psi^{\ep}).
	\end{align}
\end{corollary}
We remark that in the case the cost function $c(x,y)$ is {\it bounded}, much more can be said about the duality formulation,
see \cite{DMaGer19}.
\begin{theorem}
	[\textbf{Convexity}]
	\label{theorem:entropic-OT-convexity}	Let $\mu_0 = \Ncal(m_0, C_0)$, $\mu_1 = \Ncal(m, X)$, then
	$\OT^{\ep}_{d^2}(\mu_0, \mu_1)$ is convex in each argument. In particular, let $C_0$ be fixed,
	then the function $X\mapto F_E(X) = \OT^{\ep}_{d^2}(\Ncal(0,C_0), \Ncal(0,X))$ is convex in $X\in \Sym^{+}(\H)\cap\Tr(\H)$. Furthermore, it is strictly convex if $C_0$ is strictly positive, i.e. $\ker(C_0) = \{0\}$.
\end{theorem}

\begin{theorem}
	[\textbf{Sinkhorn divergence between Gaussian measures on Hilbert space}]
	\label{theorem:Sinkhorn-Gaussian-Hilbert}
	Let $\mu_0 = \Ncal(m_0, C_0)$, $\mu_1 = \Ncal(m_1, C_1)$. Then
	\begin{equation}
	\begin{aligned}
	\Srm^{\ep}_{d^2}(\mu_0, \mu_1) &= ||m_0 - m_1||^2 + \frac{\ep}{4}\trace\left[M^{\ep}_{00} - 2M^{\ep}_{01} + M^{\ep}_{11}\right] 
	\\
	& \quad + \frac{\ep}{4}\log\left[\frac{\det\left(I + \frac{1}{2}M^{\ep}_{01}\right)^2}{\det\left(I + \frac{1}{2}M^{\ep}_{00}\right)\det\left(I + \frac{1}{2}M^{\ep}_{11}\right)}\right].
	\end{aligned}
	\label{equation:gauss-sinkhorn-infinite}
	\end{equation}
\end{theorem}
{\bf Finite-dimensional case}. For $C_0,C_1 \in \Sym^{+}(n)$, one verifies directly that Eqs.\eqref{equation:gauss-entropic-2-Wasserstein-infinite} and \eqref{equation:gauss-sinkhorn-infinite}
reduce to Eqs.\eqref{equation:gauss-entropic-Wass-finite} and \eqref{equation:gauss-sinkhorn-finite}, respectively.

In \cite{feydy18}, the Sinkhorn divergence was proved to be {\it convex} in each variable for either a {\it compact} metric space $X$ or
for measures with {\it bounded support} on $\R^n$, with cost function $c(x,y) = ||x-y||^p$, $p=1,2$.
In \cite{janati2020debiased}, this was shown for sub-Gaussian measures on $\R^n$, $c(x,y) = ||x-y||^2$.
The following shows {\it strict convexity} for Gaussian measures on $\H$ with $c(x,y) = ||x-y||^2$.
This property is crucial for guaranteeing the uniqueness of the barycenter problem below.

\begin{theorem}
	[\textbf{Strict convexity of Sinkhorn divergence}]
	\label{theorem:strict-convexity-Sinkhorn}
	Let $\mu_0 = \Ncal(m_0,C_0)$, $\mu_1 = \Ncal(m,X)$. Then $\Srm^{\ep}_{d^2}(\mu_0, \mu_1)$ is strictly convex in each argument. In particular, let $C_0$ be fixed, then the function $X \mapto F_S(X) = \Srm^{\ep}_{d^2}[\Ncal(0,C_0), \Ncal(0,X)]$ is strictly convex in  $X \in \Sym^{+}(\H)\cap \Tr(\H)$.
\end{theorem}

In \cite{feydy18}, positivity of the Sinkhorn divergence was proved for either a {\it compact} metric space $X$ or
for measures with {\it bounded support} on $\R^n$, with cost function $c(x,y) = ||x-y||^p$, $p=1,2$.
We now show that in the Gaussian case, with $c(x,y) = ||x-y||^2$, this holds in the much more general Hilbert space setting. 
\begin{theorem}
	[\textbf{Positivity of Sinkhorn divergence}]
	\label{theorem:positivity}
The function $\Srm^{\ep}_{d^2}: \Gauss(\H) \times \Gauss(\H) \mapto \R_{\geq 0}$ satisfies
\begin{align}
\Srm^{\ep}_{d^2}(\mu_0, \mu_1) &\geq 0,\quad \quad \quad \forall \mu_0, \mu_1 \in \Gauss(\H),
\\
\Srm^{\ep}_{d^2}(\mu_0, \mu_1) &= 0 \equivalent \mu_0 = \mu_1.
\end{align} 
\end{theorem}

{\bf Differentiability}. In the finite-dimensional setting, $\Sym^{++}(n)$ is an open subset in the vector space
$\Sym(n)$ and Fr\'echet derivatives can be properly defined on this set.
In contrast, $\Sym^{++}(\H)\cap \Tr(\H)$ is {\it not} an open subset
of $\Sym(\H)$ when $\dim(\H) = \infty$.
In particular, for the exact $2$-Wasserstein distance, the function
$X \mapto W^2_2(\Ncal(0,C_0), \Ncal(0,X)) = \trace(C_0) + \trace(X) -2\trace[(C_0^{1/2}XC_0^{1/2})^{1/2}]$
is {\it not } Fr\'echet differentiable on $\Sym^{++}(\H)\cap\Tr(\H)$.

To discuss Fr\'echet differentiability of both $\OT^{\ep}_{d^2}$ and $\Srm^{\ep}_{d^2}$ in the covariance operator component, we 
can extend their definition, thanks to the regularization effect, to a larger, open set, containing $\Sym^{++}(\H)\cap \Tr(\H)$, as follows.
\begin{theorem}
	[\textbf{Differentiability of entropic Wasserstein distance and Sinkhorn divergence}]
	\label{theorem:differentiability}
	Let $C_0 \in \Sym^{+}(\H) \cap \Tr(\H)$ be fixed. Both functions $F_E$ in Theorem \ref{theorem:entropic-OT-convexity} and $F_S$ in Theorem \ref{theorem:strict-convexity-Sinkhorn} are well-defined and twice Fr\'echet differentiable 
	on the open, convex set $\Omega = \{X \in \Sym(\H)\cap \Tr(\H): I + c_{\ep}^2C_0^{1/2}XC_0^{1/2} > 0\} \supset \Sym^{+}(\H) \cap \Tr(\H)$, 
	$c_{\ep} = \frac{4}{\ep}$. Furthermore, $F_S$ is strictly convex and $F_E$ is convex on $\Omega$, 
	with strict convexity if $C_0$ is strictly positive.
\end{theorem}
{\color{black}We note that in Theorem \ref{theorem:differentiability}, the set $\Omega$ is open in the space 
	$\Sym(\H) \cap \Tr(\H)$ under the trace norm $||\;||_{\tr}$ topology (see Lemma \ref{lemma:open-set-omega})}.

The expressions for $\OT^{\ep}_{d^2}$ and $\Srm^{\ep}_{d^2}$ in Theorems \ref{theorem:OT-regularized-Gaussian} and \ref{theorem:Sinkhorn-Gaussian-Hilbert} are not intuitively close to the exact OT formula, which is the same as
in the finite-dimensional setting \cite{Gelbrich:1990Wasserstein,cuesta1996:WassersteinHilbert}. Theorem \ref{theorem:equivalent-formulas} below gives  equivalent formulas that better express the connections
between the exact and regularized settings.

\begin{theorem}[\textbf{Equivalent expressions for entropic $2$-Wasserstein distance and Sinkhorn divergence}]
\label{theorem:equivalent-formulas}
	Let $\mu_0 = \Ncal(m_0, C_0)$, $\mu_1 = \Ncal(m_1, C_1)$. 
	Define $L^{\epsilon}_{ij} = \frac{\epsilon^2}{8}(-I + (I + \frac{16}{\ep^2} C_i^{1/2}C_jC_i^{1/2})^{1/2}) = \frac{\ep^2}{8}M^{\ep}_{ij}$, $i,j=0,1$. Then
	\begin{align}
	\OT^{\epsilon}_{d^2}(\mu_0, \mu_1) &= {\color{black}||m_0 - m_1||^2} + \trace(C_0) + \trace(C_1) -2\trace[C_0^{1/2}C_1C_0^{1/2} - L^{\epsilon}_{01}]^{1/2}
	\nonumber
	\\
	&\quad +{\color{black}\frac{\ep}{2}\log\det\left(I + \frac{1}{2}M^{\ep}_{01}\right)}.
	\end{align}
	\begin{align}
	S^{\epsilon}_{d^2}(\mu_0, \mu_1) &= {\color{black}||m_0- m_1||^2} 
	\nonumber
	\\
	&\quad+ \trace[(C_0^2 - L^{\epsilon}_{00})^{1/2} - 2(C_0^{1/2}C_1C_0^{1/2} - L^{\epsilon}_{01})^{1/2}+ (C_1^2 - L^{\epsilon}_{11})^{1/2}]
	\nonumber
	\\
	&\quad +{\color{black}\frac{\ep}{4}\log\left[\frac{\det\left(I + \frac{1}{2}M^{\ep}_{01}\right)^2}{\det\left(I + \frac{1}{2}M^{\ep}_{00}\right)\det\left(I + \frac{1}{2}M^{\ep}_{11}\right)}\right]}. 
	\end{align}
	Furthermore, we verify directly that
	\begin{align}
	\lim_{\epsilon \approach 0}\OT^{\epsilon}_{d^2}(\mu_0, \mu_1) &= {\color{black}||m_0 - m_1||^2} + \trace(C_0) + \trace(C_1) -2\trace[C_0^{1/2}C_1C_0^{1/2}]^{1/2}
	\nonumber
	\\
	& = W^2_2(\mu_0, \mu_1),
	\\
	\lim_{\epsilon \approach \infty}\OT^{\epsilon}_{d^2}(\mu_0, \mu_1) &= {\color{black}||m_0 - m_1||^2} + \trace(C_0) + \trace(C_1).
	\end{align}
	\begin{align}
	\lim_{\epsilon \approach 0}S^{\epsilon}_{d^2}(\mu_0, \mu_1) &={\color{black}||m_0 - m_1||^2} + \trace(C_0) + \trace(C_1) -2\trace[C_0^{1/2}C_1C_0^{1/2}]^{1/2}
	\nonumber
	\\
	& = W^2_2(\mu_0, \mu_1),
	\\
	\lim_{\epsilon \approach \infty}S^{\epsilon}_{d^2}(\mu_0, \mu_1) &={\color{black}||m_0 - m_1||^2}.
	\end{align}
\end{theorem}

{\bf Entropic 2-Wasserstein barycenter of Gaussian measures}.

Given $N$ probability measures $\mu_i\in \Probs(\H)$, $i=1,2,..,N$, the entropic barycenter $\bar{\mu}$ with weights $w_i> 0$, $\sum_{i=1}^N w_i = 1$, is defined as the \emph{Fr\'echet mean}

\begin{equation}\label{equation:entropic-barycenter}
\bar{\mu} := \argmin\limits_{\mu \in \Probs(\H)} \sum_{i=1}^N w_i\OT^\epsilon_{d^2}(\mu, \mu_i), \quad w_i > 0, \quad \sum^N_{i=1}w_i = 1.
\end{equation}
In the current work, we consider barycenter of the $N$ Gaussian measures $\{\mu_i\}_{i=1}^N$ in the set of all Gaussian measures
on $\H$
\begin{equation}
	\label{equation:minimization-entropic-barycenter-Gaussian}
	\bar{\mu} = \argmin_{\mu \in \Gauss(\H)}\sum_{i=1}^N w_i \OT^{\ep}_{d^2}(\mu, \mu_i), \;\;\;\sum_{i=1}^Nw_i = 1, w_i > 0, 1 \leq i \leq N.
\end{equation}
The entropic barycenter problem illustrates clearly the bias effect of the entropic regularization term, as analyzed in the finite-dimensional case, e.g. \cite{janati2020debiased,Janati2020entropicOT}. This bias effect is sharp when $\dim(\H) = \infty$, with the finite-dimensional barycenter equation failing to generalize to this case. In the following, we call a barycenter {\it trivial} if it is a Dirac delta measure in $\H$.

\begin{theorem}
	[\textbf{Entropic Barycenter of Gaussians}]
	\label{theorem:entropic-barycenter-Gaussian}
	Let $\mu_i=\Ncal\left(m_i,C_i\right)$, $i=1,2,...,N$ be a set of Gaussian measures on $\H$.
	Assume at least one of the $C_i's$ is strictly positive. 
	Then on $\Gauss(\H)$, problem \eqref{equation:minimization-entropic-barycenter-Gaussian} is strictly convex and the first order minimality condition is
	\begin{equation}
	\label{equation:entropic-barycenter-Gaussian}
	\begin{aligned}
	\sum_{i=1}^N w_i \left[C_i^{1/2}\left(I+\left(I + \frac{16}{\ep^2}C_i^{1/2}XC_i^{1/2}\right)^{1/2}\right)^{-1}C_i^{1/2}\right] = \frac{\ep}{4}I.
	\end{aligned}
	\end{equation}
	\begin{enumerate}
		\item 
		If $\ep I \geq 2\sum_{i=1}^Nw_iC_i$, the unique barycenter is the Dirac delta measure centered at $\bar{m} = 
		\sum_{i=1}^Nw_im_i$. 
		
		\item If $\dim(\H) < \infty$,  a {\bf necessary condition} for the existence of a non-trivial barycenter is
		\begin{equation}
		\label{equation:condition-entropic-barycenter-Gaussian}
		0 < \ep I < 2\sum_{i=1}^Nw_iC_i.
		\end{equation}
		A {\bf sufficient condition} for the existence of a non-trivial barycenter is
		\begin{equation}
		\label{equation:condition-entropic-barycenter-Gaussian-strictlypositive}
		\exists \alpha \in \R, \alpha > 0 \text{ such that } C_i \geq \alpha I, 1\leq i \leq N, \text{and }0 < \ep < 2\alpha.
		\end{equation}
		In this case, the barycenter is unique and is the Gaussian measure $\Ncal(\bar{m}, \bar{C})$, where $\bar{C}>0$ is the unique solution of Eq.\eqref{equation:entropic-barycenter-Gaussian}.		
		Equivalently, $\bar{C}$ is the unique strictly positive solution of the following equation
		\begin{equation}
		\label{equation:entropic-barycenter-Gaussian-strictlypositive}
		X  = \frac{\epsilon}{4}\sum_{i=1}^Nw_i\left[-I + \left(I + \frac{16}{\epsilon^2}X^\frac{1}{2}C_iX^\frac{1}{2}\right)^\frac{1}{2}\right].
		\end{equation}
				
		\item If $\dim(\H) = \infty$, then Eq.\eqref{equation:entropic-barycenter-Gaussian} has {\bf no} solution
		in $\Sym^{+}(\H)$.
		If {\color{black}$\ep I \ngeq 2\sum_{i=1}^Nw_iC_i$, i.e., $\exists u\in \H, ||u||=1$, such that $0 < \ep = \ep||u||^2  < 2 \sum_{i=1}^Nw_i\la u, C_iu\ra$}, then the barycenter, if it exists, is not the 
		Dirac measure centered at $\bar{m}$.
	\end{enumerate}
\end{theorem}

{\bf Discussion of results}.
Eq.\eqref{equation:entropic-barycenter-Gaussian} is the first order optimality condition
for the strictly convex problem \eqref{equation:minimization-entropic-barycenter-Gaussian}. It has sharply different behavior when 
$\dim(\H) = \infty$ compared with the case $\dim(\H) < \infty$, as we stated.

Conditions \eqref{equation:condition-entropic-barycenter-Gaussian} and \eqref{equation:condition-entropic-barycenter-Gaussian-strictlypositive}
show that, even in the finite-dimensional setting,  a non-trivial entropic barycenter of Gaussian measures
exists if and only if $\epsilon$ is sufficiently small.
Condition \eqref{equation:condition-entropic-barycenter-Gaussian}, namely $0 < \ep I < 2 \sum_{i=1}^Nw_iC_i$, {\color{black}i.e., $\forall u\in \H, ||u||=1$, $0 < \ep = \ep||u||^2  < 2 \sum_{i=1}^Nw_i\la u, C_iu\ra$},
{\it cannot} be satisfied in the case $\dim(\H) = \infty$ since $I$ is not trace class.
{\color{black}However, we may still have $\ep I \ngeq 2\sum_{i=1}^Nw_iC_i$, i.e., $\exists u\in \H, ||u||=1$, such that $0 < \ep = \ep||u||^2  < 2 \sum_{i=1}^Nw_i\la u, C_iu\ra$. In this case, the barycenter $\Ncal(\bar{m},\bar{C})$ may still exist when $\dim(\H) = \infty$, but $\bar{C}$ can neither be the trivial solution $0$ nor
a solution of the first order optimality condition given by Eq.\eqref{equation:entropic-barycenter-Gaussian}}.

Under condition \eqref{equation:condition-entropic-barycenter-Gaussian-strictlypositive}, Eq.\eqref{equation:entropic-barycenter-Gaussian} has a unique solution, which is strictly positive. Then Eq.\eqref{equation:entropic-barycenter-Gaussian-strictlypositive} also has a unique strictly positive solution,
but it also has the trivial solution $X_0 = 0$ and {\it uncountably infinitely many positive solutions, which are singular} (see Theorem \ref{theorem:singular-fixedpoint} and Proposition \ref{proposition:entropic-barycenter-Gaussian-strictlypositive}).

\begin{remark}
	In general, for $\ep > 0$, the entropic barycenter problem \eqref{equation:minimization-entropic-barycenter-Gaussian} does {\it not} make much sense. Consider the following one-dimensional scenario, where
	$C_1 = \cdots = C_N = \sigma^2 > 0$. The unique solution 
	of \eqref{equation:entropic-barycenter-Gaussian} is
	$\bar{C} = \sigma^2 - \frac{\ep}{2}> 0 \equivalent \sigma^2 > \frac{\ep}{2}$.
	Thus if $\ep \geq 2 \sigma^2$, then Equation \eqref{equation:entropic-barycenter-Gaussian} has no positive solution.
	The above solution $\bar{C}$ is also obtained for the case $N=1$, $w_1 =1$$, C_1 = \sigma^2 > 0$, in which case
	a sensible solution should be $\bar{C} = C_1$, i.e. the barycenter of a set of one point should the point itself.
	This seemingly pathological behavior is not necessarily surprising, since $\OT^{\ep}_{d^2}$ is neither a distance nor a divergence.
\end{remark}

{\bf Sinkhorn barycenter of Gaussian measures}.
We now consider the barycenter problem with respect to the Sinkhorn divergence. 
For a set of probability measures $\{\mu_i\}_{i=1}^N$ on $\H$ and a set of weights $\sum_{i=1}^Nw_i =1$, $w_i > 0$, $1\leq i\leq N$, their Sinkhorn barycenter is defined to be
\begin{equation}
\label{equation:barycenter-Sinkhorn}
\bar{\mu} = \argmin_{\mu \in \Probs(\H)}\sum_{i=1}^N w_i \Srm^{\ep}_{d^2}(\mu, \mu_i), \;\;\;\sum_{i=1}^Nw_i = 1, w_i > 0,  1\leq i \leq N.
\end{equation}
In the current work, we consider barycenter of the $N$ Gaussian measures $\{\mu_i\}_{i=1}^N$ in the set of all Gaussian measures
on $\H$
\begin{equation}
\label{equation:barycenter-Sinkhorn-Gaussian}
\bar{\mu} = \argmin_{\mu \in \Gauss(\H)}\sum_{i=1}^N w_i \Srm^{\ep}_{d^2}(\mu, \mu_i), \;\;\;\sum_{i=1}^Nw_i = 1, w_i > 0, 1 \leq i \leq N.
\end{equation}

In contrast to the entropic barycenter problem, the debiased Sinkhorn barycenter problem has a consistent generalization to the infinite-dimensional setting, with a unique solution that is valid in both singular and nonsingular cases.
\begin{theorem}
	[\textbf{Sinkhorn barycenter of Gaussian measures}]
	\label{theorem:barycenter-sinkhorn-Gaussian}
	Consider the set of Gaussian measures $\{\Ncal(m_i, C_i)\}_{i=1}^N$ on $\H$, with $m_i \in \H$ and $C_i \in \Sym^{+}(\H) \cap \Tr(\H)$.
	Their Sinkhorn barycenter in $\Gauss(\H)$, as defined in Eq.\eqref{equation:barycenter-Sinkhorn-Gaussian}, is the unique Gaussian measure $\bar{\mu} = \Ncal(\bar{m}, \bar{C})$, where
	$\bar{m} = \sum_{i=1}^Nw_i m_i$
	and $\bar{C}$ is the unique solution of the following equation, with $c_{\ep} =\frac{4}{\ep}$,
	\begin{align}
	\label{equation:barycenter-sinkhorn-positive}
	X & = \left(I+ \left(I + c_{\ep}^2X^2\right)^{1/2} \right)^{1/2} \sum_{i=1}^Nw_i\left[C_i^{1/2}\left(I+\left(I + c_{\ep}^2C_i^{1/2}XC_i^{1/2}\right)^{1/2}\right)^{-1}C_i^{1/2}\right]
	\nonumber
	\\
	& \quad \times \left(I+ \left(I + c_{\ep}^2X^2\right)^{1/2} \right)^{1/2}.
	\end{align}
	Furthermore, $\bar{C}$ is strictly positive if and only if
	\begin{equation}
	\label{equation:barycenter-condition-strictlypositive}
	\sum_{i=1}^Nw_iC_i > 0.
	\end{equation}
	Under the additional hypothesis that $\bar{C} > 0$, $\bar{C}$ is equivalently the unique {\it strictly positive} solution of the following equation
	\begin{equation}
	\label{equation:barycenter-sinkhorn-strictlypositive}
	X  = \frac{1}{c_{\ep}}\left[- I +  \left(\sum_{i=1}^Nw_i\left(I + c_{\ep}^2X^{1/2}C_iX^{1/2}\right)^{1/2}\right)^2\right]^{1/2}.
	\end{equation}
\end{theorem}
Define the following map $\Fcal:\Sym^{+}(\H) \mapto \Sym^{+}(\H)$ by
\begin{align}
\label{equation:map-fixedpoint-barycenter-sinkhorn}
\Fcal(X) &=  \left(I+ \left(I + c_{\ep}^2X^2\right)^{1/2} \right)^{1/2} \sum_{i=1}^Nw_i\left[C_i^{1/2}\left(I+\left(I + c_{\ep}^2C_i^{1/2}XC_i^{1/2}\right)^{1/2}\right)^{-1}C_i^{1/2}\right]
\nonumber
\\
& \quad \quad \quad \quad \times \left(I+ \left(I + c_{\ep}^2X^2\right)^{1/2} \right)^{1/2}.
\end{align}
Then the unique solution of Eq.\eqref{equation:barycenter-sinkhorn-positive} is the unique fixed point of $\Fcal$.

{\bf Limiting cases}. When $\ep \approach 0$, both Eqs. \eqref{equation:entropic-barycenter-Gaussian-strictlypositive} and \eqref{equation:barycenter-sinkhorn-strictlypositive} 
become
\begin{align}
	X &= \sum_{i=1}^Nw_i(X^{1/2}C_iX^{1/2})^{1/2}.
\end{align}
In the 
finite-dimensional setting, 
this is the barycenter equation for the exact $2$-Wasserstein distance \cite{Agueh:2011barycenters}, assuming that $\bar{C} > 0$. As of the current writing, to the best of our knowledge, a rigorous proof
for the infinite-dimensional case has not yet been established. We note that the proof given in \cite{Mallasto:NIPS2017Wasserstein}, which uses the transport map in \cite{cuesta1996:WassersteinHilbert} to compute gradients, is only applicable in the case $\dim(\H) < \infty$, since the transport map is generally unbounded 
when $\dim(\H) = \infty$, see also the discussion in \cite{masarotto2019procrustes}.

\begin{theorem}
	[{Singular solutions of fixed point equations}]
	\label{theorem:singular-fixedpoint}
	Let $\dim(\H)\geq 2$. The following equations have uncountably infinitely many positive, singular solutions,
	apart from the trivial solution $X_0 = 0$. Here $c_{\ep} = \frac{4}{\ep}$.
	\begin{enumerate}
		\item Exact $2$-Wasserstein barycenter, $\sum_{i=1}^Nw_iC_i > 0$,
	\begin{align}
	X &= \sum_{i=1}^Nw_i(X^{1/2}C_iX^{1/2})^{1/2}.
	\label{equation:barycenter-exact}
	\end{align}
	Without the condition $\sum_{i=1}^Nw_iC_i > 0$, this equation always has at least one positive, nonzero singular solution.
	
	\item Entropic Wasserstein barycenter, $2\leq \dim(\H) < \infty$, $0 < \ep I < 2\sum_{i=1}^Nw_iC_i$,
	\begin{align}
	X = \frac{1}{c_{\ep}}\sum_{i=1}^Nw_i\left[-I + \left(I + c_{\ep}^2X^\frac{1}{2}C_iX^\frac{1}{2}\right)^\frac{1}{2}\right].
	\end{align}
	\item Sinkhorn barycenter (second version), $\sum_{i=1}^Nw_iC_i > 0$,
	\begin{align}
	X &= \frac{1}{c_{\ep}}\left[- I +  \left(\sum_{i=1}^Nw_i\left(I + c_{\ep}^2X^{1/2}C_iX^{1/2}\right)^{1/2}\right)^2\right]^{1/2}.
	\end{align}	
	Without the condition $\sum_{i=1}^Nw_iC_i > 0$, this equation always has at least one positive, nonzero singular solution.
\end{enumerate}
\end{theorem}

{\bf Comparison of Eqs. \eqref{equation:barycenter-sinkhorn-positive} and \eqref{equation:barycenter-sinkhorn-strictlypositive}}.
Eq.\eqref{equation:barycenter-sinkhorn-positive} is general and is always valid whether the covariance operators $C_i$'s and the barycenter $\bar{C}$ are singular or nonsingular.  
This equation always has a unique solution, which can be positive and singular or strictly positive.
Furthermore, this solution is strictly positive if and only if $\sum_{i=1}^Nw_iC_i > 0$. 

Eq.\eqref{equation:barycenter-sinkhorn-strictlypositive} has the same form as the finite-dimensional version reported in 
\cite{Mallasto2020entropyregularized} and \cite{Janati2020entropicOT}.
It is, however, only applicable for finding the barycenter in the case it is {\it strictly positive}, since it is derived under this explicit assumption.
If the solution of Eq.\eqref{equation:barycenter-sinkhorn-positive} is strictly positive, then it is also the unique strictly positive solution of Eq.\eqref{equation:barycenter-sinkhorn-strictlypositive}.
Eq.\eqref{equation:barycenter-sinkhorn-strictlypositive}, however, 
always has the trivial solution $X=0$. Furthermore, 
if $\dim(\H) \geq 2$ and at least one of the $C_i$'s is strictly positive,  
then it has {\it uncountably infinitely many positive solutions}, which are singular (Proposition \ref{proposition:G-infinitely-many-fixedpoint}). The same phenomenon happens for the
barycenter
equation \eqref{equation:barycenter-exact} in the exact, unregularized setting \cite{Agueh:2011barycenters}, i.e. when $\ep = 0$.

As we discuss in detail in Section \ref{section:compare-barycenter}, it is not straightforward to extend the
approach in \cite{Agueh:2011barycenters} for Eq.\eqref{equation:barycenter-exact} and \cite{Janati2020entropicOT} for Eq.\eqref{equation:barycenter-sinkhorn-strictlypositive}
in the finite-dimensional setting, which requires all $C_i$'s to be strictly positive for the existence
of $\bar{C}>0$, to the infinite-dimensional setting. This is because it is no longer possible to uniformly lower bound
the $C_i$'s by $\alpha I$ for some $\alpha >0$ and it is not clear whether this lower bound can be replaced by another strictly positive operator.

We also remark on our condition $\sum_{i=1}^Nw_iC_i > 0$ for the strict positivity of $\bar{C}$, which is more  general
than requiring all $C_i$'s to be strictly positive (e.g. \cite{Janati2020entropicOT}).
 In fact, we can have $\sum_{i=1}^Nw_iC_i > 0$, guaranteeing $\bar{C} > 0$,  with all $C_i$'s being singular (see Section \ref{section:compare-barycenter} for an example).

{\bf The RKHS setting}.
We now apply the abstract Hilbert space setting above to the reproducing kernel Hilbert space (RKHS) setting. In this case, we obtain an interpolation between Kernel Maximum Mean Discrepancy (MMD) and Kernelized $\Lcal^2$-Wasserstein Distance. The RKHS formulas are expressed explicitly in terms of the kernel Gram matrices, which are readily computable.

Let {$\X$} be a complete separable metric space.
Let 
{$K$} be a continuous positive definite kernel on {$\X \times \X$}. Then the reproducing kernel Hilbert space (RKHS) {$\H_K$}
induced by {$K$} is separable (\cite{Steinwart:SVM2008}, Lemma 4.33).
Let {$\Phi: \X \mapto \H_K$} be the corresponding canonical feature map, so that 
$K(x,y) = \la \Phi(x), \Phi(y)\ra_{\H_K}$ $\forall (x,y) \in \X \times \X$.
Let $\rho$ be a Borel probability measure on $\X$ such that
\begin{align}
	\label{equation:Phi-finite-secondmoment}
\int_{\X}||\Phi(x)||_{\H_K}^2d\rho(x) = \int_{\X}K(x,x)d\rho(x) < \infty.
\end{align}
Then the RKHS mean vector $\mu_{\Phi} \in \H_K$ and covariance operator {$C_{\Phi}:\H_K \mapto \H_K$} induced by the feature map $\Phi$ are both well-defined and are given by
\begin{align}
\mu_{\Phi} &= \int_{\X}\Phi(x)d\rho(x) \in \H_K, 
\\\;\;\;
C_{\Phi} &= \int_{\X}(\Phi(x)-\mu_{\Phi})\otimes (\Phi(x)-\mu_{\Phi})d\rho(x).
\end{align}
Here the rank-one operator $u \otimes v$ is defined by $(u\otimes v)w = \la v,w\ra_{\H_K}u$, $u,v,w \in \H_K$.
Then {$C_{\Phi}$} is a positive trace class operator on $\H_K$ (see e.g. \cite{Minh:Covariance2017}).


Let {$\Xbf =[x_1, \ldots, x_m]$,$m \in \Nbb$,} be a data matrix randomly sampled from {$\X$} according to a Borel probability distribution $\rho$ {\color{black}satisfying Eq.\eqref{equation:Phi-finite-secondmoment}}, where {$m \in \Nbb$} is the number of observations.
The feature map {$\Phi$} on {$\Xbf$} 
defines
the bounded linear operator
$\Phi(\Xbf): \R^m \mapto \H_K, \Phi(\Xbf)\b = \sum_{j=1}^mb_j\Phi(x_j) , \b \in \R^m$.
The corresponding empirical mean vector and covariance operator for {$\Phi(\Xbf)$}
are defined to be
\begin{align}
\mu_{\Phi(\Xbf)} &= \frac{1}{m}\sum_{j=1}^m\Phi(x_j) = \frac{1}{m}\Phi(\Xbf)\1_m,
\label{equation:mean-RKHS}
\\
C_{\Phi(\Xbf)} &= \frac{1}{m}\Phi(\Xbf)J_m\Phi(\Xbf)^{*}: \H_K \mapto \H_K,
\label{equation:covariance-operator}
\end{align}
where $J_m = I_m -\frac{1}{m}\1_m\1_m^T,\1_m = (1, \ldots, 1)^T \in \R^m$, is the centering matrix,
with $J_m^2 = J_m$ and $AJ_m$ is the matrix obtained from the (possibly infinite) matrix $A$ by subtracting the mean column.

Let {$\Xbf = [x_i]_{i=1}^m$, $\Ybf = [y_i]_{i=1}^m$}, be two random data matrices sampled from {$\X$} according to two Borel probability distributions $\rho_0$ and $\rho_1$ on $\X$, {\color{black}both satisfying 
	Eq.\eqref{equation:Phi-finite-secondmoment}}. Let $\mu_{\Phi(\Xbf)}, \mu_{\Phi(\Ybf)}$ and $C_{\Phi(\Xbf)}$, $C_{\Phi(\Ybf)}$
be the corresponding mean vectors and covariance operators induced by
$K$, respectively.
Let us derive
the explicit expression for $\OT^{\epsilon}_{d^2}(\mu_0, \mu_1)$ and $S^{\epsilon}_{d^2}(\mu_0, \mu_1)$
when $\mu_0 \sim \Ncal(\mu_{\Phi(\Xbf)}, C_{\Phi(\Xbf)})$, $\mu_1 \sim \Ncal(\mu_{\Phi(\Ybf)}, C_{\Phi(\Ybf)})$. Define the following $m \times m$ Gram matrices
\begin{equation}
\begin{aligned}
K[\Xbf] = \Phi(\Xbf)^{*}\Phi(\Xbf),\;K[\Ybf] = \Phi(\Ybf)^{*}\Phi(\Ybf), K[\Xbf,\Ybf] = \Phi(\Xbf)^{*}\Phi(\Ybf),
\\
{\color{black}(K[\Xbf])_{ij} = K(x_i, x_j), (K[\Ybf])_{ij} = K(y_i, y_j), (K[\Xbf,\Ybf])_{ij} = K(x_i, y_j)},
\\
{\color{black}1 \leq i,j \leq m}.
\end{aligned}
\end{equation}

\begin{theorem}
	\label{theorem:RKHS-distance}
	Let $\epsilon > 0$ be fixed.
	For $\mu_0 =\Ncal(\mu_{\Phi(\Xbf)}, C_{\Phi(\Xbf)})$, $\mu_1 = \Ncal(\mu_{\Phi(\Ybf)}, C_{\Phi(\Ybf)})$,
	\begin{align}
	\OT^{\epsilon}_{d^2}(\mu_0, \mu_1) &= \frac{1}{m^2}\1_m^T(K[\Xbf] + K[\Ybf] - 2K[\Xbf,\Ybf])\1_m
	\nonumber
	\\
	&+ \frac{1}{m}\trace(K[\Xbf]J_m) + \frac{1}{m}\trace(K[\Ybf]J_m) 
	\\&
	-\frac{\epsilon}{2}\trace\left[-I + \left(I + \frac{16}{\epsilon^2m^2} J_mK[\Xbf,\Ybf]J_mK[\Ybf,\Xbf]J_m\right)^{1/2}\right]
	\nonumber
	\\
	&+\frac{\epsilon}{2}\log\det\left(\frac{1}{2}I + \frac{1}{2}\left(I + \frac{16}{\epsilon^2m^2}J_mK[\Xbf,\Ybf]J_mK[\Ybf,\Xbf]J_m \right)^{1/2}\right).
	\nonumber
	\end{align}
	\begin{align}
	S^{\epsilon}_{d^2}(\mu_0, \mu_1) &= \frac{1}{m^2}\1_m^T(K[\Xbf] + K[\Ybf] - 2K[\Xbf,\Ybf])\1_m
	\nonumber
	\\
	& 
	+\frac{\epsilon}{4}\trace\left[- I + \left(I + \frac{16}{\epsilon^2m^2} (J_mK[\Xbf]J_m)^2\right)^{1/2}\right]
	\nonumber
	\\
	&+ 
	\frac{\epsilon}{4}\trace\left[- I + \left(I + \frac{16}{\epsilon^2m^2} (J_mK[\Ybf]J_m)^2\right)^{1/2}\right]
	\nonumber
	\nonumber
	\\
	& -
	\frac{\epsilon}{2}\trace\left[- I + \left(I + \frac{16}{\epsilon^2m^2} J_mK[\Xbf,\Ybf]J_mK[\Ybf,\Xbf]J_m\right)^{1/2}\right]
	\nonumber
	\nonumber
	\\
	& + \frac{\epsilon}{2}\log\det\left(\frac{1}{2}I + \frac{1}{2}\left(I + \frac{16}{\epsilon^2m^2}J_mK[\Xbf,\Ybf]J_mK[\Ybf,\Xbf]J_m \right)^{1/2}\right)
	\nonumber
	\\
	& - \frac{\epsilon}{4}\log\det\left(\frac{1}{2}I + \frac{1}{2}\left(I + \frac{16}{\epsilon^2m^2}(J_mK[\Xbf]J_m)^2 \right)^{1/2}\right)
	\nonumber
	\\
	& - \frac{\epsilon}{4}\log\det\left(\frac{1}{2}I + \frac{1}{2}\left(I + \frac{16}{\epsilon^2m^2}(J_mK[\Ybf]J_m)^2 \right)^{1/2}\right).
	\end{align}
	As $\epsilon \approach \infty$, we recover the empirical squared Kernel MMD distance \cite{Gretton:MMD12a}
	\begin{align}
	\lim_{\epsilon \approach \infty}S^{\epsilon}_{d^2}(\mu_0, \mu_1) 
	&= 
	||\mu_{\Phi(\Xbf)} - \mu_{\Phi(\Ybf)}||^2_{\H_K} 
	\nonumber
	\\
	&= \frac{1}{m^2}\1_m^T(K[\Xbf] + K[\Ybf] - 2K[\Xbf,\Ybf])\1_m.
	\end{align}
	As $\ep \approach 0$, we recover the Kernelized Wasserstein Distance \cite{zhang2019:OTRKHS,Minh:2019AlphaProcrustes}
	\begin{equation}
	\begin{aligned}
	\lim_{\ep \approach 0}S^{\epsilon}_{d^2}(\mu_0, \mu_1) 
	&=
	\frac{1}{m^2}\1^T_m[K[\Xbf] + K[\Ybf] - 2K[\Xbf,\Ybf]]\1_m
	\\
	&\quad+ \frac{1}{m}\trace(K[\Xbf]J_m) +  \frac{1}{m}\trace(K[\Ybf]J_m)
	\\
	&\quad- \frac{2}{m}\trace[J_mK[\Xbf,\Ybf]J_mK[\Ybf,\Xbf]J_m]^{1/2}.
	\end{aligned}
	\end{equation}	
\end{theorem}
\begin{remark}
	To keep our expressions simple, we have assumed that the number of data points in 
	$\Xbf = [x_i]_{i=1}^m$ and $\Ybf = [y_i]_{i=1}^n$ are the same, i.e. $m=n$.
	The extension to the case $m \neq n$ is straightforward.
\end{remark}

\section{From Gaussian measures to Gaussian processes}
\label{section:gaussian-measure-to-processes}

Let us discuss the translation of
the results for Gaussian measures on an abstract Hilbert space $\H$ into the setting of Gaussian processes,
see also \cite{Panaretos:jasa2010,Fremdt:2013testing,Pigoli:2014,Mallasto:NIPS2017Wasserstein,masarotto2019procrustes}.
Consider the following correspondence between
Gaussian measures and Gaussian processes with paths in a Hilbert space
\cite{Rajput1972gaussianprocesses}.
Let $(\Omega, \Fcal, P)$ be a probability space.  Let $T$ be an index set.
Let $(T, \Acal, \nu)$ be a measurable space, $\nu$ nonnegative, $\sigma$-finite, such that $\Lcal^2(T, \Acal, \nu)=\Lcal^2(T,\nu)$ is separable (e.g. $T \subset \R^n$ measurable, $\Acal = \Bsc(T)$, $\nu$ is the Lebesgue measure).
Let $\xi = (\xi_t)_{t \in T} = (\xi(t, \omega))_{t \in T}$ be 
a real $\Fcal/\Acal$-measurable Gaussian process on $(\Omega, \Fcal,P)$, with mean $m(t)$ and covariance function $K(s,t)$, denoted by $\GP(m,K)$. The sample paths
$\xi(\cdot, \omega) \in \H = \Lcal^2(T, \nu)$ almost $P$-surely, i.e.
$\int_{T}\xi^2(t,\omega)d\nu(t) < \infty$ {almost $P$-surely},
if and only if (\cite{Rajput1972gaussianprocesses}, Theorem 2 and Corollary 1)
\begin{align}
\label{equation:condition-Gaussian-process-paths}
\int_{T}m^2(t)d\nu(t) < \infty, \;\;\; \int_{T}K(t,t)d\nu(t) < \infty.
\end{align}
Then $\xi$ induces the following Gaussian measure $P_{\xi}$ on $(\H, \Bsc(\H))$
\begin{align}
P_{\xi}(B) = P\{\omega \in \Omega: \xi(\cdot, \omega) \in B\}, \;\;\; B \in \Bsc(\H),
\end{align}
with mean $m \in \H$ and covariance operator
$C_K: \H \mapto \H$, defined by
\begin{align}
(C_Kf)(s) = \int_{T}K(s,t)f(t)d\nu(t), \;\;\; f \in \H.
\end{align}
Conversely, let $\mu$ be a Gaussian measure on $(\H = \Lcal^2(T, \Acal, \nu), \Bsc(\H))$. Then
there is an $\Fcal/\Acal$-measurable Gaussian process $\xi = (\xi_t)_{t \in T}$ on $(\Omega, \Fcal, P)$
with sample paths in $\H$, such that the induced probability measure is $P_{\xi} = \mu$.

{\bf Correspondence between covariance function and covariance operator via Mercer Theorem}.
Covariance functions, being positive definite kernels, can be fully expressed via their induced covariance operators, as follows. In the following, let $T$ be a $\sigma$-compact metric space, that is
$T = \cup_{i=1}^{\infty}T_i$, where $T_1 \subset T_2 \subset \cdots$, with each $T_i$ compact.
Let $\nu$ be  a positive, non-degenerate Borel measure on $T$,
i.e. $\nu(B) > 0$ for any open $U \subset T$, with $\nu(T_i) < \infty \forall i \in \Nbb$.

\begin{theorem}
	[\textbf{Mercer Theorem} - version in \cite{Sun2005MercerNoncompact}]
	Let $T$ be a $\sigma$-compact metric space and $\nu$ a positive, non-degenerate Borel measure on $T$.
	Let $K: T \times T\mapto \R$ be continuous, positive definite. 
	Assume furthermore that
	\begin{align}
	\label{equation:Mercer-cond1}
	&\int_{T}K(s,t)^2d\nu(t) < \infty, \;\;\;\forall s \in T.
	\\
	&
	\int_{T \times T}K(s,t)^2d\nu(s)d\nu(t) < \infty.
	\label{equation:Mercer-cond2}
	\end{align}
	Then $C_{K}$ is Hilbert-Schmidt, self-adjoint, positive. Let $\{\lambda_k\}_{k=1}^{\infty}$ be the eigenvalues of $C_{K}$, with corresponding 
	orthonormal eigenvectors $\{\phi_k\}_{k=1}^{\infty}$. Then
	\begin{align}
	\label{equation:Mercer-expansion}
	K(s,t) = \sum_{k=1}^{\infty}\lambda_k \phi_k(s)\phi_k(t),
	\end{align}
	where the series converges absolutely for each pair $(s,t) \in T \times T$ and uniformly on each compact subset of $T \times T$.
\end{theorem}
Mercer Theorem thus describes the covariance function $K(s,t)$ fully and explicitly via its covariance operator $C_{K}$.
Since $K$ is positive definite, $K(s,t)^2 \leq K(s,s)K(t,t)$ $\forall s,t \in T \times T$.
Thus the condition $\int_{T}K(t,t)d\nu(t) < \infty$ in \eqref{equation:condition-Gaussian-process-paths}
implies both conditions \eqref{equation:Mercer-cond1} and \eqref{equation:Mercer-cond2} in Mercer Theorem and 
from \eqref{equation:Mercer-expansion}
\begin{align}
\trace(C_K) = \sum_{k=1}^{\infty}\lambda_k = \int_{T}K(t,t)d\nu(t) < \infty.
\end{align}

We now generalize ideas in  
\cite{Panaretos:jasa2010,Fremdt:2013testing,Pigoli:2014,Mallasto:NIPS2017Wasserstein,masarotto2019procrustes}, using the fact that Gaussian processes are fully determined by their mean and covariance functions.
Most importantly, the following incorporates
Mercer Theorem to quantify the correspondence $\GP(m,K) \equivalent \Ncal(m,C_K)$.
\begin{definition}
	[\textbf{Divergence between Gaussian processes}]
	\label{definition:divergence-Gaussian-Process}
	Let $T$ be a $\sigma$-compact metric space, $\nu$ a positive, non-degenerate Borel measure on $T$.
	Let $\H = \Lcal^2(T,\Bsc(T),\nu)$. Let $\xi^i = \GP(m_i,K_i)$, $i=1,2$,  be
	two Gaussian processes with mean $m_i \in \H$, covariance function $K_i$ continuous, and $\int_{T}K_i(t,t)d\nu(t)< \infty$.
	Let $D$ be a divergence function on $\Gauss(\H)\times \Gauss(\H)$. The corresponding divergence
	$D_{\GP}$ 
	between $\xi^1$ and $\xi^2$ is defined to be
	\begin{align}
	D_{\GP}(\xi^1|| \xi^2) = D(\Ncal(m_1,C_{K_1}) ||\Ncal(m_2, C_{K_2})).
	\end{align}
\end{definition}
Mercer Theorem immediately implies the following.
\begin{theorem}
	\label{theorem:divergence-Gaussian-Process}
	Assume the hypothesis in Definition \ref{definition:divergence-Gaussian-Process}. Then
	\begin{align}
	D_{\GP}(\xi^1 ||\xi^2) &\geq 0,
	\\
	D_{\GP}(\xi^1 ||\xi^2) &= 0 \equivalent m_1 = m_2, K_1 = K_2.
	\end{align}
\end{theorem}
\begin{remark}
	In our current context, we can immediately apply Definition \ref{definition:divergence-Gaussian-Process} and Theorem \ref{theorem:divergence-Gaussian-Process} to the exact Wasserstein distances and Sinkhorn divergences.
	Definition \ref{definition:divergence-Gaussian-Process} can also be extended to cover the entropic $\OT^{\ep}_c$ distances, however since they are not metrics/divergences, Theorem \ref{theorem:divergence-Gaussian-Process} no longer holds.
\end{remark}

\section{Kullback-Leibler divergence between Gaussian measures}
\label{section:KL}
In this section, we briefly review the KL divergence between Gaussian measures on Hilbert space 
and prove Theorem \ref{theorem:log-Radon-Nikodym-integral}.
For a separable Hilbert space $\H$, consider the set $\Pcal(\H)$ of probability measures on $(\H, \Bsc(\H))$,
where $\Bsc(\H)$ denotes the Borel $\sigma$-algebra on $\H$. We focus on the subset $\Pcal_2(\H)$ defined by
\begin{align}
\Pcal_2(\H) = \left\{\mu \in \Pcal(\H): \int_{\H}||x||^2d\mu(x) < \infty\right\}.
\end{align}
For $\mu \in \Pcal_2(\H)$, its mean vector $m\in \H$ and covariance operator $C:\H \mapto \H$ are well-defined and are given by
\begin{align}
\la m, u\ra &= \int_{\H}\la x, u\ra d\mu(x),\;\;\; u \in \H,
\\
\la Cu, v\ra  &= \int_{\H} \la x - m, u\ra \la x- m, v\ra d\mu(x), \;\;\; u,v \in \H.
\end{align}
In particular $C$ is a self-adjoint, positive, and trace class operator on $\H$.
 
We recall that for two measures $\mu$ and $\nu$ on a measure space $(\Omega, \Fcal)$, with $\mu$ $\sigma$-finite, 
$\nu$ is said to be {\it absolutely continuous} with respect to $\mu$, denoted by $\nu << \mu$, if for any $A \in \Fcal$, $\mu(A) = 0 \imply \nu(A) = 0$. In this case, the Radon-Nikodym derivative $\frac{d\nu}{d\mu} \in \Lcal^1(\mu)$ is 
well-defined.
The Kullback-Leibler (KL) divergence between $\nu$ and $\mu$ is defined by
\begin{align}
\KL(\nu||\mu) = 
\left\{
\begin{matrix}
\int_{\Omega}\log\left\{\frac{d\nu}{d\mu}(x)\right\}d\nu(x) & \text{if $\nu << \mu$},
\\
\infty & \text{otherwise}.
\end{matrix}
\right.
\end{align} 
If $\mu << \nu$ and $\nu << \mu$, then we say that $\mu$ and $\nu$ are {\it equivalent}, denoted by $\mu \sim \nu$.
We say that $\mu$ and $\nu$ are {\it mutually singular}, denoted by $\mu \perp \nu$, if there exist $A,B \in \Fcal$ such that
$\mu(A) = \nu(B) = 1$ and $A\cap B = \emptyset$.

{\bf Equivalence of Gaussian measures}. 
Let $Q,R$ be two self-adjoint, positive trace class operators on $\H$ such that $\ker(Q) = \ker(R) = \{0\}$. Let $m_1, m_2 \in \H$. 
A fundamental result in the theory of Gaussian measures is the Feldman-Hajek Theorem \cite{Feldman:Gaussian1958}, \cite{Hajek:Gaussian1958}, which states that 
two Gaussian measures $\mu = \Ncal(m_1,Q)$ and 
$\nu = \Ncal(m_2, R)$
are either mutually singular or
equivalent, {\color{black}that is either $\mu \perp \nu$ or $\mu \sim \nu$}.
The necessary and sufficient conditions for
the equivalence of the two Gaussian measures $\nu$ and $\mu$ are given by the following.
\begin{theorem}
	[\cite{Bogachev:Gaussian}, Corollary 6.4.11, \cite{DaPrato:PDEHilbert}, Theorems  1.3.9 and 1.3.10]
	\label{theorem:Gaussian-equivalent}
	Let $\H$ be a separable Hilbert space. Consider two Gaussian measures $\mu = \Ncal(m_1, Q)$ and
	$\nu = \Ncal(m_2, R)$ on $\H$. Then $\mu$ and $\nu$ are equivalent if and only if the following
	hold
	\begin{enumerate}
		\item $m_2 - m_1 \in \myIm(Q^{1/2})$.
		\item There exists  $S \in  \Sym(\H) \cap \HS(\H)$, without the eigenvalue $1$, such that
		$R = Q^{1/2}(I-S)Q^{1/2}$.
	\end{enumerate}
\end{theorem}

We now recall results on Kullback-Leibler divergences between two Gaussian measures $\mu = \Ncal(m_1,Q)$ and $\nu = \Ncal(m_2,R)$ on $\H$. If $\mu \perp \nu$, then
$\KL(\nu|| \mu) = \infty$. If $\mu \sim \nu$, then we have the following result. 
\begin{theorem}
	[\cite{Minh:2020regularizedDiv}]
	\label{theorem:KL-gaussian}
		Let $\mu = \Ncal(m_1, Q)$, $\nu = \Ncal(m_2, R)$, with $\ker(Q) = \ker{R} = \{0\}$, and
	$\mu \sim \nu$.
	Let $S \in \HS(\H)\cap \Sym(\H)$, $I-S > 0$, 
	be such that $R = Q^{1/2}(I-S)Q^{1/2}$, then
	\begin{align}
	\label{equation:KL-gaussian}
	\KL(\nu ||\mu) = \frac{1}{2}||Q^{-1/2}(m_2-m_1)||^2 -\frac{1}{2}\log\dettwo(I-S).
	\end{align}
\end{theorem}
%
For two equivalent Gaussian measures $\mu,\nu$ on $\H$, the Radon-Nikodym derivative involves only the means and covariance operators (\cite{Minh:2020regularizedDiv}, Theorem 11).
This motivates Theorem \ref{theorem:log-Radon-Nikodym-integral}, which seeks to extend the validity of Eq.\eqref{equation:KL-gaussian} 
by generalizing the expression $\int_{\H}\log{(\frac{d\nu}{d\mu})}d\nu$ to $\int_{\H}\log{(\frac{d\nu}{d\mu})}d\gamma$
where $\gamma \in \Pcal_2(\H)$ is any probability measure with the same mean and covariance operator as $\nu$.

To prove Theorem \ref{theorem:log-Radon-Nikodym-integral}, in the following
we utilize the concept of {\it white noise mapping}, see e.g. \cite{DaPrato:2006,DaPrato:PDEHilbert}.
For $\mu = \Ncal(m, Q)$, $\ker(Q) = \{0\}$, we  define
$\Lcal^2(\H, \mu) = \Lcal^2(\H, \Bsc(\H),\mu) = \Lcal^2(\H, \Bsc(\H), \Ncal(m,Q))$.
Consider the following mapping
\begin{align}
&W:Q^{1/2}(\H) \subset \H \mapto \Lcal^2(\H,\mu), \;\; z  \in Q^{1/2}(\H) \mapto W_z \in \Lcal^2(\H, \mu),
\\
&W_z(x) = \la x -m, Q^{-1/2}z\ra,  \;\;\; z \in Q^{1/2}(\H), x \in \H.
\end{align}
For any pair $z_1, z_2 \in Q^{1/2}(\H)$, we have by definition of the covariance operator
\begin{align}
\la W_{z_1}, W_{z_2}\ra_{\Lcal^2(\H,\mu)}
&= \int_{\H}\la x -m, Q^{-1/2}z_1\ra\la x-m, Q^{-1/2}z_2\ra\Ncal(m, Q)(dx)
\nonumber
\\
& = \la Q(Q^{-1/2}z_1), Q^{-1/2}z_2\ra 
= \la z_1, z_2\ra_{\H}.
\end{align}
Thus the map $W:Q^{1/2}(\H) \mapto \Lcal^2(\H, \mu)$ is an isometry, that is
\begin{align}
||W_z||_{\Lcal^2(\H,\mu)} = ||z||_{\H}, \;\;\; z \in Q^{1/2}(\H).
\end{align}
Since $\ker(Q) = \{0\}$, the subspace $Q^{1/2}(\H)$ is dense in $\H$ and the map $W$ can be uniquely extended to all of $\H$, as follows.
For any $z \in \H$, let $\{z_n\}_{n\in \Nbb}$ be a sequence in $Q^{1/2}(\H)$ with $\lim_{n \approach \infty}||z_n -z||_{\H} = 0$.
Then $\{z_n\}_{n \in \Nbb}$ is a Cauchy sequence in $\H$, so that by isometry, $\{W_{z_n}\}_{n\in \Nbb}$ is also
a Cauchy sequence in $\Lcal^2(\H, \mu)$, thus converging to a unique element in $\Lcal^2(\H, \mu)$.
Thus
 we can define
\begin{align}
W: \H \mapto \Lcal^2(\H, \mu),  \;\;\; z \in \H \mapto \Lcal^2(\H, \mu)
\end{align}
by the following unique limit in $\Lcal^2(\H, \mu)$
\begin{align}
W_z(x) = \lim_{n \approach \infty}W_{z_n}(x) = \lim_{n \approach \infty}\la x-m, Q^{-1/2}z_n\ra.
\end{align}
The map $W: \H \mapto \Lcal^2(\H, \mu)$ is called the {\it white noise mapping}
associated with the measure $\mu = \Ncal(m,Q)$.
%
$W_z$ can be 
expressed  explicitly in terms of the finite-rank orthogonal projections
$P_N = \sum_{k=1}^N e_k \otimes e_k$
onto the $N$-dimensional subspace of $\H$
spanned by $\{e_k\}_{k=1}^N$, $N \in \Nbb$, where $\{e_k\}_{k \in \Nbb}$ are the orthonormal eigenvectors of $Q$
{\color{black}corresponding to eigenvalues $\{\lambda_k\}_{k \in \Nbb}$, which are all strictly positive by the assumption $\ker(Q) = \{0\}$}.
For any $z \in \H$, we have
\begin{align}
P_Nz = \sum_{k=1}^N\la z,e_k \ra e_k \imply Q^{-1/2}P_Nz = \sum_{k=1}^N\frac{1}{\sqrt{\lambda_k}}\la z, e_k\ra e_k.
\end{align}
Thus $Q^{-1/2}P_Nz$ is always well-defined $\forall z \in \H$. 
Furthermore, for all $x,y \in \H$, 
\begin{align}
\la Q^{-1/2}P_Nx, y\ra = \sum_{j=1}^N\frac{1}{\sqrt{\lambda_j}}\la x, e_j\ra \la y, e_j\ra =\la x, Q^{-1/2}P_Ny\ra.
\end{align}
The operator $Q^{-1/2}P_N:\H \mapto \H$ is bounded and self-adjoint $\forall N \in \Nbb$.
Since the 
sequence $\{P_Nz\}_{N\in \Nbb}$ converges to $z$ in $\H$, we
have, in the $\Lcal^2(\H,\mu)$ sense,
\begin{align}
W_z(x) = \lim_{N \approach \infty}W_{P_Nz}(x) = \lim_{N \approach \infty}\la x-m, Q^{-1/2}P_Nz\ra.
\end{align}

The Radon-Nikodym derivative between two equivalent Gaussian measures on $\H$ is expressed explicitly via the
 white noise mapping, as follows.

\begin{theorem}
	[\cite{Minh:2020regularizedDiv}, Theorem 11]
	\label{theorem:radon-nikodym-infinite}
	Let $\mu = \Ncal(m_1, Q)$, $\nu = \Ncal(m_2,R)$, $\ker(Q) = \ker(R) = 0$ be equivalent, that is $m_2 - m_1 \in \myIm(Q^{1/2})$, $R = Q^{1/2}(I-S)Q^{1/2}$ for $S \in \Sym(\H) \cap \HS(\H)$.
	Let $\{\alpha_k\}_{k\in \Nbb}$ be the eigenvalues of $S$, with corresponding orthonormal eigenvectors
	$\{\phi_k\}_{k \in \Nbb}$. Let $W$ be the white noise mapping induced by $\mu$. 
	The Radon-Nikodym derivative $\frac{d\nu}{d\mu}$
	is given by
	\begin{align}
	\label{equation:RN-infinite}
	\frac{d\nu}{d\mu}(x) = \exp\left[-\frac{1}{2}\sum_{k=1}^{\infty}\Phi_k(x)\right]\exp\left[-\frac{1}{2}||(I-S)^{-1/2}Q^{-1/2}(m_2 - m_1)||^2\right],
	\end{align}
	where for each $k \in \Nbb$,
	\begin{align}
	\label{equation:Phik}
	\Phi_k = \frac{\alpha_k}{1-\alpha_k}W^2_{\phi_k} - \frac{2}{1-\alpha_k}\la Q^{-1/2}(m_2-m_1), \phi_k\ra W_{\phi_k}+ \log(1-\alpha_k).
	\end{align}
	The series $\sum_{k=1}^{\infty}\Phi_k$ converges in $\Lcal^1(\H,\mu)$ and $\Lcal^2(\H,\mu)$.
\end{theorem}

\begin{lemma}
	\label{lemma:whitenoise-square-integral-1}
	Let $\mu = \Ncal(m_1,Q)$, $\ker{Q} = \{0\}$. Let $W$ be its induced white noise mapping.
	Let $\nu \in \Pcal_2(\H)$, $\nu << \mu$, with mean $m_2$, where $m_2-m_1 \in \myIm(Q^{1/2})$, and covariance operator
	$R = Q^{1/2}AQ^{1/2}$, where $A \in \Sym^{+}(\H)$. Then
	\begin{align}
		\int_{\H}W_z(x)d\nu(x) &= \la Q^{-1/2}(m_2 - m_1), z\ra.
\\
	\la W_{z_1}, W_{z_2}\ra_{\Lcal^2(\H, \nu)} &
	= \la Az_1, z_2\ra, \quad \quad \quad \quad z_1, z_2 \in \H,
	\nonumber
	\\
	&\quad + \la Q^{-1/2}(m_2-m_1),z_1\ra\la Q^{-1/2}(m_2-m_1), z_2\ra.
	\\
	||W_z||^2_{\Lcal^2(\H, \nu)} & 
	= \la Az,z\ra +|\la Q^{-1/2}(m_2-m_1), z\ra|^2, \;\;z \in \H.
	\end{align}
\end{lemma}
\begin{proof}
	For $z \in \myIm(Q^{1/2})$, 
	$W_z(x) = \la x-m_1, Q^{-1/2}z\ra$. Thus for $z_1, z_2 \in \myIm(Q^{1/2})$,
	\begin{align*}
	&\la W_{z_1}, W_{z_2}\ra_{\Lcal^2(\H, \nu)} = 
	\int_{\H}W_{z_1}(x)W_{z_2}(x)d\nu(x)
	\\
	&= \int_{\H}\la x-m_2 + m_2 - m_1, Q^{-1/2}z_1\ra\la x-m_2 + m_2-m_1, Q^{-1/2}z_2\ra d\nu(x)
	\\
	&= \int_{\H}\la x-m_2, Q^{-1/2}z_1\ra\la x-m_2, Q^{-1/2}z_2\ra d\nu(x) +\la m_2-m_1, Q^{-1/2}z_1\ra \la m_2-m_1, Q^{-1/2}z_2\ra
	\\
	&
	= \la RQ^{-1/2}z_1, Q^{-1/2}z_2\ra + \la Q^{-1/2}(m_2-m_1), z_1\ra \la Q^{-1/2}(m_2-m_1), z_2\ra
	\\
	& = \la Az_1, z_2\ra + \la Q^{-1/2}(m_2-m_1), z_1\ra \la Q^{-1/2}(m_2-m_1), z_2\ra,
	\end{align*}
	since $R = Q^{1/2}AQ^{1/2}$.
	In particular, for $z \in \myIm(Q^{1/2})$,
	\begin{align*}
	&||W_z||^2_{\Lcal^2(\H, \nu)} = \int_{\H}W_z^2(x)d\nu(x) = \la Az,z\ra + |\la Q^{-1/2}(m_2-m_1), z\ra|^2
	\\
	&\leq [||A||+ ||Q^{-1/2}(m_2-m_1)||^2]||z||^2 = [||A||+||Q^{-1/2}(m_2-m_1)||^2]||W_{z}||^2_{\Lcal^2(\H, \mu)}.
	\end{align*}
	In general, for $z \in \H$, in the $\Lcal^2(\H, \mu)$ sense,
	$W_z(x) = \lim_{N \approach \infty}W_{P_Nz}(x) = \lim_{N \approach \infty}\la x-m, Q^{-1/2}P_Nz\ra$.
	By the assumption $\nu << \mu$, we have $\mysupp(\nu) \subset \mysupp(\mu)$.
	The sequence $\{W_{P_Nz}\}_{N \in \Nbb}$ is a Cauchy sequence in $\Lcal^2(\H, \mu)$ converging to $W_z$, with a subsequence converging pointwise $\mu$-almost everywhere, hence $\nu$-almost everywhere. This subsequence is also 
	a Cauchy sequence in $\Lcal^2(\H, \nu)$, converging to a unique limit, with a subsubsequence converging pointwise $\nu$-almost everywhere. Thus this limit must be $W_z$. Hence  $W_z(x) = \lim_{N \approach \infty}W_{P_Nz}(x)$
	in the $\Lcal^2(\H, \nu)$ sense also. Therefore
	\begin{align*}
	&\la W_{z_1}, W_{z_2}\ra_{\Lcal^2(\H, \nu)} = \lim_{N \approach \infty}
	\la W_{P_Nz_1}, W_{P_Nz_2}\ra_{\Lcal^2(\H, \nu)} 
	\\
	&=\lim_{N \approach \infty}\left[\la AP_Nz_1, P_Nz_2\ra + \la Q^{-1/2}(m_2-m_1), P_Nz_1\ra \la Q^{-1/2}(m_2-m_1), P_Nz_2\ra\right]
	\\
	& = \la Az_1, z_2\ra +\la Q^{-1/2}(m_2-m_1), z_1\ra \la Q^{-1/2}(m_2-m_1), z_2\ra,
	\\
	&||W_z||^2_{\Lcal^2(\H,\nu)} = \lim_{N \approach \infty}||W_{P_Nz}||^2_{\Lcal^2(\H, \nu)}
	= \la Az,z\ra + |\la Q^{-1/2}(m_2-m_1), z\ra|^2.
	\end{align*}
	For the first expression, for any $z \in \H$, $N \in \Nbb$,
	\begin{align*}
	&\int_{\H}W_{P_Nz}(x)d\nu(x) = \int_{\H}\la x-m_1, Q^{-1/2}P_Nz\ra d\nu(x)
	\\
	& = \int_{\H}\la x-m_2 +m_2-m_1, Q^{-1/2}P_Nz\ra d\nu(x) = \la m_2 - m_1, Q^{-1/2}P_Nz\ra 
	\\
	&= \la Q^{-1/2}(m_2-m_1), P_Nz\ra.
	\end{align*}
	Since $\nu$ is a probability measure, we also have $\lim_{N \approach \infty}||W_{P_Nz} - W_z||_{\Lcal^1(\H, \nu)} = 0$ by H\"older Inequality. Thus
	\begin{align*}
	\int_{\H}W_z(x)d\nu(x) = \lim_{N \approach \infty}\int_{\H}W_{P_Nz}(x)d\nu(x) 
	= \la Q^{-1/2}(m_2-m_1), z\ra.
	\end{align*}
	This completes the proof. 
	\qed
\end{proof}

\begin{proposition}
	\label{proposition:integral-L2-L1-orthonormal}
	Let $(X, \Sigma, \mu)$ be a measurable space. Let $\Phi = (\phi_k)_{k=1}^{\infty}$ be an orthonormal sequence
	in $\Lcal^2(X, \mu) = \Lcal^2(X,\Sigma, \mu)$. Let $f:X \mapto \R$ be a measurable function such that $\phi_k f \in \Lcal^1(X,\mu)$ $\forall k \in \Nbb$ and
	$b_k = \int_{X}\phi_k(x)f(x)d\mu(x)$, $k \in \Nbb$, satisfy $(b_k)_{k\in \Nbb} \in \ell^2$. Then for any $\sum_{k=1}^{\infty}a_k \phi_k \in \Lcal^2(X, \mu)$, the following integral is well-defined and finite
		\begin{align}
		\int_{X}\left(\sum_{k=1}^{\infty}a_k\phi_k(x)\right)f(x)d\mu(x) &= \int_{X}\left(\sum_{k=1}^{\infty}a_k\phi_k(x)\right)\left(\sum_{k=1}^{\infty}b_k\phi_k(x)\right)d\mu(x)
		\nonumber
		\\
		& = \sum_{k=1}^{\infty}a_kb_k.
		\end{align}
\end{proposition}
We note that if $\Phi$ is an orthonormal basis for $\Lcal^2(X, \mu)$, then the hypothesis of Proposition 
\ref{proposition:integral-L2-L1-orthonormal} becomes $f \in \Lcal^2(X,\mu)$ and the conclusion is immediate.
\begin{proof}
	Let $S_{\Phi} = \overline{\myspan\{\phi_k\}_{k \in \Nbb}}$ be the closed Hilbert subspace of $\Lcal^2(X, \mu)$ with orthonormal basis $\Phi$. We show that the following linear functional
	\begin{align*}
	A_f: S_{\Phi} \mapto \R, \;\;\; A_f(g) = \int_{X}g(x)f(x)d\mu(x),
	\end{align*}
	 is well-defined and bounded. First, by assumption, $\phi_kf \in \Lcal^1(X,\mu)$ $\forall k \in \Nbb$,
	 so that $A_f(\phi_k)$ is well-defined $\forall k \in \Nbb$, with
	 \begin{align*}
	 A_f(\phi_k) = \int_{X}\phi_k(x)f(x)d\mu(x) = b_k.
	 \end{align*}
	 It follows that $A_f$ is well-defined on the dense subspace $\myspan\{\phi_k\}_{k\in \Nbb}$, with
	 \begin{align*}
	 A_f(g_N) = \int_{X}(\sum_{k=1}^Na_k \phi_k)f(x)d\mu(x) =\sum_{k=1}^Na_kb_k\;\;\;\text{for any $g_N = \sum_{k=1}^Na_k\phi_k,\; N \in \Nbb$}.
	 \end{align*}
	 Let now $g = \sum_{k=1}^{\infty}a_k \phi_k\in \Lcal^2(X, \mu)$, $(a_k)_{k \in \Nbb}\in \ell^2$, then by linearity
	 \begin{align*}
	 A_f(g) = \sum_{k=1}^{\infty}a_kb_k \;\;\text{with}\;\; |A_fg|\leq (\sum_{k=1}^{\infty}a_k^2)^{1/2}(\sum_{k=1}^{\infty}b_k^2)^{1/2} < \infty.
	 \end{align*} 
	 Furthermore, for $g_N = \sum_{k=1}^Na_k\phi_k$, $N \in \Nbb$,
	 \begin{align*}
	 |A_f(g_N) - A_f(g)| &=\sum_{k=N+1}^{\infty}a_k b_k \leq (\sum_{k=N+1}^{\infty}a_k^2)^{1/2}(\sum_{k=N+1}^{\infty}b_k^2)^{1/2}
	 \\
	 & = ||g_N-g||_{\Lcal^2(X,\mu)}(\sum_{k=N+1}^{\infty}b_k^2)^{1/2} \approach 0
	 \end{align*} 
	 as $N \approach \infty$. Thus $A_f$ is well-defined and bounded on $S_{\Phi}$, with $||A_f||\leq (\sum_{k=1}^{\infty}b_k^2)^{1/2}$. By the Riesz Representation Theorem, there exists a unique element $h \in S_{\Phi}$ such that $A_f(g) = \la g, h\ra_{\Lcal^2(X, \mu)}$. It is clear that this element $h$ is given by
	 $h = \sum_{k=1}^{\infty}b_k\phi_k \in \Lcal^2(X, \mu)$. \qed 
\end{proof}

\begin{lemma}
	\label{lemma:cross-term-integral}
	Assume the hypothesis of Theorem \ref{theorem:log-Radon-Nikodym-integral}.
	Let $\{\alpha_k\}_{k \in \Nbb}$ be the eigenvalues of $S$, with corresponding orthonormal eigenvectors
	$\{\phi_k\}_{k \in \Nbb}$. 
	Let $g = \sum_{k=1}^{\infty}\frac{1}{1-\alpha_k}\la Q^{-1/2}(m_2- m_1), \phi_k\ra W_{\phi_k}$.
	Then $g \in \Lcal^2(\H, \gamma)$, $g \in \Lcal^1(\H, \gamma)$, and
	\begin{align}
	\int_{\H}g(x)d\gamma(x) = \la (I-S)^{-1}Q^{-1/2}(m_2-m_1), Q^{-1/2}(m_3-m_1)\ra. 
	\end{align}
\end{lemma}
\begin{proof}
	Let $a = Q^{-1/2}(m_2-m_1)$, $b = Q^{-1/2}(m_3-m_1)$. By Lemma \ref{lemma:whitenoise-square-integral-1},
	\begin{align*}
	\int_{\H}W_{\phi_k}(x)d\gamma(x) = \la Q^{-1/2}(m_3-m_1), \phi_k\ra = \la b, \phi_k\ra, \;\forall k \in \Nbb.
	\end{align*}
Using the expression for $\la W_{\phi_j}, W_{\phi_k}\ra_{\Lcal^2(\H,\gamma)}$
from Lemma \ref{lemma:whitenoise-square-integral-1}, we have
	\begin{align*}
	&||g||^2_{\Lcal^2(\H, \gamma)} = \left\|\sum_{k=1}^{\infty}\frac{\la a, \phi_k\ra}{1-\alpha_k}W_{\phi_k}\right\|^2_{\Lcal^2(\H, \gamma)}
	= \sum_{k,j=1}^{\infty}\frac{\la a, \phi_k\ra\la a, \phi_j\ra }{(1-\alpha_k)(1-\alpha_j)}\la W_{\phi_k}, W_{\phi_j}\ra_{\Lcal^2(\H,\gamma)}
	\\
	& = \sum_{k,j=1}^{\infty}\frac{\la a, \phi_k\ra\la a, \phi_j\ra }{(1-\alpha_k)(1-\alpha_j)} \la A\phi_k, \phi_j\ra + \sum_{k,j=1}^{\infty}\frac{\la a, \phi_k\ra\la a, \phi_j\ra }{(1-\alpha_k)(1-\alpha_j)}\la b, \phi_k\ra\la b, \phi_j\ra
	\\
	& = \sum_{k=1}^{\infty}\frac{\la a, \phi_k\ra}{1-\alpha_k}\la (I-S)^{-1}a, A\phi_k\ra
	+\left(\sum_{k=1}^{\infty}\frac{\la a, \phi_k\ra\la b, \phi_k\ra}{1-\alpha_k}\right)^2
	\\
	& = \la (I-S)^{-1}a, A(I-S)^{-1}a\ra + (\la (I-S)^{-1}a, b\ra)^2
	\\
	& = ||A^{1/2}(I-S)^{-1}a||^2 + (\la (I-S)^{-1}a, b\ra)^2 < \infty.
	\end{align*}
	Thus $g \in \Lcal^2(\H, \gamma)$. Furthermore, for any $N \in \Nbb$ and $T_N = \sum_{k=1}^N\phi_k \otimes \phi_k$,
	\begin{align*}
	&\sum_{k,j=N+1}^{\infty}\frac{\la a, \phi_k\ra\la a, \phi_j\ra }{(1-\alpha_k)(1-\alpha_j)} \la A\phi_k, \phi_j\ra
	=\sum_{k=N+1}^{\infty}\frac{\la a, \phi_k\ra}{1-\alpha_k}\la (I-T_N)(I-S)^{-1}a, A\phi_k\ra
	\\
	& = \la (I-T_N)(I-S)^{-1}a, A(I-T_N)(I-S)^{-1}a\ra =||A^{1/2}(I-T_N)(I-S)^{-1}a||^2,
	\\
	 &\left(\sum_{k=N+1}^{\infty}\frac{\la a, \phi_k\ra\la b, \phi_k\ra}{1-\alpha_k}\right)^2 = (\la (I-T_N)(I-S)^{-1}a,b\ra)^2.
		\end{align*}
		Let $g_N = \sum_{k=1}^N\frac{1}{1-\alpha_k}\la Q^{-1/2}(m_2-m_1), \phi_k\ra W_{\phi_k}$, then
		\begin{align*}
		&||g_N-g||^2_{\Lcal^2(\H, \gamma)} = ||A^{1/2}(I-T_N)(I-S)^{-1}a||^2 + (\la (I-T_N)(I-S)^{-1}a,b\ra)^2
		\\
		& \leq [||A|| + ||b||^2]||(I-T_N)(I-S)^{-1}a||^2 \approach 0 \text{ as $N \approach \infty$}.
		\end{align*}
		Since $\gamma$ is a probability measure, by H\"older Inequality, we have $g \in \Lcal^1(\H,\gamma)$
		and $\lim_{N \approach \infty}||g_N-g||_{\Lcal^1(\H, \gamma)} = 0$. It follows that
		\begin{align*}
		\int_{\H}g(x)d\gamma(x) = \lim_{N \approach \infty}\int_{\H}g_N(x)d\gamma(x) = \sum_{k=1}^{\infty}\frac{\la a, \phi_k\ra\la b,\phi_k\ra}{1-\alpha_k} = \la (I-S)^{-1}a,b\ra.
	\end{align*}
Letting $a=Q^{-1/2}(m_2-m_1)$, $b = Q^{-1/2}(m_3-m_1)$ gives the final answer.
	\qed
\end{proof}

We are now ready to prove Theorem \ref{theorem:log-Radon-Nikodym-integral}.
\begin{proof}
	[\textbf{of Theorem \ref{theorem:log-Radon-Nikodym-integral}}]
	Since $\nu \sim \mu$ and $\gamma << \mu$, the Radon-Nikodym derivatives
	$\frac{d\gamma}{d\mu}, \frac{d\gamma}{d\nu}$ are both well-defined. By the chain rule,
	\begin{align*}
	\KL(\gamma||\mu) &= \int_{\H}\log\left\{\frac{d\gamma}{d\mu} \right\}d\gamma = \int_{\H}\log\left\{\frac{d\gamma}{d\nu} \right\} d\gamma + \int_{\H}\log\left\{\frac{d\nu}{d\mu}\right\}d\gamma
	\\
	& = \KL(\gamma ||\nu) +  \int_{\H}\log\left\{\frac{d\nu}{d\mu}\right\}d\gamma.
	\end{align*}
Let us evaluate the second term.
	Let $\{\alpha_k\}_{k=1}^{\infty}$ be the eigenvalues of $S$, with corresponding orthonormal eigenvectors
	$\{\phi_k\}$, which forms an orthonormal basis in $\H$. 
	By the assumption $S \in \Sym(\H) \cap \HS(\H)$, we have $(\alpha_k)_{k\in \Nbb} \in \ell^2$, $\alpha_k \in 
	\R$ $\forall k \in \Nbb$ and the following quantity is finite
	\begin{align}
	\log\dettwo(I-S) = \log\det[(I-S)\exp(S)] = \sum_{k=1}^{\infty}[\alpha_k + \log(1-\alpha_k)].
	\end{align}
	By Theorem \ref{theorem:radon-nikodym-infinite},
	\begin{align}
	\label{equation:log-RN-infinite}
	\log\left\{\frac{d\nu}{d\mu}(x)\right\} = -\frac{1}{2}||(I-S)^{-1/2}Q^{-1/2}(m_2 - m_1)||^2
	-\frac{1}{2}\sum_{k=1}^{\infty}\Phi_k(x),
	\end{align}
	where for each $k \in \Nbb$,
	\begin{align*}
	\Phi_k = \frac{\alpha_k}{1-\alpha_k}W^2_{\phi_k} - \frac{2}{1-\alpha_k}\la Q^{-1/2}(m_2-m_1), \phi_k\ra W_{\phi_k}+ \log(1-\alpha_k).
	\end{align*}
	By Lemma \ref{lemma:whitenoise-square-integral-1},
	\begin{align*}
	\int_{\H}W_{\phi_k}^2(x)d\gamma(x) = \la A\phi_k, \phi_k\ra + |\la Q^{-1/2}(m_3-m_1), \phi_k\ra|^2.
	\end{align*}
	
	(i) $S \in \Tr(\H)$. Since $I-S > 0$, we have $\log(I-S) \in \Tr(\H)$ and
	$\sum_{k=1}^{\infty}\log(1-\alpha_k) = \log\det(I-S)$ is finite. By Tonelli Theorem,
	\begin{align*}
	&\int_{\H}\sum_{k=1}^{\infty}\frac{\alpha_k}{1-\alpha_k}W^2_{\phi_k}d\gamma
	= \int_{\H}\sum_{\alpha_k \geq 0}\frac{\alpha_k}{1-\alpha_k}W^2_{\phi_k}d\gamma - 
	\int_{\H}\sum_{\alpha_k < 0}\frac{-\alpha_k}{1-\alpha_k}W^2_{\phi_k}d\gamma
	\\
	& = \sum_{k=1}^{\infty}\frac{\alpha_k}{1-\alpha_k}\int_{\H}W^2_{\phi_k}d\gamma
	=\sum_{k=1}^{\infty}\frac{\alpha_k}{1-\alpha_k}[\la A\phi_k, \phi_k\ra + |\la Q^{-1/2}(m_3-m_1), \phi_k\ra|^2]
\\
& = \trace[S(I-S)^{-1}A] + \la S(I-S)^{-1}Q^{-1/2}(m_3-m_1), Q^{-1/2}(m_3-m_1)\ra.
	\end{align*}
	Combining this with Eq.\eqref{equation:log-RN-infinite} and Lemma \ref{lemma:cross-term-integral}, we obtain
\begin{align}
&\int_{\H}\log\left\{\frac{d\nu}{d\mu}(x)\right\}d\gamma(x) = -\frac{1}{2}||(I-S)^{-1/2}Q^{-1/2}(m_2 - m_1)||^2
\nonumber
\\
&\quad-\frac{1}{2}[\trace[S(I-S)^{-1}A] + \la S(I-S)^{-1}Q^{-1/2}(m_3-m_1), Q^{-1/2}(m_3-m_1)\ra]
\nonumber
\\
&\quad -\frac{1}{2}\log\det(I-S)
+ \la (I-S)^{-1}Q^{-1/2}(m_2-m_1), Q^{-1/2}(m_3-m_1)\ra.
\label{equation:log-RN-integral-11}
\end{align}	
	(ii) $S \in \Sym(\H)\cap\HS(\H)$ and $I-A \in \Sym(\H) \cap \HS(\H)$. In this case we have
\begin{align*}
\frac{\alpha_k}{1-\alpha_k}W^2_{\phi_k} + \log(1-\alpha_k) = \frac{\alpha_k}{1-\alpha_k}[W^2_{\phi_k}-1] + \left[\frac{\alpha_k}{1-\alpha_k} + \log(1-\alpha_k)\right].
\end{align*}
The second term gives the series of constants
\begin{align}
\sum_{k=1}^{\infty}\left[\frac{\alpha_k}{1-\alpha_k} + \log(1-\alpha_k)\right] 
&= \sum_{k=1}^{\infty}\frac{\alpha_k^2}{1-\alpha_k} + \sum_{k=1}^{\infty}[\alpha_k + \log(1-\alpha_k)]
\nonumber
\\
&= \trace[S^2(I-S)^{-1}] + \log\dettwo(I-S).
\label{equation:series-constants}
\end{align}
	The functions $\{\psi_k = \frac{1}{\sqrt{2}}(W^2_{\phi_k}-1)\}_{k \in \Nbb}$ form an orthonormal sequence
	in $\Lcal^2(\H, \mu)$ (\cite{DaPrato:PDEHilbert}, Proposition 1.2.6). Furthermore, $\psi_k\frac{d\gamma}{d\mu} \in \Lcal^1(\H, \mu) \forall k \in \Nbb$, with
	\begin{align*}
	b_k &= \int_{\H}\left(\psi_k\frac{d\gamma}{d\mu}(x)\right)d\mu(x) = \int_{\H}\psi_k(x)d\gamma(x)
	\\
	& = \frac{1}{\sqrt{2}}[\la A\phi_k, \phi_k\ra + |\la Q^{-1/2}(m_3-m_1), \phi_k\ra|^2 - 1]
	\\
	&= \frac{1}{\sqrt{2}}[-\la (I-A)\phi_k, \phi_k\ra + |\la Q^{-1/2}(m_3-m_1), \phi_k\ra|^2 ].
	\end{align*}
	The sequence $(b_k)_{k \in \Nbb}$ satisfies
	\begin{align*}
	&\sum_{k=1}^{\infty}b_k^2 = \frac{1}{2}\sum_{k=1}^{\infty}[-\la (I-A)\phi_k, \phi_k\ra + |\la Q^{-1/2}(m_3-m_1), \phi_k\ra|^2 ]^2
	\\ 
	&\leq \sum_{k=1}^{\infty}\la (I-A)\phi_k, \phi_k\ra^2 + \sum_{k=1}^{\infty}|\la Q^{-1/2}(m_3-m_1), \phi_k\ra|^4
	\\
	& \leq \sum_{k=1}^{\infty}||(I-A)\phi_k||^2 + (\sum_{k=1}^{\infty} |\la Q^{-1/2}(m_3-m_1), \phi_k\ra|^2)^2
\\
& = ||I-A||^2_{\HS} + ||Q^{-1/2}(m_3-m_1)||^4 < \infty.
	\end{align*}
	Since $(\alpha_k)_{k \in \Nbb} \in \ell^2$, we apply Proposition \ref{proposition:integral-L2-L1-orthonormal} to $(\psi_k)_{k \in \Nbb}$ and $f=\frac{d\gamma}{d\mu}$ to obtain
	\begin{align*}
	&\int_{\H}\sum_{k=1}^{\infty}\frac{\alpha_k}{1-\alpha_k}(W^2_{\phi_k}-1)d\gamma(x) = \int_{\H}\sum_{k=1}^{\infty}\frac{\sqrt{2}\alpha_k}{1-\alpha_k}\frac{1}{\sqrt{2}}(W^2_{\phi_k}-1)\frac{d\gamma}{d\mu}(x)d\mu(x)
	\\
	& = \sum_{k=1}^{\infty}\frac{\sqrt{2}\alpha_kb_k}{1-\alpha_k} = \sum_{k=1}^{\infty}\frac{\alpha_k}{1-\alpha_k} [-\la (I-A)\phi_k, \phi_k\ra + |\la Q^{-1/2}(m_3-m_1), \phi_k\ra|^2]
	\\
	& = -\trace[S(I-S)^{-1}(I-A)] + \la S(I-S)^{-1}Q^{-1/2}(m_3-m_1), Q^{-1/2}(m_3-m_1)\ra.
	\end{align*}
	Combining this with Eqs.\eqref{equation:log-RN-infinite}, \eqref{equation:series-constants}, and Lemma \ref{lemma:cross-term-integral}, we obtain
	\begin{align}
	&\int_{\H}\log\left\{\frac{d\nu}{d\mu}(x)\right\}d\gamma(x) =
	-\frac{1}{2}||(I-S)^{-1/2}Q^{-1/2}(m_2 - m_1)||^2
	\nonumber
	\\
	&\quad-\frac{1}{2}[\la S(I-S)^{-1}Q^{-1/2}(m_3-m_1), Q^{-1/2}(m_3-m_1)\ra]
	\nonumber
	\\
	&\quad +\frac{1}{2}\trace[S(I-(I-S)^{-1}A)]- \frac{1}{2}\log\dettwo(I-S)
	\nonumber
	\\
	&\quad + \la (I-S)^{-1}Q^{-1/2}(m_2-m_1), Q^{-1/2}(m_3-m_1)\ra.
	\label{equation:log-RN-integral-12} 
	\end{align}
	For $S \in \Tr(\H)$, we have $\log\dettwo(I-S) = \trace(S) + \log\det(I-S)$, so that
	\eqref{equation:log-RN-integral-12} reduces to \eqref{equation:log-RN-integral-11}.
	For $m_3 = m_2$ and $A = I-S$, \eqref{equation:log-RN-integral-12} simplifies to 
	\begin{align*}
		&\int_{\H}\log\left\{\frac{d\nu}{d\mu}(x)\right\}d\gamma(x) &= \frac{1}{2}||Q^{-1/2}(m_2-m_1)||^2 -\frac{1}{2}\log\dettwo(I-S)
		\\
		& = \KL(\nu ||\mu).
	\end{align*}
	Thus in the case we have $\KL(\gamma ||\mu) = \KL(\gamma||\nu) + \KL(\nu ||\mu)$ and hence
	$\KL(\gamma||\mu) \geq \KL(\nu ||\mu)$, with equality if and only $\gamma = \nu$.
	\qed
\end{proof}

\section{Mutual information of Gaussian measures on Hilbert space}
\label{section:mutual-info-Gauss}

In this section, we prove Theorem \ref{theorem:minimum-Mutual-Info-Gaussian}.
For completeness, we give a new, shorter proof of the mutual information between two Gaussian measures \cite{Baker1978capacity}.
We start by reviewing
joint measures and cross-covariance operators on Hilbert space. 
%


{\bf Joint measures}. Following \cite{Baker1973CrossCovariance}, let $(\H_1, \la, \ra_1)$ and $(\H_2, \la,\ra_2)$ be two real separable Hilbert spaces
with
corresponding Borel $\sigma$-algebras
$\Bsc(\H_1)$ and $\Bsc(\H_2)$. Let $\H_1 \times \H_2$ be the Hilbert space with inner product
defined by
$\la(u_1, u_2), (v_1, v_2)\ra_{12} = \la u_1, v_1\ra_1 + \la u_2, v_2\ra_2$
and Hilbert norm
$||(u,v)||^2_{12} = {||u||^2_1 + ||v||^2_{2}}$.
A {\it joint measure }$\mu_{XY}$ is a probability measure defined on $(\H_1 \times \H_2, \Bsc(\H_1) \times \Bsc(\H_2))$.
Let $\mu_X$ and $\mu_Y$ be its marginal probability measures defined on $(\H_1, \Bsc(\H_1))$ and $(\H_2, \Bsc(\H_2))$, respectively.
Assume further that $\mu_{XY} \in \Pcal_2(\H_1 \times \H_2)$.
Under this assumption, the mean vector $m_{XY} \in \H_1 \times \H_2$ and covariance operator $\Gamma_{XY}: \H_1 \times \H_2 \mapto \H_1 \times \H_2$ are well-defined. 
Furthermore, since
\begin{align}
\int_{\H_1 \times \H_2}||(u,v)||^2_{12}d\mu_{XY}(u,v) 
& = \int_{\H_1}||u||^2_1d\mu_X(u) + \int_{\H_2}||v||^2_2d\mu_Y(v),
\end{align}
it follows that $\mu_{XY} \in \Pcal_2(\H_1 \times \H_2)$ if and only if
$\mu_X \in \Pcal_2(\H_1)$ and $\mu_{Y}\in \Pcal_2(\H_2)$.
{\it Subsequently, throughout the paper, we assume that $\mu_{XY} \in \Pcal_2(\H_1 \times \H_2)$
	for all joint measures $\mu_{XY}$ on $(\H_1 \times \H_2, \Bsc(\H_1) \times \Bsc(\H_2))$.}

{\bf Cross-covariance operators}. 
Let $(m_X, C_X)$ and $(m_Y, C_Y)$ be the means and covariance operators of $\mu_X$ and $\mu_Y$, respectively. 
Since $\mu_{XY} \in \Pcal_2(\H_1 \times \H_2)$, 
the following linear functional $G:\H_1 \times \H_2 \mapto \R$ is well-defined and bounded,
\begin{align}
G(u,v) &= \int_{\H_1 \times \H_2}\la x- m_X, u\ra_1\la y- m_Y, v\ra_2 d\mu_{XY}(x,y).
\end{align}
By the Riesz Representation Theorem, there exist single-valued, bounded, linear maps $C_{XY}: \H_2 \mapto \H_1$, $C_{YX}:\H_1 \mapto \H_2$ such that
\begin{align}
\la u, C_{XY}v\ra_1 = G(u,v), \;\;\; \la C_{YX}u, v\ra_2 = G(u,v), \; u\in\H_1, v\in \H_2.
\end{align}
Clearly, $C_{YX} = C_{XY}^{*}$. In operator tensor product notation,
\begin{align}
C_{XY} = \int_{\H_1\times \H_2}(x-m_X) \otimes (y-m_Y)d\mu_{XY}(x,y).
\end{align}
$C_{XY}$ is called the {\it cross-covariance operator} of $\mu_{XY}$. It is closely related to the covariance operators
of the marginals $\mu_X$ and $\mu_Y$ via the following.
\begin{theorem}
	[\cite{Baker1973CrossCovariance}]
	\label{theorem:CXY-representation}
	Consider the projection operators $P_X: \H_1 \mapto \overline{\myIm(C_X)}$, $P_Y: \H_2 \mapto \overline{\myIm(C_Y)}$.
	Then $C_{XY}$ admits the following representation
	\begin{align}
	C_{XY} = C_{X}^{1/2}VC_Y^{1/2}
	\end{align}
	where $V:\H_2 \mapto \H_1$ is a unique bounded linear operator such that $||V||\leq 1$ and $V = P_XVP_Y$. 
\end{theorem}
In column vector notation, the mean vector of the joint measure $\mu_{XY}$ is given by
$m_{XY} = \begin{pmatrix}m_X \\ m_Y\end{pmatrix} \in \H_1 \times \H_2$.
In operator-valued matrix notation, the covariance operator $\Gamma_{XY}:\H_1 \times \H_2 \mapto \H_1 \times \H_2$ of  $\mu_{XY}$ is given by
\begin{align}
\Gamma_{XY} &= \begin{pmatrix} C_X & C_{XY}
\\
C_{YX} & C_Y
\end{pmatrix},
\Gamma_{XY}\begin{pmatrix}u \\v \end{pmatrix}  = \begin{pmatrix}
C_Xu + C_{XY}v\\
C_{YX}u + C_Yv
\end{pmatrix},  u \in \H_1, v \in \H_2.
\end{align}

{\bf Mutual information of Gaussian measures}. We now consider
joint Gaussian measures.
The covariance operator for $\mu_X \otimes \mu_Y$ is
$
\Gamma_0 =\begin{pmatrix}
C_1 & 0
\\
0 & C_2
\end{pmatrix}. 
$
Let $V\in \Lcal(\H_2, \H_1)$ be 
as in Theorem \ref{theorem:CXY-representation}.
The covariance operator for $\mu_{XY}$ is 
\begin{align}
\Gamma = \begin{pmatrix}
C_1 & C_1^{1/2}VC_2^{1/2}
\\
C_2^{1/2}V^{*}C_1^{1/2} & C_2
\end{pmatrix}
= \Gamma_0^{1/2}\begin{pmatrix}
I & V\\
V^{*} & I
\end{pmatrix}\Gamma_0^{1/2}.
\label{equation:COV-joint}
\end{align}

\begin{lemma}
	\label{lemma:Hilbert-Schmidt-double}
	Let $V \in \Lcal(\H_2, \H_1)$.
	Then the operator
	$\begin{pmatrix}
	0 & V\\
	V^{*} & 0
	\end{pmatrix}: \H_1 \times \H_2 \mapto \H_1\times \H_2$
	is Hilbert-Schmidt on $\H_1 \times \H_2$ if and only if $V \in \HS(\H_2,\H_1)$, 
	{\color{black}equivalently, if and only if $V^{*}V \in \Tr(\H_2)$ (or, equivalently, $VV^{*} \in \Tr(\H_1)$)}.
\end{lemma}
\begin{proof} Let $\{e^i_j\}_{j \in \Nbb}$ be an orthonormal basis in $\H_i$, $i=1,2$. Then
	$\left\{\begin{pmatrix}
	e^1_j 
	\\
	0\end{pmatrix},\begin{pmatrix}
	0 
	\\
	e^2_j\end{pmatrix}\right \}_{j\in \Nbb}$ is an orthonormal basis for $\H_1 \times \H_2$ and
	$
	\begin{pmatrix}
	0 & V\\
	V^{*} & 0
	\end{pmatrix}\begin{pmatrix}
	e^1_j 
	\\
	0
	\end{pmatrix} = \begin{pmatrix}0 \\V^{*}e^1_j\end{pmatrix}, 
	\begin{pmatrix}
	0 & V\\
	V^{*} & 0
	\end{pmatrix}\begin{pmatrix}
	0
	\\
	e^2_j
	\end{pmatrix} =\begin{pmatrix}Ve^2_j \\ 0 \end{pmatrix}.
	$
	Thus
	$\left\|
	\begin{pmatrix}
	0 & V\\
	V^{*} & 0
	\end{pmatrix}
	\right\|^2_{\HS} = \sum_{j=1}^{\infty}[||V^{*}e^1_j||_2^2 + ||Ve^2_j||_1^2] = ||V^{*}||^2_{\HS} + ||V||^2_{\HS}
	= \trace(VV^{*}) + \trace(V^*V)$, which is finite if and only if 
	$\trace(VV^{*}) = ||V^{*}||^2_{\HS} = ||V||^2_{\HS} = \trace(V^{*}V)$ is finite.
	\qed
\end{proof}

\begin{lemma}
	\label{lemma:det2VV}
	Let $V \in \HS(\H_2,\H_1)$. Then $\begin{pmatrix}0 & V \\ V^{*} & 0\end{pmatrix}$
	has nonzero eigenvalues $\{\pm \sqrt{\gamma_k}\}_{k \in \Nbb}$, with $\{\gamma_k\}_{k \in \Nbb}$
	being the 
	nonzero 
	eigenvalues of $VV^{*}:\H_1 \mapto \H_1$. Moreover, $\left\|\begin{pmatrix}0 & V \\ V^{*} & 0\end{pmatrix}\right\| < 1 \equivalent ||V||<1$, in which case the following quantity is finite
	\begin{align}
	\log\dettwo\begin{pmatrix}
	I & V\\
	V^{*} & I
	\end{pmatrix} = \log\det(I-V^{*}V).
	\end{align}
	Furthermore, $\begin{pmatrix}0 & V \\ V^{*} & 0\end{pmatrix} \in \Tr(\H_1 \times \H_2) \equivalent (VV^{*})^{1/2} \in \Tr(\H_1){\color{black}\equivalent (V^{*}V)^{1/2} \in \Tr(\H_2)}$.
	In this case {\color{black}the Fredholm determinant $\det\begin{pmatrix}
			I & V\\
			V^{*} & I
		\end{pmatrix}$ is also finite and} $\dettwo\begin{pmatrix}
	I & V\\
	V^{*} & I
	\end{pmatrix}$ $= \det\begin{pmatrix}
	I & V\\
	V^{*} & I
	\end{pmatrix} = \det(I-V^{*}V)$.
\end{lemma}
{\color{black}
\begin{remark} As stated in Lemmas \ref{lemma:Hilbert-Schmidt-double} and \ref{lemma:det2VV}, for $V \in \HS(\H_2,\H_1)$, equivalently $VV^{*} \in \Tr(\H_1), V^{*}V \in \Tr(\H_2)$, 
	we have $\begin{pmatrix}
		0 & V\\
		V^{*} & 0
	\end{pmatrix} \in \HS(\H_1 \times \H_2)$, so that
	the Hilbert-Carleman determinant of
	$\begin{pmatrix}
		I & V\\
		V^{*} & I
	\end{pmatrix}$ is well-defined and finite and is given by
	  $\dettwo\begin{pmatrix}
		I & V\\
		V^{*} & I
	\end{pmatrix} = \det(I-V^{*}V)$. 
Under the stronger assumption $(VV^{*})^{1/2} \in \Tr(\H_1)$, or equivalently $(V^{*}V)^{1/2} \in \Tr(\H_2)$,
	we have $\begin{pmatrix}
	0 & V\\
	V^{*} & 0
\end{pmatrix} \in \Tr(\H_1 \times \H_2)$, so that
 the Fredholm determinant of $\begin{pmatrix}
	I & V\\
	V^{*} & I
\end{pmatrix}$ is also well-defined and finite and is given by $\det\begin{pmatrix}
	I & V\\
	V^{*} & I
\end{pmatrix} = \det(I-V^{*}V)$, i.e., in this case the Hilbert-Carleman and Fredholm determinants of 
$\begin{pmatrix}
	I & V\\
	V^{*} & I
\end{pmatrix}$ have the same value. We utilize the latter expression in Lemma \ref{lemma:logdet(I-V*V)-concave} below.
\end{remark}
}
\begin{proof}
	[\textbf{of Lemma \ref{lemma:det2VV}}]
	By Lemma \ref{lemma:Hilbert-Schmidt-double},
	the operator $\begin{pmatrix}
	0 & V
	\\
	V^{*} & 0
	\end{pmatrix}$ is self-adjoint and Hilbert-Schmidt on $\H_1 \times \H_2$, thus admits a countable sequence
	of eigenvalues, with corresponding eigenvectors forming an orthonormal basis for $\H_1 \times \H_2$. Consider the square
	$\begin{pmatrix}
	0 & V
	\\
	V^{*} & 0
	\end{pmatrix}^2
	= \begin{pmatrix}
	VV^{*} & 0
	\\
	0 & V^{*}V
	\end{pmatrix}$.
	The nonzero eigenvalues of the last operator are precisely the nonzero eigenvalues of $VV^{*}:\H_1 \mapto\H_1$ and $V^{*}V:\H_2 \mapto \H_2$, which are 
	identical. Consider the following eigenvalue equation, where $\lambda \neq 0$,
	\begin{align*}
	\begin{pmatrix}
	0 & V
	\\
	V^{*} & 0
	\end{pmatrix}
	\begin{pmatrix}
	a\\
	b
	\end{pmatrix} = \begin{pmatrix}
	Vb
	\\
	V^{*}a 
	\end{pmatrix}
	=\lambda \begin{pmatrix}
	a
	\\
	b
	\end{pmatrix} \equivalent 
	\begin{pmatrix}
	Vb = \lambda a
	\\
	V^{*}a = \lambda b
	\end{pmatrix}, \;a\in \H_1,b\in \H_2.
	\end{align*}
	Combining these two equations, we obtain
	$VV^{*}a = \lambda^2 a$ and
	$V^{*}Vb = \lambda^2 b$.
	Thus $a$ and $b$ are necessarily eigenvectors of $VV^{*}$ and $V^{*}V$, respectively, under the same eigenvalue $\lambda^2$.
	Due to the square factor, it follows that $-\lambda$ is an eigenvalue of $\begin{pmatrix}
	0 & V
	\\
	V^{*} & 0
	\end{pmatrix}$ corresponding to eigenvector $\begin{pmatrix}a\\ -b\end{pmatrix}$.
	Let $\{\gamma_k\}_{k=1}^{\infty}$ be the nonzero eigenvalues of $VV^{*}$, then the nonzero eigenvalues of
	$\begin{pmatrix}
	0 & V
	\\
	V^{*} & 0
	\end{pmatrix}$ are $\{\pm \sqrt{\gamma_k}\}_{k=1}^{\infty}$.
	Thus $\left\|\begin{pmatrix}
	0 & V
	\\
	V^{*} & 0
	\end{pmatrix}\right\| < 1 \equivalent ||VV^{*}|| = ||V||^2 < 1 \equivalent ||V||< 1$.
	
	Let $\{e_j^{\pm}\}_{j=1}^{\infty}$
	denote the corresponding orthonormal eigenvectors. Then
	\begin{align*}
	\begin{pmatrix}
	0 & V
	\\
	V^{*} & 0
	\end{pmatrix} = \sum_{j=1}^{\infty}[\sqrt{\gamma_j}e_j^{+} \otimes e_j^{+} - \sqrt{\gamma_j}e_j^{-}\otimes e_j^{-}].
	\end{align*}
	By definition of the Hilbert-Carleman determinant, 
	\begin{align*}
	&\dettwo\begin{pmatrix}
	I & V\\
	V^{*} & I
	\end{pmatrix} = \det\left[\begin{pmatrix}
	I & V\\
	V^{*} & I
	\end{pmatrix}
	\exp\left(-\begin{pmatrix}
	0 & V\\
	V^{*} & 0
	\end{pmatrix}\right)
	\right]
	\\
	&= \prod_{j=1}^{\infty}(1+\sqrt{\gamma_j})(1-\sqrt{\gamma_j})e^{\sqrt{\gamma_j}}e^{-\sqrt{\gamma_j}} = \prod_{j=1}^{\infty}(1-\gamma_j) 
	= \det(I - V^{*}V),
	\end{align*}
	giving the $\log\dettwo$ formula. Since
	$
	\left|\begin{pmatrix}
	0 & V
	\\
	V^{*} & 0
	\end{pmatrix}\right|
	= \begin{pmatrix}
	(VV^{*})^{1/2} & 0
	\\
	0 & (V^{*}V)^{1/2}
	\end{pmatrix}$,
	$\begin{pmatrix}
	0 & V
	\\
	V^{*} & 0
	\end{pmatrix} \in \Tr(\H_1 \times \H_2) \equivalent (VV^{*})^{1/2} \in \Tr(\H_1)\equivalent (V^{*}V)^{1/2} \in \Tr(\H_2)$.
	In this case, $\det(I - V^{*}V) = \prod_{j=1}^{\infty}(1-\gamma_j) =\dettwo\begin{pmatrix}
	I & V\\
	V^{*} & I
	\end{pmatrix} = \det\begin{pmatrix}
	I & V\\
	V^{*} & I
	\end{pmatrix}$.
	\qed
\end{proof}

By applying Formula (\ref{equation:KL-gaussian}) in Theorem \ref{theorem:KL-gaussian} and Lemma \ref{lemma:det2VV} to our setting, we obtain the expression for $\KL(\mu_{XY}||\mu_X \otimes \mu_Y)$, first proved in \cite{Baker1978capacity} (Proposition 2) by a direct approach. The case $\H_1 = \H_2 = \H$ was proved in \cite{baker1970mutual}.
As we now show, by employing the more general result on the KL divergence between Gaussian measures in Theorem \ref{theorem:KL-gaussian}, 
this result is obtained immediately.
\begin{theorem}
	\label{theorem:MutualInfo-Gaussian}
	Let $\H_1,\H_2$ be two separable Hilbert spaces.
	Let $\mu_X = \Ncal(m_X, C_X)\in \Gauss(\H_1)$, $\mu_Y = \Ncal(m_Y, C_Y) \in \Gauss(\H_2)$, $\ker(C_X) = \{0\}$,
	$\ker(C_Y) = \{0\}$.
	Assume that $\mu_{XY} \in \Gauss(\mu_X, \mu_Y)$. Then $\mu_{XY} \sim \mu_X \otimes \mu_Y \equivalent ||V||<1, V \in \HS(\H_2,\H_1)$,
	where   $C_{XY} = C_X^{1/2}VC_Y^{1/2}$.
	Furthermore,
	\begin{align}
	\KL\left(\mu_{XY}||{\mu_X \otimes \mu_Y}\right) = 
	\left\{
	\begin{matrix}
	-\frac{1}{2}\log\det(I-V^{*}V) & \text{if $\mu_{XY} \sim \mu_X \otimes \mu_Y$},
	\\
	\infty & \text{if $\mu_{XY} \perp \mu_X \otimes \mu_Y$}.
	\end{matrix}
	\right.
	\end{align}
\end{theorem}
\begin{proof}
	[\textbf{of Theorem \ref{theorem:MutualInfo-Gaussian}}]
	The condition  for $\mu_{XY} \sim \mu_X \otimes \mu_Y$ follows from Lemmas \ref{lemma:Hilbert-Schmidt-double},\ref{lemma:det2VV} and Theorem \ref{theorem:Gaussian-equivalent}. 
	Assume that $\mu_{XY} \sim \mu_X \otimes \mu_Y$. Since
	the mean vector for $\mu_{XY}$ and $\mu_X \otimes \mu_Y$ is the same, namely $\begin{pmatrix}m_X\\m_Y\end{pmatrix}$,
	the first term in Eq.\eqref{equation:KL-gaussian} is equal to zero.
	From Eqs.\eqref{equation:COV-joint} and \eqref{equation:KL-gaussian}, we have
	\begin{align*}
	\KL\left(\mu_{XY}||{\mu_X \otimes \mu_Y}\right) = -\frac{1}{2}\log\dettwo\begin{pmatrix}
	I & V\\
	V^{*} & I
	\end{pmatrix} = -\frac{1}{2}\log\det(I-V^{*}V),
	\end{align*}
	where the last equality follows from Lemma \ref{lemma:det2VV}.\qed
\end{proof}

We are now ready to prove Theorem \ref{theorem:minimum-Mutual-Info-Gaussian}, which we restate here for clarity.
\begin{theorem}[\textbf{Minimum Mutual Information of Joint Gaussian Measures}]
	\label{theorem:minimum-Mutual-Info-Gaussian-1}
	Let $\H_1,\H_2$ be two separable Hilbert spaces.
	Let $\mu_X = \Ncal(m_X, C_X)\in \Gauss(\H_1)$, $\mu_Y = \Ncal(m_Y, C_Y)\in \Gauss(\H_2)$,
	$\ker(C_X) = \ker(C_Y) =\{0\}$.
	Let $\gamma \in \ADM(\mu_X, \mu_Y), \gamma_0 \in \Gauss(\mu_X, \mu_Y)$,
	$\gamma_0 \sim \mu_X \otimes \mu_Y$. Assume that $\gamma$ and $\gamma_0$ have the same covariance operator $\Gamma$ and that $\mu_X \otimes \mu_Y$ has covariance operator $\Gamma_0$. 
	Then
	\begin{align}
	\KL(\gamma||\mu_X \otimes \mu_Y) \geq \KL(\gamma_0 ||\mu_X \otimes \mu_Y) = -\frac{1}{2}\log\det(I-V^{*}V).
	\end{align}
	Equality happens if and only if $\gamma = \gamma_0$.
	Here $V$ is the unique bounded linear operator satisfying $V \in \HS(\H_2,\H_1)$, $||V||< 1$, 
	such that $\Gamma = \Gamma_0^{1/2}\begin{pmatrix}I & V \\ V^{*} & I\end{pmatrix}\Gamma_0^{1/2}$.
\end{theorem}

\begin{proof}
	[\textbf{of Theorem \ref{theorem:minimum-Mutual-Info-Gaussian-1}}]
	We can assume that $\gamma << \mu_X \otimes \mu_Y$, since otherwise $\KL(\gamma ||\mu_X \otimes \mu_Y) =\infty$
	and the inequality is obviously true.
	Since $\gamma_0 \sim \mu_X \otimes \mu_Y$, we also have $\gamma << \gamma_0$
	and the Radon-Nikodym derivatives $\frac{d\gamma}{d\gamma_0}$ and $\frac{d\gamma}{d(\mu_X \otimes \mu_Y)}$ are both well-defined. 
	Since $\gamma, \gamma_0$, and $\mu_X \otimes \mu_Y$ all have the same mean, namely 
	$\begin{pmatrix}m_X \\m_Y\end{pmatrix}$, and $\gamma, \gamma_0$ have the same covariance operators, we have by Theorem \ref{theorem:log-Radon-Nikodym-integral}, 
	\begin{align*}
	\KL(\gamma ||\mu_X \otimes \mu_Y) &= \KL(\gamma||\gamma_0) + \KL(\gamma_0||\mu_X \otimes \mu_Y)
	\\
	&\geq \KL(\gamma_0||\mu_X \otimes \mu_Y) = -\frac{1}{2}\log\det(I-V^{*}V) \;\text{by Theorem \ref{theorem:MutualInfo-Gaussian}}.
	\end{align*}
	Thus $\KL(\gamma ||\mu_X \otimes \mu_Y) \geq \KL(\gamma_0||\mu_X \otimes \mu_Y)$, with equality if and only if $\gamma = \gamma_0$, that is if and only if $\gamma$ is Gaussian.
	\qed
\end{proof}

\section{Entropic regularized $2$-Wasserstein distance between Gaussian measures on Hilbert space}
\label{section:OT-entropic-Gaussian}

Consider the entropic OT problem \eqref{equation:OT-entropic}, which we restate here
\begin{equation}
\label{equation:main-KL}
\OT^{\ep}_c(\mu_0, \mu_1) = \min_{\gamma \in \ADM(\mu_0, \mu_1)}\left\{\bE_{\gamma}{c(x,y)} + \ep \KL(\gamma || \mu_0 \otimes \mu_1) \right\}.
\end{equation}

{\bf The nonsingular case}.
We first solve \eqref{equation:main-KL} for $c(x,y) = ||x-y||^2$ under the nonsingular Gaussian setting. 
Let $\mu_X, {\color{black}\mu_Y} \in \Pcal_2(\H)$, with means $m_X, m_Y$ and covariance operators $C_X, C_Y$, respectively, then $C_X, C_Y\in \Tr(\H)$ and
\begin{align}
\bE_{\mu_X}||x-m_X||^2 = \trace(C_X),\;\;\;
\bE_{\mu_Y}||y-m_Y||^2 = \trace(C_Y).
\end{align}

\begin{lemma} 
	\label{lemma:cross-covariance-expected-value}
	Let $\mu_X, \mu_Y \in \Pcal_2(\H)$, $\mu_{XY} \in \Joint(\mu_X, \mu_Y)$. Then $C_{XY} \in \Tr(\H)$ and for any $A,B \in \Lcal(\H)$,
	\begin{align}
	\bE_{\mu_{XY}}\la A(x- m_X), B(y-m_Y)\ra &= \trace(AC_{XY}B^{*}).
	\end{align}
\end{lemma}
\begin{proof} By Theorem \ref{theorem:CXY-representation}, $C_{XY} = C_X^{1/2}VC_Y^{1/2}$. 
	Since $C_X$ and $C_Y$ are both trace class, $C_X^{1/2}$ and $C_Y^{1/2}$ are both Hilbert-Schmidt. Thus it follows
	that $C_{XY}$ is trace class.
	Recall that $C_{XY}: \H \mapto \H$ is given by
	$C_{XY} = \int_{\H \times \H} (x-m_X)\otimes(y- m_Y) d\mu_{XY}(x,y)$.
	It suffices to consider $m_X = m_Y = 0$.
	Let $\{e_j\}_{j=1}^{\infty}$ be an orthonormal basis on $\H$, then for each $j \in \Nbb$,
	\begin{align*}
	&\int_{\H \times \H}\la Ax, e_j\ra\la By, e_j\ra d\mu_{XY}(x,y)
	\\
	&= \int_{\H \times \H}\la x, A^{*}e_j\ra\la y, B^{*}e_j\ra d\mu_{XY}(x,y) = \la A^{*}e_j, C_{XY}B^{*}e_j\ra.
	\end{align*}
	For each $N \in \Nbb$, $\sum_{j=1}^N|\la Ax,e_j\ra\la By, e_j\ra| 
	\leq [\sum_{j=1}^N|\la Ax, e_j\ra|^2]^{1/2}[\sum_{j=1}^N|\la By, e_j\ra|^2]^{1/2}$ $\leq ||Ax||\;||By|| \leq \frac{||A||\;||B||}{2}[||x||^2 + ||y||^2]$, with
	\begin{align*}
	\int_{\H \times \H}[||x||^2 + ||y||^2]d\mu_{XY}(x,y) &= \int_{\H}||x||^2d\mu_X(x) + \int_{\H}||y||^2d\mu_Y(y) 
	\\
	&= \trace(C_X) + \trace(C_Y) < \infty.
	\end{align*}
	By Lebesgue Dominated Convergence Theorem,
	\begin{align*}
	&\bE_{\mu_{XY}}\la Ax, By\ra = \int_{\H \times \H}\la Ax, By\ra d\mu_{XY}(x,y)
	= \int_{\H \times \H}\sum_{j=1}^{\infty}\la Ax, e_j\ra\la By, e_j\ra d\mu_{XY}(x,y)
	\\
	&= \sum_{j=1}^{\infty} \int_{\H \times \H}\la x, A^{*}e_j\ra\la y, B^{*}e_j\ra d\mu_{XY}(x,y) 
	= \sum_{j=1}^{\infty}\la A^{*}e_j, C_{XY}B^{*}e_j\ra = \trace(AC_{XY}B^{*}).
	\end{align*}
	\qed
\end{proof}

\begin{corollary} 
	\label{corollary:square-moment-covariance-operator}
	Let $\mu_X, \mu_Y \in \Pcal_2(\H)$, $\mu_{XY} \in \Joint(\mu_X, \mu_Y)$. Then
	\begin{align}
	\bE_{\mu_{XY}}||x-y||^2 &= ||m_X-m_Y||^2 + \trace(C_X) + \trace(C_Y) -2\trace(C_{XY})
	\\
	& = ||m_X - m_Y||^2 + \trace(C_X) + \trace(C_Y) - 2\trace(C_X^{1/2}VC_Y^{1/2}).
	\nonumber
	\end{align}
\end{corollary}

\begin{proposition}
	\label{proposition:entropic-OT-nonsingular}
Let $\mu_0 = \Ncal(m_0, C_0)$, $\mu_1 = \Ncal(m_1, C_1)$, with 
$\ker(C_0) = \ker(C_1) = \{0\}$. If
$\gamma^{\ep}$ is the minimizer of problem \eqref{equation:OT-ep-square}, then necessarily $\gamma^{\ep} \in \Gauss(\mu_0, \mu_1)$ and $\gamma^{\ep} \sim \mu_0 \otimes \mu_1$.
\end{proposition}
\begin{proof}
By Corollary \ref{corollary:square-moment-covariance-operator}, 
for any $\gamma \in \Joint(\mu_0, \mu_1)$,
\begin{align*}
&I_{\ep}(\gamma) = \bE_{\gamma}||x-y||^2 + \ep \KL(\gamma ||\mu_0 \otimes \mu_1) 
\\
&= ||m_0 - m_1||^2 + \trace(C_0) + \trace(C_1) - 2\trace(C_0^{1/2}VC_1^{1/2}) + \ep\KL(\gamma ||\mu_0 \otimes \mu_1),
\end{align*}
with the first terms depending solely on the means and covariance operators.
First, we must have $\gamma << \mu_0 \otimes \mu_1$, since otherwise $\KL(\gamma ||\mu_0 \otimes \mu_1) =\infty$
and such a $\gamma$ cannot be a minimizer.
By Theorem \ref{theorem:minimum-Mutual-Info-Gaussian}, if $\gamma << \mu_0 \otimes \mu_1$ and $\gamma_0 \in \Gauss(\mu_0, \mu_1)$, $\gamma_0 \sim \mu_0 \otimes \mu_1$,
has the same covariance operator as $\gamma$, then $\KL(\gamma||\mu_0 \otimes \mu_1) \geq \KL(\gamma_0||\mu_0 \otimes \mu_1)$, with equality if and only if $\gamma = \gamma_0$.
Thus we must have $\gamma^{\ep} \in \Gauss(\mu_0, \mu_1)$ and $\gamma^{\ep} \sim \mu_0 \otimes \mu_1$.
\qed
\end{proof}

The following technical lemmas are used extensively throughout the paper.

\begin{lemma}
	\label{lemma:adjont-switch}
	Let $A:\H \mapto \H$ be a compact operator. Then
	\begin{align}
	A(I+A^{*}A)^{1/2} = (I+AA^{*})^{1/2}A.
	\end{align}
\end{lemma}
\begin{proof}
	This is a special case of Lemma 10 and Corollary 2 in \cite{Minh:LogDetIII2018}.
\end{proof}

The following result is then immediate.
\begin{lemma}
	\label{lemma:adjoint-switch-2}
	Let $A:\H \mapto \H$ be a compact operator. Then
	\begin{align}
	[I+(I+AA^{*})^{1/2}]A & = A[I + (I+A^{*}A)^{1/2}],
	\\
	A[I + (I+A^{*}A)^{1/2}]^{-1} &= [I+(I+AA^{*})^{1/2}]^{-1}A.
	\end{align}
\end{lemma}

\begin{lemma}
	\label{lemma:adjoint-switch-CX}
	Let $C \in \Sym^{+}(\H)$, $X \in \Sym^{+}(\H)$, {\color{black}at least one of which is compact},  be fixed.
	$\forall a \in \R$, $a \neq 0$,
	\begin{equation}
	\begin{aligned}
	X^{1/2}C^{1/2}\left(I + \left(I + a^2C^{1/2}XC^{1/2}\right)^{1/2}\right)^{-1}C^{1/2}X^{1/2} 
	\\
	= 
	-\frac{1}{a^2}I + \frac{1}{a^2}\left(I + a^2 X^{1/2}CX^{1/2}\right)^{1/2}.
	\end{aligned}
	\end{equation}
\end{lemma}
\begin{proof}
	The desired equality is
	\begin{equation*}
	\begin{aligned}
	&X^{1/2}C^{1/2}\left(I + \left(I + a^2C^{1/2}XC^{1/2}\right)^{1/2}\right)^{-1}C^{1/2}X^{1/2}  
	\\
	&\quad\quad = -\frac{1}{a^2}I + \frac{1}{a^2}\left(I + a^2 X^{1/2}CX^{1/2}\right)^{1/2}
	\\
	&\equivalent a^2 X^{1/2}C^{1/2}\left(I + \left(I + a^2C^{1/2}XC^{1/2}\right)^{1/2}\right)^{-1}C^{1/2}X^{1/2} 
	\\
	&\quad \quad = 
	-I + \left(I + a^2 X^{1/2}CX^{1/2}\right)^{1/2}
	\\
	&\equivalent
	a^2 X^{1/2}C^{1/2}\left(I + \left(I + a^2C^{1/2}XC^{1/2}\right)^{1/2}\right)^{-1}C^{1/2}X^{1/2}
	\\  
	&\quad \quad = a^2 X^{1/2}CX^{1/2}\left(I+\left(I + a^2 X^{1/2}CX^{1/2}\right)^{1/2}\right)^{-1}
	\\
	& \equivalent
	a^2 X^{1/2}C^{1/2}\left(I + \left(I + a^2C^{1/2}XC^{1/2}\right)^{1/2}\right)^{-1}
	\\
	&\quad \quad \times C^{1/2}X^{1/2} 
	\left(I + \left(I + a^2 X^{1/2}CX^{1/2}\right)^{1/2}\right) = a^2 X^{1/2}CX^{1/2}.
	\end{aligned}
	\end{equation*}
	This last equality is valid as a consequence of Lemma \ref{lemma:adjont-switch}, which gives
	\begin{equation*}
	\begin{aligned}
	&C^{1/2}X^{1/2} 
	\left(I + \left(I + a^2 X^{1/2}CX^{1/2}\right)^{1/2}\right)
	\\ 
	&= C^{1/2}X^{1/2} + C^{1/2}X^{1/2}\left(I + a^2 X^{1/2}CX^{1/2}\right)^{1/2}
	\\
	& =  C^{1/2}X^{1/2} + \left(I + a^2 C^{1/2}XC^{1/2}\right)^{1/2}C^{1/2}X^{1/2}
	\\
	& = \left(I +\left(I + a^2 C^{1/2}XC^{1/2}\right)^{1/2} \right)C^{1/2}X^{1/2}.
	\end{aligned}
	\end{equation*}
	Together with the left hand side of the previous expression, this gives the desired equality.
	\qed
\end{proof}

We apply the following result on the log-concavity of the Fredholm determinant from \cite{Minh:LogDet2016}, which is a generalization of Ky Fan's inequality for the log-concavity of the determinant on the set of symmetric positive definite matrices
\cite{KyFan:1950}.
\begin{proposition}
	[\textbf{Proposition 7 in \cite{Minh:LogDet2016}}]
	\label{proposition:Fredholm-concave}
	Let $A,B \in \Sym(\H) \cap \Tr(\H)$ be such that $I+A> 0$, $I+B > 0$. Then 
	for any fixed $0 < \alpha < 1$, 
	\begin{align}
	\det(I+\alpha A + (1-\alpha) B) \geq \det(I+A)^{\alpha}\det(I+B)^{1-\alpha},
	\end{align}
	with equality if and only if $A = B$.
\end{proposition}

\begin{lemma}
	\label{lemma:logdet(I-V*V)-concave}
	Let $\Omega = \{X \in \Tr(\H), ||X|| < 1\}$. The function $f:\Omega \mapto \R$ defined by $f(X) = \log\det(I-X^{*}X)$ is strictly concave, i.e. for $0< \alpha<1$ fixed,
	\begin{align}
	&\log\det[I - (\alpha A +(1-\alpha)B)^{*}(\alpha A +(1-\alpha)B)] 
	\nonumber
	\\
	&\quad \geq \alpha \log\det(I-A^{*}A) 
	+ (1-\alpha)\log\det(I-B^{*}B),
	\end{align}
	$\forall A,B \in \Omega$. Equality happens if and only if $A=B$.
\end{lemma}
\begin{proof} By Lemma \ref{lemma:det2VV}, for $X \in \Tr(\H), ||X||<1$,
	$\left\|\begin{pmatrix}0 & X \\ X^{*} & 0\end{pmatrix}\right\| < 1$ and
	thus $\begin{pmatrix}I & X \\ X^{*} & I\end{pmatrix} > 0 $, with 
	$\det(I-X^{*}X) = \det\begin{pmatrix}I & X \\ X^{*} & I\end{pmatrix}$.
	For any $A,B \in \Omega$ and a fixed $0< \alpha < 1$, we have
	$\alpha A + (1-\alpha)B \in \Omega$. Thus by Proposition \ref{proposition:Fredholm-concave},
	\begin{align*}
	&\det[I - (\alpha A + (1-\alpha)B)^{*}(\alpha A + (1-\alpha) B)] 
	\\
	& =\det\begin{pmatrix} I & \alpha A + (1-\alpha)B\\ \alpha A^{*} + (1-\alpha) B^{*} & I\end{pmatrix}
	\\
	& =\det\left[I + \alpha \begin{pmatrix}0 & A \\ A^{*} & 0\end{pmatrix} + (1-\alpha)\begin{pmatrix}0 & B \\ B^{*} & 0\end{pmatrix}\right]\geq \det\begin{pmatrix}I & A \\ A^{*} & I\end{pmatrix}^{\alpha}\det\begin{pmatrix}I & B \\ B^{*} & I\end{pmatrix}^{1-\alpha}
	\\
	& = \det(I-A^{*}A)^{\alpha}\det(I-B^{*}B)^{1-\alpha},
	\end{align*}
	with equality if and only $A=B$, from which the desired result follows.
	\qed
\end{proof}

\begin{theorem} 
	[\textbf{Optimal entropic transport plan and entropic Wasserstein distance - the nonsingular case}]
	\label{theorem:entropic-gaussian-direct}
	Let $\mu_X = \Ncal(m_0, C_0)$, $\mu_Y = \Ncal(m_1, C_1)$, with $\ker(C_0) = \ker(C_1) = \{0\}$. Then
	\begin{align}
	\label{equation:entropic-min-1}
	&\min_{\gamma \in \Joint(\mu_X, \mu_Y)}\left\{\bE_{\gamma}||x-y||^2 + \epsilon \KL\left({\gamma}||{\mu_X \otimes \mu_Y}\right)\right\}
	\\
	&= ||m_0-m_1||^2 + \trace(C_0) + \trace(C_1) 
	\nonumber
	\\
	&- \max_{V \in \HS(\H), ||V||< 1}\left\{2\trace(VC_1^{1/2}C_0^{1/2}) +\frac{\epsilon}{2}\log\det(I - V^{*}V)\right\}
	\label{equation:entropic-min-2}
	\\
	& = ||m_0-m_1||^2 + \trace(C_0) + \trace(C_1) - \frac{\ep}{2}\trace(M^{\ep}_{01})+ \frac{\ep}{2}\log\det\left(I + \frac{1}{2}M^{\ep}_{01}\right).
	\nonumber
	\end{align}
	The unique minimizer is $V = \frac{4}{\ep}\left(I + (I + \frac{16}{\ep^2}C_0^{1/2}C_1C_0^{1/2})^{1/2}\right)^{-1}C_0^{1/2}C_1^{1/2}$.
	This corresponds to the unique minimizing Gaussian measure 
	\begin{align}
	\label{equation:minimizing-Gaussian-measure-1}
	\gamma^{\ep} &= \Ncal\left(\begin{pmatrix} m_0 \\ m_1 \end{pmatrix},
	\begin{pmatrix} C_0 & C_{XY}
	\\
	C_{XY}^{*} & C_1\end{pmatrix}
	\right),
	\\
	\text{where  } C_{XY} &= \frac{4}{\ep}C_0^{1/2}\left(I + (I + \frac{16}{\ep^2}C_0^{1/2}C_1C_0^{1/2})^{1/2}\right)^{-1}C_0^{1/2}C_1.
	\end{align}
\end{theorem}

\begin{proof}
	[\textbf{of Theorem \ref{theorem:entropic-gaussian-direct}}]
	By Proposition \ref{proposition:entropic-OT-nonsingular},
	a minimizer of Eq.\eqref{equation:entropic-min-1} must necessarily be Gaussian and satisfy $\gamma \sim \mu_X \otimes \mu_Y$.
	By Theorem \ref{theorem:MutualInfo-Gaussian},
	\begin{align*}
	&\bE_{\gamma}||x-y||^2 + \epsilon \KL\left(\gamma||{\mu_X \otimes \mu_Y}\right), \;\;\;\text{with $\gamma \in \Gauss(\mu_X, \mu_Y)$}
	\\
	&= ||m_0-m_1||^2 + \trace(C_0) + \trace(C_1) - 2\trace(C_0^{1/2}VC_1^{1/2}) -\frac{\epsilon}{2}\log\det(I - V^{*}V)
	\\
	& = ||m_0-m_1||^2 + \trace(C_0) + \trace(C_1) - 2\trace(VC_1^{1/2}C_0^{1/2}) -\frac{\epsilon}{2}\log\det(I - V^{*}V),
	\end{align*}
	where $V \in \HS(\H), ||V|| < 1$. It follows that
	\begin{align*}
	&\min_{\gamma \in \Gauss(\mu_X, \mu_Y)}\left\{\bE_{\gamma}||x-y||^2 + \epsilon \KL\left({\gamma}||{\mu_X \otimes \mu_Y}\right)\right\}
	\\
	& = ||m_0 - m_1||^2 + \trace(C_0) + \trace(C_1) 
	\\
	& \quad - \max_{V \in \HS(\H), ||V||< 1}\left\{2\trace(VC_1^{1/2}C_0^{1/2}) +\frac{\epsilon}{2}\log\det(I - V^{*}V)\right\}.
	\end{align*}
	Let $g: \Lcal(\H) \mapto \Sym(\H)$ be defined by $g(X) = X^{*}X$, then
	\begin{align}
	Dg(X_0)(X) = X_0^{*}X + X^{*}X_0, \;\;\; X_0, X \in \Lcal(\H).
	\end{align}
	Let $\Omega = \{X \in \HS(\H), ||X||< 1\}$.
	Let $f:\Omega \mapto \R$ be defined by
	\begin{align}
	\label{equation:f-maximize}
	f(X) = 2\trace(XC_1^{1/2}C_0^{1/2}) + \frac{\ep}{2}\log\det(I-X^{*}X).
	\end{align}
	By the chain rule {\color{black}and Lemmas \ref{lemma:derivative-trace} and \ref{lemma:derivative-logdet}}, we have for $X_0 \in \Omega$, $X \in \HS(\H)$
	\begin{equation}
	\begin{aligned}
	Df(X_0)(X) &= 2\trace(XC_1^{1/2}C_0^{1/2}) - \frac{\ep}{2}\trace[(I-X_0^{*}X_0)^{-1}(X_0^{*}X + X^{*}X_0)]
	\\
	& = 2\trace(XC_1^{1/2}C_0^{1/2}) -\ep \trace[(I-X_0^{*}X_0)^{-1}X_0^{*}X].
	\end{aligned}
	\end{equation}
	Thus $Df(X_0)(X) = 0$ $\forall X \in \HS(\H)$ if and only if
	\begin{equation}
	\label{equation:X0-0}
	(I-X_0^{*}X_0)^{-1}X_0^{*} = \frac{2}{\ep}C_1^{1/2}C_0^{1/2}.
	\end{equation}
	Since the right hand side of Eq.\eqref{equation:X0-0} is trace class, any solution $X_0$ must necessarily satisfy
	$X_0 \in \Tr(\H)$. By Lemma \ref{lemma:logdet(I-V*V)-concave}, the function $f$, as defined in 
	Eq.\eqref{equation:f-maximize}, is strictly concave in the set $\Omega_2 = \{X \in \Tr(\H), ||X||<1\}$,
	since the first term is linear in $X$. Thus any solution of Eq.\eqref{equation:X0-0}
	is necessarily unique and is the unique maximizer of $f$ on the larger set $\Omega = \{X \in \HS(\H), ||X||<1\}$,
	corresponding therefore to the unique minimizer of Eq.\eqref{equation:entropic-min-1}.
	
	We claim that, with $c_{\ep} =\frac{4}{\ep}$, the unique solution of Eq.\eqref{equation:X0-0} in $\Omega_2$ is
	\begin{align}
	X_0 = c_{\ep}(I + (I + c_{\ep}^2C_0^{1/2}C_1C_0^{1/2})^{1/2})^{-1}C_0^{1/2}C_1^{1/2}.
	\label{equation:X0-1}
	\end{align}
	Clearly $X_0 \in \Tr(\H)$.
	By Lemma \ref{lemma:adjoint-switch-2}, we also have
	\begin{align}
	\label{equation:X0-2}
	X_0 =c_{\ep} C_0^{1/2}C_1^{1/2}(I + (I + c_{\ep}^2C_1^{1/2}C_0C_1^{1/2})^{1/2})^{-1}.
	\end{align}
	We first show that with the above expression for $X_0$,
	\begin{equation}
	\label{equation:I-X*X}
	I-X_0^{*}X_0 = \left(\frac{1}{2}I + \frac{1}{2}(I + c_{\ep}^2C_1^{1/2}C_0C_1^{1/2})^{1/2}\right)^{-1}. 
	\end{equation}
	This gives $||X_0||^2 = ||X_0^{*}X_0|| < 1$. Eq.\eqref{equation:I-X*X} is 
	equivalent to
	\begin{align*}
	&X_0^{*}X_0 = I - \left(\frac{1}{2}I + \frac{1}{2}(I + c_{\ep}^2C_1^{1/2}C_0C_1^{1/2})^{1/2}\right)^{-1}
	\\
	&= \left(-I + (I + c_{\ep}^2C_1^{1/2}C_0C_1^{1/2})^{1/2}\right)
	\left(I + (I + c_{\ep}^2C_1^{1/2}C_0C_1^{1/2})^{1/2}\right)^{-1}
	\\
	& = c_{\ep}^2C_1^{1/2}C_0C_1^{1/2}\left(I + (I + c_{\ep}^2C_1^{1/2}C_0C_1^{1/2})^{1/2}\right)^{-2}
	\\
	& = c_{\ep}^2C_1^{1/2}C_0^{1/2}\left(I + (I + c_{\ep}^2C_0^{1/2}C_1C_0^{1/2})^{1/2}\right)^{-1}C_0^{1/2}C_1^{1/2}\left(I + (I + c_{\ep}^2C_1^{1/2}C_0C_1^{1/2})^{1/2}\right)^{-1}
	\end{align*}
by Lemma \ref{lemma:adjoint-switch-2}.
	The last equality is valid by Eqs.\eqref{equation:X0-1} and \eqref{equation:X0-2}.
	It follows from Eq.\eqref{equation:I-X*X}, by invoking Lemma \ref{lemma:adjoint-switch-2}, that
	\begin{align*}
	&(I-X_0^{*}X_0)^{-1}X_0^{*} = \frac{4}{\ep}\left(\frac{1}{2}I + \frac{1}{2}(I + c_{\ep}^2C_1^{1/2}C_0C_1^{1/2})^{1/2}\right)C_1^{1/2}C_0^{1/2}
	\\
	&\quad \quad \quad \quad \quad \quad \quad \quad \quad \times \left(I + (I + c_{\ep}^2C_0^{1/2}C_1C_0^{1/2})^{1/2}\right)^{-1}
	\\
	& = \frac{2}{\ep}C_1^{1/2}C_0^{1/2}\left[I + c_{\ep}^2(I + c_{\ep}^2C_0^{1/2}C_1C_0^{1/2})^{1/2}\right]
	{\color{black}\left(I + \left(I + c_{\ep}^2C_0^{1/2}C_1C_0^{1/2}\right)^{1/2}\right)^{-1}}
	\\
	& = \frac{2}{\ep}C_1^{1/2}C_0^{1/2}, \;\;\text{which is \eqref{equation:X0-0}}.
	\end{align*}
	With $X_0$ as given in Eq.\eqref{equation:X0-1}, we obtain the unique minimizing Gaussian measure
	$\gamma^{\ep}$ of Eq.\eqref{equation:entropic-min-1}, with $C_{XY} = C_0^{1/2}X_0C_1^{1/2}$.
	Furthermore, 
	\begin{align*}
	\trace[X_0C_1^{1/2}C_0^{1/2}] &= \trace\left[c_{\ep}(I + (I + c_{\ep}^2C_0^{1/2}C_1C_0^{1/2})^{1/2})^{-1}C_0^{1/2}C_1C_0^{1/2}\right]
	\\
	& = \frac{1}{c_{\ep}}\trace\left[-I + (I + c_{\ep}^2C_0^{1/2}C_1C_0^{1/2})^{1/2}\right] = \frac{\ep}{4}\trace[M^{\ep}_{01}],
	\\
	\log\det(I-X_0^{*}X_0) &= -\log\det(\frac{1}{2}I + \frac{1}{2}(I + c_{\ep}^2C_1^{1/2}C_0C_1^{1/2})^{1/2}) 
	\\
	&= -\log\det(I+ \frac{1}{2}M^{\ep}_{10}) = -\log\det(I+ \frac{1}{2}M^{\ep}_{01}).
	\end{align*}
	Combining the last two expressions gives the entropic distance formula.
	\qed
	%
\end{proof}

Let us now compute the Radon-Nikodym density of the optimal entropic transport plan $\gamma^{\ep}$
in Theorem \ref{theorem:entropic-gaussian-direct}
with respect to $\mu_0 \otimes \mu_1$. We note that the optimal $V$ in Theorem \ref{theorem:entropic-gaussian-direct}
is trace class. Consequently, we can apply the following result (see Proposition 1.3.11 in \cite{DaPrato:PDEHilbert}
or Corollary 2 in \cite{Minh:2020regularizedDiv}). 
\begin{proposition}
	\label{proposition:Radon-Nikodym-traceclass}
	Let $\mu = \Ncal(m,Q)$, $\nu = \Ncal(m, R)$, with $\ker{Q} = \{0\}$ and $\mu \sim \nu$.
	Assume that $R = Q^{1/2}(I-S)Q^{1/2}$ with $S\in \Sym(\H)\cap\Tr(\H)$. Then
	\begin{align}
	\frac{d\nu}{d\mu}(x) = \det[(I-S)^{-1/2}]\exp\left\{-\frac{1}{2}\la Q^{-1/2}(x-m), S(I-S)^{-1}Q^{-1/2}(x-m)\right\}
	\end{align}
	where, in the $\Lcal^1(\H, \mu)$ sense,
	\begin{align}
	&\la Q^{-1/2}(x-m), S(I-S)^{-1}Q^{-1/2}(x-m)\ra 
	\nonumber
	\\
	&\doteq \lim_{N \approach \infty}
	\la Q^{-1/2}P_N(x-m), S(I-S)^{-1}Q^{-1/2}P_N(x-m)\ra.
	\end{align}
\end{proposition}
The following result can be obtained by direct verification.
\begin{lemma}
	\label{lemma:operator-block-inverse}
	Let $B,C \in \Lcal(\H)$ be such that $(I-BC)$ is invertible, then the 
	block operator $\begin{pmatrix}
	I \; &\;  B 
	\\
	C\; & \; I
	\end{pmatrix}: \H \times \H \mapto \H \times \H$ is invertible, with
	\begin{align}
	\begin{pmatrix}
	I \; &\;  B 
	\\
	C\; & \; I
	\end{pmatrix}^{-1} = 
	\begin{pmatrix}
	(I-BC)^{-1} \; & \; -(I-BC)^{-1}B\\
	-C(I-BC)^{-1} \; & \; I + C (I-BC)^{-1}B
	\end{pmatrix} \in \Lcal(\H \times \H).
	\end{align}
\end{lemma}

\begin{proposition}
	\label{proposition:Radon-Nikodym-minimizer}. Assume the hypothesis of Theorem \ref{theorem:entropic-gaussian-direct}.
	The Radon-Nikodym density of the optimal entropic transport plan $\gamma^{\ep}$ with respect to $\mu_0\otimes \mu_1$ is
	\begin{align}
	\label{equation:Radon-Nikodym-minimizer-1}
	\frac{d\gamma^{\ep}}{d(\mu_0 \otimes \mu_1)}(x,y) &=\sqrt{\det\left(I + \frac{1}{2}M^{\ep}_{01}\right)}
	\exp(\la x - m_0, A_{\ep}(x-m_0)\ra)
	\\
	&\quad \times \exp(\la y-m_1, B_{\ep}(y-m_1))\exp\left(\frac{2}{\ep}\la x-m_0, y-m_1\ra\right).
	\nonumber
	\end{align}
	Here
	$A_{\ep} = -\frac{2}{\ep^2}C_1^{1/2}\left(I + \frac{1}{2}M^{\ep}_{10}\right)^{-1}C_1^{1/2}$,
	$B_{\ep} =-\frac{2}{\ep^2}C_0^{1/2}\left(I + \frac{1}{2}M^{\ep}_{01}\right)^{-1}C_0^{1/2}$.
\end{proposition}
Eq.\eqref{equation:Radon-Nikodym-minimizer-1} 
is equivalent to
Eq.\eqref{equation:gamma-opt-gauss} in Theorem \ref{theorem:optimal-joint-square-gauss} 
via the identity
\begin{align}
||x-y||^2 
&= ||x-m_0||^2 + ||y-m_1||^2 + ||m_0-m_1||^2 - 2\la x-m_0, y-m_1\ra 
\nonumber
\\
&\quad+ 2\la x - m_0, m_0 - m_1\ra - 2\la y-m_1, m_0 - m_1\ra.
\end{align}
\begin{proof}
	For $||V||< 1$, $(I-VV^{*})$ is invertible
	and by Lemma \ref{lemma:operator-block-inverse},
	\begin{align*}
	\begin{pmatrix}
	I \; & \; V\\
	V^{*}\; & \; I
	\end{pmatrix}^{-1} = \begin{pmatrix}
	(I-VV^{*})^{-1} \; & \; -(I-VV^{*})^{-1}V\\
	-V^{*}(I-VV^{*})^{-1} \; & \; I + V^{*}(I-VV^{*})^{-1}V
	\end{pmatrix}.
	\end{align*}
	By the identity $V^{*}(I-VV^{*})^{-1} = (I-V^{*}V)^{-1}V^{*}$,
	\begin{align*}
	\begin{pmatrix}
	0 \; & \; V\\
	V^{*}\; & \; 0
	\end{pmatrix}
	\begin{pmatrix}
	I \; & \; V\\
	V^{*}\; & \; I
	\end{pmatrix}^{-1} = \begin{pmatrix}
	-VV^{*}(I-VV^{*})^{-1} \; & \; (I-VV^{*})^{-1}V\\
	V^{*}(I-VV^{*})^{-1} \; & \; -V^{*}(I-VV^{*})^{-1}V
	\end{pmatrix}
	\\
	=\begin{pmatrix}
	-(I-VV^{*})^{-1}VV^{*} \; & \; (I-VV^{*})^{-1}V\\
	(I-V^{*}V)^{-1}V^{*} \; & \; -(I-V^{*}V)^{-1}V^{*}V
	\end{pmatrix}.
	\end{align*}
	For $V = c_{\ep}(I + (I + c_{\ep}^2C_0^{1/2}C_1C_0^{1/2})^{1/2})^{-1}C_0^{1/2}C_1^{1/2} = \frac{c_{\ep}}{2}(I+\frac{1}{2}M_{01}^{\ep})^{-1}C_0^{1/2}C_1^{1/2}$, where $c_{\ep} = \frac{4}{\ep}$, we have $V \in \Tr(\H)$.
	Similar to the proof of Theorem \ref{theorem:entropic-gaussian-direct},
	\begin{align*}
	I-V^{*}V &= \left(\frac{1}{2}I + \frac{1}{2}(I + c_{\ep}^2C_1^{1/2}C_0C_1^{1/2})^{1/2}\right)^{-1} = \left(I+\frac{1}{2}M^{\ep}_{10}\right)^{-1},
	\\
	I-VV^{*} &= \left(\frac{1}{2}I + \frac{1}{2}(I + c_{\ep}^2C_0^{1/2}C_1C_0^{1/2})^{1/2}\right)^{-1} = \left(I+\frac{1}{2}M^{\ep}_{01}\right)^{-1},
	\\
	(I-V^{*}V)^{-1}V^{*} &=\frac{c_{\ep}}{2}C_1^{1/2}C_0^{1/2},
	\;\;\;
	(I-VV^{*})^{-1}V = \frac{c_{\ep}}{2}C_0^{1/2}C_1^{1/2},
	\\
	(I-VV^{*})^{-1}VV^{*} &= \frac{1}{4}c_{\ep}^2C_0^{1/2}C_1^{1/2}\left(I + \frac{1}{2}M^{\ep}_{10}\right)^{-1}C_1^{1/2}C_0^{1/2},
	\\
	(I-V^{*}V)^{-1}V^{*}V &= \frac{1}{4}c_{\ep}^2C_1^{1/2}C_0^{1/2}\left(I + \frac{1}{2}M^{\ep}_{01}\right)^{-1}C_0^{1/2}C_1^{1/2}.
	\end{align*}
	Since $V \in \Tr(\H)$, $S = -\begin{pmatrix}0 & V \\ V^{*} & 0\end{pmatrix} \in \Tr(\H\times \H)$ by Lemma 
	\ref{lemma:det2VV}, and 
	\begin{align*}
	\det(I-S)
	=\det(I-V^{*}V) = \left[\det(I+\frac{1}{2}M^{\ep}_{10})\right]^{-1} = \left[\det(I+\frac{1}{2}M^{\ep}_{01})\right]^{-1},
	\end{align*}
	since $C_0^{1/2}C_1C_0^{1/2}$ and $C_1^{1/2}C_0C_1^{1/2}$ have the same eigenvalues.
	Let $Q = \begin{pmatrix}
	C_0 & 0\\
	0 & C_1
	\end{pmatrix}$, 
	$A_{\ep} = -\frac{c_{\ep}^2}{8}C_1^{1/2}\left(I + \frac{1}{2}M^{\ep}_{10}\right)^{-1}C_1^{1/2}$,
	$B_{\ep} =-\frac{c_{\ep}^2}{8}C_0^{1/2}\left(I + \frac{1}{2}M^{\ep}_{01}\right)^{-1}C_0^{1/2}$,
	then
	\begin{align*}
	S(I-S)^{-1} = -
	Q^{1/2}
	\begin{pmatrix}
	2A_{\ep}& \; \frac{1}{2}c_{\ep}I 
	\\
	\frac{1}{2}c_{\ep}I & 2B_{\ep}
	\end{pmatrix}
	Q^{1/2}.
	\end{align*}
	Thus 
	for any $x,y \in \H$ and $N \in \Nbb$,
	\begin{align*}
	&\lim_{N \approach \infty}\left\la Q^{-1/2}P_N\begin{pmatrix}x - m_0\\y- m_1\end{pmatrix}, S(I-S)^{-1}Q^{-1/2}P_N\begin{pmatrix}x - m_0\\y- m_1\end{pmatrix}\right \ra
	\\
	& = -\left\la \begin{pmatrix}x - m_0\\y- m_1\end{pmatrix}, \begin{pmatrix}
	2A_{\ep}& \; \frac{1}{2}c_{\ep}I 
	\\
	\frac{1}{2}c_{\ep}I & 2B_{\ep}
	\end{pmatrix}\begin{pmatrix}x - m_0\\y- m_1\end{pmatrix}\right \ra
	\\
	&=
	-2\la x-m_0, A_{\ep}(x-m_0)\ra
	-c_{\ep}\la x-m_0, y-m_1\ra 
	-2\la y-m_1, B_{\ep}(y-m_1)\ra.
	\end{align*}
	Combining this with $\det(I-S)$ and Proposition \ref{proposition:Radon-Nikodym-traceclass} gives the desired result.
	\qed
\end{proof}

{\bf The general case}. We note that while Theorem \ref{theorem:entropic-gaussian-direct} and Proposition 
\ref{proposition:Radon-Nikodym-minimizer} are proved under the hypothesis that $\ker(C_0) = \ker(C_1) = \{0\}$,
the optimal solution obtained is clearly mathematically valid without this assumption. We now confirm that this is indeed the case via a different approach.

Let $k(x,y) = \exp\left(-\frac{c(x,y)}{\ep}\right)$, $\ep > 0$. 
It has been shown that, see e.g.
\cite{BorLewNus94,Csi75,DMaGer19,GigTam18,RusIPFP},
problem \eqref{equation:main-KL} has a unique minimizer $\gamma^{\ep}$ if and only if there exist functions $\alpha^{\ep}, \beta^{\ep}$ satisfying the {\it Schr\"odinger system}
\begin{equation}
\begin{aligned}
\alpha^{\ep}(x)\bE_{\mu_1}[\beta^{\ep}(y)k(x,y)] = 1,
\\
\beta^{\ep}(y)\bE_{\mu_0}[\alpha^{\ep}(x)k(x,y)] = 1.
\end{aligned}
\label{equation:Schrodinger-1}
\end{equation}
In this case, the unique minimizer $\gamma^{\ep}$ is the probability measure whose Radon-Nikodym derivative with respect to $\mu_0 \otimes \mu_1$ is given by
\begin{align}
\label{equation:gamma-Radon-Nikodym}
\frac{d\gamma^{\ep}}{d(\mu_0 \otimes \mu_1)}(x,y) = \alpha^{\ep}(x)\beta^{\ep}(y)k(x,y).
\end{align}
Motivated by Proposition \ref{proposition:Radon-Nikodym-minimizer}, we now solve the Schr\"odinger system 
\eqref{equation:Schrodinger-1}
when $c(x,y) = ||x-y||^2$ on $\H$, leading to another proof of Theorem \ref{theorem:optimal-joint-square-gauss},
which is valid in the general setting, where $C_0$ and $C_1$ can be singular.
 
We make use of the following results on Gaussian integrals on Hilbert spaces.
\begin{theorem}
	[\cite{DaPrato:PDEHilbert}, Proposition 1.2.8]
\label{theorem:gaussian-integral-1}
Consider the Gaussian measure $\Ncal(0,C)$ on $\H$.
	Assume that $M$ is a self-adjoint operator on $\H$ such that $\la C^{1/2}MC^{1/2} x, x \ra < ||x||^2$ $\forall 
	x \in \H, x \neq 0$. Let $b \in \H$. Then 
	\begin{align}
	&\int_{\H}\exp\left(\frac{1}{2}\la M y, y\ra + \la b,y\ra\right)d\Ncal(0,C)(y)
	\\
	& = [\det(I - C^{1/2}MC^{1/2})]^{-1/2}\exp\left(\frac{1}{2}||(I-C^{1/2}MC^{1/2})^{-1/2}C^{1/2}b||^2\right).
	\nonumber
	\end{align}
\end{theorem}
The following result then follows immediately
\begin{corollary}
\label{corollary:gaussian-integral-2}
Consider the Gaussian measure $\Ncal(m,C)$ on $\H$.
Assume that $M$ is a self-adjoint operator on $\H$ such that $\la C^{1/2}MC^{1/2} x, x \ra < ||x||^2$ $\forall 
x \in \H, x \neq 0$. Let $b \in \H$. Then 
\begin{align}
&\int_{\H}\exp\left(\frac{1}{2}\la M (y-m), (y-m)\ra + \la b,y-m\ra\right)d\Ncal(m,C)(y)
\\
& = [\det(I - C^{1/2}MC^{1/2})]^{-1/2}\exp\left(\frac{1}{2}||(I-C^{1/2}MC^{1/2})^{-1/2}C^{1/2}b||^2\right).
\nonumber
\end{align}	
\end{corollary}

\begin{proof}
[\textbf{of Theorem \ref{theorem:optimal-joint-square-gauss}: optimal entropic transport plan - the general case}]
We first have
\begin{align*}
&||x-y||^2 = ||(x-m_0) - (y-m_1) + (m_0-m_1)||^2 
\\
&= ||x-m_0||^2 + ||y-m_1||^2 + ||m_0-m_1||^2 - 2\la x-m_0, y-m_1\ra 
\\
&+ 2\la x - m_0, m_0 - m_1\ra - 2\la y-m_1, m_0 - m_1\ra.
\end{align*}
Expanding $\alpha^{\ep}(x)\beta^{\ep}(y)\exp\left(-\frac{||x-y||^2}{\ep}\right)$ under the assumptions 
\begin{align*}
\alpha^{\ep}(x) &= \exp\left(\la x-m_0, A(x-m_0)\ra + \frac{2}{\ep}\la x-m_0, m_0 - m_1\ra + a\right),
\\
\beta^{\ep}(y)  &= \exp\left(\la y-m_1, B(y-m_1)\ra + \frac{2}{\ep}\la y-m_1, m_1 - m_0\ra + b\right),
\end{align*}
we obtain
\begin{align*}
&\alpha^{\ep}(x)\beta^{\ep}(y)\exp\left(-\frac{||x-y||^2}{\ep}\right) = \exp(a+b)\exp\left(-\frac{||m_0-m_1||^2}{\ep}\right)
\\
&\times \exp\left(\left\la x-m_0, \left(A-\frac{1}{\ep}I\right) (x-m_0)\right\ra\right)
\exp\left(\left\la y-m_1, \left(B-\frac{1}{\ep}I\right)(y-m_1)\right\ra \right)
\\
& \times \exp\left(\frac{2}{\ep}\la x- m_0, y- m_1\ra\right).
\end{align*}
The Schr\"odinger system \eqref{equation:Schrodinger-1} then becomes
\begin{equation}
\begin{aligned}
1 = &{\exp(a+b)}\exp\left(-\frac{||m_0-m_1||^2}{\ep}\right)\exp\left(\left\la x-m_0,\left(A - \frac{1}{\epsilon}I\right)(x-m_0)\right\ra\right)
\\
& 
\times \int_{\H} \exp\left(\left\la y-m_1, \left(B-\frac{1}{\epsilon}I \right)(y-m_1)\right\ra +\frac{2}{\epsilon}\la x-m_0, y-m_1\ra\right)d\mu_1(y),
\\
1 = &{\exp(a+b)}\exp\left(-\frac{||m_0-m_1||^2}{\ep}\right)\exp\left(\left\la(y-m_1), \left(B - \frac{1}{\epsilon}I\right) (y-m_1)\right\ra\right)\\
& \times\int_{\H} \exp\left(\left\la x-m_0, \left(A-\frac{1}{\epsilon}I
\right)(x-m_0)\right\ra+\frac{2}{\epsilon}\la y-m_1,x-m_0\ra\right)d\mu_0(x).
\end{aligned}
\label{equation:gaussian_system}
\end{equation}
Let $A_{\ep} = A - \frac{1}{\epsilon}I$, $B_{\ep} = B - \frac{1}{\ep}I$, then  by Corollary \ref{corollary:gaussian-integral-2},
\begin{align*}
&\int_{\H} \exp\left(\la x-m_0,A_{\ep}(x-m_0)\ra+\frac{2}{\epsilon}\la (y-m_1), x-m_0\ra\right)d\mu_0(x)
\\
&= [\det(I - 2C_0^{1/2}A_{\ep}C_0^{1/2})]^{-1/2}\exp\left(\frac{2}{\ep^2}||(I-2C_0^{1/2}A_{\ep}C_0^{1/2})^{-1/2}C_0^{1/2}(y-m_1)||^2\right),
\\
&\int_{\H} \exp\left(\la (y-m_1),B_{\ep}(y-m_1)\ra+\frac{2}{\epsilon}\la x-m_0, y-m_1\ra\right)d\mu_1(y)
\\
&= [\det(I - 2C_1^{1/2}B_{\ep}C_1^{1/2})]^{-1/2}\exp\left(\frac{2}{\ep^2}||(I-2C_1^{1/2}B_{\ep}C_1^{1/2})^{-1/2}C_1^{1/2}(x-m_0)||^2\right).
\end{align*}
Thus for the system of equations \eqref{equation:gaussian_system} to hold $\forall x,y \in \H$, we must have
\begin{equation}
	\begin{aligned}
B_{\ep} &= -\frac{2}{\ep^2}C_0^{1/2}(I-2C_0^{1/2}A_{\ep}C_0^{1/2})^{-1}C_0^{1/2},
\\
A_{\ep} &= -\frac{2}{\ep^2}C_1^{1/2} (I-2C_1^{1/2}B_{\ep}C_1^{1/2})^{-1}C_1^{1/2},
\\
\exp(a+b) & = \exp\left(\frac{||m_0-m_1||^2}{\ep}\right)[\det(I - 2C_0^{1/2}A_{\ep}C_0^{1/2})]^{1/2},
\\
\exp(a+b) & = \exp\left(\frac{||m_0-m_1||^2}{\ep}\right)[\det(I - 2C_1^{1/2}B_{\ep}C_1^{1/2})]^{1/2}.
\end{aligned}
\label{equation:gaussian-system-2}
\end{equation}
We claim that the following $A_{\ep}$ and $B_{\ep}$ solve the system of equations \eqref{equation:gaussian-system-2}
\begin{align*}
A_{\ep} &= -\frac{2}{\epsilon^2}C_1^{1/2}\left[\frac{1}{2}I + \frac{1}{2}\left(I + \frac{16}{\epsilon^2}C_1^{1/2}C_0C_1^{1/2}\right)^{1/2}\right]^{-1}C_1^{1/2},
\\
B_{\ep} &= -\frac{2}{\epsilon^2}C_0^{1/2}\left[\frac{1}{2}I + \frac{1}{2}\left(I + \frac{16}{\epsilon^2}C_0^{1/2}C_1C_0^{1/2}\right)^{1/2}\right]^{-1}C_0^{1/2}.
\end{align*}
Let us verify the first equation in \eqref{equation:gaussian-system-2} (the second one is analogous).
We have
\begin{align}
\label{equation:C0AC0}
C_0^{1/2}A_{\ep}C_0^{1/2} = -\frac{2}{\epsilon^2}C_0^{1/2}C_1^{1/2}\left[\frac{1}{2}I + \frac{1}{2}\left(I + \frac{16}{\epsilon^2}C_1^{1/2}C_0C_1^{1/2}\right)^{1/2}\right]^{-1}C_1^{1/2}C_0^{1/2},
\\
\label{equation:I-C0AC0}
I- 2C_0^{1/2}A_{\ep}C_0^{1/2} = I + \frac{4}{\epsilon^2}C_0^{1/2}C_1^{1/2}\left[\frac{1}{2}I + \frac{1}{2}\left(I + \frac{16}{\epsilon^2}C_1^{1/2}C_0C_1^{1/2}\right)^{1/2}\right]^{-1}C_1^{1/2}C_0^{1/2}.
\end{align}
Thus in order to have $B_{\ep} = -\frac{2}{\ep^2}C_0^{1/2}(I-2C_0^{1/2}A_{\ep}C_0^{1/2})^{-1}C_0^{1/2}$,
we need
\begin{align}
\frac{1}{2}I + \frac{1}{2}\left(I + \frac{16}{\epsilon^2}C_0^{1/2}C_1C_0^{1/2}\right)^{1/2} = I- 2C_0^{1/2}A_{\ep}C_0^{1/2}
\label{equation:I-C0AC0-2}
\end{align}
Replacing the right hand side with the expression in Eq.\eqref{equation:I-C0AC0}, this is
\begin{align*}
&\frac{1}{2}I + \frac{1}{2}\left(I + \frac{16}{\epsilon^2}C_0^{1/2}C_1C_0^{1/2}\right)^{1/2}
\\
&\quad \quad=
I + \frac{4}{\epsilon^2}C_0^{1/2}C_1^{1/2}\left[\frac{1}{2}I + \frac{1}{2}\left(I + \frac{16}{\epsilon^2}C_1^{1/2}C_0C_1^{1/2}\right)^{1/2}\right]^{-1}C_1^{1/2}C_0^{1/2}
\\
&\equivalent
-I + \left(I + \frac{16}{\epsilon^2}C_0^{1/2}C_1C_0^{1/2}\right)^{1/2} 
\\
&\quad \quad= \frac{16}{\epsilon^2}C_0^{1/2}C_1^{1/2}\left[I + \left(I + \frac{16}{\epsilon^2}C_1^{1/2}C_0C_1^{1/2}\right)^{1/2}\right]^{-1}C_1^{1/2}C_0^{1/2}.
\end{align*}
This is precisely Lemma \ref{lemma:adjoint-switch-CX} with $a = \frac{4}{\ep}$.
Similarly, we have
\begin{align}
\frac{1}{2}I + \frac{1}{2}\left(I + \frac{16}{\epsilon^2}C_1^{1/2}C_0C_1^{1/2}\right) = I - 2C_1^{1/2}B_{\ep}C_1^{1/2}.
\end{align}
Finally, for the expression of $\exp(a+b)$, we note that it is clear that from Eq.\eqref{equation:C0AC0} that 
$C_0^{1/2}A_{\ep}C_0^{1/2} \in \Tr(\H)$ and from Eqs. \eqref{equation:I-C0AC0} and \eqref{equation:I-C0AC0-2} that $I - 2C_0^{1/2}A_{\ep}C_0^{1/2} > 0$.
Thus the Fredholm determinant of the latter expression is well-defined and positive. From Eq.\eqref{equation:I-C0AC0-2},
\begin{align*}
&\det(I - 2C_0^{1/2}A_{\ep}C_0^{1/2}) = \det\left(\frac{1}{2}I + \frac{1}{2}\left(I + \frac{16}{\epsilon^2}C_0^{1/2}C_1C_0^{1/2}\right)^{1/2}\right)\\
& =\det\left(\frac{1}{2}I + \frac{1}{2}\left(I + \frac{16}{\epsilon^2}C_1^{1/2}C_0C_1^{1/2}\right)^{1/2}\right)
= \det(I - 2C_1^{1/2}B_{\ep}C_1^{1/2}),
\end{align*}
since the nonzero eigenvalues of $C_0^{1/2}C_1C_0^{1/2}$ and $C_1^{1/2}C_0C_1^{1/2}$
are equal. \qed
\end{proof}

\begin{lemma}
	\label{lemma:integral-quadratic-prob-1}
Let $\mu \in \Pcal_2(\H)$ with
mean $m$ and covariance operator $C$.
Let $A \in \Lcal(\H)$. Then
\begin{align}
\int_{\H}\la x-m, A(x-m)\ra d\mu(x) = \trace(CA).
\end{align}
\end{lemma}
\begin{proof}
	It suffices to consider the case $m = 0$.  
	
	(i) Suppose $A \in \Sym^{+}(\H)$.
	Let $\{e_k\}_{k=1}^{\infty}$ be an orthonormal basis in $\H$. 
	Then
	\begin{align*}
	&\int_{\H}\la x, Ax\ra d\mu(x) = \int_{\H}||A^{1/2}x||^2 d\mu(x) = \int_{\H}\sum_{k=1}^{\infty}\la A^{1/2}x, e_k\ra^2 d\mu(x)
	\\
	& = \int_{\H}\sum_{k=1}^{\infty}\la x, A^{1/2}e_k\ra^2 d\mu(x) = \sum_{k=1}^{\infty}\int_{\H}\la x, A^{1/2}e_k\ra^2 d\mu(x)
	\\
	& \quad \quad \text{by Lebesgue Monotone Convergence Theorem}
	\\
	& = \sum_{k=1}^{\infty} \la CA^{1/2}e_k, A^{1/2}e_k\ra =\sum_{k=1}^{\infty}\la A^{1/2}CA^{1/2}e_k, e_k\ra = \trace(A^{1/2}CA^{1/2}) = \trace(CA).
	\end{align*}
	(ii) Suppose now $A \in \Sym(\H)$, then $A = A_1 - A_2$, where $A_1 = \frac{1}{2}(|A|+A) \in \Sym^{+}(\H), A_2 = \frac{1}{2}(|A|-A) \in \Sym^{+}(\H)$. Thus this case reduces to case (i).
	
	(iii) For any $A \in \Lcal(\H)$, $\int_{\H}\la x, Ax\ra d\mu(x) = \frac{1}{2}\int_{\H}\la x, (A+A^{*})x\ra d\mu(x)$. Using the fact $\trace(CA^{*}) = \trace(AC) = \trace(CA)$, this case reduces to case (ii). \qed
\end{proof}

\begin{corollary}
	[\textbf{Entropic Wasserstein distance between Gaussian measure on Hilbert space - the general case}]
	\label{corollary:OT-regularized-Gaussian}
	Let $\mu_0 = \Ncal(m_0, C_0)$ and $\mu_1 = \Ncal(m_1, C_1)$. For each fixed $\ep > 0$,
\begin{equation}
\begin{aligned}
\OT^{\ep}_{d^2}(\mu_0, \mu_1) &= ||m_0-m_1||^2 + \trace(C_0) + \trace(C_1) - \frac{\ep}{2}\trace(M^{\ep}_{01})
\\
&\quad +\frac{\ep}{2}\log\det\left(I + \frac{1}{2}M^{\ep}_{01}\right).
\end{aligned}
\end{equation}
\end{corollary}
\begin{proof} For the optimal $\frac{d\gamma^{\ep}}{d(\mu_0 \otimes \mu_1)}(x,y) = \alpha^{\ep}(x)\beta^{\ep}(y)\exp(-\frac{||x-y||^2}{\ep})$, 
	\begin{align*}
	\OT^{\ep}_{d^2}(\mu_0, \mu_1) &= \bE_{\gamma^{\ep}}||x-y||^2  + \ep\int_{\H \times \H}\log\left\{\frac{d\gamma^{\ep}}{d(\mu_0 \otimes \mu_1)}(x,y)\right\}d\gamma^{\ep}(x,y) 
	\\
	& = \ep\int_{\H \times \H}\log{\alpha^{\ep}(x)}d\gamma^{\ep}(x,y) + \ep\int_{\H \times \H}\log{\beta^{\ep}(y)}d\gamma^{\ep}(x,y)
	\\
	& = \ep\int_{\H}\log{\alpha^{\ep}(x)}d\mu_0(x) + \ep\int_{\H}\log{\beta^{\ep}(y)}d\mu_1(y).
	\end{align*}
Recall that
{\small
\begin{align*}
\alpha^{\ep}(x) &= \exp\left(\la x-m_0, A(x-m_0)\ra + \frac{2}{\ep}\la x-m_0, m_0 - m_1\ra + a\right),
\\
\beta^{\ep}(y)  &= \exp\left(\la y-m_1, B(y-m_1)\ra + \frac{2}{\ep}\la y-m_1, m_1 - m_0\ra + b\right),
\\
A &= \frac{1}{\ep}I-\frac{2}{\epsilon^2}C_1^{1/2}\left[\frac{1}{2}I + \frac{1}{2}\left(I + \frac{16}{\epsilon^2}C_1^{1/2}C_0C_1^{1/2}\right)^{1/2}\right]^{-1}C_1^{1/2},
\\
B &= \frac{1}{\ep}I-\frac{2}{\epsilon^2}C_0^{1/2}\left[\frac{1}{2}I + \frac{1}{2}\left(I + \frac{16}{\epsilon^2}C_0^{1/2}C_1C_0^{1/2}\right)^{1/2}\right]^{-1}C_0^{1/2},
\\
M^\epsilon_{01} &= -I + \left(I + \frac{16}{\epsilon^2}C_0^{1/2}C_1C_0^{1/2}\right)^\frac{1}{2}
= \frac{16}{\epsilon^2}C_0^{1/2}C_1C_0^{1/2}\left[I + \left(I + \frac{16}{\epsilon^2}C_0^{1/2}C_1C_0^{1/2}\right)^{1/2}\right]^{-1},
\\
M^\epsilon_{10} &= -I + \left(I + \frac{16}{\epsilon^2}C_1^{1/2}C_0C_1^{1/2}\right)^\frac{1}{2}
 = \frac{16}{\epsilon^2}C_1^{1/2}C_0C_1^{1/2}\left[I + \left(I + \frac{16}{\epsilon^2}C_1^{1/2}C_0C_1^{1/2}\right)^{1/2}\right]^{-1},
\\
\exp(a+b) &= \exp\left(\frac{||m_0-m_1||^2}{\ep}\right)\sqrt{\det\left(I + \frac{1}{2}M^{\ep}_{01}\right)}
\\
\equivalent (a+b) &= \frac{||m_0-m_1||^2}{\ep} + \frac{1}{2}\log\det\left(I + \frac{1}{2}M^{\ep}_{01}\right).
\end{align*}
}
It follows that
\begin{align*}
\OT_{d^2}^\epsilon(\mu_0, \mu_1)
=& \ep(a+b) + \epsilon
\bE_{X\sim\mu_0}\left[\la X-m_0, A(X-m_0)\ra + \frac{2}{\ep}\la X-m_0, m_0 - m_1\ra\right] 
\\
&+ 
\ep\bE_{Y\sim\mu_1}\left[\la Y-m_1, B(Y-m_1)\ra + \frac{2}{\ep}\la Y-m_1, m_1 - m_0\ra \right]
\\
=& \ep(a+b) + \epsilon\left(
\trace\left[C_0A\right] + \trace\left[C_1B\right]
\right).
\end{align*}
Here we have invoked Lemma \ref{lemma:integral-quadratic-prob-1}.
For the first trace term, we have
\begin{align*}
\trace[C_0A] &= \frac{1}{\ep}\trace(C_0) - \frac{2}{\ep^2}\trace\left[C_1^{1/2}C_0C_1^{1/2}\left(\frac{1}{2}I + \frac{1}{2}\left(I + \frac{16}{\epsilon^2}C_1^{1/2}C_0C_1^{1/2}\right)^{1/2}\right)^{-1}\right]
\\
 &= \frac{1}{\ep}\trace(C_0) - \frac{1}{4}\trace(M^{\ep}_{10}).
\end{align*}
Similarly, the second term is
\begin{align*}
&\trace(C_1B) = \frac{1}{\ep}\trace(C_1)-\frac{1}{4}\trace(M^{\ep}_{01}).
\end{align*}
Combining all the previous expressions, we obtain
\begin{align*}
&\OT^{\ep}_{d^2}(\mu_0, \mu_1) = ||m_0-m_1||^2 + \trace(C_0) + \trace(C_1) - \frac{\ep}{4}\trace(M^{\ep}_{01}) - \frac{\ep}{4}\trace(M^{\ep}_{10})
\\
& \quad \quad \quad \quad \quad \quad \quad + \frac{\ep}{2}\log\det\left(I + \frac{1}{2}M^{\ep}_{01}\right)
\\
&= ||m_0 - m_1||^2+ \trace(C_0) + \trace(C_1) - \frac{\ep}{2}\trace(M^{\ep}_{01}) +\frac{\ep}{2}\log\det\left(I + \frac{1}{2}M^{\ep}_{01}\right).
\end{align*}
Here we have used the fact that the nonzero eigenvalues of $C_0^{1/2}C_1C_0^{1/2}$ and $C_1^{1/2}C_0C_1^{1/2}$ are the same,
so that $\trace(M^{\ep}_{01}) = \trace(M^{\ep}_{10})$. \qed
\end{proof}

\begin{proof}
	[\textbf{of Corollary \ref{corollary:dual-attain} - dual formulation}]
	With $\varphi^{\ep} = \ep \log{\alpha^{\ep}},\psi^{\ep} = \ep\log{\beta^{\ep}}$,
	\begin{align*}
	&\bE_{\mu_0}\left[\exp\left(\frac{1}{\ep}\varphi^{\ep}\right)\right] = \bE_{\mu_0}[\alpha^{\ep}] = \int_{\H}\alpha^{\ep}(x)d\mu_0(x)
	\\
	&=\int_{\H} \exp\left(\la x-m_0,A(x-m_0)\ra+\frac{2}{\epsilon}\la x-m_0, m_0-m_1\ra + a\right)d\mu_0(x)
	\\
	&= \exp(a)[\det(I - 2C_0^{1/2}AC_0^{1/2})]^{-1/2}
	\\
	&\quad \times \exp\left(\frac{2}{\ep^2}||(I-2C_0^{1/2}A_{\ep}C_0^{1/2})^{-1/2}C_0^{1/2}(m_0-m_1)||^2\right)
	\end{align*}
	by Corollary \ref{corollary:gaussian-integral-2}. Clearly $0 < \bE_{\mu_0}\left(\frac{1}{\ep}\varphi^{\ep}\right) < \infty$ and
	thus $\varphi^{\ep} \in \Lexp(\H, \mu_0)$. Similarly $\psi^{\ep} \in \Lexp(\H, \mu_1)$. Let us now compute $D(\varphi^{\ep}, \psi^{\ep})$.
	We have
	\begin{align*}
	&\bE_{\mu_0}[\varphi^{\ep}] = \ep\bE_{\mu_0}[\log{\alpha^{\ep}}], \;\;\;\bE_{\mu_1}[\psi^{\ep}] = \ep\bE_{\mu_1}[\log{\beta^{\ep}}], 
	\\
	&\int_{\H \times \H}\left[\exp\left(\frac{\varphi^{\ep}(x) + \psi^{\ep}(y) - d^2(x,y)}{\ep}\right)-1\right]d(\mu_0\otimes \mu_1)(x,y)
	\\
	&=
	-1 + \int_{\H \times \H} \alpha^{\ep}(x)\beta^{\ep}(y)\exp\left(-\frac{||x-y||^2}{\ep}\right)d(\mu_0 \otimes \mu_1)(x,y)
	\\
	&= -1 + \int_{\H \times \H}\frac{d\gamma^{\ep}}{d(\mu_0\otimes \mu_1)}(x,y)d(\mu_0 \otimes \mu_1)(x,y)
	= -1 + \int_{\H \times \H}d\gamma^{\ep}(x,y) = 0.
	\end{align*}
	It follows that $D(\varphi^{\ep}, \psi^{\ep}) = \ep\bE_{\mu_0}[\log{\alpha^{\ep}}] + \ep\bE_{\mu_1}[\log{\beta^{\ep}}]=
	\OT^{\ep}_{d^2}(\mu_0, \mu_1)$, as in the proof of Corollary \ref{corollary:OT-regularized-Gaussian}.
	\qed
\end{proof}

\begin{corollary}
[\textbf{Sinkhorn divergence between Gaussian measures on Hilbert space}]
	\label{corollary:Sinkhorn-Gaussian-Hilbert}
	Let $\mu_0 = \Ncal(m_0, C_0)$, $\mu_1 = \Ncal(m_1, C_1)$. Then
	\begin{equation}
	\begin{aligned}
	\Srm^{\ep}_{d^2}(\mu_0, \mu_1) &= ||m_0 - m_1||^2 + \frac{\ep}{4}\trace\left[M^{\ep}_{00} - 2M^{\ep}_{01} + M^{\ep}_{11}\right] 
	\\
	& \quad + \frac{\ep}{4}\log\left[\frac{\det\left(I + \frac{1}{2}M^{\ep}_{01}\right)^2}{\det\left(I + \frac{1}{2}M^{\ep}_{00}\right)\det\left(I + \frac{1}{2}M^{\ep}_{11}\right)}\right].
	\end{aligned}
	\end{equation}
\end{corollary}

\begin{proof}
	By definition of the Sinkhorn divergence and Theorem \ref{theorem:OT-regularized-Gaussian},
	\begin{align*}
	&\Srm^{\ep}_{d^2}(\mu_0, \mu_1) = \OT^{\ep}_{d^2}(\mu_0, \mu_1) - \frac{1}{2}\OT^{\ep}_{d^2}(\mu_0, \mu_0) - \frac{1}{2}\OT^{\ep}_{d^2}(\mu_1, \mu_1)
	\\ 
	&= \quad||m_0 - m_1||^2 + \trace(C_0) + \trace(C_1) -\frac{\ep}{2}\trace(M^{\ep}_{01}) + \frac{\ep}{2}\log\det\left(I + \frac{1}{2}M^{\ep}_{01}\right)
	\\
	& \quad - \frac{1}{2}\left[2\trace(C_0) -\frac{\ep}{2}\trace(M^{\ep}_{00}) + \frac{\ep}{2}\log\det\left(I + \frac{1}{2}M^{\ep}_{00}\right)\right]
	\\
	& \quad - \frac{1}{2}\left[2\trace(C_1) -\frac{\ep}{2}\trace(M^{\ep}_{11}) + \frac{\ep}{2}\log\det\left(I + \frac{1}{2}M^{\ep}_{11}\right)\right]
	\\
	& = ||m_0 - m_1||^2 + \frac{\ep}{4}\trace\left[M^{\ep}_{00} - 2M^{\ep}_{01} + M^{\ep}_{11}\right] 
	+ \frac{\ep}{4}\log\left[\frac{\det\left(I + \frac{1}{2}M^{\ep}_{01}\right)^2}{\det\left(I + \frac{1}{2}M^{\ep}_{00}\right)\det\left(I + \frac{1}{2}M^{\ep}_{11}\right)}\right].
	\end{align*}
	This completes the proof. \qed
\end{proof}

\begin{proof}
	\textbf{of Theorem \ref{theorem:equivalent-formulas} - Equivalent expressions}.
	Let $c_{\ep} = \frac{4}{\ep}$.
	Then
	\begin{align*}
	C_i^{1/2}C_jC_i^{1/2} - L^{\ep}_{ij} &= c_{\ep}^2(C_i^{1/2}C_jC_i^{1/2})^2\left(I+(I + c_{\ep}^2C_i^{1/2}C_jC_i^{1/2})^{1/2}\right)^{-2},
	\\
	(C_i^{1/2}C_jC_i^{1/2} - L^{\ep}_{ij})^{1/2} &= c_{\ep}C_i^{1/2}C_jC_i^{1/2}\left(I+(I + c_{\ep}^2C_i^{1/2}C_jC_i^{1/2})^{1/2}\right)^{-1}
	\\
	& =\frac{1}{c_{\ep}}\left( -I + (I + c_{\ep}^2C_i^{1/2}C_jC_i^{1/2})^{1/2}\right) =\frac{\ep}{4} M^{\ep}_{ij}
	\\
	& =-\frac{\ep}{4}I + \left(\frac{\ep^2}{16}I + C_i^{1/2}C_jC_i^{1/2}\right)^{1/2}.
	\end{align*}
	Using these expressions, we see that the formulas for $\OT^{\ep}_{d^2}$ and $\Srm^{\ep}_{d^2}$ coincide with those in Theorems \ref{theorem:OT-regularized-Gaussian} and \ref{theorem:Sinkhorn-Gaussian-Hilbert}.
	Furthermore, it is immediately clear that
	\begin{align*}
	\lim_{\ep \approach 0}\trace(C_i^{1/2}C_jC_i^{1/2} - L^{\ep}_{ij})^{1/2} &= \trace(C_i^{1/2}C_jC_i^{1/2})^{1/2},
	\\
	\lim_{\ep \approach \infty}\trace(C_i^{1/2}C_jC_i^{1/2} - L^{\ep}_{ij})^{1/2} &= 0.
	\end{align*}
	By L'Hopital's rule, we have $\forall x \geq 0$,
	$\lim_{\ep \approach 0}\ep\log(\frac{1}{2} + \frac{1}{2}(1+\frac{16}{\ep^2}x)^{1/2}) = \lim_{\ep \approach \infty}\ep\log(\frac{1}{2} + \frac{1}{2}(1+\frac{16}{\ep^2}x)^{1/2}) = 0$.
	It follows that
	\begin{align*}
	\lim_{\ep \approach 0}\ep\log\det(I + \frac{1}{2}M^{\ep}_{ij}) = \lim_{\ep \approach \infty}\ep\log\det(I + \frac{1}{2}M^{\ep}_{ij}) = 0.
	\end{align*}
	Combining these limits gives the limiting behavior of $\OT^{\ep}_{d^2}$ and $\Srm^{\ep}_{d^2}$. \qed
\end{proof}

\begin{proof}
	\textbf{of Theorem \ref{theorem:RKHS-distance} - RKHS setting}. Let $\lambda(A)$ denote the set of nonzero eigenvalues of a compact operator $A$. Then
	\begin{align*}
	\lambda(C_{\Phi(\Xbf)}) & = \frac{1}{m}\lambda([\Phi(\Xbf) J_m \Phi(\Xbf)^{*}]) = \frac{1}{m}\lambda([\Phi(\Xbf)^{*}\Phi(\Xbf) J_m ])
	\\
	& = \frac{1}{m}\lambda(K[\Xbf]J_m) = \frac{1}{m}\lambda(J_mK[\Xbf]J_m), \;\;\text{since $J_m^2 = J_m$}
	\\
	\lambda(C_{\Phi(\Ybf)}) &= \frac{1}{m}\lambda(K[\Ybf]J_m) = \frac{1}{m}\lambda(J_mK[\Ybf]J_m),
	\\
	\lambda(C_{\Phi(\Xbf)}^2) &= \frac{1}{m^2}\lambda([\Phi(\Xbf) J_m \Phi(\Xbf)^{*}][\Phi(\Xbf) J_m \Phi(\Xbf)^{*}])
	\\
	&= \frac{1}{m^2}\lambda([\Phi(\Xbf)^{*}\Phi(\Xbf) J_m \Phi(\Xbf)^{*}\Phi(\Xbf) J_m ]) = \frac{1}{m^2}\lambda[(K[\Xbf]J_m)^2]
	\\
	& = \frac{1}{m^2}\lambda[(J_mK[\Xbf]J_m)^2],
	\\
	\lambda(C_{\Phi(\Ybf)}^2) &=\frac{1}{m^2}\lambda[(K[\Ybf]J_m)^2] = \frac{1}{m^2}\lambda[(J_mK[\Ybf]J_m)^2],
	\\
	\lambda(C_{\Phi(\Ybf)}^{1/2}C_{\Phi(\Xbf)}C_{\Phi(\Ybf)}^{1/2}) &=\lambda(C_{\Phi(\Xbf)}^{1/2}C_{\Phi(\Ybf)}C_{\Phi(\Xbf)}^{1/2}) = \lambda(C_{\Phi(\Ybf)}C_{\Phi(\Xbf)}) 
	\\
	&= \frac{1}{m^2}\lambda([\Phi(\Ybf) J_m \Phi(\Ybf)^{*}][\Phi(\Xbf) J_m \Phi(\Xbf)^{*}])
	\\
	&= \frac{1}{m^2}\lambda(K[\Xbf,\Ybf]J_mK[\Ybf,\Xbf]J_m)
	\\
	& = \frac{1}{m^2}\lambda(J_mK[\Xbf,\Ybf]J_mK[\Ybf,\Xbf]J_m)
	\end{align*}
	Combining these with the expressions for $\OT^{\ep}_{d^2}$ and $\Srm^{\ep}_{d^2}$ in Theorems 
	\ref{theorem:OT-regularized-Gaussian} and \ref{theorem:Sinkhorn-Gaussian-Hilbert} gives the desired results. \qed
\end{proof}

\section{Entropic $2$-Wasserstein barycenter}
\label{section:entropic-barycenter}
  
In this section, we prove Theorem \ref{theorem:entropic-OT-convexity} on the convexity of $\OT^{\ep}_{d^2}$ and
Theorem \ref{theorem:entropic-barycenter-Gaussian} on the $\OT^{\ep}_{d^2}$-based barycenter.
We first recall the concept of the Fr\'echet derivative on Banach spaces (see e.g. \cite{Jost:1998}). Let {$V,W$} be Banach spaces
and {$\Lcal(V,W)$} be the Banach space of bounded linear maps between {$V$} and {$W$}. Assume that {$f:\Omega \mapto W$} is well-defined, where
{$\Omega$} is an open subset of $V$. Then the map {$f$} is said to be Fr\'echet differentiable at {$x_0 \in \Omega$} if there exists a bounded 
linear map {$Df(x_0): V \mapto W$} such that
	\begin{align*}
	\lim_{h \approach 0}\frac{||f(x_0+h) - f(x_0) - Df(x_0)(h)||_W}{||h||_V} = 0.
	\end{align*}
The map {$Df(x_0)$} is called the Fr\'echet derivative of $f$ at $x_0$. 
Let now $W = \R$.
If {$x_0$} is a local minimizer for {$f$}, then  necessarily
(see e.g. Theorem 1.33 in \cite{Convex:2015})
	\begin{align}
	Df(x_0) = 0.
	\end{align}
	Furthermore, if $f$ is Fr\'echet differentiable, then (Proposition 3.11 in \cite{Convex:2015}) for $\Omega \subset V$ open and convex,
	\begin{equation}
	\label{equation:Frechet-derivative-strictly-convex}
	f \text{ is strictly convex } \equivalent f(y) > f(x) + Df(x)(y-x), \forall x, y \in \Omega.
	\end{equation}
	Thus $Df(x_0) = 0$ implies that $x_0$ is the unique global minimizer for $f$.

If the map {$Df:\Omega \mapto \Lcal(V,W)$} is differentiable at {$x_0$}, then its Fr\'echet derivative at {$x_0$},
denoted by {$D^2f(x_0): V \mapto \Lcal(V,W)$}, is called the second order derivative of {$f$} at {$x_0$}.
The bounded linear map {$D^2f(x_0) \in\Lcal(V,\Lcal(V,W))$}, can be identified with a bounded bilinear map from {$V \times V \mapto W$}, via
\begin{align*}
D^2f(x_0)(x,y)  = (D^2f(x_0)(x))(y), \;\; x, y\in V.
\end{align*}
Under this identification, {$D^2f(x_0)$} is a symmetric, continuous bilinear map from {$V \times V \mapto W$},
so that
$D^2f(x_0)(x,y) = D^2f(x_0)(y,x) \;\; \forall x,y \in V$.
For $W = \R$,
if $f$ is twice differentiable on $\Omega$, then in addition to \eqref{equation:Frechet-derivative-strictly-convex},
\begin{align}
&\text{$f$ is convex } \equivalent D^2f(x_0)(x,x) \geq 0 \;\forall x_0 \in \Omega, \forall x \in V,
\\
& D^2f(x_0)(x,x) > 0 \;\forall x_0 \in \Omega, \forall x \in V, x\neq 0 \imply
\text{$f$ is strictly convex }.
\end{align}

In the following, we focus on the Banach space $\Lcal(\H)$ of bounded operators, the subspace $\Sym(\H)\subset \Lcal(\H)$ of self-adjoint bounded operators, and the Banach space $\Tr(\H)$ of trace class operators on $\H$, respectively.

The following two results are straightforward.

\begin{lemma}
		Let $A, B \in \Lcal(\H)$ be fixed and consider the function $f:\Lcal(\H) \mapto \Lcal(\H)$ defined by
$f(X) = AXB$. Then
\begin{align}
Df(X_0)(X) = AXB, \;\;\; X_0, X \in \Lcal(\H).
\end{align}
\end{lemma}
\begin{lemma}
	\label{lemma:derivative-trace}

	For the function $\trace:\Tr(\H) \mapto \R$, 
	\begin{align}
	D\trace(X_0)(X) = \trace(X), \;\;\; X_0, X \in \Tr(\H).
	\end{align}
\end{lemma}
The following are special cases of Lemmas 3 and 4 in \cite{Minh:2019AlphaBeta}, respectively.

\begin{lemma}
	\label{lemma:derivative-det-Fredholm}
	Let $\Omega = \{A \in \Tr(\H): I+A \text{ is invertible}\}$. Define $f: \Omega \mapto \R$ by
	$f(X) = \det(I+X)$. Then $Df(X_0): \Tr(\H) \mapto \R$, $X_0 \in \Omega$, is given by
	\begin{align}
	Df(X_0)(X) = \det(I+X_0)\trace[(I+X_0)^{-1}X], \;\;\; X \in \Tr(\H).
	\end{align}	
\end{lemma}

\begin{lemma}
	\label{lemma:derivative-logdet}
	Let $\Omega = \{A \in \Tr(\H): I+A > 0\}$. Let $f: \Omega \mapto \R$ be defined by
	$f(X) = \log\det(I+X)$. Then $Df(X_0): \Tr(\H) \mapto \R$, $X_0 \in \Omega$, is given by 
	\begin{align}
	Df(X_0)(X) = \trace[(I+X_0)^{-1}(X)], \;\;\; X \in \Tr(\H).
	\end{align}
\end{lemma}

\begin{lemma}
\label{lemma:derivative-square-root}
Let $\mysqrt:\Sym^{+}(\H) \mapto \Sym^{+}(\H)$ be defined by $\mysqrt(X) = X^{1/2}$. 
Let $\Omega \subset \Sym^{+}(\H)$ be an open subset.
The derivative $D\mysqrt(X_0): 
\Sym(\H) \mapto \Sym(\H)$, $X_0 \in \Omega \subset \Sym^{+}(\H)$, is given by 
\begin{align}
(X_0)^{1/2}D\mysqrt(X_0)(X) + D\mysqrt(X_0)(X)(X_0)^{1/2} = X, \;\;\; X \in \Sym(\H).
\end{align}
If, furthermore, $X_0$ is invertible and $X \in \Sym(\H) \cap \Tr(\H)$, then
\begin{align}
\trace[D\mysqrt(X_0)(X)] = \frac{1}{2}\trace[(X_0)^{-1/2}X].
\end{align}
In general, let $f:\Lcal(\H) \mapto \Lcal(\H)$ be such that $f(X_0)$ and $X_0$ commute, then
\begin{align}
\trace [f(X_0)D\mysqrt(X_0)(X)] = \frac{1}{2}\trace[(X_0)^{-1/2}f(X_0)X].
\end{align}
\end{lemma}
\begin{proof}
	i) For the function $\mysq(X) = X^2$, we have
	\begin{align}
	D\mysq(X_0)(X) = X_0X + XX_0.
	\end{align}
	Let $f(X) = X = \mysq(\mysqrt(X))$, 
	the first identity follows from
	the chain rule
	\begin{align*}
	Df(X_0)(X) = X &= D\mysq(\mysqrt(X_0)) \compose D\mysqrt(X_0)(X)
	\\
	 &=(X_0)^{1/2}D\mysqrt(X_0)(X) + D\mysqrt(X_0)(X)(X_0)^{1/2}.
	\end{align*}
	ii) For the second identity, we note that
	\begin{align*}
	&(X_0)^{1/2}D\mysqrt(X_0)(X) + D\mysqrt(X_0)(X)(X_0)^{1/2} = X
	\\
	&\equivalent D\mysqrt(X_0)(X) + (X_0)^{-1/2}D\mysqrt(X_0)(X)(X_0)^{1/2} = (X_0)^{-1/2}X.
	\end{align*}
	Taking trace on both sides gives
	$2\trace[D\mysqrt(X_0)(X)] = \trace[(X_0)^{-1/2}X]$.
	
	iii) For the third expression,
	\begin{align*}
	f(X_0)(X_0)^{1/2}D\mysqrt(X_0)(X) + f(X_0)D\mysqrt(X_0)(X)(X_0)^{1/2} = f(X_0)X.
	\end{align*}
	Since $f(X_0)$ and $X_0^{1/2}$ commute, this is the same as
	\begin{align*}
		&(X_0)^{1/2}f(X_0)D\mysqrt(X_0)(X) + f(X_0)D\mysqrt(X_0)(X)(X_0)^{1/2} = f(X_0)X
		\\
		& \equivalent f(X_0)D\mysqrt(X_0)(X) + (X_0)^{-1/2}f(X_0)D\mysqrt(X_0)(X)(X_0)^{1/2} = (X_0)^{-1/2}f(X_0)X.
	\end{align*}
	Taking trace on both sides gives
	\begin{align*}
	\trace [f(X_0)D\mysqrt(X_0)(X)] = \frac{1}{2}\trace[(X_0)^{-1/2}f(X_0)X].
	\end{align*}
\end{proof}
{\color{black}
In the following, consider the set of $p$th Schatten class operators (see e.g. \cite{gohberg1978nonselfadjoint}) $\Csc_p(\H) = \{A \in \Lcal(\H):
||A||_p = (\trace[(A^{*}A)^{p/2}])^{1/p} < \infty\}$, $1 \leq p \leq \infty$, with $\Csc_1(\H) = \Tr(\H)$,
$||\;||_{1} = ||\;||_{\tr}$,  $\Csc_2(\H) = \HS(\H)$, $||\;||_2 = ||\;||_{\HS}$, and $\Csc_{\infty}(\H)$ being the set of compact operators on $\H$ under $||\;||_{\infty} = ||\;||$, the operator norm.
Here $||A||_p \leq ||A||_q$ for $1 \leq q \leq p \leq \infty$.

	\begin{lemma}
		\label{lemma:open-set-omega}
		Let $C \in \Sym^{+}(\H)$ and $c \in \R$ be fixed. 
		Let $1 \leq p \leq \infty$ be fixed.
		The set 
		$\Omega_1 = \{X\in \Sym(\H) \cap \Csc_p(\H): I+c^2C^{1/2}XC^{1/2} > 0\}$ is open in 
		$\Sym(\H)\cap \Csc_p(\H)$ in 
		the $||\;||_{q}$ norm topology $\forall q, p \leq q \leq \infty$.
		The set 	$\Omega_2 = \{X\in \Sym(\H): I+c^2C^{1/2}XC^{1/2} > 0\}$ is open in 
		$\Sym(\H)$ in the
		operator $||\;||$ norm topology.
		\end{lemma}
	\begin{proof}
		Assume that $X_0 \in \Omega_1$, then $\exists M_{X_0} > 0$ such that
		$||x||^2 + c^2 \la x, C^{1/2}X_0C^{1/2}x\ra \geq M_{X_0}||x||^2$ $\forall x \in \H$.
		 Recall that $||X||_q \leq ||X||_p$ for $p \leq q \leq \infty$, so that $\Csc_p(\H) \subset \Csc_q(\H) \subset \Csc_{\infty}(\H)$.
		 We show that $\exists \ep > 0$ such that $X \in \Omega_1$ $\forall X \in \Sym(\H) \cap \Csc_p(\H)$ satisfying $||X-X_0||_q < \ep$.
		 Since $||X-X_0|| \leq ||X-X_0||_q$, $1 \leq q \leq \infty$, we have $||X-X_0||_q < \ep \imply
		 ||X-X_0||<\ep \imply -\ep||x||^2 \leq \la x, (X-X_0)x\ra \leq \ep||x||^2$ $\forall x \in \H$. 
		 It follows that for any $X\in \Sym(\H)\cap \Csc_p(\H)$ with $||X- X_0||_q < \ep$, $p \leq q \leq \infty$,
		 \begin{align*}
		 &||x||^2 + c^2\la x, C^{1/2}XC^{1/2}x\ra
		 \\
		  &= ||x||^2 + c^2\la x, C^{1/2}X_0C^{1/2}x\ra + c^2\la x, C^{1/2}(X-X_0)C^{1/2}x\ra  
		  \\
		  & \geq M_{X_0}||x||^2 - c^2\ep ||C^{1/2}x||^2 \geq (M_{X_0} - c^2\ep ||C||)||x||^2.
		 \end{align*}
	 Thus if we choose $\ep > 0$ such that $M_{X_0} - c^2\ep||C|| > 0$, then $I+ c^2 C^{1/2}XC^{1/2} > 0$, that is $X \in \Omega_1$.
	 The proof for $\Omega_2$ is entirely similar.\qed
	\end{proof}
	
}
\begin{lemma}
	\label{lemma:derivative-trace-square-root}
Let $C \in \Sym^{+}(\H)\cap \Tr(\H)$ be fixed.
Let $\Omega = \{X\in \Sym(\H) \cap \Tr(\H): I+c^2C^{1/2}XC^{1/2} > 0\}$, $c \in \R$. 
	Let $f:\Omega \mapto \R$ be defined by
	$f(X) = \trace\left[-I + \left(I + c^2C^{1/2}XC^{1/2}\right)^{1/2}\right]$. Then
	$Df(X_0): \Sym(\H) \cap \Tr(\H) \mapto \R$, $X_0 \in \Omega$,
	is given by
	\begin{align}
	Df(X_0)(X) = \frac{c^2}{2}\trace\left[C^{1/2}\left(I+c^2C^{1/2}X_0 C^{1/2}\right)^{-1/2}C^{1/2}X\right].
	\end{align}
\end{lemma}

\begin{proof}
Let $g(X) = \left(I + c^2C^{1/2}XC^{1/2}\right)^{1/2}$, $f(X) = \trace[-I + g(X)]$,
then
\begin{align*}
Df(X_0)(X) &= Df(g(X_0)) \compose Dg(X_0)(X) =\trace[Dg(X_0)(X)] 
\\
& = c^2\trace\left[D\mysqrt(I+c^2C^{1/2}X_0 C^{1/2})(C^{1/2}XC^{1/2})\right].
\end{align*}
By Lemma \ref{lemma:derivative-square-root},
\begin{align*}
&\trace\left[D\mysqrt(I+c^2C^{1/2}X_0 C^{1/2})(C^{1/2}XC^{1/2})\right]
\\
&\quad = \frac{1}{2}\trace\left[\left(I+c^2C^{1/2}X_0 C^{1/2}\right)^{-1/2}C^{1/2}XC^{1/2}\right].
\end{align*}
It thus follows that 
\begin{align*}
Df(X_0)(X)
& = \frac{c^2}{2}\trace\left[\left(I+c^2C^{1/2}X_0 C^{1/2}\right)^{-1/2}(C^{1/2}XC^{1/2})\right]
\\
& = \frac{c^2}{2}\trace\left[C^{1/2}\left(I+c^2C^{1/2}X_0 C^{1/2}\right)^{-1/2}C^{1/2}X\right].
\end{align*}
\end{proof}
\begin{lemma}
	\label{lemma:derivative-logdet-square-root}
	Let $C \in \Sym^{+}(\H) \cap \Tr(\H)$ be fixed.
	Let $\Omega = \{X\in \Sym(\H) \cap \Tr(\H): I+c^2C^{1/2}XC^{1/2} > 0\}$, $c \in \R$.
	Let $f:\Omega \mapto \R$ be defined by
	  $f(X) = \log\det\left[\frac{1}{2}I + \frac{1}{2}\left(I + c^2C^{1/2}XC^{1/2}\right)^{1/2}\right]$.
Then $Df(X_0): \Sym(\H) \cap \Tr(\H) \mapto \R$, $X_0 \in \Omega$,
is given by
{\small
\begin{align}
Df(X_0)(X) =\frac{c^2}{2}\trace
\biggl[C^{1/2}\left(\left(I + c^2C^{1/2}X_0C^{1/2}\right)^{1/2}+ \left(I + c^2C^{1/2}X_0C^{1/2}\right)\right)^{-1}C^{1/2}X\biggr]. 
\end{align}
}
\end{lemma}
\begin{proof}
Let $g(X) = -\frac{1}{2}I + \frac{1}{2}\left(I + c^2C^{1/2}XC^{1/2}\right)^{1/2}$, then $f(X) = \log\det[I+g(X)]$.
By Lemma \ref{lemma:derivative-logdet},
\begin{align*}
&Df(X_0)(X) = Df(g(X_0))\compose Dg(X_0)(X) = \trace[(I+g(X_0))^{-1}Dg(X_0)(X)].
\end{align*}
As in the proof of Lemma \ref{lemma:derivative-trace-square-root},
\begin{align*}
&Dg(X_0)(X) = \frac{c^2}{2}D\mysqrt\left(I + c^2C^{1/2}X_0C^{1/2}\right)(C^{1/2}XC^{1/2}).
\end{align*}
Thus it follows from Lemma \ref{lemma:derivative-square-root} that
{\small
\begin{equation*}
\begin{split}
&Df(X_0)(X) = \frac{c^2}{2}\trace
\biggl[\left(\frac{1}{2}I+ \frac{1}{2}\left(I + c^2C^{1/2}X_0C^{1/2}\right)^{1/2}\right)^{-1}
\\
&\quad \quad \quad \quad \quad \quad \quad D\mysqrt\left(I + c^2C^{1/2}X_0C^{1/2}\right)(C^{1/2}XC^{1/2})
\biggr]
\\
=& \frac{c^2}{4}\trace
\biggl[\left(I + c^2C^{1/2}X_0C^{1/2}\right)^{-1/2}\left(\frac{1}{2}I+ \frac{1}{2}\left(I + c^2C^{1/2}X_0C^{1/2}\right)^{1/2}\right)^{-1}(C^{1/2}XC^{1/2})\biggr]
\\
=& \frac{c^2}{2}\trace
\biggl[C^{1/2}\left(\left(I + c^2C^{1/2}X_0C^{1/2}\right)^{1/2}+ \left(I + c^2C^{1/2}X_0C^{1/2}\right)\right)^{-1}C^{1/2}X\biggr].
\end{split}
\end{equation*}
}
\end{proof}
\begin{lemma}
	\label{lemma:trace-functional-zero}
Let $A \in \Lcal(\H)$. Then
\begin{align}
\trace[AX] = 0 \;\;\;\forall X \in \Tr(\H) \equivalent A = 0.
\end{align}
If, furthermore, $A \in \Sym(\H)$, then
\begin{align}
\trace[AX] = 0 \;\;\;\forall X \in \Sym(\H) \cap \Tr(\H) \equivalent A = 0.
\end{align}
\end{lemma}
\begin{proof}
The first statement follows from the fact that $\Lcal(\H) = [\Tr(\H)]^{*}$, that is the map
$A \mapto \trace(A \cdot)$ is an isometric isomorphism between $\Lcal(\H)$ and the dual space
$[\Tr(\H)]^{*}$ of $\Tr(\H)$ (see e.g. Theorem VI.26 in \cite{ReedSimon:Functional}).

For the second statement, let $Y \in \Tr(\H)$ be arbitrary, then
\begin{align*}
\trace[AY] &= \frac{1}{2}[\trace(AY) + \trace((AY)^{*})] = \frac{1}{2}[\trace(AY) + \trace(Y^{*}A)]
\\
&= \frac{1}{2}[\trace(AY) + \trace(AY^{*})] = \frac{1}{2}\trace[A(Y+Y^{*})] = 0.
\end{align*}
Since $Y$ is arbitrary, this implies $A = 0$ by the first statement.
\qed
\end{proof}

\begin{lemma}
	\label{lemma:derivative-inverse}
	Let $\Omega = \{X \in \Lcal(\H): I+X \; \text{invertible}\}$. 
	For the map $f: \Omega \mapto \Omega$ defined by
	$f(X) = (I+X)^{-1}$, $Df(X_0): \Lcal(\H) \mapto \Lcal(\H)$ is given by
	\begin{align}
	Df(X_0)(X) = -(I+X_0)^{-1}X(I+X_0)^{-1}, \;\; X_0 \in \Omega, X \in  \Lcal(\H).
	\end{align}
\end{lemma}
\begin{proof}
	Using the identity $A^{-1} - B^{-1} = -A^{-1}(A-B)B^{-1}$, we have
	\begin{align*}
	&(I+X_0 + tX)^{-1} - (I+X_0)^{-1} 
	=-t(I+X_0+tX)^{-1}X(I+X_0)^{-1},
	\\
	&(I+X_0 + tX)^{-1} - (I+X_0)^{-1} + t(I+X_0)^{-1}X(I+X_0)^{-1}
	\\
	&=-t [(I+X_0 + tX)^{-1} - (I+X_0)^{-1}]X(I+X_0)^{-1}
	\\
	& = t^2(I+X_0+tX)^{-1}X(I+X_0)^{-1}X(I+X_0)^{-1}.
	\end{align*}
	Thus $\lim_{t \approach 0}\frac{||f(X_0+tX) - f(X_0)- Df(X_0)(tX))||}{|t|\;||X||} = 0$.
	\qed
\end{proof}
\begin{lemma}
	\label{lemma:derivative-trace-functional}
	For the map $f: \Tr(\H) \mapto \Lcal(\Lcal(\H),\R)$ defined by $f(X)(Y) = \trace[XY]$, the Fr\'echet derivative
	$Df(X_0):\Tr(\H) \mapto \Lcal(\Lcal(\H),\R)$ is given by
	\begin{align}
	[Df(X_0)(X)](Y) = \trace[XY], \;\;\; X \in \Tr(\H), Y \in \Lcal(\H).
	\end{align}
\end{lemma}

\begin{proof}
	For $g \in W = \Lcal(\Lcal(\H),\R)$, $g: \Lcal(\H) \mapto \R$, we have $||g||_W =  \sup_{||X||\leq 1}|g(X)|$.
	Let $V = \Tr(\H)$, then with $f:V \mapto W$,
	\begin{align*}
	&\lim_{t \approach 0}\frac{||f(X_0 + tX) - f(X_0) - Df(X_0)(tX)||_{W}}{|t|\;|X||_V}
	\\
	&= \lim_{t \approach 0}\sup_{Y \in \Lcal(\H), ||Y|| \leq 1}\frac{|f(X_0 + tX)(Y) - f(X_0)(Y) - Df(X_0)(tX)(Y)|}{|t|\;|X||_V}
	\\
	& = \lim_{t \approach 0}\sup_{Y \in \Lcal(\H), ||Y||\leq 1}\frac{|\trace[(X_0+tX)Y] - \trace(X_0Y) - t\trace(XY)|}{|t|\;||X||_{\tr}} = 0.
	\end{align*}
\end{proof}

\begin{lemma}
	\label{lemma:product-triple-operator-positive}
	Let $A \in \Sym(\H)\cap \HS(\H)$, $B \in \Sym^{+}(\H)$, $C\in \Sym^{+}(\H)$. Assume further that $B$
	and $C$ commute. Then
	\begin{align}
	\trace[C(AB)^2] = ||C^{1/2}B^{1/2}AB^{1/2}||^2_{\HS} \geq 0.
	\end{align} 
	If in addition $B,C$ are invertible, then equality happens if and only if $A = 0$.
\end{lemma}
We remark that the condition that $B$ and $C$ commute is crucial in Lemma \ref{lemma:product-triple-operator-positive}.
It can be verified numerically that, without this condition, the stated inequality is generally false even if $A$ is also positive.
\begin{proof} By the assumption that $B$ and $C$ commute,
	\begin{align*}
	&\trace[C(AB)^2] = \trace[CABAB]= \trace[C^{1/2}AB^{1/2}B^{1/2}AB^{1/2}C^{1/2}B^{1/2}]
	\\
	&=\trace[C^{1/2}B^{1/2}AB^{1/2}B^{1/2}AB^{1/2}C^{1/2}] = ||C^{1/2}B^{1/2}AB^{1/2}||^2_{\HS} \geq 0.
	\end{align*}
	Equality happens if and only if $C^{1/2}B^{1/2}AB^{1/2} = 0$. If $B$ and $C$ are invertible, then this happens if and only if $A = 0$. \qed
\end{proof}

\begin{lemma}
	\label{lemma:trace-DX-X-mysqrt}
	Let $Y_0, Z_0 \in \Sym^{+}(\H)$.
	Assume that $Y_0$ and $Z_0$ commute, then $\forall X \in \Sym(\H) \cap \HS(\H)$,
	\begin{align}
	\trace[ D\mysqrt(Z_0)(X)Y_0XY_0] = 2 ||Z_0^{1/4}Y_0^{1/2}D\mysqrt(Z_0)(X)Y_0^{1/2}||_{\HS}^2 \geq 0.
	\end{align}
	If in addition $Y_0,Z_0$ are invertible, then equality happens if and only if $X=0$.
\end{lemma}
\begin{proof}
	By Lemma \ref{lemma:derivative-square-root}, for any $X \in \Sym(\H)\cap \Tr(\H)$,
	\begin{align}
	\label{equation:mysqrt-Frechet-identity}
	Z_0^{1/2}D\mysqrt(Z_0)(X) + D\mysqrt(Z_0)(X)Z_0^{1/2} = X.
	\end{align}
	Pre- and post-multiplying both sides by $D\mysqrt(Z_0)(X)Y_0$ and $Y_0$, 
	respectively,
	\begin{align*}
	&D\mysqrt(Z_0)(X)Y_0Z_0^{1/2}D\mysqrt(Z_0)(X)Y_0 
	+D\mysqrt(Z_0)(X)Y_0D\mysqrt(Z_0)(X)Z_0^{1/2}Y_0 
	\\
	&= D\mysqrt(Z_0)(X)Y_0XY_0.
	\end{align*}
	Taking trace on both sides and applying Lemma \ref{lemma:product-triple-operator-positive} gives
	\begin{align*}
	&\trace[ D\mysqrt(Z_0)(X)Y_0XY_0] 
	= \trace\left(Z_0^{1/2}\left[D\mysqrt(Z_0)(X)Y_0\right]^2 +Z_0^{1/2}\left[Y_0D\mysqrt(Z_0)(X)\right]^2 \right)
	\\
	&= 2 \trace\left(Z_0^{1/2}\left[D\mysqrt(Z_0)(X)Y_0\right]^2\right)
	= 2 ||Z_0^{1/4}Y_0^{1/2}D\mysqrt(Z_0)(X)Y_0^{1/2}||_{\HS}^2 \geq 0.
	\end{align*}
	If $Y_0, Z_0$ are invertible, by Lemma \ref{lemma:product-triple-operator-positive}, the zero equality 
	happens if and only if $D\mysqrt(Z_0)(X) = 0$, which is equivalent to $X=0$ by Eq.\eqref{equation:mysqrt-Frechet-identity}
	\qed
\end{proof}

\begin{lemma}
	\label{lemma:derivative-functional-trace-inverse}
	Let $C \in \Sym^{+}(\H) \cap \Tr(\H)$. Let $\Omega = \{X \in \Sym(\H): I+c^2C^{1/2}XC^{1/2} > 0\}$,
	$c \in \R, c \neq 0$.
	Define
	$f: \Omega \mapto \Lcal(\Lcal(\H), \R)$ by 
	\begin{align}
	f(X)(Y) = \trace\left[C^{1/2}\left(I+(I+ c^2C^{1/2}XC^{1/2})^{1/2}\right)^{-1}C^{1/2}Y\right].
	\end{align}
	The Fr\'echet derivative $Df(X_0): \Sym(\H) \mapto \Lcal(\Lcal(\H),\R)$ is given by
	\begin{align}
	&[Df(X_0)(X)](Y) \;\;\;\;\;\;\;\;\;\;\;\;\; X \in \Sym(\H), Y \in \Lcal(\H) 
	\\
	&=-c^2\trace[D\mysqrt(Z_0)(C^{1/2}XC^{1/2})(I+Z_0^{1/2})^{-1}C^{1/2}YC^{1/2}(I+Z_0^{1/2})^{-1}],
	\nonumber
	\end{align}
	where $Z_0 =I+ c^2C^{1/2}X_0C^{1/2}$. In particular, for $Y = X$,
	\begin{align}
	&[Df(X_0)(X)](X) 
	\\
	& = -2c^2 ||Z_0^{1/4}(I+Z_0^{1/2})^{-1/2}D\mysqrt(Z_0)(C^{1/2}XC^{1/2})(I+Z_0^{1/2})^{-1/2}||_{\HS}^2 \leq 0.
	\nonumber
	\end{align}
	For $c \neq 0$,
	equality happens if and only $C^{1/2}XC^{1/2} = 0$. 
	If $C$ is strictly positive, then equality happens if and only if $X=0$.
\end{lemma}
\begin{proof}
	Let $h(X) = (I+ c^2C^{1/2}XC^{1/2})^{1/2}$, $g(X) = (I+h(X))^{-1}$, and $f(X)(Y) = \trace[C^{1/2}g(X)C^{1/2}Y]$.
	By the chain rule and Lemma \ref{lemma:derivative-inverse},
	\begin{align*}
	&Dg(X_0)(X) = [Dg(h(X_0))\compose Dh(X_0)](X) 
	\\
	&= -c^2(I+h(X_0))^{-1}D\mysqrt(I+ c^2C^{1/2}X_0C^{1/2})(C^{1/2}XC^{1/2})(I+h(X_0))^{-1}.
	\end{align*}
	Let $Z_0 = I+ c^2C^{1/2}X_0C^{1/2}$
	and $h(X_0) = Z_0^{1/2}$.
	By Lemma \ref{lemma:derivative-trace-functional}, 
	\begin{align*}
	&[Df(X_0)(X)](Y) = [Df(g(X_0))\compose Dg(X_0)](X)(Y) = \trace[C^{1/2}Dg(X_0)(X)C^{1/2}Y]
	\\
	&= -c^2\trace[C^{1/2}(I+h(X_0))^{-1}D\mysqrt(I+ c^2C^{1/2}X_0C^{1/2})(C^{1/2}XC^{1/2})(I+h(X_0))^{-1}C^{1/2}Y]
	\\
	& = -c^2\trace[D\mysqrt(Z_0)(C^{1/2}XC^{1/2})(I+Z_0^{1/2})^{-1}C^{1/2}YC^{1/2}(I+Z_0^{1/2})^{-1}].
	\end{align*}
	In particular, for $Y = X$, by Lemma \ref{lemma:trace-DX-X-mysqrt},
	\begin{align*}
	&[Df(X_0)(X)](X) =  -c^2\trace[D\mysqrt(Z_0)(C^{1/2}XC^{1/2})(I+Z_0^{1/2})^{-1}C^{1/2}XC^{1/2}(I+Z_0^{1/2})^{-1}]
	\\
	& = -2c^2 ||Z_0^{1/4}(I+Z_0^{1/2})^{-1/2}D\mysqrt(Z_0)(C^{1/2}XC^{1/2})(I+Z_0^{1/2})^{-1/2}||_{\HS}^2 \leq 0.
	\end{align*}
	Since $Z_0$ and $(I+Z_0^{1/2})$ are invertible, equality happens if and only if $C^{1/2}XC^{1/2} = 0$.
	If $C$ is strictly positive, then $C^{1/2}XC^{1/2} = 0 \equivalent X = 0$.
	\qed
\end{proof}

\begin{lemma}
	\label{lemma:AB-nondegenerate}
	Let $A,B \in \Lcal(\H)$. Assume that $B$ is compact, self-adjoint, and $B > 0$. Then
	\begin{equation}
	AB = 0 \equivalent A = 0;\;\;\;
	BA = 0 \equivalent A = 0.
	\end{equation}
\end{lemma}
\begin{proof}
	Let $\{\lambda_k\}_{k \in \Nbb}$ be the eigenvalues of $B$, $\lambda_k > 0 \forall k \in \Nbb$, with corresponding orthonormal eigenvectors $\{e_k\}_{k \in \Nbb}$
	forming an orthonormal basis in $\H$. Then
	$0 = ABe_k = \lambda_k Ae_k \imply Ae_k = 0 \;\forall k \in \Nbb \imply Ax = 0 \;\forall x \in \H \imply A = 0$.
	The second expression then follows by via the adjoint operation.\qed
\end{proof}

\begin{proposition}
	\label{proposition:logdet-trace-convex}
	Let $C \in \Sym^{+}\cap\Tr(\H)$ be fixed. Let $\Omega = \{X \in \Sym(\H): I + c^2C^{1/2}XC^{1/2} > 0\}$,
	$c \in \R$, $c \neq 0$.
	Let $f:\Omega \mapto \R$ be defined by
	\begin{align}
	f(X) &= \log\det\left(\frac{1}{2}I + \frac{1}{2}(I +c^2 C^{1/2}XC^{1/2})^{1/2}\right)
	\nonumber
	\\
	&\quad- \trace\left[-I + (I + c^2C^{1/2}XC^{1/2})^{1/2}\right].
	\end{align}
	Then $f$ is convex.
	 Furthermore, $f$ is strictly convex if $C$ is strictly positive.
\end{proposition}
\begin{proof} For any $X_0 \in \Omega$, $X \in \Sym(\H)$, by Lemmas \ref{lemma:derivative-logdet-square-root}
	and \ref{lemma:derivative-trace-square-root}, 
	\begin{align*}
	&Df(X_0)(X) 
	\\
	&= \frac{c^2}{2}\trace\left[C^{1/2}\left((I+c^2C^{1/2}X_0C^{1/2})^{1/2} + (I+c^2C^{1/2}X_0C^{1/2})\right)^{-1}C^{1/2}X\right]
	\\
	&\quad -\frac{c^2}{2}\trace[C^{1/2}(I+c^2C^{1/2}X_0C^{1/2})^{-1/2}C^{1/2}X] 
	\\
	& = -\frac{c^2}{2}\trace\left[C^{1/2}\left(I+(I+c^2C^{1/2}X_0C^{1/2})^{1/2}\right)^{-1}C^{1/2}X\right].
	\end{align*}
	Thus we have the map $Df:\Omega \mapto \Lcal(\Lcal(\H), \R)$, with
	\begin{align*}
	Df(X)(Y) = -\frac{c^2}{2}\trace\left[C^{1/2}\left(I+(I+c^2C^{1/2}XC^{1/2})^{1/2}\right)^{-1}C^{1/2}Y\right].
	\end{align*}
	Differentiating this map gives the second-order Fr\'echet derivative
	\begin{align*}
	&[D^2f(X_0)](X,Y) = [D^2f(X_0)(X)](Y)
	\\ 
	&=\frac{c^4}{2}\trace[D\mysqrt(Z_0)(C^{1/2}XC^{1/2})(I+Z_0^{1/2})^{-1}C^{1/2}YC^{1/2}(I+Z_0^{1/2})^{-1}],
	\end{align*}
	where $Z_0 =I+ c^2C^{1/2}X_0C^{1/2}$, by Lemma \ref{lemma:derivative-functional-trace-inverse}. In particular, for $Y = X$,
	\begin{align*}
	&[D^2f(X_0)](X,X) = [D^2f(X_0)(X)](X) 
	\\
	& = c^4 ||Z_0^{1/4}(I+Z_0^{1/2})^{-1/2}D\mysqrt(Z_0)(C^{1/2}XC^{1/2})(I+Z_0^{1/2})^{-1/2}||_{\HS}^2 \geq 0.
	\nonumber
	\end{align*}
	Equality happens if and only if $C^{1/2}XC^{1/2} = 0$. If $C$ is strictly positive, then equality happens if and only if
	$X = 0$ by Lemma \ref{lemma:AB-nondegenerate}.
	Thus $f$ is convex on $\Omega$ and furthermore, it is strictly convex if $C$ is strictly positive. 
	\qed
\end{proof}

\begin{proof}
	[\textbf{of Theorem \ref{theorem:entropic-OT-convexity} - Convexity of entropic Wasserstein distance}]
	By the strict convexity of the square Hilbert norm $||\;||^2$, the function
	$m \mapto ||m-m_0||^2$ is strictly convex in $m$. For the covariance part, by Theorem 
	\ref{theorem:OT-regularized-Gaussian},
	\begin{align*}
	&F(X) = \OT^{\ep}_{d^2}(\Ncal(0,C_0), \Ncal(0,X))
	\\
	& = 
	\trace(X) + \trace(C_0) - \frac{\ep}{2}\trace\left[-I + \left(I + c_{\ep}^2C_0^{1/2}XC_0^{1/2}\right)^{1/2}\right]
	\\
	&\quad +\frac{\ep}{2}\log\det\left(\frac{1}{2}I + \frac{1}{2}\left(I + c_{\ep}^2C_0^{1/2}XC_0^{1/2}\right)^{1/2}\right), \;\; c_{\ep} = \frac{4}{\ep}.
	\end{align*}
	Since $\Tr(X)$ is linear in $X$ and the remaining part is convex in $\Sym^{+}(\H)\cap \Tr(\H)$ by Proposition \ref{proposition:logdet-trace-convex}, $F$ is convex in $X\in \Sym^{+}(\H)\cap \Tr(\H)$, with strict convexity if $C_0$ is strictly positive.
	\qed
\end{proof}

\begin{lemma}
	\label{lemma:entropic-OT-barycenter-1D}
	Assume that $\sum_{i=1}^Nw_i\sigma_i^2 >0$.
	For a fixed $\ep > 0$, 
	define the 
	function
	$f:[0, \infty) \mapto \R$ by 
	$f(x) = \sum_{i=1}^Nw_if_i(x)$, where for $1 \leq i \leq N$, $c_{\ep} = \frac{4}{\ep}$,
	\begin{align}
	f_i(x) &= \sigma_i^2 + x -\frac{\ep}{2}\left[-1 + \left(1 + c_{\ep}^2\sigma_i^2 x\right)^{1/2}\right] + \frac{\ep}{2}\log\left[\frac{1}{2} +\frac{1}{2}\left(1 + c_{\ep}^2\sigma_i^2 x\right)^{1/2}\right],
	\end{align}
	Then $f$ is strictly convex. For the minimum of $f$,
	there are two scenarios
	\begin{enumerate}
		\item $\ep \geq 2\sum_{i=1}^Nw_i\sigma_i^2$: in this case $\min{f} = f(0) = \sum_{i=1}^Nw_i\sigma_i^2$.
		\item $0 < \ep < 2\sum_{i=1}^Nw_i\sigma_i^2$: in this case $\min{f} = f(x^{*})$, where $x^{*} > 0$ is
		the unique solution of the following equation
		\begin{equation}
		\label{equation:entropic-OT-barycenter-1D}
		\sum_{i=1}^Nw_i \sigma_i^2\left[1 +\left(1 + c_{\ep}^2\sigma_i^2 x\right)^{1/2}\right]^{-1} = \frac{\ep}{4}.
		\end{equation}
		Equivalently, $x^{*}$ is the unique positive solution of the equation
		\begin{equation}
		\label{equation:entropic-OT-barycenter-1D-2}
		x = \frac{\ep}{4}\sum_{i=1}^Nw_i \left[-1 + \left(1+ c_{\ep}^2\sigma_i^2 x\right)^{1/2}\right],
		\end{equation}
		which also has the solution $x_0 = 0$.
	\end{enumerate}
\end{lemma}
\begin{proof} We have
	$f'(x) = 1 - \frac{4}{\ep}\sum_{i=1}^Nw_i \sigma_i^2\left[1 +\left(1 + c_{\ep}^2\sigma_i^2 x\right)^{1/2}\right]^{-1}$,
	$f''(x)  = \frac{32}{\ep^3}\sum_{i=1}^Nw_i\sigma_i^4\left[1 +\left(1 + c_{\ep}^2\sigma_i^2 x\right)^{1/2}\right]^{-2}\left(1 + c_{\ep}^2\sigma_i^2 x\right)^{-1/2} > 0 \;\forall x \geq 0$.
	Thus $f'(x)$ is a strictly increasing function on $[0,\infty)$. 
	Furthermore, $f'(0)  = 1- \frac{2}{\ep}\sum_{i=1}^Nw_i \sigma_i^2$,
	$\lim_{x \approach \infty}f'(x) = 1$.
	We have the following three scenarios
	\begin{enumerate}
		\item $\ep > 2\sum_{i=1}^Nw_i\sigma_i^2$.  In this case $f'(0) > 0$ and thus $f'(x) > 0$ $\forall x > 0$ and thus the minimum for $f$ on $[0,\infty)$ is $f(0) = \sum_{i=1}^Nw_i\sigma_i^2$.
		
		\item $\ep = 2\sum_{i=1}^Nw_i\sigma_i^2$. In this case $f'(0) = 0$, $f'(x) > 0$ $\forall x > 0$ and thus the minimum for $f$ on $[0,\infty)$
		is $f(0) = \sum_{i=1}^Nw_i \sigma_i^2$.
		\item $0< \ep < 2\sum_{i=1}^Nw_i\sigma_i^2$. In this case $f'(0) < 0$ and thus there exists a unique $x^{*} > 0$ such that $f'(x^{*}) = 0$.
		This is the unique global minimizer of $f$ and 
		\begin{align*}
		f'(x^{*}) = 0 &\equivalent\sum_{i=1}^Nw_i \sigma_i^2\left[1 +\left(1 + c_{\ep}^2\sigma_i^2 x^{*}\right)^{1/2}\right]^{-1} = \frac{\ep}{4}
		\\
		&\equivalent x^{*} = \frac{\ep}{4}\sum_{i=1}^Nw_i \left[-1 + \left(1+ c_{\ep}^2\sigma_i^2 x^{*}\right)^{1/2}\right],
		\end{align*}
		which can be verified via the identity $-1+(1+a^2)^{1/2} = a^2[1+(1+a^2)^{1/2}]^{-1}$.
		We note the last equation also has the solution $x=0$. \qed
	\end{enumerate}
\end{proof}

\begin{proof}
	[\textbf{of Theorem \ref{theorem:entropic-barycenter-Gaussian} - Infinite-dimensional setting}]
	By Theorem \ref{theorem:OT-regularized-Gaussian}, 
	$\OT^{\ep}_{d^2}(\Ncal(m_0, C_0), \Ncal(m_1, C_1))$ decomposes into the squared
	Euclidean distance $||m_0 - m_1||^2$ and the distance $\OT^{\ep}_{d^2}(\Ncal(0, C_0), \Ncal(0, C_1))$.
	It follows that we can compute the barycentric mean and covariance operator separately.
	The barycentric mean is obviously the Euclidean mean $\bar{m} = \sum_{i=1}^Nw_im_i$.
	Consider now the centered Gaussian measures $\{\Ncal(0, C_i)\}_{i=1}^N$.
	We define the following functions $F, F_i: \Sym^{+}(\H) \cap\Tr(\H) \mapto \R$,
	$1 \leq i \leq N$,
\begin{equation}
	\begin{aligned}
		F(C) &= \sum_{i=1}^Nw_i F_i(C),
		\\
	F_i(C)&=\OT^{\ep}_{d^2}(\Ncal(0,C), \Ncal(0, C_i))
	\\ 
	&= \trace(C) + \trace(C_i) -\frac{\ep}{2}\trace(M^{\ep}_{01}) + \frac{\ep}{2}\log\det\left(I + \frac{1}{2}M^{\ep}_{01}\right)
	\\
	&= \trace(C) + \trace(C_i) - \frac{\ep}{2}\trace\left[-I + \left(I + c_{\ep}^2C_i^{1/2}CC_i^{1/2}\right)^{1/2}\right]
	\\
	&\quad +\frac{\ep}{2}\log\det\left(\frac{1}{2}I + \frac{1}{2}\left(I + c_{\ep}^2C_i^{1/2}CC_i^{1/2}\right)^{1/2}\right),\;\; c_{\ep} = \frac{4}{\ep}.
	\end{aligned}
	\end{equation}
	Each function $F_i$ is well-defined on the larger, open, convex set $\Omega_i =\{X \in \Sym(\H)\cap\Tr(\H):
	I+c_{\ep}^2C_i^{1/2}XC_i^{1/2} > 0\}$, $c_{\ep}=\frac{4}{\ep}$. 
	If $C_i$ is strictly positive, then $F_i$ is strictly convex by Theorem \ref{theorem:entropic-OT-convexity}.
	Then $F$ is well-defined on the open, convex set 
	$\Omega = \cap_{i=1}^N\Omega_i = \{X \in \Sym(\H)\cap\Tr(\H): I+c_{\ep}^2C_i^{1/2}XC_i^{1/2}>0, i=1, \ldots, N\}$.
	On $\Omega$, $F$ is Fr\'echet differentiable. Since we assume that at least one of the $C_i's$ is strictly positive, $F$ is strictly convex. Thus a minimizer $X_0\in \Omega$ of $F$ must necessarily be unique and satisfy $DF(X_0)=0$.
	 
	Combining Lemmas \ref{lemma:derivative-trace}, \ref{lemma:derivative-trace-square-root}, and \ref{lemma:derivative-logdet-square-root},
	we obtain the Fr\'echet derivative $DF_i(X_0): \Sym(\H) \cap \Tr(\H) \mapto \R$, 
	$X_0 \in \Omega$, 
	$X \in \Sym(\H) \cap \Tr(\H)$, as follows. With $c_{\ep} = \frac{4}{\ep}$, 
	\begin{align*}
	&DF_i(X_0)(X) = \trace(X) - \frac{4}{\ep}\trace\left[C_i^{1/2}\left(I+c_{\ep}^2C_i^{1/2}X_0 C_i^{1/2}\right)^{-1/2}C_i^{1/2}X\right]
	\\
	&\quad +\frac{4}{\ep}\trace
	\biggl[C_i^{1/2}\left(\left(I + c_{\ep}^2C_i^{1/2}X_0C_i^{1/2}\right)^{1/2}+ \left(I + c_{\ep}^2C_i^{1/2}X_0C_i^{1/2}\right)\right)^{-1}C_i^{1/2}X\biggr]
	\\
	& = \trace(X) - \frac{4}{\ep}\trace\left[C_i^{1/2}\left(I+\left(I + c_{\ep}^2C_i^{1/2}X_0C_i^{1/2}\right)^{1/2}\right)^{-1}C_i^{1/2}X\right].
	\end{align*} 
	Summing over $i$, $i =1, \ldots, N$, we obtain
	\begin{align*}
	&DF(X_0)(X) = \sum_{i=1}^Nw_i DF_i(X_0)(X) 
	\\
	&=\sum_{i=1}^Nw_i \trace\left\{\left(I - \frac{4}{\ep}\left[C_i^{1/2}\left(I+\left(I + c_{\ep}^2C_i^{1/2}X_0C_i^{1/2}\right)^{1/2}\right)^{-1}C_i^{1/2}\right]\right)X\right\}.
	\end{align*}
	By Lemma \ref{lemma:trace-functional-zero}, $DF(X_0)(X) = 0$ $\forall X \in \Sym(\H) \cap \Tr(\H)$ if and only if
	\begin{equation}
	\label{equation:identity-DFX0=0}
	\begin{aligned}
	&\sum_{i=1}^Nw_i \left\{I - \frac{4}{\ep}\left[C_i^{1/2}\left(I+\left(I + c_{\ep}^2C_i^{1/2}X_0C_i^{1/2}\right)^{1/2}\right)^{-1}C_i^{1/2}\right]\right\} = 0
	\\
	& \equivalent \sum_{i=1}^N w_i \left[C_i^{1/2}\left(I+\left(I + c_{\ep}^2C_i^{1/2}X_0C_i^{1/2}\right)^{1/2}\right)^{-1}C_i^{1/2}\right] = \frac{\ep}{4}I.
	\end{aligned}
	\end{equation}
	
	When $\dim(\H) = \infty$, this identity is impossible, since the left hand side is a trace class operator, whereas the identity operator is not trace class. Thus the function $F$ does not have a global minimum
	on the open set $\Omega$.
	
	Consider the possible global minima of $F$ on $\Sym^{+}(\H)\cap\Tr(\H) \subsetneq \Omega$.
	
	(i) Consider the case $\ep I \geq 2\sum_{i=1}^Nw_iC_i$. 
	For any $X \in \Sym(\H) \cap \Tr(\H)$,
	\begin{align*}
	DF(0)(X) = \sum_{i=1}^Nw_i\trace\left[\left(I - \frac{2}{\ep}C_i\right)X\right] = \trace\left[\left(I - \frac{2}{\ep}\sum_{i=1}^Nw_iC_i\right)X\right].
	\end{align*}
	By Lemma \ref{lemma:operator-monotone-quadratic}, we have 
	\begin{align*}
	A\geq 0, B \geq 0 \implies A^{1/2}BA^{1/2}\geq 0 \imply \trace(AB) = \trace(A^{1/2}BA^{1/2}) \geq 0.
	\end{align*}
	
	Then $\forall Y \in \Sym^{+}(\H) \cap \Tr(\H)$, we have by Eq.\eqref{equation:Frechet-derivative-strictly-convex}
	\begin{align*}
	F(Y) > F(0) + DF(0)(Y) = F(0) + \trace\left[\left(I - \frac{2}{\ep}\sum_{i=1}^Nw_iC_i\right)Y\right] \geq F(0).
	\end{align*}
	Thus $X_0 = 0$ is the unique global minimizer of $F$ in $\Sym^{+}(\H) \cap \Tr(\H)$.
	
	(ii) Assume now that $\ep I \ngeq 2\sum_{i=1}^Nw_iC_i$. We show that $X_0 = 0$ is not  
	a global minimum of $F$ in $\Sym^{+}(\H) \cap \Tr(\H)$.
	For any $X_0 \in \Sym^{+}(\H)\cap \Tr(\H)$,
	\begin{align*}
	&DF(X_0)(X_0)
	\\ 
	&= \sum_{i=1}^Nw_i\trace\left[X_0 - \frac{4}{\ep}\left[X_0^{1/2}C_i^{1/2}\left(I+\left(I + c_{\ep}^2C_i^{1/2}X_0C_i^{1/2}\right)^{1/2}\right)^{-1}C_i^{1/2}X_0^{1/2}\right]\right]
	\\
	& = \sum_{i=1}^Nw_i\trace\left[X_0 +\frac{\ep}{4}I -\frac{\ep}{4}\left(I + c_{\ep}^2X_0^{1/2}C_iX_0^{1/2}\right)^{1/2} \right] \;\text{by Lemma \ref{lemma:adjoint-switch-CX}}.
	\end{align*}
	The assumption 
	$\ep I \ngeq 2\sum_{i=1}^Nw_iC_i$ means that $\exists u \in \H, ||u||=1$, such that
	$0 < \ep = \ep||u||^2 < 2\sum_{i=1}^Nw_i\la u, C_iu\ra$.
	Consider $X_u = x(u \otimes u)$. By Lemma \ref{lemma:rank-one-operator-power},
	\begin{align*}
	DF(X_u)(X_u) &= \sum_{i=1}^Nw_i\trace\left[\left(x +\frac{\ep}{4}- \frac{\ep}{4}\left(1+ c_{\ep}^2x\la u, C_iu\ra\right)^{1/2}\right)(u\otimes u)\right]
	\\
	&= x - \frac{\ep}{4}\sum_{i=1}^Nw_i \left[-1 + \left(1+ c_{\ep}^2x\la u, C_iu\ra\right)^{1/2}\right].
	\end{align*}
	According to Lemma \ref{lemma:entropic-OT-barycenter-1D}, the property $0 < \ep < 2\sum_{i=1}^Nw_i \la u, C_iu\ra$
implies that there exists a unique $x^{*} > 0$ such that
\begin{align*}
x^{*} = \frac{\ep}{4}\sum_{i=1}^Nw_i \left[-1 + \left(1+ c_{\ep}^2x^{*}\la u, C_iu\ra\right)^{1/2}\right].
\end{align*}	
Thus with $X_u^{*} = x^{*}(u \otimes u)$, we have
$DF(X_u^{*})(X_u^{*}) = 0$.
By Eq.\eqref{equation:Frechet-derivative-strictly-convex}, 
\begin{align*}
F(0) > F(X_u^{*}) -DF(X_u^{*})(X_u^{*}) = F(X_u^{*}).
\end{align*}
Thus $X_0 = 0$ is not a global minimum of $F$ in $\Sym^{+}(\H)\cap\Tr(\H)$.
\qed
	\end{proof}

	\begin{proof}
	    	[\textbf{of Theorem \ref{theorem:entropic-barycenter-Gaussian} - Finite-dimensional setting}]
	When $\dim(\H) < \infty$, if we impose the additional condition that $\ker(X_0) = \{0\}$, that is $X_0$ is invertible, then the identity \eqref{equation:identity-DFX0=0} is equivalent to 
	\begin{align*}
	& \sum_{i=1}^N w_i \left[X_0^{1/2}C_i^{1/2}\left(I+\left(I + c_{\ep}^2C_i^{1/2}X_0C_i^{1/2}\right)^{1/2}\right)^{-1}C_i^{1/2}X_0^{1/2}\right] = \frac{\ep}{4} X_0.
	\end{align*} 
	By Lemma \ref{lemma:adjoint-switch-CX},
	on the left hand side,
	\begin{align*}
	&X_0^{1/2}C_i^{1/2}\left(I+\left(I + c_{\ep}^2C_i^{1/2}X_0C_i^{1/2}\right)^{1/2}\right)^{-1}C_i^{1/2}X_0^{1/2} 
	\\
	&\quad = 
	-\frac{1}{c_{\ep}^2}I + \frac{1}{c_{\ep}^2}\left(I +  c_{\ep}^2X_0^{1/2}C_iX_0^{1/2}\right)^{1/2}.
	\end{align*}
	Substituting into the previous equation, we obtain
	\begin{align*}
	X_0 = \frac{\ep}{4}\sum_{i=1}^N w_i\left[-I + \left(I +  c_{\ep}^2X_0^{1/2}C_iX_0^{1/2}\right)^{1/2}\right].
	\end{align*}
	This gives Eq.\eqref{equation:entropic-barycenter-Gaussian-strictlypositive}.
	Define the following map $\Fcal: \Sym^{+}(\H) \mapto \Sym^{+}(\H)$ by
\begin{equation}
\Fcal(X) = \sum_{i=1}^N w_i \left[C_i^{1/2}\left(I+\left(I + c_{\ep}^2C_i^{1/2}XC_i^{1/2}\right)^{1/2}\right)^{-1}C_i^{1/2}\right],
\end{equation}
then we have by the monotonicity of the square root function
\begin{equation}
\begin{aligned}
\Fcal(0) = \frac{1}{2}\sum_{i=1}^Nw_iC_i, \;\;\;
\Fcal(X)  \leq \frac{1}{2}\sum_{i=1}^Nw_iC_i, \;\;\; \forall X \geq 0.
\end{aligned}
\end{equation}
Thus it is clear that

(i) If $\ep I > 2\sum_{i=1}^Nw_iC_i$, then Eq.\eqref{equation:entropic-barycenter-Gaussian} has
no solution $X_0 \geq 0$.

(ii) If $\ep I = 2\sum_{i=1}^Nw_iC_i$, then Eq.\eqref{equation:entropic-barycenter-Gaussian} has the solution
$X_0 = 0$, which is necessarily unique due to the strict convexity of the entropic OT distance.

(iii) The condition $0 < \ep I < 2 \sum_{i=1}^Nw_iC_i$ is thus necessary if  Eq.\eqref{equation:entropic-barycenter-Gaussian} is to have a solution $X_0 \geq 0$, $X \neq 0$.
By Proposition \ref{proposition:entropic-barycenter-Gaussian-strictlypositive}, 
if $C_i \geq \alpha I$, $1\leq i \leq N$, 
and $0 < \ep < 2 \alpha$, then Eq.\eqref{equation:entropic-barycenter-Gaussian} has a strictly positive solution,
which is necessarily unique due to the strict convexity of the entropic OT distance.
\qed
\end{proof}

\begin{proposition}
	\label{proposition:entropic-barycenter-Gaussian-strictlypositive}
	Define the following map $\Gcal: \Sym^{+}(n) \mapto \Sym^{+}(n)$, $c_{\ep}=\frac{4}{\ep}$,
	\begin{equation}
	\label{equation:map-entropic-barycenter-Gaussian-strictlypositive}
	\Gcal(X) = \frac{\epsilon}{4}\sum_{i=1}^Nw_i\left[-I + \left(I + c_{\ep}^2X^\frac{1}{2}C_iX^\frac{1}{2}\right)^\frac{1}{2}\right].
	\end{equation}
\begin{enumerate}
	\item Suppose $\exists \alpha \in \R, \alpha > 0$ such that $C_i \geq \alpha I$, $1 \leq i \leq N$ and 
	$0 < \ep < 2 \alpha$. Then $\Gcal$ has a strictly positive fixed point, that is Eq.\eqref{equation:entropic-barycenter-Gaussian-strictlypositive}
	has a strictly positive solution.
	\item Let $u \in \R^n$, $||u|| =1$. If $0 < \ep < 2\sum_{i=1}^Nw_i \la u, C_iu\ra$, then $X_u = x_u (u\otimes u) = x_u uu^T$
	is a fixed point of $\Gcal$, where $x_u$ is the unique positive solution of the one-dimensional fixed point equation
	\begin{equation}
	x = \frac{\ep}{4}\sum_{i=1}^Nw_i\left[-1 + \left(1+ c_{\ep}^2x\la u, C_iu\ra\right)^{1/2}\right].
	\end{equation}
\end{enumerate}
\end{proposition}
If $0 < \ep < 2\sum_{i=1}^Nw_iC_i$, then for $\forall u \in \H, ||u||=1$, 
$0 < \ep = \ep||u||^2 < 2\sum_{i=1}^Nw_i \la u, C_iu\ra$. In particular, if
$C_i \geq \alpha I$, $1 \leq i \leq N$,
and $0 < \ep < 2\alpha$, then
$0 < \ep <  2\sum_{i=1}^Nw_i \la u, C_iu\ra$. Hence under these assumptions, the map $\Gcal$
 in Eq.\eqref{equation:map-entropic-barycenter-Gaussian-strictlypositive}, has {\it uncountably many fixed points}, which are positive but singular.

\begin{proof}
[\textbf{of Proposition \ref{proposition:entropic-barycenter-Gaussian-strictlypositive}}]
	For the first part,
let $\gamma \in \R, \gamma > 0$ be such that $C_i \leq \gamma I$, $1 \leq i \leq N$. Let $\beta_{\ep} = \alpha - \frac{\ep}{2} > 0$ and consider the following set
\begin{equation}
\Kcal_{\ep} = \{X \in \Sym^{+}(n): \beta_{\ep} I \leq X \leq \gamma I\}.
\end{equation}
This is a compact, convex set in $\Sym(n)$.
By Lemma \ref{lemma:operator-monotone-quadratic}, the operator monotonicity of the square root function, 
and the inequality $(1+a^2)^{1/2} \leq 1+a$, $a \in \R, a\geq  0$, we have with $c_{\ep} = \frac{4}{\ep}$,
\begin{align*}
&0 \leq X \leq \gamma I \imply \Gcal(X) \leq \frac{\ep}{4}\left[- I + \left(I + c_{\ep}^2\gamma^2I \right)^{1/2}\right] \leq \gamma I,
\\
&X \geq \beta_{\ep}I \imply \Gcal(X) \geq \frac{\ep}{4}\left[- I + \left(I + c_{\ep}^2\alpha \left(\alpha-\frac{\ep}{2}\right)I \right)^{1/2}\right] = \left(\alpha -\frac{\ep}{2}\right)I = \beta_{\ep}I.
\end{align*}
Thus the continuous map $\Gcal$ maps the compact convex set $\Kcal_{\ep}$ into itself.
By Brouwer Fixed Point Theorem, $\Gcal$ has at least a fixed point in $\Kcal_{\ep}$. 

For the second part, by Lemma \ref{lemma:rank-one-operator-power},
\begin{align*}
(u \otimes u)^{1/2}C_i(u \otimes u)^{1/2} = (u \otimes u)C_i(u \otimes u) = \la u, C_iu\ra (u \otimes u).
\end{align*}
Therefore, for $X_u = x(u \otimes u)$,
\begin{align*}
I + c_{\ep}^2X_u^{1/2}C_iX_u^{1/2} &= (I-u\otimes u) + \left(1+c_{\ep}^2x\la u, C_iu\ra\right) (u \otimes u).
\end{align*}
Since $(I-u\otimes u)^2 = (I-u \otimes u)$ and $(I-u \otimes u)(u \otimes u) = 0$,
by Lemma \ref{lemma:rank-one-operator-power},
\begin{align*}
\left(I + c_{\ep}^2X_u^{1/2}C_iX_u^{1/2}\right)^{1/2} = (I - u \otimes u) + \left(1+ c_{\ep}^2x\la u, C_iu\ra\right)^{1/2}(u \otimes u).
\end{align*}
It follows that
$\Gcal(X_u) = \frac{\ep}{4}\sum_{i=1}^Nw_i\left[-1 + \left(1+ c_{\ep}^2x\la u, C_iu\ra\right)^{1/2}\right]
(u \otimes u)$.
Thus the
equation $X_u = \Gcal(X_u)$ is equivalent to 
$x = \frac{\ep}{4}\sum_{i=1}^Nw_i\left[-1 + \left(1+ c_{\ep}^2x\la u, C_iu\ra\right)^{1/2}\right]$.
By the assumption $0 < \ep < 2\sum_{i=1}^Nw_i \la u, C_iu\ra$, this equation
has a unique positive solution by Lemma \ref{lemma:entropic-OT-barycenter-1D}.
\qed
\end{proof}

	\section{Sinkhorn barycenter of Gaussian measures}
	\label{section:barycenter-Sinkhorn}
	
	In this section, we prove Theorem \ref{theorem:strict-convexity-Sinkhorn} on the {\it strict convexity} of the Sinkhorn divergence and Theorem \ref{theorem:barycenter-sinkhorn-Gaussian} on the {\it barycenter} of a set of Gaussian measures on $\H$ under the Sinkhorn divergence. We need the following technical results.

	\begin{lemma}
		\label{lemma:derivative-trace-square-root-2}
		Let $f: \Sym(\H) \cap \Tr(\H) \mapto \R$ be defined by $f(X) = \trace\left[ -I + \left(I + c^2X^2\right)^{1/2}\right]$, $c \in \R$. 
		$Df(X_0): \Sym(\H) \cap \Tr(\H) \mapto \R$, $X_0 \in \Sym(\H) \cap \Tr(\H)$, is given by
		\begin{align}
		Df(X_0)(X) = c^2\trace\left[\left(I + c^2X_0^2\right)^{-1/2}X_0X\right], \;X \in \Sym(\H) \cap\Tr(\H).
		\end{align}
	\end{lemma}
	\begin{proof}
		Let $g(X) = -I + \left(I + c^2X^2\right)^{1/2}$, then $f(X) = \trace[g(X)]$ and
		\begin{align*}
		&Df(X_0)(X) = Df(g(X_0)) \compose Dg(X_0)(X) = \trace[Dg(X_0)(X)]
		\\
		& = c^2\trace\left[D\mysqrt\left(I + c^2X_0^2\right)\left(X_0X + XX_0\right)\right]
		\\
		& = \frac{1}{2}c^2\trace\left[\left(I + c^2X_0^2\right)^{-1/2}\left(X_0X + XX_0\right)\right]
		 \;\;\;\text{by Lemma \ref{lemma:derivative-square-root}}
		\\
		& 
		= c^2\trace\left[\left(I + c^2X_0^2\right)^{-1/2}X_0X\right]. 
		\end{align*}
	\end{proof}
	
	\begin{lemma}
		\label{lemma:derivative-logdet-square-root-2}
		Let $f: \Sym(\H) \cap \Tr(\H) \mapto \R$ be defined by 
		$f(X) = \log\det\left[ \frac{1}{2}I + \frac{1}{2}\left(I + c^2X^2\right)^{1/2}\right]$.
		Then
		$Df(X_0): \Sym(\H) \cap \Tr(\H) \mapto \R$, $X_0 \in \Sym(\H) \cap \Tr(\H)$, is defined by
		\begin{align}
		Df(X_0)(X) = c^2\trace\left[\left(\left(I + c^2X_0^2\right)^{1/2} + \left(I + c^2X_0^2\right)\right)^{-1}X_0X\right].
		\end{align}
	\end{lemma}
	\begin{proof}
		Let $g(X) = -I + \left(I + c^2X^2\right)^{1/2}$, then $f(X) = \log\det[I + \frac{1}{2}g(X)]$. By Lemmas \ref{lemma:derivative-logdet} and \ref{lemma:derivative-square-root},
		\begin{align*}
		&Df(X_0)(X) = \frac{1}{2}\trace[(I+\frac{1}{2}g(X_0))^{-1}Dg(X_0)(X)]
		\\
		&= \frac{1}{2}c^2\trace\left[\left(\frac{1}{2}I + \frac{1}{2}\left(I + c^2X_0^2\right)^{1/2}\right)^{-1}
		D\mysqrt\left(I + c^2X_0^2\right)(X_0X + XX_0)\right]
		\\
		& =\frac{1}{4}c^2\trace\left[\left(I + c^2X_0^2\right)^{-1/2}\left(\frac{1}{2}I + \frac{1}{2}\left(I + c^2X_0^2\right)^{1/2}\right)^{-1}(X_0X + XX_0)\right]
		\\
		& = c^2\trace\left[\left(\left(I + c^2X_0^2\right)^{1/2} + \left(I + c^2X_0^2\right)\right)^{-1}X_0X\right].
		\end{align*}
	\end{proof}
	
	\begin{lemma}
		\label{lemma:Frechet-derivative-product-rule}
		Let $W$ be a Banach algebra and $\Omega \subset W$ be an open subset.
		Let $g:\Omega \mapto W$ be Fr\'echet differentiable at $X_0$.
		Let $f(X) = g(X)X$ and $h(X) = Xg(X)$. Then $f$ and $h$ are Fr\'echet differentiable at $X_0$, with
		$Df(X_0)(X) = Dg(X_0)(X)X_0 + g(X_0)X$ and $Dh(X_0)(X) = X_0Dg(X_0)(X) + Xg(X_0)$. 
	\end{lemma}
	\begin{proof}
		By assumption, $\lim_{t \approach 0}\frac{||g(X_0+tX) - g(X_0) - tDg(X_0)(X)||_W}{|t|\;||X||_W} = 0$. Thus
		\begin{align*}
		&\lim_{t \approach 0}\frac{||f(X_0 + tX) - f(X_0) - Df(X_0)(tX)||_{W}}{|t|\;||X||_W} 
		\\
		&= \lim_{t \approach 0}\frac{||g(X_0+tX)(X_0+tX) - g(X_0)X_0 - tDg(X_0)(X)X_0 - tg(X_0)X||_W}{|t|\; ||X||_W}
		\\
		& \leq \lim_{t \approach 0}||X_0||_W\frac{||g(X_0+tX) - g(X_0)- tDg(X_0)(X)||_W}{|t|\;||X||_W} 
		\\
		&\quad + \lim_{t \approach 0}||g(X_0+tX) - g(X_0)||_W = 0.
		\end{align*}
		This proves the formula for $Df(X_0)$.
		Here we use the fact the Fr\'echet Differentiability implies continuity.
		The proof for $Dh(X_0)$ is entirely similar.
		\qed
	\end{proof}

	\begin{lemma}
		\label{lemma:derivative-second-order-trace-square-root-2}
		Define the function $f:\Sym(\H)\mapto \Sym(\H)$ by
		$f(X) = (I+(I+c^2X^2)^{1/2})^{-1}X$. Then $Df(X_0): \Sym(\H) \mapto \Sym(\H)$ is given by
		\begin{align}
		Df(X_0)(X)  &= -{c^2}(I+h(X_0))^{-1}D\mysqrt(I+c^2X_0^2)(X_0X+XX_0)(I+h(X_0))^{-1}X_0
		\nonumber
		\\
		&\quad + (I+h(X_0))^{-1}X
		\label{equation:derivative-Frechet-second-order-1}
		\\
		&= -{c^2}X_0(I+h(X_0))^{-1}D\mysqrt(I+c^2X_0^2)(X_0X+XX_0)(I+h(X_0))^{-1}
		\nonumber
		\\
		&\quad + X(I+h(X_0))^{-1}.
		\label{equation:derivative-Frechet-second-order-2}
		\end{align}
		Here $X,X_0 \in \Sym(\H)$ and $h(X_0) = (I+c^2X_0^2)^{1/2}$.
	\end{lemma}
	\begin{proof}
		Let $h(X) = (I+ c^2X^2)^{1/2}$ and $g(X) = (I+h(X))^{-1}$. By the chain rule,
		\begin{align*}
		&Dg(X_0)(X) = [Dg(h(X_0))\compose Dh(X_0)](X)
		\\
		& = - (I+h(X_0))^{-1}Dh(X_0)(X)(I+h(X_0))^{-1}
		\\
		& = -{c^2}(I+h(X_0))^{-1}D\mysqrt(I+c^2X_0^2)(X_0X+XX_0)(I+h(X_0))^{-1}.
		\end{align*}
		By Lemma \ref{lemma:Frechet-derivative-product-rule}, with $f(X) = g(X)X$,
		\begin{align*}
		&Df(X_0)(X) = Dg(X_0)(X)X_0 + g(X_0)X 
		\\
		&= -{c^2}(I+h(X_0))^{-1}D\mysqrt(I+c^2X_0^2)(X_0X+XX_0)(I+h(X_0))^{-1}X_0
		\\
		&\quad + (I+h(X_0))^{-1}X.
		\end{align*}
		Since $g(X)$ and $X$ commute, we also have $f(X) = Xg(X)$ and thus
		\begin{align*}
		&Df(X_0)(X) = X_0Dg(X_0)(X) + Xg(X_0)
		\\
		&= -{c^2}X_0(I+h(X_0))^{-1}D\mysqrt(I+c^2X_0^2)(X_0X+XX_0)(I+h(X_0))^{-1}
		\\
		&\quad + X(I+h(X_0))^{-1}.
		\end{align*}
		This gives the second, equivalent, expression for $Df(X_0)(X)$.
		\qed
	\end{proof}
	
	\begin{lemma}
		\label{lemma:derivative-trace-functional-square}
		Let $f: \Sym(\H) \cap \Tr(\H) \mapto \Lcal(\Lcal(\H),\R)$ be defined by
		\begin{align}
		f(X)(Y) = \trace[(I+(I+c^2X^2)^{1/2})^{-1}XY], \;\;\; Y \in \Lcal(\H).
		\end{align}
		Then $Df(X_0): \Sym(\H)\cap\Tr(\H) \mapto \Lcal(\Lcal(\H),\R)$ is given by
		\begin{align}
		&[Df(X_0)(X)](Y) = \frac{1}{2}\trace[(I+Z_0^{1/2})^{-1}(XY+YX)]
		\\
		&\quad-\frac{c^2}{2}\trace[(I+Z_0^{1/2})^{-1}D\mysqrt(Z_0)(X_0X+XX_0)(I+Z_0^{1/2})^{-1}(X_0Y+YX_0)].
		\nonumber
		\end{align}
		In particular, for $Y = X$,
		\begin{align}
		&[Df(X_0)(X)](X) = \trace[(I+Z_0^{1/2})^{-1}X^2]
		\\
		&\quad -\frac{c^2}{2}\trace[(I+Z_0^{1/2})^{-1}D\mysqrt(Z_0)(X_0X+XX_0)(I+Z_0^{1/2})^{-1}(X_0X + XX_0)].
		\nonumber
		\end{align}
		Here $X,X_0 \in \Sym(\H) \cap \Tr(\H)$ and $Z_0= I+c^2X_0^2$.
	\end{lemma}
	\begin{proof}
		Let $g(X)= (I+(I+c^2X^2)^{1/2})^{-1}X$, $h(X) = (I+c^2X^2)^{1/2}$, then $f(X)(Y) = \trace[g(X)Y]$.
		By Lemmas \ref{lemma:derivative-functional-trace-inverse} and \ref{lemma:derivative-second-order-trace-square-root-2},
		combining Eqs.\eqref{equation:derivative-Frechet-second-order-1} and\eqref{equation:derivative-Frechet-second-order-2},
		we obtain, with $Z_0 = I+c^2X_0^2$,
		\begin{align*}
		&[Df(X_0)(X)](Y) = [Df(g(X_0)) \compose Dg(X_0)(X)](Y) = \trace[Dg(X_0)(X)Y]
		\\
		&=-\frac{c^2}{2}\trace[(I+h(X_0))^{-1}D\mysqrt(I+c^2X_0^2)(X_0X+XX_0)(I+h(X_0))^{-1}(X_0Y+YX_0)]
		\\
		& \quad + \frac{1}{2}\trace[(I+h(X_0))^{-1}(XY+YX)]
		\\
		&=-\frac{c^2}{2}\trace[(I+Z_0^{1/2})^{-1}D\mysqrt(Z_0)(X_0X+XX_0)(I+Z_0^{1/2})^{-1}(X_0Y+YX_0)]
		\\
		& \quad + \frac{1}{2}\trace[(I+Z_0^{1/2})^{-1}(XY+YX)].
		\end{align*}
		In particular, for $Y = X$,
		\begin{align*}
		&[Df(X_0)(X)](X) = \trace[(I+Z_0^{1/2})^{-1}X^2]
		\\
		&\quad -\frac{c^2}{2}\trace[(I+Z_0^{1/2})^{-1}D\mysqrt(Z_0)(X_0X+XX_0)(I+Z_0^{1/2})^{-1}(X_0X + XX_0)].
		\end{align*}
	\end{proof}
	
	\begin{lemma}
		\label{lemma:trace-DXX-X2-strict-inequality}
		Let $X_0\in \Sym(\H)$ be a fixed compact operator. Let $Z_0 = I +c^2X_0^2$, $c \in \R$. Then
		$\forall X \in \Sym(\H)\cap\HS(\H)$, $X \neq 0$,
		\begin{align}
		&\frac{c^2}{2}\trace[(I+Z_0^{1/2})^{-1}D\mysqrt(Z_0)(X_0X+XX_0)(I+Z_0^{1/2})^{-1}(X_0X + XX_0)]
		\nonumber
		\\
		&\quad < \trace[(I+Z_0^{1/2})^{-1}X^2].
		\end{align}
		
	\end{lemma}
	\begin{proof}
		Let $Y_0 = (I+Z_0^{1/2})^{-1} = (I+(I+c^2X_0^2)^{1/2})^{-1}$. 
		By Lemma \ref{lemma:trace-DX-X-mysqrt}, since $Y_0$ and $Z_0$ commute, we have
		\begin{align*}
		\trace[ D\mysqrt(Z_0)(X)Y_0XY_0] = 2 ||Z_0^{1/4}Y_0^{1/2}D\mysqrt(Z_0)(X)Y_0^{1/2}||_{\HS}^2.
		\end{align*}
		Let $\{\lambda_k\}_{k\in \Nbb}$ be the eigenvalues of $X_0$, with corresponding orthonormal eigenvectors
		$\{e_k\}_{k\in \Nbb}$ forming an orthonormal basis in $\H$.
		Then $Z_0$ and $Y_0$ have eigenvalues $\{z_k = (1+c^2\lambda_k^2)\}_{k \in \Nbb}$ and
		$\{y_k = (1+z_k^{1/2})^{-1}\}_{k\in \Nbb}$, respectively, with the same eigenvectors. We then have
		$\forall X \in \Sym(\H) \cap \HS(\H)$,
		\begin{align}
		\label{equation:Z0Y0D-HS-norm-1}
		&\frac{1}{2}\trace[ D\mysqrt(Z_0)(X)Y_0XY_0] =
		||Z_0^{1/4}Y_0^{1/2}D\mysqrt(Z_0)(X)Y_0^{1/2}||_{\HS}^2
		\nonumber
		\\
		& = \sum_{k=1}^{\infty}||Z_0^{1/4}Y_0^{1/2}D\mysqrt(Z_0)(X)Y_0^{1/2}e_k||^2
		= \sum_{k=1}^{\infty}y_k ||Z_0^{1/4}Y_0^{1/2}D\mysqrt(Z_0)(X)e_k||^2 
		\nonumber
		\\
		&= \sum_{k=1}^{\infty}y_k\sum_{j=1}^{\infty}\la Z_0^{1/4}Y_0^{1/2}D\mysqrt(Z_0)(X)e_k, e_j\ra^2
		\nonumber
		\\
		& = \sum_{j,k=1}^{\infty}y_ky_jz_j^{1/2}\la e_k, D\mysqrt(Z_0)(X)e_j\ra^2.
		\end{align}
		We recall the following identity from Lemma \ref{lemma:derivative-square-root}
		\begin{align*}
		Z_0^{1/2}D\mysqrt(Z_0)(X) + D\mysqrt(Z_0)(X)Z_0^{1/2} = X.
		\end{align*}
		Applying the inner product with $e_k$ and $e_j$ on both sides, taking into account that $D\mysqrt(Z_0)(X)
		\in \Sym(\H)$ and $z_k > 0 \forall k \in \Nbb$, gives
		\begin{align*}
		\la e_k, D\mysqrt(Z_0)(X)e_j\ra = \frac{1}{z_k^{1/2} + z_j^{1/2}}\la e_k, Xe_j\ra.
		\end{align*}
		Since $X,X_0 \in \Sym(\H)$, $\la e_k, (XX_0+X_0X)e_j\ra = (\lambda_k + \lambda_j)\la e_k, Xe_j\ra$ and thus
		\begin{align*}
		\la e_k, D\mysqrt(Z_0)(XX_0 + X_0X)e_j\ra = \frac{\lambda_j + \lambda_k}{z_k^{1/2} + z_j^{1/2}}\la e_k, Xe_j\ra
		\end{align*}
		Substituting this into Eq.\eqref{equation:Z0Y0D-HS-norm-1} gives
		\begin{align*}
		& \frac{1}{2}\trace[ D\mysqrt(Z_0)(XX_0+X_0X)Y_0(XX_0 + X_0X)Y_0]
		\\
		&= ||Z_0^{1/4}Y_0^{1/2}D\mysqrt(Z_0)(XX_0+X_0X)Y_0^{1/2}||_{\HS}^2
		\\
		& =\sum_{k,j=1}^{\infty}y_ky_jz_j^{1/2}\left(\frac{\lambda_j + \lambda_k}{z_k^{1/2} + z_j^{1/2}}\right)^2\la e_k, Xe_j\ra^2.
		\end{align*}
		With $z_k = 1+c^2\lambda_k^2$, $y_k = (1+z_k^{1/2})^{-1}$, we have $z_k^{1/2}y_k < 1 \forall k \in \Nbb$ and
		\begin{align*}
		c^2\left(\frac{\lambda_k + \lambda_j}{z_k^{1/2} + z_j^{1/2}}\right)^2 =\left(\frac{c\lambda_k +c\lambda_j}{(1+c^2\lambda_k^2)^{1/2}+ (1+c^2\lambda_j^2)^{1/2}}\right)^2 < 1,\;\forall k,j\in \Nbb.
		\end{align*}
		It follows that $\forall X \in \Sym(\H) \cap\HS(\H)$, $X \neq 0$,
		\begin{align*}
		&c^2||Z_0^{1/4}Y_0^{1/2}D\mysqrt(Z_0)(XX_0+X_0X)Y_0^{1/2}||_{\HS}^2 < \sum_{k,j=1}^{\infty}y_k\la Xe_k, e_j\ra^2
		\\
		& = \sum_{k=1}^{\infty}y_k||Xe_k||^2 = \sum_{k=1}^{\infty}||XY_0^{1/2}e_k||^2= ||XY_0^{1/2}||^2_{\HS} = \trace(Y_0X^2).
		\end{align*}
		Substituting $Y_0 = (I+Z_0^{1/2})^{-1}$ gives the desired result.\qed
	\end{proof}
	
	\begin{proposition}
		\label{proposition:trace-log-square-strictly-convex}
		Let $f: \Sym(\H) \cap \Tr(\H) \mapto \R$ be defined by
		\begin{align}
		f(X) = \trace\left[-I + (I+c^2X^2)^{1/2}\right] - \log\det\left(\frac{1}{2}I + \frac{1}{2}(I+c^2X^2)^{1/2}\right),
		\end{align}
		where $c \in \R, c \neq 0$. Then $f$ is at least twice Fr\'echet differentiable and strictly convex
		on $\Sym(\H) \cap \Tr(\H)$.
	\end{proposition}
	\begin{proof}
		By Lemmas \ref{lemma:derivative-trace-square-root-2}, \ref{lemma:derivative-logdet-square-root-2}, 
		$Df(X_0): \Sym(\H) \cap \Tr(\H) \mapto \R$ is given by,  $\forall X_0, X \in \Sym(\H) \cap \Tr(\H)$, 
		\begin{align*}
		Df(X_0)(X) &= c^2\trace\left[\left(I + c^2X_0^2\right)^{-1/2}X_0X\right]
		\\
		&-c^2\trace\left[\left(\left(I + c^2X_0^2\right)^{1/2} + \left(I + c^2X_0^2\right)\right)^{-1}X_0X\right]
		\\
		& = c^2\trace\left[\left(I+ \left(I + c^2X_0^2\right)^{1/2} \right)^{-1}X_0X\right].
		\end{align*}
		Thus we have the map $Df: \Sym(\H)\cap \Tr(\H) \mapto \Lcal(\Lcal(\H),\R)$, with
		\begin{align*}
		Df(X)(Y) = c^2\trace\left[\left(I+ \left(I + c^2X^2\right)^{1/2} \right)^{-1}XY\right].
		\end{align*}
		Differentiating this map gives the 2nd-order Fr\'echet derivative, by Lemma \ref{lemma:derivative-trace-functional-square},
		\begin{align*}
		&[D^2f(X_0)](X,Y) = [D^2f(X_0)(X)](Y) =
		\frac{c^2}{2}\trace[(I+Z_0^{1/2})^{-1}(XY+YX)]
		\\
		&\quad-\frac{c^4}{2}\trace[(I+Z_0^{1/2})^{-1}D\mysqrt(Z_0)(X_0X+XX_0)(I+Z_0^{1/2})^{-1}(X_0Y+YX_0)].
		\nonumber
		\end{align*}
		In particular, for $Y = X$,
		\begin{align*}
		&[D^2f(X_0)](X,X) = c^2\trace[(I+Z_0^{1/2})^{-1}X^2]
		\\
		&\quad -\frac{c^4}{2}\trace[(I+Z_0^{1/2})^{-1}D\mysqrt(Z_0)(X_0X+XX_0)(I+Z_0^{1/2})^{-1}(X_0X + XX_0)] \geq  0,
		\end{align*}
		with the strict inequality being valid $\forall X \in \Sym(\H) \cap \Tr(\H), X \neq 0$ by Lemma \ref{lemma:trace-DXX-X2-strict-inequality}. Thus $f$ is strictly convex on $\Sym(\H) \cap \Tr(\H)$.
		\qed
	\end{proof}

\begin{proof}
	[\textbf{of Theorem \ref{theorem:strict-convexity-Sinkhorn} - Strict convexity of Sinkhorn divergence}]
	By the strict convexity of the square Hilbert norm $||\;||^2$, the function $m \mapto ||m-m_0||^2$ is strictly convex in $m$. For the covariance part,
	by Theorem \ref{theorem:Sinkhorn-Gaussian-Hilbert},
	\begin{equation*}
	\begin{aligned}
	&F(X) = \Srm^{\ep}_{d^2}(\Ncal(0,C_0), \Ncal(0, X)) 
	\\
	& = \frac{\ep}{4}\trace\left[\left(I + \frac{16}{\ep^2}C_0^2\right)^{1/2} - 2\left(I + \frac{16}{\ep^2}C_0^{1/2}XC_0^{1/2}\right)^{1/2} 
	+ \left(I + \frac{16}{\ep^2}X^2\right)^{1/2}
	\right]
	\\
	& \quad + \frac{\ep}{2}\log\det\left(\frac{1}{2}I + \frac{1}{2}\left(I + \frac{16}{\ep^2}C_0^{1/2}XC_0^{1/2}\right)^{1/2} \right)
	\\
	& \quad -\frac{\ep}{4}\log\det\left(\frac{1}{2}I + \frac{1}{2}\left(I + \frac{16}{\ep^2}X^2\right)^{1/2} \right)
	-\frac{\ep}{4}\log\det\left(\frac{1}{2}I + \frac{1}{2}\left(I + \frac{16}{\ep^2}C_0^2\right)^{1/2} \right).
	\end{aligned}
	\end{equation*}
	By Proposition \ref{proposition:logdet-trace-convex}, the function
	$f(X) = \log\det\left(\frac{1}{2}I + \frac{1}{2}(I +c^2 C_0^{1/2}XC_0^{1/2})^{1/2}\right)
	- \trace\left[-I + (I + c^2C_0^{1/2}XC_0^{1/2})^{1/2}\right]$ is convex on $\Sym^{+}(\H) \cap \Tr(\H)$.
	By Proposition \ref{proposition:trace-log-square-strictly-convex},
	$g(X) = \trace\left[-I+\left(I + \frac{16}{\ep^2}X^2\right)^{1/2}\right] + \log\det \left(\frac{1}{2}I + \frac{1}{2}\left(I + \frac{16}{\ep^2}X^2\right)^{1/2} \right)$ is strictly convex in
	$\Sym(\H) \cap \Tr(\H)$.
	Thus $F$ is strictly convex on $\Sym^{+}(\H) \cap \Tr(\H)$.
	\qed
\end{proof}
	\begin{proof}
		[\textbf{of Theorem \ref{theorem:differentiability} - Differentiability}].
		The differentiability and convexity of $F_E$ follows from Lemmas \ref{lemma:derivative-trace-square-root}, \ref{lemma:derivative-logdet-square-root}, and Proposition \ref{proposition:logdet-trace-convex}.
		The differentiability and strict convexity of $F_S$ follows from Lemmas \ref{lemma:derivative-trace-square-root}, \ref{lemma:derivative-logdet-square-root}, 
		\ref{lemma:derivative-trace-square-root-2}, \ref{lemma:derivative-logdet-square-root-2},
		and Proposition \ref{proposition:logdet-trace-convex} and \ref{proposition:trace-log-square-strictly-convex}.
		\qed
	\end{proof}

\begin{proof}
	[\textbf{of Theorem \ref{theorem:positivity} - Positivity}]
	Let $\mu_0 = \Ncal(m_0, C_0)$ and $\mu_1 = \Ncal(m, X)$. It suffices to prove for the case $m_0 = m = 0$.
	Let $C_0$ be fixed.
	By Theorem \ref{theorem:differentiability}, the function $F_S: X \mapto \Srm^{\ep}_{d^2}[\Ncal(0,C_0), \Ncal(0,X)]$ is well-defined, twice Fr\'echet differentiable, and  strictly convex
	on the open, convex set $\Omega = \{X \in \Sym(\H) \cap \Tr(\H): I+c_{\ep}^2C_0^{1/2}XC_0^{1/2} > 0\}\supset \Sym^{+}(\H) \cap \Tr(\H)$. 
	Thus a minimizer of $F_S$ in $\Omega$ is necessarily unique.
	Proceeding as in Proposition \ref{proposition:derivation-equation-barycenter}, the Fr\'echet derivative for $F_S$ is given by, for $X_0 \in \Omega$, $X \in \Sym(\H) \cap \Tr(\H)$,
	\begin{align}
	DF_S(X_0)(X) &= -c_{\ep}\trace
	\biggl[C_0^{1/2}\left(I+\left(I + c_{\ep}^2C_0^{1/2}X_0C_0^{1/2}\right)^{1/2}\right)^{-1}C_0^{1/2}X\biggr]
	\nonumber
	\\
	&\quad + c_{\ep}\trace\left[\left(I+ \left(I + c_{\ep}^2X_0^2\right)^{1/2} \right)^{-1}X_0X\right].
	\end{align}
	By Lemma \ref{lemma:trace-functional-zero}, $DF_S(X_0)(X) = 0$ $\forall X \in \Sym(\H) \cap \Tr(\H)$ if and only if
	\begin{align*}
	&\left(I+ \left(I + c_{\ep}^2X_0^2\right)^{1/2} \right)^{-1}X_0 = C_0^{1/2}\left(I+\left(I + c_{\ep}^2C_0^{1/2}X_0C_0^{1/2}\right)^{1/2}\right)^{-1}C_0^{1/2}
	\\
	&\equivalent X_0^{1/2}\left(I+ \left(I + c_{\ep}^2X_0^2\right)^{1/2} \right)^{-1}X_0^{1/2} = C_0^{1/2}\left(I+\left(I + c_{\ep}^2C_0^{1/2}X_0C_0^{1/2}\right)^{1/2}\right)^{-1}C_0^{1/2}.
	\end{align*}
	This equation obviously has solution $X_0 = C_0$, which must be unique since $F_S$ is strictly convex in $\Omega$.
	Thus the unique global minimum of $F_S$ in $\Omega$, and hence in $\Sym^{+}(\H)\cap \Tr(\H)$, 
	is $F_S(C_0) = 0$. Hence $F_S(X) \geq 0$ $\forall X \in \Sym^{+}(\H)\cap \Tr(\H)$, with $F_S(X) = 0 \equivalent
	X = C_0$. \qed
\end{proof}

	{\bf Derivation of the barycenter equations}.
	We start by deriving Eqs.\eqref{equation:barycenter-sinkhorn-positive} and \eqref{equation:barycenter-sinkhorn-strictlypositive}
	for the barycenter of Gaussian measures, which we restate here. 
	
	\begin{proposition}
		\label{proposition:derivation-equation-barycenter}
		Consider the Gaussian measures $\Ncal(0,C), \Ncal(0, C_i)$, $1 \leq i \leq N$.  
		Define the following function $F: \Sym^{+}(\H) \cap\Tr(\H)\mapto \R$,
		$1 \leq i \leq N$, by
		\begin{equation}
		F(C) = \sum_{i=1}^Nw_i\Srm^{\ep}_{d^2}(\Ncal(0, C), \Ncal(0, C_i)).
		\end{equation}
		Then $F$ is well-defined on the larger, open set $\Omega = \{X \in \Sym(\H)\cap \Tr(\H):
I+\frac{16}{\ep^2}C_i^{1/2}XC_i^{/2} >0, i=1,\ldots, N\}$.	$F$ is Fr\'echet differentiable on $\Omega$ and
the condition $DF(X_0) = 0$ is equivalent to
	{\small
		\begin{align}
		\label{equation:fixedpoint-barycenter-sinkhorn-gaussian}
		X_0 &= \left(I+ \left(I + \frac{16}{\ep^2}X_0^2\right)^{1/2} \right)^{1/2} \sum_{i=1}^Nw_i\left[C_i^{1/2}\left(I+\left(I + \frac{16}{\ep^2}C_i^{1/2}X_0C_i^{1/2}\right)^{1/2}\right)^{-1}C_i^{1/2}\right]
		\nonumber
		\\
		& \quad \quad \times \left(I+ \left(I + \frac{16}{\ep^2}X_0^2\right)^{1/2} \right)^{1/2}. 
		\end{align}
	}
		A solution $X_0$ of Eq.\eqref{equation:fixedpoint-barycenter-sinkhorn-gaussian} must necessarily satisfy
		$X_0 \in \Sym^{+}(\H) \cap\Tr(\H)$.
		Under the additional hypothesis that $X_0> 0$, Eq.\eqref{equation:fixedpoint-barycenter-sinkhorn-gaussian} is equivalent to 
		{\small
		\begin{equation}
		\begin{aligned}
		X_0 = \frac{\ep}{4}\left[- I +  \left(\sum_{i=1}^Nw_i\left(I + \frac{16}{\ep^2}X_0^{1/2}C_iX_0^{1/2}\right)^{1/2}\right)^2\right]^{1/2}.
		\end{aligned}
		\end{equation}
	}
	\end{proposition}
\begin{proof}
Let $F_i(C) = \Srm^{\ep}_{d^2}(\Ncal(0, C), \Ncal(0, C_i))$.
By Theorem \ref{theorem:Sinkhorn-Gaussian-Hilbert}, with $c_{\ep} = \frac{4}{\ep}$,
\begin{equation*}
\begin{aligned}
&F_i(C) 
=  \frac{\ep}{4}\trace\left[M^{\ep}_{00} - 2M^{\ep}_{01} + M^{\ep}_{11}\right] 
+ \frac{\ep}{4}\log\left[\frac{\det\left(I + \frac{1}{2}M^{\ep}_{01}\right)^2}{\det\left(I + \frac{1}{2}M^{\ep}_{00}\right)\det\left(I + \frac{1}{2}M^{\ep}_{11}\right)}\right]
\\
& = \frac{\ep}{4}\trace\left[\left(I +  c_{\ep}^2C_i^2\right)^{1/2} - 2\left(I +  c_{\ep}^2C_i^{1/2}CC_i^{1/2}\right)^{1/2} 
+ \left(I +  c_{\ep}^2C^2\right)^{1/2}
\right]
\\
& \quad + \frac{\ep}{2}\log\det\left(\frac{1}{2}I + \frac{1}{2}\left(I +  c_{\ep}^2C_i^{1/2}CC_i^{1/2}\right)^{1/2} \right)
\\
& \quad -\frac{\ep}{4}\log\det\left(\frac{1}{2}I + \frac{1}{2}\left(I +  c_{\ep}^2C^2\right)^{1/2} \right)
-\frac{\ep}{4}\log\det\left(\frac{1}{2}I + \frac{1}{2}\left(I +  c_{\ep}^2C_i^2\right)^{1/2} \right).
\end{aligned}
\end{equation*}
Clearly $F_i$ is well-defined on the larger, open set $\Omega_i = \{X \in \Sym(\H)\cap \Tr(\H): I+\frac{16}{\ep^2}C_i^{1/2}XC_i^{1/2} > 0\}$, hence $F$ is well-defined on
$\Omega = \cap_{i=1}^N\Omega_i$. 
The Fr\'echet derivative of $F_i$ at each $X_0 \in \Omega_i$ 
 is a linear map 
$DF_i(X_0): \Sym(\H) \cap \Tr(\H) \mapto \R$.  
Combining Lemmas \ref{lemma:derivative-trace-square-root}, \ref{lemma:derivative-logdet-square-root}, \ref{lemma:derivative-trace-square-root-2}, \ref{lemma:derivative-logdet-square-root-2}, we obtain, $\forall X \in \Sym(\H) \cap \Tr(\H)$, 
\begin{align*}
&DF_i(X_0)(X) = -c_{\ep}\trace\left[C_i^{1/2}\left(I+c_{\ep}^2C_i^{1/2}X_0 C_i^{1/2}\right)^{-1/2}C_i^{1/2}X\right]
\\
&\quad +c_{\ep} \trace\left[\left(I + c_{\ep}^2X_0^2\right)^{-1/2}X_0X\right]
\\
&\quad +c_{\ep} \trace
\biggl[C_i^{1/2}\left(\left(I + c_{\ep}^2C_i^{1/2}X_0C_i^{1/2}\right)^{1/2}+ \left(I + c_{\ep}^2C_i^{1/2}X_0C_i^{1/2}\right)\right)^{-1}C_i^{1/2}X\biggr]
\\
&\quad -c_{\ep}\trace\left[\left(\left(I + c_{\ep}^2X_0^2\right)^{1/2} + \left(I + c_{\ep}^2X_0^2\right)\right)^{-1}X_0X\right]
\\
& \quad = -c_{\ep}\trace
\biggl[C_i^{1/2}\left(I+\left(I + c_{\ep}^2C_i^{1/2}X_0C_i^{1/2}\right)^{1/2}\right)^{-1}C_i^{1/2}X\biggr]
\\
&\quad + c_{\ep}\trace\left[\left(I+ \left(I + c_{\ep}^2X_0^2\right)^{1/2} \right)^{-1}X_0X\right].
\end{align*} 
Summing over $i$, $1 \leq i \leq N$, we obtain the Fr\'echet derivative for $F$, namely
\begin{align}
DF(X_0)(X) &= -c_{\ep}\sum_{i=1}^Nw_i\trace
\biggl[C_i^{1/2}\left(I+\left(I + c_{\ep}^2C_i^{1/2}X_0C_i^{1/2}\right)^{1/2}\right)^{-1}C_i^{1/2}X\biggr]
\nonumber
\\
&\quad + c_{\ep}\trace\left[\left(I+ \left(I + c_{\ep}^2X_0^2\right)^{1/2} \right)^{-1}X_0X\right].
\end{align}
By Lemma \ref{lemma:trace-functional-zero}, the first order optimality condition $DF(X_0)(X) = 0 \;\forall X \in \Sym(\H)\cap \Tr(\H)$ is equivalent to
$DF(X_0) = 0$, which in turn is equivalent to
{\small
\begin{align*}
&\left(I+ \left(I + c_{\ep}^2X_0^2\right)^{1/2} \right)^{-1}X_0 = \sum_{i=1}^Nw_i\left[C_i^{1/2}\left(I+\left(I + c_{\ep}^2C_i^{1/2}X_0C_i^{1/2}\right)^{1/2}\right)^{-1}C_i^{1/2}\right]
\\
& \equivalent \left(I+ \left(I + c_{\ep}^2X_0^2\right)^{1/2} \right)^{-1/2}X_0\left(I+ \left(I + c_{\ep}^2X_0^2\right)^{1/2} \right)^{-1/2}
\\
&\quad =  \sum_{i=1}^Nw_i\left[C_i^{1/2}\left(I+\left(I + c_{\ep}^2C_i^{1/2}X_0C_i^{1/2}\right)^{1/2}\right)^{-1}C_i^{1/2}\right]
\\
& \equivalent X_0 = \left(I+ \left(I + c_{\ep}^2X_0^2\right)^{1/2} \right)^{1/2} \sum_{i=1}^Nw_i\left[C_i^{1/2}\left(I+\left(I + c_{\ep}^2C_i^{1/2}X_0C_i^{1/2}\right)^{1/2}\right)^{-1}C_i^{1/2}\right]
\\
& \quad \quad \quad \quad \times \left(I+ \left(I + c_{\ep}^2X_0^2\right)^{1/2} \right)^{1/2}.
\end{align*}
}
This gives the first equation. Clearly any solution $X_0$ of this equation must necessarily satisfy $X_0 \in \Sym^{+}(\H)\cap \Tr(\H)$.

Assume now that
$X_0 > 0$.
We rewrite 
$DF(X_0) = 0$ as
{\small 
\begin{equation*}
\begin{aligned}
&\left(I+ \left(I + c_{\ep}^2X_0^2\right)^{1/2} \right)^{-1}X_0 = \sum_{i=1}^Nw_i\left[C_i^{1/2}\left(I+\left(I + c_{\ep}^2C_i^{1/2}X_0C_i^{1/2}\right)^{1/2}\right)^{-1}C_i^{1/2}\right]
\\
& \equivalent X_0^{1/2}\left(I+ \left(I + c_{\ep}^2X_0^2\right)^{1/2} \right)^{-1}X_0^{1/2} 
\\
&\quad \quad = \sum_{i=1}^Nw_i\left[C_i^{1/2}\left(I+\left(I + c_{\ep}^2C_i^{1/2}X_0C_i^{1/2}\right)^{1/2}\right)^{-1}C_i^{1/2}\right].
\end{aligned}
\end{equation*}
}
Under the condition $X_0 > 0$, by Lemma \ref{lemma:AB-nondegenerate}, pre- and post-multiplying $X_0^{1/2}$ on both sides gives the equivalent expression
\begin{align*}
&X_0\left(I+ \left(I + c_{\ep}^2X_0^2\right)^{1/2} \right)^{-1}X_0
\\
&\quad \quad=\sum_{i=1}^Nw_i\left[X_0^{1/2}C_i^{1/2}\left(I+\left(I + c_{\ep}^2C_i^{1/2}X_0C_i^{1/2}\right)^{1/2}\right)^{-1}C_i^{1/2}X_0^{1/2}\right]
\\
& \equivalent \frac{1}{c_{\ep}^2}\left[-I+ \left(I + c_{\ep}^2X_0^2\right)^{1/2} \right]
= \frac{1}{c_{\ep}^2}\sum_{i=1}^Nw_i\left[-I + \left(I + c_{\ep}^2X_0^{1/2}C_iX_0^{1/2}\right)^{1/2}\right],
\end{align*}
where the right hand side follows from  Lemma \ref{lemma:adjoint-switch-CX}. This in turn is
\begin{align*}
&
\left(I + c_{\ep}^2X_0^2\right)^{1/2} = \sum_{i=1}^Nw_i\left(I + c_{\ep}^2X_0^{1/2}C_iX_0^{1/2}\right)^{1/2}
\\
& \equivalent X_0^2 = \frac{1}{c_{\ep}^2} \left[- I +  \left(\sum_{i=1}^Nw_i\left(I + c_{\ep}^2X_0^{1/2}C_iX_0^{1/2}\right)^{1/2}\right)^2\right]
\\
& \equivalent X_0= \frac{1}{c_{\ep}}\left[- I +  \left(\sum_{i=1}^Nw_i\left(I + c_{\ep}^2X_0^{1/2}C_iX_0^{1/2}\right)^{1/2}\right)^2\right]^{1/2}.
\end{align*}
This completes the proof. \qed
\end{proof}

{\bf Existence of the Fixed Point}.
It is clear from Eq.\eqref{equation:barycenter-sinkhorn-positive} that if it has a solution $X_0$, then necessarily $X_0 \geq 0$.
We now prove that Eq.\eqref{equation:barycenter-sinkhorn-positive} has at least one solution $X_0$, which is then necessarily unique 
by the strict convexity of the Sinkhorn divergence. 
%
%
This is done via the
{\it Schauder Fixed Point Theorem} (see e.g.\cite{conwayFunctionalAnalysis2007}). Let $E$ be a Banach space and $M \subset E$. 
We recall that a mapping $f:M \mapto E$ is said to be {\it compact} if it is continuous and maps bounded subsets into relatively compact subsets of $E$, that is subsets whose closures are compact.
\begin{theorem}
	[\textbf{Schauder Fixed Point Theorem}]
	Let $M$ be a bounded closed convex subset of a Banach space $E$. Assume that $f:M \mapto M$ is a compact mapping.
	Then $f$ has at least one fixed point in $M$.
\end{theorem}
Consider the map $\Fcal:\Sym^{+}(\H) \mapto \Sym^{+}(\H)$ as defined in Eq.\eqref{equation:map-fixedpoint-barycenter-sinkhorn}.
The proof of the existence of a fixed point of $\Fcal$ consists of two steps
\begin{enumerate}
	\item We show that $\Fcal$ is compact.
	\item Let $\gamma  \in \R, \gamma > 0$ be such that $C_i \leq \gamma I$, $1\leq i \leq N$, and consider the set
	\begin{equation}
	\Kcal = \{X \in \Sym^{+}(\H): 0 \leq X \leq \gamma I\}.
	\end{equation}
	We show that $\Fcal$ maps $\Kcal$ into itself. We can then apply Schauder Fixed Point Theorem to obtain the existence of a fixed point of $\Fcal$ in $\Kcal$.
\end{enumerate}

\begin{lemma}
	\label{lemma:operator-quadratic-positive}
	Let $B \in \Lcal(\H) \cap \Sym(\H)$. Let $A \in \Lcal(\H)$, $\ker(A^{*}) = \{0\}$. Then
	\begin{equation}
	A^{*}BA \geq 0 \equivalent B \geq 0.
	\end{equation}
\end{lemma}
\begin{proof}
	If $B \geq 0$, then by Lemma \ref{lemma:operator-monotone-quadratic}, we have $A^{*}BA \geq 0$.
	Assume now that $A^{*}BA \geq 0$, which means that $\la x, A^{*}BAx \ra  = \la Ax, BAx\ra \geq 0$ $\forall x \in \H$.
	In particular, for $y = Ax \in \myIm(A)$, we have $\la y, By\ra \geq 0$.
	Since $\ker(A^{*}) = \{0\}$, we have $\overline{\myIm(A)} = \ker(A^{*})^{\perp} = \H$, hence $\myIm(A)$ is dense in $\H$. Thus $\forall y \in \H$, there is a sequence
	$\{y_n\}_{n \in \Nbb}$ in $\myIm(A)$ such that $\lim_{n \approach \infty}||y_n - y|| = 0$.
	Then
	\begin{align*}
	|\la y, By\ra - \la y_n, By_n\ra| &= |\la y-y_n, By\ra + \la y_n, B(y - y_n)\ra| 
	\\
	&\leq ||y_n-y||\;||B||(||y|| + ||y_n||) \approach 0 \text{ as } n \approach \infty.
	\end{align*}
	Thus $\la y, By\ra = \lim_{n \approach \infty}\la y_n, By_n \ra \geq 0$.
	Since this holds $\forall y \in \H$, we have $B \geq 0$. \qed
\end{proof}

\begin{remark}
	Lemma \ref{lemma:operator-quadratic-positive} is generally not true without the condition $\ker(A^{*}) = \{0\}$.
	As an example, consider the case $\H = \R^2$ and $A = \begin{pmatrix} 1 & 0 \\ 0 & 0\end{pmatrix}$, $B = (b_{ij})_{i,j,=1,2}
	$. Then $ABA = \begin{pmatrix}b_{11} & 0 \\ 0 & 0 \end{pmatrix} \geq 0 \equivalent b_{11} \geq 0$. 
\end{remark}

In the following, {\color{black}recall} the set of $p$th Schatten class operators $\Csc_p(\H) = \{A \in \Lcal(\H):
||A||_p = (\trace[(A^{*}A)^{p/2}])^{1/p} < \infty\}$, $1 \leq p \leq \infty$, with $\Csc_1(\H) = \Tr(\H)$, $\Csc_2(\H) = \HS(\H)$, and $\Csc_{\infty}(\H)$ being the set of compact operators on $\H$.
\begin{lemma}
	[\textbf{Corollary 3.2 in \cite{Kitta:InequalitiesV}}]
	\label{lemma:inequality-Schatten}
	For any two positive operators $A,B$ on $\H$ such that $A \geq cI > 0$, $B \geq cI > 0$,
	for any bounded operator $X$ on $\H$,
	\begin{equation}
	||A^rX - XB^r||_p\leq rc^{r-1}||AX-XB||_p, 0 < r \leq 1, 1 \leq p \leq \infty.
	\end{equation}
\end{lemma}
\begin{corollary}
	\label{corollary:trace-class-square-root}
	For two operators $A, B \in \Sym^{+}(\H) \cap \Csc_p(\H)$, $1 \leq p \leq \infty$,
	\begin{align}
	||(I+A)^{r} - (I+B)^{r}||_{p} \leq r||A-B||_{p}, 0\leq r \leq 1.
	\end{align}
\end{corollary}
\begin{theorem}
	[\textbf{Theorem 2.3 in \cite{Kitta:InequalitiesV}}]
	Let $A,B$ be two positive operators on $\H$ and $f$ any operator monotone function with $f(0) = 0$. 
	Then
	\begin{equation}
	||f(A) - f(B)|| \leq f(||A-B||).
	\end{equation}
\end{theorem}
By Proposition \ref{proposition:operator-monotone-rth}, the following result is then immediate.
\begin{corollary}
	\label{corollary:continuity-norm-square-root}
	Let $A,B$ be two positive bounded operators on $\H$. Then
	\begin{equation}
	||A^{r} - B^{r}|| \leq ||A-B||^{r}, \;\;\; 0 \leq r \leq 1.
	\end{equation}
\end{corollary}

\begin{proposition}
	\label{proposition:compact-map-Sym-sqrt}
	Let $C \in \Sym^{+}(\H) \cap \Tr(\H)$. The following maps are compact
	\begin{enumerate}
		\item $F_1:\Sym^{+}(\H) \mapto \Lcal(\H)$ defined by $F_1(X) = X^{1/2}C^{1/2}$.
		\item $F_2:\Sym^{+}(\H) \mapto \Lcal(\H)$ defined by $F_2(X) = C^{1/2}X^{1/2}$.
		\item $F_3:\Sym^{+}(\H) \mapto \Sym^{+}(\H)$ defined by $F_3(X) = X^{1/2}CX^{1/2}$.
		\item $F_4:\Sym^{+}(\H) \mapto \Sym^{+}(\H)$ defined by $F_4(X) = C^{1/2}XC^{1/2}$.
	\end{enumerate}
\end{proposition}
\begin{proof}
	(i)	By Corollary \ref{corollary:continuity-norm-square-root}, we have
	\begin{align*}
	||F_1(X) - F_1(Y)|| \leq ||X^{1/2}- Y^{1/2}||\;||C^{1/2}|| \leq ||C||^{1/2}||X-Y||^{1/2}.
	\end{align*}
	Thus the map $F_1$ is continuous on $\Sym^{+}(\H)$.
	Since $C^{1/2}$ is a compact operator on $\H$, it maps bounded subsets of $\H$ into relatively compact subsets of $\H$.
	Consider the set $\Ycal = \{C^{1/2}x: x \in \H, ||x||\leq 1\} \subset \H$, then $\Ycal$ being relatively compact means that
	every sequence $\{y_n = C^{1/2}x_n\}_{n \in \Nbb}$ in $\Ycal$ contains a subsequence $\{y_{n_k} = C^{1/2}x_{n_k}\}_{k \in \Nbb}$ that converges in $\H$, that is $\exists y \in \H$ such that $\lim_{k \approach \infty}||C^{1/2}x_{n_k} - y|| = 0$.
	
	Since $C \in \Sym^{+}(\H) \cap \Tr(\H)$, we have $C^{1/2} \in \HS(\H)$ and $X^{1/2}C^{1/2} \in \HS(\H)$ $\forall X \in \Sym^{+}(\H)$. In particular $X^{1/2}C^{1/2}$ is a compact operator.
	
	Let $\{\lambda_k\}_{k \in \Nbb}$ be the eigenvalues of $C$, arranged in decreasing order, with corresponding orthonormal eigenvectors $\{e_k\}_{k \in \Nbb}$.
	Let $N \in \Nbb$ be fixed and consider the finite-rank operator
	$C_N^{1/2} = \sum_{k=1}^N\sqrt{\lambda_k} (e_k \otimes e_k)$.
	Then for any $\ep > 0$, there exists $N(\ep) \in \Nbb$ such that
	\begin{align*}
	||C_N^{1/2} - C^{1/2}||_{\HS} = (\sum_{k=N+1}^{\infty}\lambda_k)^{1/2} < \ep, \;\;\;\forall N > N(\ep).
	\end{align*}
	Consequently, for $||X||\leq 1$, $N > N(\ep)$,
	\begin{align*}
	||X^{1/2}C_N^{1/2} - X^{1/2}C^{1/2}||_{\HS} \leq ||X^{1/2}||\;||C_N^{1/2} - C^{1/2}||_{\HS} < \ep.
	\end{align*}
	For a fixed $N \in \Nbb$, consider 
	the set $\Zcal_N = \{X^{1/2}C_N^{1/2}: ||X|| \leq 1\} \subset \HS(\H)$ and a sequence $\{X_n^{1/2}C_N^{1/2}\}_{n\in \Nbb} \subset \Zcal_N$. We now show that this sequence has a convergent subsequence in $\HS(\H)$. We have
	\begin{align*}
	||X_n^{1/2}C_N^{1/2}||^2_{\HS} = \sum_{j=1}^{\infty}||X_n^{1/2}C_N^{1/2}e_j||^2 \leq ||C_N^{1/2}||^2_{\HS} \leq \trace(C) < \infty.
	\end{align*}
	The sequence $\{||X_n^{1/2}C_N^{1/2}||_{\HS}\}_{n \in \Nbb}$ is a bounded sequence of non-negative numbers and thus, by the Bolzano-Weierstrass Theorem, has a convergent subsequence
	$\{||X_{n_0}^{1/2}C_N^{1/2}||_{\HS}\}_{n_0 \in \Nbb}$, with
	\begin{align*}
	\lim_{n_0 \approach \infty}||X_{n_0}^{1/2}C_N^{1/2}||_{\HS} = B, \;\;\;\text{for some constant $B \geq 0$}.
	\end{align*}
	For $j=1$, the sequence $\{X_{n_0}^{1/2}C_N^{1/2}e_1\}_{n_0 \in \Nbb}$, belonging to a relatively compact set in $\H$, contains a convergent subsequence $\{X_{n_1}^{1/2}C_N^{1/2}e_1\}$, i.e. 
	\begin{align*}
	\lim_{n_1 \approach \infty}||X_{n_1}^{1/2}C_N^{1/2}e_1 - y_1|| = 0, \;\;\;\text{for some $y_1 \in \H$}.
	\end{align*} 
	Similarly, for $j =2$, the sequence $\{X_{n_1}^{1/2}C_N^{1/2}e_2\}_{n_1 \in \Nbb}$ contains a convergent subsequence $\{X_{n_2}^{1/2}C_N^{1/2}e_2\}$, i.e. $\exists y_2 \in \H$ such that
	\begin{align*}
	\lim_{n_2 \approach \infty}||X_{n_2}^{1/2}C_N^{1/2}e_2 - y_2|| =0 \text{ and at the same time }\lim_{n_2 \approach \infty}||X_{n_2}^{1/2}C_N^{1/2}e_1 - y_1|| = 0.
	\end{align*}
	Carrying out this procedure iteratively, we obtain  a subsequence $\{X_{n_N}^{1/2}C_N^{1/2}\}$ in $\HS(\H)$ and $(y_j)_{j=1}^{N}, y_j \in \H$,
	such that
	\begin{align*}
	\lim_{n_{N} \approach \infty}||X_{n_N}^{1/2}C_N^{1/2}e_j - y_j|| = 0, \;\;\;1 \leq j \leq N.
	\end{align*}
	Furthermore, 
	\begin{align*}
	\sum_{j=1}^{N}||y_j||^2 &= \sum_{j=1}^{N}\lim_{n_N \approach \infty}||X_{n_N}^{1/2}C_N^{1/2}e_j||^2 = \lim_{n_N \approach \infty}\sum_{j=1}^{N}||X_{n_N}^{1/2}C_N^{1/2}e_j||^2 
	\\
	&= \lim_{n_N \approach \infty}||X_{n_N}^{1/2}C_N^{1/2}||^2_{\HS} = B^2< \infty.
	\end{align*}
	Define the following finite-rank operator $Y_N \in \Lcal(\H)$ by 
	\begin{equation}
	Y_Ne_j = 
	\left\{
	\begin{matrix}
	y_j & \text{for $1\leq j \leq N$},
	\\
	0 & \text{else}.
	\end{matrix}
	\right.
	\end{equation}
	Then $||Y_N||^2_{\HS} = \sum_{j=1}^{\infty}||Y_Ne_j||^2 = \sum_{j=1}^{N}||y_j||^2  = B^2 < \infty$ and
	\begin{align*}
	\lim_{n_N \approach \infty}||X_{n_N}^{1/2}C_N^{1/2} - Y_N||_{\HS}^2 &=\lim_{n_N \approach \infty} \sum_{j=1}^{N}||(X_{n_N}^{1/2}C_N^{1/2} - Y_N)e_j||^2
	\\
	& = \sum_{j=1}^{N}\lim_{n_N \approach \infty}||X_{n_N}^{1/2}C_N^{1/2}e_j - y_j||^2 = 0.
	\end{align*}
	Thus $\{X_{n_N}^{1/2}C_N^{1/2}\}$ is the desired convergent subsequence, with limit $Y_N \in \HS(\H)$.
	This shows that the set $\Zcal_N = \{X^{1/2}C_N:||X||\leq 1\}$ is relatively compact $\forall N \in \Nbb$ in $\HS(\H)$, so that
	$\forall \ep >0$, there is a finite $\ep$-net $\{Z_i\}_{i=1}^{N_2(N,\ep)}$ in $\HS(\H)$ such that
	$\{X^{1/2}C_N^{1/2}: ||X||\leq 1\}  \subset \cup_{i=1}^{N_2(N,\ep)}B_{\HS(\H)}(Z_i, \ep)$.
	Consequently, for $N > N(\ep)$,
	$\{X^{1/2}C^{1/2}:||X||\leq 1\} \subset \cup_{i=1}^{N_2(N,\ep)}B_{\HS(\H)}(Z_i, 2\ep)$.
	This shows that the set $\{X^{1/2}C^{1/2}:||X||\leq 1\}$ is relatively compact in $\HS(\H)$, hence in $\Lcal(\H)$ and thus $F_1$ is a compact map on $\Sym^{+}(\H)$. Moreover, each sequence $\{X_n^{1/2}C^{1/2}:||X||\leq 1\}$ contains a convergent subsequence $\{X_{k(n)}^{1/2}C^{1/2}\}_{n \in \Nbb}$ in $\HS(\H)$, i.e. $\exists Y \in \HS(\H)$ such that
	\begin{equation}
	\lim_{k(n) \approach \infty}||X_{k(n)}^{1/2}C^{1/2} - Y||_{\HS} = 0.
	\end{equation}

	(ii) Similarly, $F_2$ is a continuous map on $\Sym^{+}(\H)$. Each sequence $\{C^{1/2}X_n^{1/2}, ||X_n||\leq 1\}_{n \in \Nbb}$
	contains a convergent subsequence $\{C^{1/2}X_{k(n)}^{1/2}\}$, with
	$\lim_{k(n) \approach \infty}||C^{1/2}X_{k(n)}^{1/2} - Y^{*}||_{\HS} = 0$,
	where $Y \in \HS(\H)$ is as in Part (i). Thus $F_2$ is a compact map on $\Sym^{+}(\H)$.
	
	(iii) Since $F_3(X) = F_1(X)F_2(X)$, $F_3$ is continuous on $\Sym^{+}(\H)$. Furthermore,
	each sequence $\{X_n^{1/2}CX_n^{1/2}, ||X_n||\leq 1\}_{n \in \Nbb}$
	contains a convergent subsequence $\{X_{k(n)}^{1/2}CX_{k(n)}^{1/2}\}$, with
	$\lim_{k(n) \approach \infty}||X_{k(n)}^{1/2}CX_{k(n)}^{1/2} - YY^{*}||_{\HS} = 0$,
	where $Y \in \HS(\H)$ is as in Part (i). Thus $F_2$ is a compact map on $\Sym^{+}(\H)$.
	
	(iv) Entirely analogous to $F_3$, the map $F_4$ is compact on $\Sym^{+}(\H)$. \qed
\end{proof}
\begin{corollary}
	\label{corollary:inverse-square-root-sequence}
	Let $A, B \in \Sym^{+}(\H)$ be given. Let $1 \leq p \leq \infty$. Then
	\begin{equation}
	||(I+(I+A)^{1/2})^{-1} - (I+(I+B)^{1/2})^{-1}||_p \leq \frac{1}{8}||A-B||_p.
	\end{equation}
	In particular, let $A\in \Sym^{+}(\H)$, $\{A_n\}_{n \in \Nbb}$, $A_n \in \Sym^{+}(\H)$ $\forall n \in \Nbb$ be such that
	$\lim_{n \approach \infty}||A_n - A||_p = 0$. Then
	\begin{equation}
	\lim_{n \approach \infty}||(I+(I+A_n)^{1/2})^{-1} - (I+(I+A)^{1/2})^{-1}||_p = 0.
	\end{equation}
\end{corollary}
\begin{proof} By Corollary \ref{corollary:trace-class-square-root},
	\begin{align*}
	&||(I+(I+A)^{1/2})^{-1} - (I+(I+B)^{1/2})^{-1}||_p
	\\
	&= ||(I+(I+A)^{1/2})^{-1}[(I+(I+A)^{1/2}) - (I+(I+B)^{1/2})](I+(I+B)^{1/2})^{-1}||_p
	\\
	& \leq ||(I+(I+A)^{1/2})^{-1}||\;||(I+A)^{1/2} - (I+B)^{1/2})||_p\;||(I+(I+B)^{1/2})^{-1}||
	\\
	& \leq \frac{1}{8}||A - B||_p.
	\end{align*}
	The second result is then immediate. \qed
\end{proof}
\begin{corollary}
	\label{corollary:compact-map-inverse-square-root}
	Let $C \in \Sym^{+}(\H) \cap \Tr(\H)$. 
	Let $a\in \R, a \neq 0$. Consider the map $F:\Sym^{+}(\H) \mapto \Sym^{+}(\H)$ defined by
	\begin{equation}
	F(X) = \left(I + \left(I + a^2C^{1/2}XC^{1/2}\right)^{1/2}\right)^{-1}.
	\end{equation}
	Then $F$ is a compact map on $\Sym^{+}(\H)$.
\end{corollary}

\begin{proof}
	By Corollary \ref{corollary:inverse-square-root-sequence}, 
	\begin{align*}
	||F(X) - F(Y)|| \leq \frac{a^2}{8}||C^{1/2}XC^{1/2} - C^{1/2}YC^{1/2}|| \leq \frac{a^2}{8}||C^{1/2}||^2||X-Y||.
	\end{align*}
	Thus $F(X)$ is a continuous map on $\Sym^{+}(\H)$. 	
	By Proposition \ref{proposition:compact-map-Sym-sqrt}, the map $g:\Sym^{+}(\H)\mapto \Sym^{+}(\H)$ defined by
	$g(X) = C^{1/2}XC^{1/2}$ is compact, so that each sequence $\{C^{1/2}X_nC^{1/2}, ||X_n|| \leq 1\}_{n \in \Nbb}$ contains
	a convergent subsequence $\{C^{1/2}X_{k(n)}C^{1/2}\}$ with 
	$\lim_{k(n) \approach \infty}||C^{1/2}X_{k(n)}C^{1/2} - Y^{*}Y||_{\HS} = 0$,
	where $Y \in \HS(\H)$ is as defined in the proof of Proposition 
	\ref{proposition:compact-map-Sym-sqrt}. By Corollary \ref{corollary:inverse-square-root-sequence},
	\begin{align*}
	||F(X_{k(n)}) - (I+(I + a^2Y^{*}Y)^{1/2})^{-1}|| &\leq \frac{a^2}{8}||C^{1/2}X_{k(n)}C^{1/2}- Y^{*}Y|| 
	\\
	&\leq \frac{a^2}{8}||C^{1/2}X_{k(n)}C^{1/2}- Y^{*}Y||_{\HS} \approach 0
	\end{align*}
	as $k(n) \approach \infty$. Thus the set $\{F(X), ||X||\leq 1\} \subset \Sym^{+}(\H)$ is relatively compact, showing that $F$ is compact. \qed 
\end{proof}
\begin{lemma}
	\label{lemma:compact-map-quadratic}
	Let $C\in \Sym^{+}(\H)\cap \Tr(\H)$. 
	Let $a \in \R, a > 0$.
	The following map $F: \Sym^{+}(\H) \mapto \Lcal(\H)$ is compact
	\begin{equation}
	F(X) =(I+(I+a^2X^2)^{1/2})^{1/2}C^{1/2}. 
	\end{equation}
\end{lemma}

\begin{proof}
	Define the map $g:\Sym^{+}(\H) \mapto \Sym^{+}(\H)$ by $g(X) = (1+(1+a^2X^2)^{1/2})^{1/2}$.
	By the inequality $(1+a^2)^{1/2}\leq 1 +a$ for $a\in \R, a\geq 0$, 
	\begin{align*}
	||g(X)||\leq 1 +(1+a^2||X||^2)^{1/4} \leq 1 + (1+a||X||)^{1/2}\leq 2 + \sqrt{a} ||X||^{1/2}.
	\end{align*}
	Thus the set $\{g(X):||X||\leq 1\}$ is bounded, with 
	$\max_{||X||\leq 1}||g(X)|| \leq 2 +\sqrt{a}$.
	Applying Corollary \ref{corollary:trace-class-square-root} twice, we obtain
	\begin{align*}
	||g(X) - g(Y)|| &= ||(I+(I+a^2X^2)^{1/2})^{1/2} - (I+(I+a^2Y^2)^{1/2})^{1/2}|| 
	\\
	&\leq \frac{1}{2}||(I+a^2X^2)^{1/2} - (I+a^2Y^2)^{1/2}||
	\\
	& \leq \frac{a^2}{4}||X^2 - Y^2|| \leq \frac{a^2}{4}(||X||+||Y||)||X-Y||.
	\end{align*}
	This shows that $g$ is continuous on $\Sym^{+}(\H)$. Hence $F$ is continuous on $\Sym^{+}(\H)$.
	As in the proof of Proposition \ref{proposition:compact-map-Sym-sqrt}, for each sequence
	$\{F(X_n) = g(X_n)C^{1/2}, ||X_n||\leq 1\}$, there exists a subsequence $\{F(X_{k(n)})\}$ and 
	an operator $Z \in \HS(\H)$ such that
	$\lim_{k(n) \approach \infty}||F(X_{k(n)}) - Z||_{\HS} = 0$.
	Thus the set $\{F(X): ||X||\leq 1\}$ is relatively compact, proving that 
	$F$ is compact.
	\qed
\end{proof}

\begin{lemma}
	\label{lemma:compact-map-sum-product}
	Let $E$ be a Banach algebra and $M \subset E$.
	Let $f,g:M \mapto E$ be compact. Then the sum and  product maps $h_1, h_2: M \mapto E$ defined by
	$h_1(X) = f(X) + g(X)$and $h_2(X) = f(X)g(X)$ are compact.
\end{lemma}
\begin{proof}
	Let us show that the product map is compact.
	Let $M_B$ be any bounded, non-empty subset of $M$.
	By assumption of compactness, each sequence $\{f(X_n), X_n \in M_B\}_{n \in \Nbb}$ contains a convergent subsequence
	$\{f(X_{n_1})\}$ with limit $Y_1$ in $E$.  Next, the sequence $\{g(X_{n_1})\}$ contains a convergent subsequence
	$\{g(X_{n_2})\}$ with limit $Y_2$ in $E$. We then have
	\begin{align*}
	||f(X_{n_2})g(X_{n_2}) - Y_1Y_2|| &\leq ||(f(X_{n_2})-Y_1)g(X_{n_2})|| + ||Y_1(g(X_{n_2}) - Y_2)||
	\\
	& \leq ||f(X_{n_2}) - Y_1||\;||g(X_{n_2})|| + ||Y_1||\;||g(X_{n_2}) - Y_2|| \approach 0
	\end{align*}
	as $n_2 \approach \infty$. Thus the set $\{h_2(X) = f(X)g(X):X \in M_B\}$ is relatively compact.
	An analogous argument shows that $h_2$ is continuous, hence $h_2$ is a compact map on $M$. \qed
\end{proof}
\begin{proposition}
	\label{proposition:continuous-fixedpoint-map}
	Let $C_i \geq 0$, $1 \leq i \leq N$, be fixed.
	The map 
	$\Fcal: \Sym^{+}(\H)\mapto \Sym^{+}(\H)$ as defined in Eq.\eqref{equation:map-fixedpoint-barycenter-sinkhorn} 
	is continuous in the operator $||\;||$ norm.
		In particular, if $0 \leq C_i \leq \gamma I$, $i=1, \ldots, N$, and $0 \leq X,Y \leq \gamma I$, then
	\begin{equation}
	\begin{aligned}
	||\Fcal(X)- \Fcal(Y)||  \leq \frac{8\gamma^2}{\ep^2}\left(1 + \sqrt{\frac{\gamma}{\ep}}\right)\left(3 + \sqrt{\frac{\gamma}{\ep}}\right)||X-Y||.
	\end{aligned}
	\end{equation}
\end{proposition}

\begin{proof}
	Let $c_{\ep} = \frac{4}{\ep}$ and $g(X) = \left(I+ \left(I + c_{\ep}^2X^2\right)^{1/2}\right)^{1/2}$
	and for $1 \leq i \leq N$, define $h_i(X) =C_i^{1/2}\left(I+\left(I + c_{\ep}^2C_i^{1/2}XC_i^{1/2}\right)^{1/2}\right)^{-1}C_i^{1/2}$, then
	\begin{align}
	\label{equation:continuous-inequality-0}
	&||\Fcal(X) - \Fcal(Y)|| = ||g(X)\sum_{i=1}^Nw_ih_i(X)g(X) - g(Y)\sum_{i=1}^Nw_ih_i(Y)g(Y)||
	\nonumber
	\\
	&\leq ||g(X)\sum_{i=1}^Nw_ih_i(X)||\;||g(X) - g(Y)|| + ||g(X)||\sum_{i=1}^Nw_i||h_i(X) - h_i(Y)||\;||g(Y)||
	\nonumber
	\\
	& + ||g(X) - g(Y)||\;||\sum_{i=1}^Nw_ih_i(Y)g(Y)||.
	\end{align}
	Using the inequality $(1+a^2)^{1/2} \leq 1+a$ $\forall a \geq 0$, we obtain
	\begin{equation}
	\label{equation:continuous-inequality-1}
	\begin{aligned}
	||g(X)|| \leq 2 + \frac{2}{\sqrt{\ep}}||X||^{1/2}, \;\;\;
	||g(Y)||  \leq 2 + \frac{2}{\sqrt{\ep}}||Y||^{1/2}.
	\end{aligned}
	\end{equation}
	Applying Corollary \ref{corollary:trace-class-square-root} twice gives
	\begin{equation}
	\label{equation:continuous-inequality-3}
	\begin{aligned}
	&||g(X) - g(Y)|| \leq \frac{4}{\ep^2}(||X|| + ||Y||)||X - Y||.
	\end{aligned}
	\end{equation} 
For $\sum_{i=1}^Nw_ih_i(X)$,
	\begin{align}
	\label{equation:continuous-inequality-2}
	||\sum_{i=1}^Nw_ih_i(X)|| &= \left\|\sum_{i=1}^Nw_i\left[C_i^{1/2}\left(I+\left(I + c_{\ep}^2C_i^{1/2}XC_i^{1/2}\right)^{1/2}\right)^{-1}C_i^{1/2}\right]\right\|
	\nonumber
	\\
	&\leq \frac{1}{2}\sum_{i=1}^Nw_i||C_i||.
	\end{align}
	
	 By Corollary \ref{corollary:continuity-norm-square-root}, 
	 {\small
	\begin{align}
	\label{equation:continuous-inequality-4}
	&||h_i(X) - h_i(Y)||
	\nonumber
	\\
	&= \left\|C_i^{1/2}\left[\left(I+\left(I + c_{\ep}^2C_i^{1/2}XC_i^{1/2}\right)^{1/2}\right)^{-1}
	-\left(I+\left(I + c_{\ep}^2C_i^{1/2}YC_i^{1/2}\right)^{1/2}\right)^{-1}
	\right] C_i^{1/2}\right\|
	\nonumber
	\\
	&\leq \frac{1}{8}{c_{\ep}^2}||C_i^{1/2}||^2\;||C_i^{1/2}XC_i^{1/2}-C_i^{1/2}YC_i^{1/2}||
	\leq \frac{2}{\ep^2}||C_i||^2||X - Y||.
	\end{align}
}
	Combining Eqs.\eqref{equation:continuous-inequality-0}, \eqref{equation:continuous-inequality-1}, \eqref{equation:continuous-inequality-2},\eqref{equation:continuous-inequality-3}, and \eqref{equation:continuous-inequality-4} gives
{\small
	\begin{equation*}
	\begin{aligned}
	&||\Fcal(X)- \Fcal(Y)||  \leq \left(2 + \frac{2}{\sqrt{\ep}}||X||^{1/2}\right)\left(\frac{1}{2}\sum_{i=1}^Nw_i||C_i||\right)\frac{4}{\ep^2}(||X|| + ||Y||)||X - Y||
	\\
	& \quad+ \left(2 + \frac{2}{\sqrt{\ep}}||X||^{1/2}\right)\left(2 + \frac{2}{\sqrt{\ep}}||Y||^{1/2}\right)\left(\frac{2}{\ep^2}\sum_{i=1}^Nw_i||C_i||^2||X - Y||\right)
	\\
	&\quad +\left(2 + \frac{2}{\sqrt{\ep}}||Y||^{1/2}\right)\left(\frac{1}{2}\sum_{i=1}^Nw_i||C_i||\right)\frac{4}{\ep^2}(||X|| + ||Y||)||X - Y||
	\\
	& =\frac{4}{\ep^2} \left(\sum_{i=1}^Nw_i||C_i||\right)\left(2 + \frac{1}{\sqrt{\ep}}||X||^{1/2} + \frac{1}{\sqrt{\ep}}||Y||^{1/2}\right)(||X|| + ||Y||)||X - Y||
	\\
	&\quad + \frac{8}{\ep^2}\left(\sum_{i=1}^Nw_i||C_i||^2\right)\left(1 + \frac{1}{\sqrt{\ep}}||X||^{1/2}\right)\left(1 + \frac{1}{\sqrt{\ep}}||Y||^{1/2}\right)||X-Y||.
	\end{aligned}
	\end{equation*}
}
	This shows that $\Fcal$ is continuous in the operator norm $||\;||$.
	In particular, for $0 \le X\leq \gamma I$, $0 \leq Y \leq \gamma I$, 
	and $0 \leq C_i \leq \gamma I$, $1 \leq i \leq N$,
	\begin{equation*}
	\begin{aligned}
	||\Fcal(X)- \Fcal(Y)||  &\leq \frac{16\gamma^2}{\ep^2} \left(1 + \sqrt{\frac{\gamma}{\ep}}\right)||X - Y||
	+ \frac{8\gamma^2}{\ep^2}\left(1 + \sqrt{\frac{\gamma}{\ep}}\right)^2||X-Y||
	\\
	& = \frac{8\gamma^2}{\ep^2}\left(1 + \sqrt{\frac{\gamma}{\ep}}\right)\left(3 + \sqrt{\frac{\gamma}{\ep}}\right)||X-Y||.
	\end{aligned}
	\end{equation*}
	This completes the proof. \qed
\end{proof}

\begin{proposition}
	\label{proposition:compact-fixedpoint-map}
	Let $C_i \in \Sym^{+}(\H)\cap \Tr(\H)$, $1 \leq i \leq N$.
	Consider the map $\Fcal: \Sym^{+}(\H) \mapto \Sym^{+}(\H)$ as defined in
	Eq.\eqref{equation:map-fixedpoint-barycenter-sinkhorn}.
	Then $\Fcal$ is compact.
\end{proposition}
\begin{proof}
	[\textbf{of Proposition \ref{proposition:compact-fixedpoint-map}}]
	By Proposition \ref{proposition:continuous-fixedpoint-map}, $\Fcal$ is continuous in the operator norm $||\;||$.	
	Let $g_i(X) = \left(I+ \left(I + \frac{16}{\ep^2}X^2\right)^{1/2} \right)^{1/2}C_i^{1/2}$, $1 \leq i \leq N$,
	then $g_i$ is compact by Lemma \ref{lemma:compact-map-quadratic}
	and similarly so is $l_i(X) = C_i^{1/2}\left(I+ \left(I + \frac{16}{\ep^2}X^2\right)^{1/2} \right)^{1/2}$, $1 \leq i \leq N$.
	Let $h_i(X) = \left(I+\left(I + \frac{16}{\ep^2}C_i^{1/2}XC_i^{1/2}\right)^{1/2}\right)^{-1}$, $1 \leq i \leq N$,
	then $h_i$ is compact by Corollary \ref{corollary:compact-map-inverse-square-root}.
	Then we have the summation
	$\Fcal(X) = \sum_{i=1}^Nw_ig_i(X)h_i(X)l_i(X)$,
	which is compact by Lemma \ref{lemma:compact-map-sum-product}.\qed
\end{proof}
\begin{lemma}
	\label{lemma:XFX-upperbound}
	Let $\gamma \in \R,\gamma > 0$ be fixed.
	Assume that $0 \leq C_i \leq \gamma I$, $1 \leq i \leq N$.
	Consider the map $\Fcal: \Sym^{+}(\H) \mapto \Sym^{+}(\H)$ as defined in
	Eq.\eqref{equation:map-fixedpoint-barycenter-sinkhorn}.
	Then
	\begin{equation}
	0 \leq X \leq \gamma I \imply 0 \leq X\Fcal(X)X \leq \gamma X^2.
	\end{equation}
\end{lemma}
\begin{proof}
	Since $0 \leq X \leq \gamma I$, by Lemma \ref{lemma:operator-monotone-quadratic},
	\begin{align*}
	&0 \leq X \leq \gamma I \imply X^2 = X^{1/2}(X)X^{1/2}\leq \gamma X \imply c_{\ep}^2C_i^{1/2}X^2C_i^{1/2} \leq \gamma c_{\ep}^2C_i^{1/2}XC_i^{1/2} 
	\\
	&\equivalent \gamma I + c_{\ep}^2C_i^{1/2}X^2C_i^{1/2}\leq \gamma
	\left(I + c_{\ep}^2C_i^{1/2}XC_i^{1/2}\right)
	\\
	& \imply \left(I + \frac{1}{\gamma} c_{\ep}^2C_i^{1/2}X^2C_i^{1/2}\right)^{1/2} \leq \left(I + c_{\ep}^2C_i^{1/2}XC_i^{1/2}\right)^{1/2}
	\\
	& \equivalent I + \left(I + \frac{1}{\gamma} c_{\ep}^2C_i^{1/2}X^2C_i^{1/2}\right)^{1/2}\leq I + \left(I + c_{\ep}^2C_i^{1/2}XC_i^{1/2}\right)^{1/2}
	\\
	& \equivalent \left(I + \left(I + c_{\ep}^2C_i^{1/2}XC_i^{1/2}\right)^{1/2}\right)^{-1} \leq 
	\left(I + \left(I + \frac{1}{\gamma} c_{\ep}^2C_i^{1/2}X^2C_i^{1/2}\right)^{1/2}\right)^{-1}.
	\end{align*}
	By Lemma \ref{lemma:operator-monotone-quadratic}, pre- and post-multiplying by $C_i^{1/2}$ gives
	{\small
	\begin{align*}
	& \imply C_i^{1/2}\left(I + \left(I + c_{\ep}^2C_i^{1/2}XC_i^{1/2}\right)^{1/2}\right)^{-1}C_i^{1/2} 
	&\leq 
	C_i^{1/2}\left(I + \left(I + \frac{c_{\ep}^2}{\gamma} C_i^{1/2}X^2C_i^{1/2}\right)^{1/2}\right)^{-1}C_i^{1/2}.
	\end{align*}
}
		Once again applying Lemma \ref{lemma:operator-monotone-quadratic}, pre- and post-multiplying by $X$ gives
	{\small
	\begin{align*}
	&XC_i^{1/2}\left(I + \left(I + c_{\ep}^2C_i^{1/2}XC_i^{1/2}\right)^{1/2}\right)^{-1}C_i^{1/2}X
	\\
	 \leq 
	&XC_i^{1/2}\left(I + \left(I + \frac{c_{\ep}^2}{\gamma} C_i^{1/2}X^2C_i^{1/2}\right)^{1/2}\right)^{-1}C_i^{1/2}X
	= \frac{\gamma}{c_{\ep}^2}\left[-I + \left(I + \frac{c_{\ep}^2}{\gamma}XC_iX\right)^{1/2}\right],
	\end{align*}
}
	where the last expression follows from Lemma \ref{lemma:adjoint-switch-CX}. Since $0 \leq C_i \leq \gamma I$, we have
	$XC_iX \leq X(\gamma I) X = \gamma X^2$ by Lemma \ref{lemma:operator-monotone-quadratic}. Thus,
	\begin{align*}
	XC_i^{1/2}\left(I + \left(I + c_{\ep}^2C_i^{1/2}XC_i^{1/2}\right)^{1/2}\right)^{-1}C_i^{1/2}X &\leq \frac{\gamma}{c_{\ep}^2}\left[-I + \left(I + c_{\ep}^2X^2\right)^{1/2}\right]
	\\
	& = \gamma X^2\left[I + \left(I + c_{\ep}^2X^2\right)^{1/2}\right]^{-1}.
	\end{align*}
	By Lemma \ref{lemma:operator-monotone-quadratic},
	{\small
	\begin{align*}
	&X\Fcal(X)X = \left(I+ \left(I + c_{\ep}^2X^2\right)^{1/2} \right)^{1/2}
	\\
	&\times \sum_{i=1}^Nw_i\left[XC_i^{1/2}\left(I+\left(I + c_{\ep}^2C_i^{1/2}XC_i^{1/2}\right)^{1/2}\right)^{-1}C_i^{1/2}X\right]
	\left(I+ \left(I + c_{\ep}^2X^2\right)^{1/2} \right)^{1/2}
	\\
	& \leq \left(I+ \left(I + c_{\ep}^2X^2\right)^{1/2} \right)^{1/2} \sum_{i=1}^N w_i \gamma 
	X^2\left[I + \left(I + c_{\ep}^2X^2\right)^{1/2}\right]^{-1}
	\left(I+ \left(I + c_{\ep}^2X^2\right)^{1/2} \right)^{1/2} 
	\\
	& = 
	\gamma X^2.
	\end{align*}
}
	This completes the proof. \qed
\end{proof}

\begin{proposition}
	\label{proposition:FX-upperbound}
	Let $\gamma \in \R,\gamma > 0$ be fixed.
	Assume that $0 \leq C_i \leq \gamma I$, $1 \leq i \leq N$.
	Consider the map $\Fcal: \Sym^{+}(\H) \mapto \Sym^{+}(\H)$ as defined in
	Eq.\eqref{equation:map-fixedpoint-barycenter-sinkhorn}.
	Then under either one of the following  two additional conditions
	\begin{enumerate}
		\item $X$ is strictly positive,
		\item $C_i$, $1\leq i \leq N$, and $X$ are compact, not necessarily strictly positive,
	\end{enumerate}
	the following holds
	\begin{equation}
	0 \leq X \leq \gamma I \imply 0 \leq \Fcal(X) \leq \gamma I.
	\end{equation}
\end{proposition}
\begin{proof}
	[\textbf{of Proposition \ref{proposition:FX-upperbound}}]
	It is clear that $\Fcal(X) \geq 0$. For $X = 0$, we have
	\begin{align*}
	\Fcal(0) =\sum_{i=1}^Nw_iC_i \leq \gamma I \;\text{ since } C_i \leq \gamma I \;\forall i=1, \ldots, N. 
	\end{align*}
	Assume now that $X \neq 0$.
	By Lemma \ref{lemma:XFX-upperbound}, we have
	\begin{align*}
	X\Fcal(X)X \leq \gamma X^2 \equivalent X[\gamma I - \Fcal(X)]X \geq 0.
	\end{align*}
	If $X$ is strictly positive, that is $\ker(X) = \{0\}$, then by Lemma \ref{lemma:operator-quadratic-positive}, the previous 
	inequality implies $\gamma I - \Fcal(X) \geq 0 \equivalent \Fcal(X) \leq \gamma I$.
	
	Assume now that $X$ is compact and singular, $X \neq 0$. Let 
	$\{\lambda_k\}_{k\in \Nbb}$ be the set of eigenvalues of $X$, $\lambda_k \geq 0$, arranged in decreasing order, with corresponding 
	orthonormal eigenvectors $\{e_k\}_{k \in \Nbb}$ forming an orthonormal basis in $\H$. Then
	\begin{align*}
	X = \sum_{k=1}^{\infty}\lambda_k e_k \otimes e_k \; \text{ and } X \leq \gamma I \equivalent \lambda_1 \leq \gamma.
	\end{align*}
	For any $0 < \delta \leq \lambda_1$, define the following operator
		\begin{align*}
	X' = \sum_{k=1}^{\infty}\lambda_k' e_k \otimes e_k, \; \text{ where }
	\lambda_k' = 
	\left\{\begin{matrix} \lambda_k & \text{if $\lambda_k > 0$,}
	\\
	\frac{\delta}{k^2} & \text{if $\lambda_k = 0$.}
	\end{matrix}\right.
	\end{align*}
	Then $X'$ is compact, strictly positive, with $X' \leq \gamma I$ and $||X-X'|| <\delta$.
	Thus $\Fcal(X') \leq \gamma I$ by the first part of the proposition.
	By Proposition \ref{proposition:continuous-fixedpoint-map},
	\begin{align*}
	||\Fcal(X) - \Fcal(X')|| \leq \frac{8\gamma^2}{\ep^2}\left(1 + \sqrt{\frac{\gamma}{\ep}}\right)\left(3 + \sqrt{\frac{\gamma}{\ep}}\right)||X-X'||. 
	\end{align*}
	This implies that
	\begin{align*}
	||\Fcal(X)|| &\leq ||\Fcal(X')|| + \frac{8\gamma^2}{\ep^2}\left(1 + \sqrt{\frac{\gamma}{\ep}}\right)\left(3 + \sqrt{\frac{\gamma}{\ep}}\right)||X-X'||
	\\
	&\leq \gamma + \frac{8\gamma^2}{\ep^2}\left(1 + \sqrt{\frac{\gamma}{\ep}}\right)\left(3 + \sqrt{\frac{\gamma}{\ep}}\right)\delta.
	\end{align*}
	Since $\delta$ can be arbitrarily close to zero, this implies that $||\Fcal(X)|| \leq \gamma $.
	With the additional conditions that $C_i$, $1 \leq i \leq N$ are compact,
	the operator $\Fcal(X)$ is self-adjoint, compact, positive, and thus $0 \leq \Fcal(X) \leq \gamma I$. \qed
\end{proof}
\begin{proof}
	[\textbf{of Theorem \ref{theorem:barycenter-sinkhorn-Gaussian}}]
	As with the entropic $2$-Wasserstein distance, the Sinkhorn divergence $\Srm^{\ep}_{d^2}(\Ncal(m_0, C_0), \Ncal(m_1, C_1))$
	is the sum of the squared Euclidean distance $||m_0 - m_1||^2$ and the Sinkhorn divergence $\Srm^{\ep}_{d^2}(\Ncal(0, C_0),\Ncal(0, C_1))$.
	We can thus consider the means and covariance operators separately.
	The barycentric mean is obviously the Euclidean mean $\bar{m} = \sum_{i=1}^Nw_i m_i$. 
	
	Consider now the centered Gaussian measures $\Ncal(0,C), \Ncal(0, C_i)$, $1 \leq i \leq N$.  
	Define the function $F: \Sym^{+}(\H) \cap\Tr(\H)\mapto \R$,
	$1 \leq i \leq N$, by
	\begin{align*}
	F(C) = \sum_{i=1}^Nw_i\Srm^{\ep}_{d^2}(\Ncal(0, C), \Ncal(0, C_i)).
	\end{align*}
	Then $F$ is strictly convex, since $\Srm^{\ep}_{d^2}$ is strictly convex, thus
	its minimum, if it exists, is unique.
	By Proposition \ref{proposition:derivation-equation-barycenter},
	\begin{align*}
	DF(X_0) = 0 \equivalent X_0 = \Fcal(X_0)
	\end{align*}
	where $\Fcal:\Sym^{+}(\H) \mapto \Sym^{+}(\H)$ is the map defined by
	Eq.\eqref{equation:map-fixedpoint-barycenter-sinkhorn}.
	By Proposition \ref{proposition:compact-fixedpoint-map}, $\Fcal$ is a compact map on $\Sym^{+}(\H)$.
	Let $\gamma  \in \R, \gamma > 0$ be such that $C_i \leq \gamma I$, $1\leq i \leq N$ and consider the set
	\begin{align}
	\Kcal = \{X \in \Sym^{+}(\H): 0 \leq X \leq \gamma I\}.
	\end{align}
	Then $\Kcal$ is a closed, bounded, convex subset of $\Lcal(\H)$.
	By Proposition \ref{proposition:FX-upperbound}, $X \in \Kcal \imply \Fcal(X) \in \Kcal$.
	Thus by Schauder Fixed Point Theorem, there exists $X_0 \in \Kcal$ such that $X_0 = \Fcal(X_0)$,
	which must be the unique global minimizer of the strictly convex function $F$.
		Clearly, with $c_{\ep} = \frac{4}{\ep}$,
	\begin{align*}
	X_0 > 0 \equivalent \sum_{i=1}^Nw_iC_i^{1/2}\left(I+\left(I + c_{\ep}^2C_i^{1/2}X_0C_i^{1/2}\right)^{1/2}\right)^{-1}C_i^{1/2} > 0.
	\end{align*}
	Since $0 \leq X_0 \leq \gamma I, 0 \leq C_i \leq \gamma I$, $1 \leq i \leq N$, 
	\begin{align*}
	\left(1+\left(1 + c_{\ep}^2\gamma^2\right)^{1/2}\right)^{-1}\sum_{i=1}^Nw_iC_i &\leq \sum_{i=1}^Nw_iC_i^{1/2}\left(I+\left(I + c_{\ep}^2C_i^{1/2}X_0C_i^{1/2}\right)^{1/2}\right)^{-1}C_i^{1/2}
	\\
	&\leq \frac{1}{2}\sum_{i=1}^Nw_iC_i.
	\end{align*}
	Thus it follows that
	$X_0 > 0 \equivalent \sum_{i=1}^Nw_iC_i > 0$.
	\qed
\end{proof}
	
\section{Comparison of barycenter fixed point equations}
\label{section:compare-barycenter}
We now show that for $\dim(\H) \geq 2$, $\sum_{i=1}^Nw_iC_i > 0$, Eq.\eqref{equation:barycenter-sinkhorn-strictlypositive}
has uncountably
infinitely 
many positive, singular solutions.

Consider first the case $\H = \R$ and $C_i = \sigma_i^2$, $i=1, \ldots, N$.

\begin{lemma}
	\label{lemma:F-fixedpoint-1D}
	Assume that $\sum_{i=1}^Nw_i\sigma_i^2 > 0$. The function
	\begin{equation}
	\label{equation:F-fixedpoint-1D}
	\Fcal(x) = \left(1+ \left(1+c_{\ep}^2x^2\right)^{1/2}\right)\sum_{i=1}^Nw_i \sigma_i^2\left(1+\left(1+ c_{\ep}^2\sigma_i^2x\right)^{1/2}\right)^{-1}, \; x\geq 0,
	\end{equation}
	has a unique fixed point $x^{*}$, which satisfies $x^{*} > 0$. The function
	\begin{equation}
	\Gcal(x) = \frac{\ep}{4}\left(-1 + \left(\sum_{i=1}^Nw_i\left(1+c_{\ep}^2\sigma_i^2 x\right)^{1/2}\right)^2\right)^{1/2}, \; x\geq 0,
	\end{equation}
	has two fixed points, namely $x^{*}$ and $x_0 = 0$.
\end{lemma}
\begin{proof}
	The fixed point equation $x = \Fcal(x)$ is equivalent to
	\begin{equation*}
	\left[1+ \left(1+ c_{\ep}^2x^2\right)^{1/2}\right]^{-1}x - \sum_{i=1}^Nw_i\sigma_i^2 \left(1 + \left(1+ c_{\ep}^2\sigma_i^2x\right)^{1/2}\right)^{-1} = 0.
	\end{equation*}
	Consider the left hand side, which is $f(x) = \left(1+ \left(1+c_{\ep}^2x^2\right)^{1/2}\right)^{-1}[x - \Fcal(x)]$,
	with $f(0) =  -\frac{1}{2}\sum_{i=1}^Nw_i\sigma_i^2 < 0, \;\;\; \lim_{x \approach \infty}f(x) = \frac{\ep}{4} > 0$,
	and
	\begin{equation*}
	\begin{aligned}
	f'(x) & = \left[1+ \left(1+ c_{\ep}^2x^2\right)^{1/2}\right]^{-1}\left[\left(1+ c_{\ep}^2x^2\right)+ \left(1+ c_{\ep}^2x^2\right)^{1/2}\right]^{-1}
	\\
	&\quad + \frac{8}{\ep^2}\sum_{i=1}^Nw_i\sigma_i^4\left(1 + \left(1+ c_{\ep}^2\sigma_i^2x\right)^{1/2}\right)^{-2}\left(1+ c_{\ep}^2\sigma_i^2x\right)^{-1/2} > 0 \;\;\; \forall x \geq 0.
	\end{aligned}
	\end{equation*}
	Thus $f(x)$ is strictly increasing on $[0, \infty)$ and hence there must exist a unique $x^{*} > 0$ at which $f(x^{*}) = 0 \equivalent x^{*} = \Fcal(x^{*})$.
	Since $x^{*} > 0$, by Proposition \ref{proposition:derivation-equation-barycenter},
	\begin{equation*}
	x^{*}= \Fcal(x^{*}) \equivalent x^{*} = \Gcal(x^{*}) = \frac{\ep}{4}\left(-1 + \left(\sum_{i=1}^Nw_i\left(1+c_{\ep}^2\sigma_i^2 x^{*}\right)^{1/2}\right)^2\right)^{1/2}.
	\end{equation*}
	It is obvious that $\Gcal$ has another fixed point $x_0=0$. \qed 
\end{proof}

\begin{lemma}
	\label{lemma:rank-one-operator-power}
	Let $u \in \H$, $||u|| = 1$. 
	Consider the rank-one operator $u \otimes u: \H\mapto \H$ defined by $(u \otimes u)x = \la u,x\ra u$.
	Then for any $A \in \Lcal(\H)$ and any $k \in \Nbb$, 
	\begin{align}
	[(u\otimes u) A (u \otimes u)]^{k} &= \la u, Au\ra^k (u \otimes u).
	\end{align}
	If $A$ is self-adjoint, positive, then
	\begin{align}
	[(u\otimes u)A(u \otimes u)]^{1/k} &= \la u, Au\ra^{1/k} (u \otimes u).
	\end{align}
\end{lemma}

\begin{proof}
	Consider the cases $k=1$ and $k=2$. For any $x \in \H$,
	\begin{align*}
	(u \otimes u)A(u \otimes u)x = (u \otimes u)A \la u,x\ra u = \la u, x\ra\la u, Au\ra u = \la u, Au\ra (u \otimes u)x.
	\end{align*}
	Since $(u\otimes u)^2 = (u \otimes u)$, we have
	\begin{align*}
	[(u \otimes u)A(u \otimes u)]^2x = \la u, Au\ra (u \otimes u)A(u \otimes u)x = \la u, Au\ra^2 (u \otimes u)x.
	\end{align*}
	For the first expression, the general case then follows by induction.
	
	If $A$ is self-adjoint, positive, then $(u\otimes u)A(u \otimes u)$ is self-adjoint, positive, and
	$[(u\otimes u)A(u \otimes u)]^{1/k}$ is well-defined and unique \cite{Brown1980:nthRoot}. We have
	\begin{align*}
	[\la u, Au\ra^{1/k} (u \otimes u)]^k = \la u, Au\ra (u \otimes u)^k = \la u, Au\ra (u \otimes u) = (u \otimes u)A(u \otimes u),
	\end{align*}
	from which the second identity follows. \qed
\end{proof}

\begin{proposition}
	\label{proposition:G-infinitely-many-fixedpoint}
	Let $C_i \in \Sym^{+}(\H)$, $1 \leq i \leq N$ be fixed.
	Consider the following map $\Gcal: \Sym^{+}(\H) \mapto \Sym^{+}(\H)$, defined by
	\begin{equation}
	\Gcal(X) = \frac{\ep}{4}\left[- I +  \left(\sum_{i=1}^Nw_i\left(I + c_{\ep}^2X^{1/2}C_iX^{1/2}\right)^{1/2}\right)^2\right]^{1/2}.
	\end{equation}
	Then $X_0= 0$ is a fixed point of $\Gcal$. Let $u \in \H, ||u||=1$ be such that $\sum_{i=1}^Nw_i\la u,C_iu\ra > 0$. Then
	$X_u = x_u (u \otimes u)$
	is a fixed point of $\Gcal$, where $x_u$ is the unique positive solution of the following one-dimensional fixed point equation
	\begin{equation}
	x = \frac{\ep}{4}\left[- 1 +  \left(\sum_{i=1}^Nw_i\left(1 + c_{\ep}^2x \la u, C_iu\ra\right)^{1/2}\right)^2\right]^{1/2}.
	\end{equation}
\end{proposition}
The condition $\sum_{i=1}^Nw_i\la u,C_iu\ra > 0$ is satisfied for at least one $u \in \H$, $u\neq0$, since otherwise
$C_1 = \cdots C_N = 0$.
If $\sum_{i=1}^Nw_iC_i > 0$, then $\sum_{i=1}^Nw_i \la u,C_iu \ra > 0$$\forall u \in \H, u \neq 0$.
Thus under this assumption, for $\dim(\H) \geq 2$,
$\Gcal$ has {\it uncountably infinitely many fixed points} of the form $X_u = x_u(u \otimes u)$.
\begin{proof}
	[\textbf{of Proposition \ref{proposition:G-infinitely-many-fixedpoint}}]
	Clearly $X=0$ is always a fixed point of $\Gcal$.
	Consider the rank-one operator $u \otimes u$, $||u||=1$, with eigenvalue $1$ and eigenvector $u$, we have $(u\otimes u)^{1/2}= u \otimes u $. By Lemma \ref{lemma:rank-one-operator-power},
	\begin{align*}
	(u \otimes u)^{1/2}C_i(u \otimes u)^{1/2} = (u \otimes u)C_i(u \otimes u) = \la u, C_iu\ra (u \otimes u).
	\end{align*}
	Therefore, for $X_u = x(u \otimes u)$,
	\begin{align*}
	I + c_{\ep}^2X_u^{1/2}C_iX_u^{1/2} &= (I-u\otimes u) + \left(1+c_{\ep}^2x\la u, C_iu\ra\right) (u \otimes u).
	\end{align*}
	Since $(I-u\otimes u)^2 = (I-u \otimes u)$ and $(I-u \otimes u)(u \otimes u) = 0$,
		by Lemma \ref{lemma:rank-one-operator-power},
	\begin{align*}
	\left(I + c_{\ep}^2X_u^{1/2}C_iX_u^{1/2}\right)^{1/2} = (I - u \otimes u) + \left(1+ c_{\ep}^2x\la u, C_iu\ra\right)^{1/2}(u \otimes u).
	\end{align*}
	Applying the same argument and using the fact that $\sum_{i=1}^Nw_i =1$, we have
	\begin{align*}
	\Gcal(X_u) = \frac{\ep}{4}\left[- I +  \left(\sum_{i=1}^Nw_i\left(1 + c_{\ep}^2x \la u, C_i u\ra\right)^{1/2}\right)^2\right]^{1/2} (u \otimes u).
	\end{align*}
	Thus the fixed point equation $X_u = \Gcal(X_u)$ becomes
	\begin{align*}
	x  = \frac{\ep}{4}\left[- 1 +  \left(\sum_{i=1}^Nw_i\left(1 + c_{\ep}^2x \la u, C_iu\ra\right)^{1/2}\right)^2\right]^{1/2}.
	\end{align*}
	As shown in Lemma \ref{lemma:F-fixedpoint-1D}, under the condition $\sum_{i=1}^Nw_i \la u, C_i u\ra > 0$,
	this one-dimensional fixed point equation has a unique positive solution $x^{*}_u$.
	Thus $X_u = x^{*}_u (u \otimes u)$ is a fixed point of $\Gcal$.\qed
\end{proof}

When $\ep=0$, the fixed points $X_u$ of $\Gcal$ in Proposition \ref{proposition:G-infinitely-many-fixedpoint} 
admit a closed form.
\begin{proposition}
	\label{proposition:G-ep0-infinitely-many-fixedpoint}
	Let $C_i \in \Sym^{+}(\H)$, $1 \leq i \leq N$ be fixed. Consider the following map $\Gcal: \Sym^{+}(\H) \mapto \Sym^{+}(\H)$, defined by
	\begin{equation}
	\Gcal(X) = \sum_{i=1}^Nw_i(X^{1/2}C_iX^{1/2})^{1/2}.
	\end{equation}
	Then $X_0 =0$ is a fixed point of $\Gcal$. Let $u \in \H, ||u|| =1$ be such that
	$\sum_{i=1}^Nw_i \la u, C_i u\ra^{1/2} > 0$, then
	the following is a fixed point of $\Gcal$
	\begin{equation}
	\begin{aligned}
	X_u & = \left(\sum_{i=1}^Nw_i\la u, C_iu\ra^{1/2}\right)^2(u \otimes u).
	\end{aligned}
	\end{equation}
\end{proposition}
The condition $\sum_{i=1}^Nw_i\la u,C_iu\ra^{1/2} > 0$ is satisfied for at least one $u \in \H$, $u\neq0$, since otherwise
$C_1 = \cdots C_N = 0$.
If $\sum_{i=1}^Nw_iC_i > 0$, then $\sum_{i=1}^Nw_i \la u,C_iu \ra > 0$$\forall u \in \H, u \neq 0$. Since
$w_i > 0$, $i=1,\ldots, N$, for each $u \neq 0$, there must be at least one $i$ for which $\la u,C_iu\ra > 0$.
This implies that under this assumption, $\sum_{i=1}^Nw_i \la u, C_i u\ra^{1/2} > 0$ is satisfied for all $u \in \H$, $u \neq 0$.
In this case, when $\dim(\H) \geq 2$, $\Gcal$ has {\it uncountably infinitely many fixed points} of the form $X_u$.

{\bf Example}. Consider the simplest setting $C_1 = \cdots = C_N = C > 0$, then
in both Propositions \ref{proposition:G-infinitely-many-fixedpoint} and \ref{proposition:G-ep0-infinitely-many-fixedpoint},
\begin{equation}
X = \Gcal(X) \equivalent X = (X^{1/2}CX^{1/2})^{1/2}.
\end{equation} 
One can immediately see that some of the solutions of the above equation include
$X = 0$, $X=C$, $X_u = \la u, Cu\ra (u \otimes u)$, for any $u \in \H$, $||u||=1$, including $X_k = \lambda_k (e_k \otimes e_k)$, $k \in \Nbb$,
where $\{\lambda_k\}_{k \in \Nbb}$ are the eigenvalues of $C$, with corresponding orthonormal eigenvectors
$\{e_k\}_{k \in \Nbb}$.

\begin{proof}
	[\textbf{of Proposition \ref{proposition:G-ep0-infinitely-many-fixedpoint}}]
	By Lemma \ref{lemma:rank-one-operator-power}, 
	$X_u = x(u \otimes u)$, $x > 0$, we have
	\begin{align*}
	\Gcal(X_u) = \sqrt{x}\left(\sum_{i=1}^Nw_i\la u, C_iu\ra^{1/2}\right)(u \otimes u).
	\end{align*}
	Thus the fixed point equation  $X_u = \Gcal(X_u)$ becomes
	$x = \left(\sum_{i=1}^Nw_i\la u, C_iu\ra^{1/2}\right)^2 >0$
	by the assumption $\sum_{i=1}^Nw_i\la u, C_iu\ra^{1/2} > 0$.
	Hence $\left(\sum_{i=1}^Nw_i\la u, C_iu\ra^{1/2}\right)^2(u \otimes u)$ is a fixed point of $\Gcal$ $\forall u \in \H$, $||u||=1$.\qed
\end{proof}

\begin{proof}
	[\textbf{of Theorem \ref{theorem:singular-fixedpoint} - Singular solutions of fixed point equations}]
	This is the combination of Propositions \ref{proposition:entropic-barycenter-Gaussian-strictlypositive}, \ref{proposition:G-infinitely-many-fixedpoint}, and \ref{proposition:G-ep0-infinitely-many-fixedpoint}.
	\qed
\end{proof}

{\bf Comparison with the finite-dimensional setting}.
In the case $\H = \R^n$, the existence of the strictly positive solution of Eq.\eqref{equation:barycenter-sinkhorn-strictlypositive}
is proved via the Brouwer Fixed Point Theorem (see e.g. \cite{Borwein:2000ConvexAnalysis}) as follows. This is the technique employed
by \cite{Agueh:2011barycenters} for the case $\ep = 0$ and \cite{Janati2020entropicOT} for the case $\ep > 0$. 

\begin{theorem}
[\textbf{Brouwer Fixed Point Theorem}]
Let $M \subset \R^n$ be a compact convex subset and $f: M \mapto M$ be continuous. Then $f$ has a fixed point in $M$.
\end{theorem}
\begin{remark}
	Unlike the Banach Fixed Point Theorem, both Brouwer and Schauder Fixed Point Theorems guarantee the existence of one fixed point but {\it not} its uniqueness. This needs to be proved via other means, e.g. strict convexity.
\end{remark}
Assume that
$\exists \alpha, \beta \in \R$, $\alpha > 0, \beta > 0$ such that $\alpha I \leq C_i \leq \beta I$, $1 \leq i \leq N$.
Consider the set $\Kcal_2 = \{X \in \Sym^{++}(n): \alpha I \leq X \leq \beta I\}$. Then
by Lemma \ref{lemma:operator-monotone-quadratic}
\begin{align*}
\alpha^2 I \leq \alpha X = X^{1/2}\alpha X^{1/2 } \leq X^{1/2}C_iX^{1/2} \leq X^{1/2}\beta  X^{1/2} = \beta X \leq \beta^2 I.
\end{align*}
It follows that
\begin{equation}
\alpha I \leq \Gcal(X) = \frac{\ep}{4}\left[- I +  \left(\sum_{i=1}^Nw_i\left(I + c_{\ep}^2X^{1/2}C_iX^{1/2}\right)^{1/2}\right)^2\right]^{1/2}\leq \beta I.
\end{equation}   
Thus the continuous map $\Gcal$ maps the compact convex set $\Kcal_2$ into itself, thus $\Gcal$ has a fixed point in $\Kcal_2$.
This strictly positive solution of Eq.\eqref{equation:barycenter-sinkhorn-strictlypositive} is then precisely the unique solution of Eq.\eqref{equation:barycenter-sinkhorn-positive}.
It is {\it not} clear, however, whether this proof technique can be extended to the {\it infinite-dimensional} setting. This is because one can no longer assume that
there is a uniform lower bound of the form $\alpha I$ for the $C_i$'s and $X$ as above.
One might assume instead that there is an operator $C > 0$ such that $C_i \geq C$, $1 \leq i \leq N$
and consider the set $\Kcal_3 = \{X \in \Sym^{+}(\H): C \leq X \leq \beta I\}$. This does not help, however, since
the condition $X \geq C > 0$ does {\it not} imply that $X^{1/2}CX^{1/2} \geq C^2$. The following is a counterexample,
where it can be verified that $X \geq C$ but $X^{1/2}CX^{1/2} \ngeq C^2$
{\small
\begin{equation}
X = 
\begin{pmatrix}
1.6254  & -0.6825 &  -1.2503
\\
-0.6825 &   1.9105 &   0.0516
\\
-1.2503  &  0.0516 &   2.2376
\end{pmatrix},
C = \begin{pmatrix}
0.2867  & -0.3297  &  0.1976
\\
-0.3297  &  0.6925  & -0.2484
\\
0.1976  & -0.2484   & 0.1392
\end{pmatrix}.
\end{equation}
}
%
We note also that the condition $0 < \alpha I \leq C_i \leq \beta I$, $1 \leq i \leq N$, above
is more restrictive that the condition $\sum_{i=1}^Nw_i C_i > 0$ stated in Theorem \ref{theorem:barycenter-sinkhorn-Gaussian},
which guarantees the existence of a strictly positive solution of Eq.\eqref{equation:barycenter-sinkhorn-strictlypositive}. As an example,
let $\H=\R^N$ and $\{e_i\}_{i=1}^N$ be an orthonormal basis in $\H$.
Define
\begin{equation}
C_i = e_i \otimes e_i \;\;\;\text{ then }\sum_{i=1}^Nw_iC_i = \sum_{i=1}^Nw_i(e_i \otimes e_i) > 0,
\end{equation}
guaranteeing that $\bar{C} > 0$, even though all the $C_i$'s are singular.

\section{Miscellaneous Technical Results}
\label{section:misc}

We briefly review the concept of operator monotone functions.
Let $A,B\in \Sym(\H)$, 
then we say $A \leq B$ if $B - A \geq  0$.
Let $I \subset \R$ be an interval. A function $f:I \mapto \R$ is said to be {\it operator monotone} if $A \leq B \imply f(A) \leq f(B)$.
\begin{proposition}
	[see e.g. \cite{Pedersen1972:OperatorMonotone}]
	\label{proposition:operator-monotone-rth}
	The function $f(t)= t^{r}$ on $[0, \infty)$ is operator monotone if and only if $0 \leq r \leq 1$.
\end{proposition}
For $r=1/2$, we have
$0 \leq A \leq B \imply 0 \leq A^{1/2} \leq B^{1/2}$.
The function $f(t) = t^2$ on $[0, \infty)$, on the other hand, is {\it not} operator monotone. However, we still have 
$0 \leq A \leq \lambda I \imply 0 \leq A^2 \leq \lambda^2 I$ and $A \geq \lambda I \imply A^2 \geq \lambda^2 I$, $\lambda > 0$.

The following is the generalization of Proposition V.1.6 in \cite{Bhatia:1997Matrix} to the infinite-dimensional setting,
with the additional assumption that $A,B$ be invertible, since if $\dim(\H) = \infty$, then $A > 0$ does not imply that $A$ is invertible.
\begin{proposition}
	The function $f(t) = -\frac{1}{t}$ is operator monotone on $(0, \infty)$. Thus if $A,B \in \Lcal(\H)$ are invertible
	then $0 < A \leq B \imply A^{-1} \geq B^{-1}$.
\end{proposition}
\begin{lemma}
	[see \cite{Bhatia:1997Matrix}, Lemma V.1.5]
	\label{lemma:operator-monotone-quadratic}
	\begin{equation}
	A \leq B \imply X^{*}AX \leq X^{*}BX \;\;\;\forall X \in \Lcal(\H).
	\end{equation}
\end{lemma}

{\bf Acknowledgments}. The author would like to thank Augusto Gerolin and Anton Mallasto for their comments and feedback. {\color{black}In addition, he would like to thank the anonymous referee for the many valuable comments,
	suggestions, and corrections, which helped improve the manuscript}.
This work was partially supported by JSPS KAKENHI Grant Number JP20H04250.

{\bf Data availability}. This paper contains no associated data.

\bibliographystyle{spmpsci}
\bibliography{bib_infinite}

\begin{thebibliography}{10}
\providecommand{\url}[1]{{#1}}
\providecommand{\urlprefix}{URL }
\expandafter\ifx\csname urlstyle\endcsname\relax
  \providecommand{\doi}[1]{DOI~\discretionary{}{}{}#1}\else
  \providecommand{\doi}{DOI~\discretionary{}{}{}\begingroup
  \urlstyle{rm}\Url}\fi

\bibitem{Agueh:2011barycenters}
Agueh, M., Carlier, G.: Barycenters in the wasserstein space.
\newblock SIAM Journal on Mathematical Analysis \textbf{43}(2), 904--924 (2011)

\bibitem{amari2018information}
Amari, S.i., Karakida, R., Oizumi, M.: Information geometry connecting
  {W}asserstein distance and {K}ullback--{L}eibler divergence via the
  entropy-relaxed transportation problem.
\newblock Information Geometry \textbf{1}(1), 13--37 (2018)

\bibitem{arsigny06}
Arsigny, V., Fillard, P., Pennec, X., Ayache, N.: Log-euclidean metrics for
  fast and simple calculus on diffusion tensors.
\newblock Magnetic Resonance in Medicine: An Official Journal of the
  International Society for Magnetic Resonance in Medicine \textbf{56}(2),
  411--421 (2006)

\bibitem{LogEuclidean:SIAM2007}
Arsigny, V., Fillard, P., Pennec, X., Ayache, N.: Geometric means in a novel
  vector space structure on symmetric positive-definite matrices.
\newblock SIAM J. on Matrix An. and App. \textbf{29}(1), 328--347 (2007)

\bibitem{baker1970mutual}
Baker, C.R.: Mutual information for {Gaussian} processes.
\newblock SIAM Journal on Applied Mathematics \textbf{19}(2), 451--458 (1970)

\bibitem{Baker1973CrossCovariance}
Baker, C.R.: Joint measures and cross-covariance operators.
\newblock Transactions of the American Mathematical Society \textbf{186},
  273--289 (1973)

\bibitem{Baker1978capacity}
Baker, C.R.: Capacity of the {Gaussian} channel without feedback.
\newblock Information and Control \textbf{37}(1), 70--89 (1978)

\bibitem{barrio2020entropic}
del Barrio, E., Loubes, J.M.: The statistical effect of entropic regularization
  in optimal transportation.
\newblock preprint arxiv:2006.05199  (2020)

\bibitem{Bhatia:1997Matrix}
Bhatia, R.: Matrix Analysis.
\newblock Springer (1997)

\bibitem{bigot19}
Bigot, J., Cazelles, E., Papadakis, N.: Penalization of barycenters in the
  wasserstein space.
\newblock SIAM Journal on Mathematical Analysis \textbf{51}(3), 2261--2285
  (2019)

\bibitem{Bogachev:Gaussian}
Bogachev, V.: Gaussian Measures.
\newblock American Mathematical Society (1998)

\bibitem{Borwein:2000ConvexAnalysis}
Borwein, J., Lewis, A.: Convex analysis and nonlinear optimization.
\newblock CMS Books in Mathematics. Springer (2000)

\bibitem{BorLewNus94}
Borwein, J.M., Lewis, A.S., Nussbaum, R.D.: Entropy minimization, {D}{A}{D}
  problems, and doubly stochastic kernels.
\newblock Journal of Functional Analysis \textbf{123}(2), 264--307 (1994)

\bibitem{Brown1980:nthRoot}
Brown, D., O'Malley, M.: On nth roots of positive operators.
\newblock The American Mathematical Monthly \textbf{87}(5), 380--382 (1980)

\bibitem{Chebbi:2012Means}
Chebbi, Z., Moakher, M.: Means of {H}ermitian positive-definite matrices based
  on the log-determinant $\alpha$-divergence function.
\newblock Linear Algebra and its Applications \textbf{436}(7), 1872--1889
  (2012)

\bibitem{ciccone2020regularizedtransport}
Ciccone, V., Chen, Y., Georgiou, T.T., Pavon, M.: Regularized transport between
  singular covariance matrices.
\newblock preprint arxiv:2006.10000  (2020)

\bibitem{cichocki15}
Cichocki, A., Cruces, S., Amari, S.i.: Log-determinant divergences revisited:
  Alpha-beta and gamma log-det divergences.
\newblock Entropy \textbf{17}(5), 2988--3034 (2015)

\bibitem{Congedo:BCIreview2017}
Congedo, M., Barachant, A., Bhatia, R.: Riemannian geometry for {EEG}-based
  brain-computer interfaces; a primer and a review.
\newblock Brain-Computer Interfaces \textbf{4}(3), 155--174 (2017)

\bibitem{conwayFunctionalAnalysis2007}
Conway, J.: A course in functional analysis, \emph{Graduate Texts in
  Mathematics}, vol.~96, 2nd edn.
\newblock Springer (2007)

\bibitem{CoverThomas1991:InformationTheory}
Cover, T., Thomas, J.: Elements of Information Theory.
\newblock Wiley, New York (1991)

\bibitem{Csi75}
Csisz{\'a}r, I.: I-divergence geometry of probability distributions and
  minimization problems.
\newblock The Annals of Probability pp. 146--158 (1975)

\bibitem{cuesta1996:WassersteinHilbert}
Cuesta-Albertos, J., Matr{\'a}n-Bea, C., Tuero-Diaz, A.: On lower bounds for
  the {L2}-{Wasserstein metric} in a {Hilbert} space.
\newblock Journal of Theoretical Probability \textbf{9}(2), 263--283 (1996)

\bibitem{cuturi13}
Cuturi, M.: Sinkhorn distances: {L}ightspeed computation of optimal transport.
\newblock In: Advances in neural information processing systems, pp. 2292--2300
  (2013)

\bibitem{peyre17}
Cuturi, M., Peyr{\'e}, G.: Computational optimal transport.
\newblock Foundations and Trends{\textregistered} in Machine Learning
  \textbf{11}(5-6), 355--607 (2019)

\bibitem{DaPrato:2006}
Da~Prato, G.: An introduction to infinite-dimensional analysis.
\newblock Springer Science \& Business Media (2006)

\bibitem{DaPrato:PDEHilbert}
Da~Prato, G., Zabczyk, J.: Second order partial differential equations in
  Hilbert spaces, vol. 293.
\newblock Cambridge University Press (2002)

\bibitem{DMaGer19}
Di~Marino, S., Gerolin, A.: An optimal transport approach for the
  {Schr{\"o}dinger} bridge problem and convergence of {Sinkhorn} algorithm.
\newblock Journal of Scientific Computing \textbf{85}(2), 1--28 (2020)

\bibitem{dowson82}
Dowson, D.C., Landau, B.V.: The {F}r{\'e}chet distance between multivariate
  normal distributions.
\newblock Journal of multivariate analysis \textbf{12}(3), 450--455 (1982)

\bibitem{Dryden:2009}
Dryden, I., Koloydenko, A., Zhou, D.: Non-{E}uclidean statistics for covariance
  matrices, with applications to diffusion tensor imaging.
\newblock Annals of Applied Statistics \textbf{3}, 1102--1123 (2009)

\bibitem{KyFan:1950}
Fan, K.: On a theorem of {W}eyl concerning eigenvalues of linear
  transformations: {II}.
\newblock Proceedings of the National Academy of Sciences of the United States
  of America \textbf{36}(1), 31 (1950)

\bibitem{Feldman:Gaussian1958}
Feldman, J.: Equivalence and perpendicularity of {Gaussian} processes.
\newblock Pacific Journal of Mathematics \textbf{8}(4), 699--708 (1958)

\bibitem{feydy18}
Feydy, J., S{\'e}journ{\'e}, T., Vialard, F., Amari, S., Trouve, A., Peyr{\'e},
  G.: Interpolating between optimal transport and {MMD} using {S}inkhorn
  divergences.
\newblock In: The 22nd International Conference on Artificial Intelligence and
  Statistics, pp. 2681--2690 (2019)

\bibitem{FraLor89}
Franklin, J., Lorenz, J.: On the scaling of multidimensional matrices.
\newblock Linear Algebra and its applications \textbf{114}, 717--735 (1989)

\bibitem{Fremdt:2013testing}
Fremdt, S., Steinebach, J., Horv{\'a}th, L., Kokoszka, P.: Testing the equality
  of covariance operators in functional samples.
\newblock Scandinavian Journal of Statistics \textbf{40}(1), 138--152 (2013)

\bibitem{galsal}
Galichon, A., Salani{\'e}, B.: Matching with trade-offs: Revealed preferences
  over competing characteristics  (2010)

\bibitem{Gelbrich:1990Wasserstein}
Gelbrich, M.: On a formula for the {L2} {Wasserstein} metric between measures
  on {Euclidean} and {Hilbert} spaces.
\newblock Mathematische Nachrichten \textbf{147}(1), 185--203 (1990)

\bibitem{genevay16}
Genevay, A., Cuturi, M., Peyr{\'e}, G., Bach, F.: Stochastic optimization for
  large-scale optimal transport.
\newblock In: Advances in Neural Information Processing Systems, pp. 3440--3448
  (2016)

\bibitem{genevay17}
Genevay, A., Peyre, G., Cuturi, M.: Learning {G}enerative {M}odels with
  {S}inkhorn {D}ivergences.
\newblock In: International Conference on Artificial Intelligence and
  Statistics (2018), pp. 1608--1617 (2018)

\bibitem{GerGroGor19}
Gerolin, A., Grossi, J., Gori-Giorgi, P.: Kinetic correlation functionals from
  the entropic regularisation of the strictly-correlated electrons problem.
\newblock Journal of Chemical Theory and Computation \textbf{16}(1), 488--498
  (2020)

\bibitem{GigTam18}
{Gigli}, N., {Tamanini}, L.: Second order differentiation formula on
  ${RCD}^*({K},{N})$ spaces.
\newblock J. Eur. Math. Soc. (JEMS) \textbf{23}, 1727–1795

\bibitem{GigTamBB18}
Gigli, N., Tamanini, L.: Benamou-{B}renier and duality formulas for the
  entropic cost on $ {R}{C}{D}^{*}({K}, {N}) $ spaces.
\newblock Probab. Theory Related Fields  (2018)

\bibitem{givens84}
Givens, C.R., Shortt, R.M.: A class of {W}asserstein metrics for probability
  distributions.
\newblock The Michigan Mathematical Journal \textbf{31}(2), 231--240 (1984)

\bibitem{gohberg1978nonselfadjoint}
Gohberg, I., Krein, M.: Introduction to the theory of linear nonselfadjoint
  operators, vol.~18.
\newblock American Mathematical Society (1978)

\bibitem{Gretton:MMD12a}
Gretton, A., Borgwardt, K.M., Rasch, M.J., Sch{{\"o}}lkopf, B., Smola, A.: A
  kernel two-sample test.
\newblock Journal of Machine Learning Research \textbf{13}(25), 723--773 (2012)

\bibitem{Hajek:Gaussian1958}
H\'ajek, J.: On a property of normal distributions of any stochastic process.
\newblock Czechoslovak Mathematical Journal \textbf{08}(4), 610--618 (1958).
\newblock \urlprefix\url{http://eudml.org/doc/11961}

\bibitem{Covariance:CVPR2014}
Harandi, M., Salzmann, M., Porikli, F.: Bregman divergences for infinite
  dimensional covariance matrices.
\newblock In: CVPR (2014)

\bibitem{janati2020debiased}
Janati, H., Cuturi, M., Gramfort, A.: Debiased sinkhorn barycenters.
\newblock ICML  (2020)

\bibitem{Janati2020entropicOT}
Janati, H., Muzellec, B., Peyr{\'e}, G., Cuturi, M.: Entropic optimal transport
  between (unbalanced) {Gaussian} measures has a closed form.
\newblock Advances in Neural Information Processing Systems  (2020)

\bibitem{Jost:1998}
Jost, J.: Postmodern analysis.
\newblock Springer (1998)

\bibitem{Kadison:1983}
Kadison, R., Ringrose, J.: Fundamentals of the theory of operator algebras.
  Volume I: Elementary Theory.
\newblock Pure and Applied Mathematics. Academic Press (1983)

\bibitem{Kitta:InequalitiesV}
Kittaneh, F., Kosaki, H.: Inequalities for the {Schatten p-norm V}.
\newblock Publications of the Research Institute for Mathematical Sciences
  \textbf{23}(2), 433--443 (1987)

\bibitem{knott84}
Knott, M., Smith, C.S.: On the optimal mapping of distributions.
\newblock Journal of Optimization Theory and Applications \textbf{43}(1),
  39--49 (1984)

\bibitem{kum2020penalization}
Kum, S., Duong, M.H., Lim, Y., Yun, S.: Penalization of barycenters for
  $\varphi$-exponential distributions (2020)

\bibitem{Larotonda:2007}
Larotonda, G.: Nonpositive curvature: A geometrical approach to
  {H}ilbert-{S}chmidt operators.
\newblock Differential Geometry and its Applications \textbf{25}, 679--700
  (2007)

\bibitem{LeoSurvey}
L{\'e}onard, C.: A survey of the {S}chr{\"o}dinger problem and some of its
  connections with optimal transport.
\newblock Discrete \& Continuous Dynamical Systems-A \textbf{34}(4), 1533--1574
  (2014)

\bibitem{Lunz18}
Lunz, S., {\"O}ktem, O., Sch{\"o}nlieb, C.B.: Adversarial regularizers in
  inverse problems.
\newblock In: Advances in Neural Information Processing Systems, pp. 8507--8516
  (2018)

\bibitem{Mallasto:NIPS2017Wasserstein}
Mallasto, A., Feragen, A.: Learning from uncertain curves: The 2-{Wasserstein}
  metric for {Gaussian} processes.
\newblock In: Advances in Neural Information Processing Systems, pp. 5660--5670
  (2017)

\bibitem{Mallasto2020entropyregularized}
Mallasto, A., Gerolin, A., Minh, H.: Entropy-regularized 2-{Wasserstein}
  distance between {Gaussian} measures.
\newblock Information Geometry  (2021)

\bibitem{masarotto2019procrustes}
Masarotto, V., Panaretos, V.M., Zemel, Y.: Procrustes metrics on covariance
  operators and optimal transportation of gaussian processes.
\newblock Sankhya A \textbf{81}(1), 172--213 (2019)

\bibitem{Minh:2019AlphaProcrustes}
Minh, H.: Alpha {Procrustes} metrics between positive definite operators: a
  unifying formulation for the {Bures-Wasserstein} and
  {Log-Euclidean/Log-Hilbert-Schmidt} metrics.
\newblock Linear Algebra and Its Applications \textbf{636}, 25--68

\bibitem{Minh:LogDetIII2018}
Minh, H.: Infinite-dimensional {Log-Determinant} divergences {III}:
  {Log-Euclidean} and {Log-Hilbert--Schmidt} divergences.
\newblock In: Information Geometry and its Applications IV, pp. 209--243.
  Springer (2018)

\bibitem{Minh:GSI2019}
Minh, H.: A unified formulation for the {Bures-Wasserstein} and
  {Log-Euclidean/Log-Hilbert-Schmidt} distances between positive definite
  operators.
\newblock In: International Conference on Geometric Science of Information.
  Springer (2019)

\bibitem{Minh:Positivity2020}
Minh, H.: Infinite-dimensional {Log-Determinant} divergences between positive
  definite {Hilbert-Schmidt} operators.
\newblock Positivity \textbf{24}, 631--662 (2020)

\bibitem{Minh:2020regularizedDiv}
Minh, H.: Regularized divergences between covariance operators and {Gaussian}
  measures on {Hilbert} spaces.
\newblock Journal of Theoretical Probability  (2020)

\bibitem{Minh:Covariance2017}
Minh, H., Murino, V.: Covariances in computer vision and machine learning.
\newblock Synthesis Lectures on Computer Vision \textbf{7}(4), 1--170 (2017)

\bibitem{Minh:LogDet2016}
Minh, H.Q.: Infinite-dimensional {Log-Determinant} divergences between positive
  definite trace class operators.
\newblock Linear Algebra and Its Applications \textbf{528}, 331--383 (2017)

\bibitem{Minh:2019AlphaBeta}
Minh, H.Q.: {Alpha-Beta Log-Determinant} divergences between positive definite
  trace class operators.
\newblock Information Geometry \textbf{2}(2), 101--176 (2019)

\bibitem{MinhSB:NIPS2014}
Minh, H.Q., San~Biagio, M., Murino, V.: Log-{H}ilbert-{S}chmidt metric between
  positive definite operators on {H}ilbert spaces.
\newblock In: Advances in Neural Information Processing Systems 27 (NIPS 2014),
  pp. 388--396 (2014)

\bibitem{olkin82}
Olkin, I., Pukelsheim, F.: The distance between two random vectors with given
  dispersion matrices.
\newblock Linear Algebra and its Applications \textbf{48}, 257--263 (1982)

\bibitem{Panaretos:jasa2010}
Panaretos, V., Kraus, D., Maddocks, J.: Second-order comparison of {Gaussian}
  random functions and the geometry of {DNA} minicircles.
\newblock Journal of the American Statistical Association \textbf{105}(490),
  670--682 (2010)

\bibitem{patrini18}
Patrini, G., Berg, R.v.d., Forre, P., Carioni, M., Bhargav, S., Welling, M.,
  Genewein, T., Nielsen, F.: Sinkhorn {A}utoencoders.
\newblock Uncertainty in Artificial Intelligence Conference  (2020)

\bibitem{Pedersen1972:OperatorMonotone}
Pedersen, G.: Some operator monotone functions.
\newblock Proceedings of the American Mathematical Society \textbf{36}(1),
  309--310 (1972)

\bibitem{Pennec:IJCV2006}
Pennec, X., Fillard, P., Ayache, N.: A {R}iemannian framework for tensor
  computing.
\newblock International Journal of Computer Vision \textbf{66}(1), 41--66
  (2006)

\bibitem{Petryshyn:1962}
Petryshyn, W.: Direct and iterative methods for the solution of linear operator
  equations in {H}ilbert spaces.
\newblock Transactions of the American Mathematical Society \textbf{105},
  136--175 (1962)

\bibitem{Convex:2015}
Peypouquet, J.: Convex optimization in normed spaces: theory, methods and
  examples.
\newblock Springer (2015)

\bibitem{Pigoli:2014}
Pigoli, D., Aston, J., Dryden, I., Secchi, P.: Distances and inference for
  covariance operators.
\newblock Biometrika \textbf{101}(2), 409--422 (2014)

\bibitem{Rajput1972gaussianprocesses}
Rajput, B.S., Cambanis, S.: Gaussian processes and {Gaussian} measures.
\newblock The Annals of Mathematical Statistics pp. 1944--1952 (1972)

\bibitem{ramdas2017}
Ramdas, A., Trillos, N., Cuturi, M.: On {W}asserstein two-sample testing and
  related families of nonparametric tests.
\newblock Entropy \textbf{19}(2), 47 (2017)

\bibitem{ReedSimon:Functional}
Reed, M., Simon, B.: Methods of Modern Mathematical Physics: Functional
  analysis.
\newblock Academic Press (1975)

\bibitem{RipThesis}
Ripani, L.: The {S}chr\"odinger problem and its links to optimal transport and
  functional inequalities.
\newblock Ph.D. thesis, University Lyon 1 (2017)

\bibitem{RusIPFP}
Ruschendorf, L.: Convergence of the iterative proportional fitting procedure.
\newblock The Annals of Statistics \textbf{23}(4), 1160--1174 (1995)

\bibitem{rus93}
R{\"u}schendorf, L., Thomsen, W.: Note on the schr{\"o}dinger equation and
  i-projections.
\newblock Statistics \& probability letters \textbf{17}(5), 369--375 (1993)

\bibitem{rus98}
R{\"u}schendorf, L., Thomsen, W.: Closedness of sum spaces and the generalized
  {S}chr{\"o}dinger problem.
\newblock Theory of Probability \& Its Applications \textbf{42}(3), 483--494
  (1998)

\bibitem{Schr31}
Schr{\"o}dinger, E.: {\"U}ber die umkehrung der naturgesetze.
\newblock Verlag Akademie der wissenschaften in kommission bei Walter de
  Gruyter u. Company (1931)

\bibitem{Simon:1977}
Simon, B.: Notes on infinite determinants of {H}ilbert space operators.
\newblock Advances in Mathematics \textbf{24}, 244--273 (1977)

\bibitem{Sommerfeld2017WassersteinDO}
Sommerfeld, M.: Wasserstein distance on finite spaces: Statistical inference
  and algorithms (2017)

\bibitem{Steinwart:SVM2008}
Steinwart, I., Christmann, A.: Support vector machines.
\newblock Springer Science \& Business Media (2008)

\bibitem{Sun2005MercerNoncompact}
Sun, H.: Mercer theorem for {RKHS} on noncompact sets.
\newblock Journal of Complexity \textbf{21}(3), 337--349 (2005)

\bibitem{thanwerdas19}
Thanwerdas, Y., Pennec, X.: Exploration of balanced metrics on symmetric
  positive definite matrices.
\newblock In: International Conference on Geometric Science of Information, pp.
  484--493. Springer (2019)

\bibitem{Tosato:PAMI2013}
Tosato, D., Spera, M., Cristani, M., Murino, V.: Characterizing humans on
  {R}iemannian manifolds.
\newblock TPAMI \textbf{35}(8), 1972--1984 (2013)

\bibitem{tuzel08}
Tuzel, O., Porikli, F., Meer, P.: Pedestrian detection via classification on
  riemannian manifolds.
\newblock IEEE transactions on pattern analysis and machine intelligence
  \textbf{30}(10), 1713--1727 (2008)

\bibitem{villani2016}
Villani, C.: Topics in Optimal Transportation, \emph{Graduate Studies in
  Mathematics}, vol.~58.
\newblock American Mathematical Society (2016)

\bibitem{Zam15}
Zambrini, J.C.: The research program of stochastic deformation (with a view
  toward geometric mechanics).
\newblock In: Stochastic analysis: a series of lectures, pp. 359--393. Springer
  (2015)

\bibitem{zhang2019:OTRKHS}
Zhang, Z., Wang, M., Nehorai, A.: Optimal transport in reproducing kernel
  {Hilbert} spaces: Theory and applications.
\newblock IEEE transactions on pattern analysis and machine intelligence
  (2019)

\bibitem{ProbDistance:PAMI2006}
Zhou, S.K., Chellappa, R.: From sample similarity to ensemble similarity:
  Probabilistic distance measures in reproducing kernel {H}ilbert space.
\newblock TPAMI \textbf{28}(6), 917--929 (2006)

\end{thebibliography}

\end{document}